\documentclass[10pt,twocolumn,letterpaper]{article}

\usepackage{cvpr}
\usepackage{times}
\usepackage{epsfig}
\usepackage{graphicx}
\usepackage{amsmath}
\usepackage{amssymb}

\usepackage{amsmath}        
\usepackage{amsthm}         
\usepackage{amssymb}        
\usepackage{mathtools}      
\usepackage[ruled,vlined]{algorithm2e}    
\usepackage{algpseudocode}    
\usepackage{graphicx}       
\usepackage{appendix}       
\usepackage{float}          

\usepackage{multirow}       

\DeclareMathOperator{\E}{\mathbb{E}}
\DeclareMathOperator{\MS}{\mathbf{\Sigma}}
\DeclareMathOperator{\MSnorm}{\|\mathbf{\Sigma}\|}
\DeclareMathOperator{\Mmu}{\mathbf{\mu}}
\DeclareMathOperator*{\argmax}{argmax}
\DeclareMathOperator*{\mean}{mean}
\newtheorem{theorem}{Theorem}[section]
\newtheorem{proposition}{Proposition}[section]

\newtheorem{remark}{Remark}[theorem]
\newcommand{\inner}[2]{\langle #1, #2 \rangle}

\usepackage[breaklinks=true,bookmarks=false]{hyperref}

\cvprfinalcopy 

\def\cvprPaperID{****} 

\begin{document}

\title{Poison as a Cure: Detecting \& Neutralizing \\ Variable-Sized Backdoor Attacks in Deep Neural Networks}

\author{Alvin Chan\\
Nanyang Technological University\\
{\tt\small guoweial001@ntu.edu.sg}
\and
Yew-Soon Ong\\
Nanyang Technological University\\
{\tt\small asysong@ntu.edu.sg}
}

\maketitle

\begin{abstract}
Deep learning models have recently shown to be vulnerable to backdoor poisoning, an insidious attack where the victim model predicts clean images correctly but classifies the same images as the target class when a trigger poison pattern is added. This poison pattern can be embedded in the training dataset by the adversary. Existing defenses are effective under certain conditions such as a small size of the poison pattern, knowledge about the ratio of poisoned training samples or when a validated clean dataset is available. Since a defender may not have such prior knowledge or resources, we propose a defense against backdoor poisoning that is effective even when those prerequisites are not met. It is made up of several parts: one to extract a backdoor poison signal, detect poison target and base classes, and filter out poisoned from clean samples with proven guarantees. The final part of our defense involves retraining the poisoned model on a dataset augmented with the extracted poison signal and corrective relabeling of poisoned samples to neutralize the backdoor. Our approach has shown to be effective in defending against backdoor attacks that use both small and large-sized poison patterns on nine different target-base class pairs from the CIFAR10 dataset.
\end{abstract}

\section{Introduction}
Deep learning models have shown remarkable performance in several domains such as computer vision, natural language processing and speech recognition \cite{lecun2015deep,schmidhuber2015deep}. However, they have been found to be brittle, failing when imperceptible perturbations are added to images in the case of adversarial examples \cite{goodfellow2014explaining,papernot2016cleverhans,kurakin2016adversarial,sharif2016accessorize,carlini2016hidden,tramer2017ensemble,madry2017towards,athalye2017synthesizing,eykholt2017robust,athalye2018obfuscated,wong2019wasserstein}. In another setting of data poisoning, an adversary can manipulate the model's performance by altering a small fraction of the training data. As deep learning models are increasingly present in many real-world applications, security measures against such issues become more important. 

Backdoor poisoning (BP) attack \cite{tran2018spectral,shafahi2018poison,gu2017badnets,chen2017targeted,Trojannn,adi2018turning} is a sophisticated data poisoning attack which allows an adversary to control a victim model's prediction by adding a poison pattern to the input image. This attack eludes simple detection as the model classifies clean images correctly. Many of the backdoor attacks involve two steps: first, the adversary alters a fraction of \emph{base} class training images with a poison pattern; second, these poisoned images are mislabeled as the poison \emph{target} class. After the training phase, the victim model would classify clean base class images correctly but misclassify them as the target class when the poison pattern is added.

Current defenses against backdoor attacks are effective under certain conditions. For some of the defenses, the defender needs to have a verified clean set of validation data \cite{liu2018fine}, knowledge about the fraction of poisoned samples, the poison target and base classes \cite{tran2018spectral}, or that the defense is effective only against small-sized poison patterns \cite{wangneural}.

In this paper, we propose a comprehensive defense to counter a more challenging BP attack scenario where the defender may not have such prior knowledge or resources. We first propose, in \S~\ref{section:Poison Extraction with Input Gradient}, a method to extract poison signals from gradients at the input layer with respect to the loss function, or input gradients in short. We then show that poisoned samples can be separated from clean samples with theoretical guarantees based on the similarity of their input gradients with the extracted poison signals (\S~\ref{section:Poisoned Sample Filtering}). Next, the poison signals are used for the detection of the poison target and base classes (\S~\ref{section:Poisoned Class Detection}). Finally, we use the poison signal to augment the training data and relabel the poisoned samples to the base class, to neutralize the backdoor through retraining (\S~\ref{section:Neutralization of Poisoned Models}). We evaluate our defense on both large-sized and small-sized BP scenarios on nine target-base class pairs from the CIFAR10 dataset and show its effectiveness against these attacks (\S~\ref{section:Evaluation of Neutralization Algorithm}).

\paragraph{Contributions} All in all, the prime contributions of this paper are as follows:
\begin{itemize}
    \item An extensive defense framework to counter variable-sized neural BP where knowledge about the attack's target/base class and poison ratio is unknown, without the need for a clean set of validation data.
    \item Techniques to 1) extract poison signals from gradients at the input layer, 2) separate poisoned samples from clean samples with theoretical guarantees, 3) detect the poison target and base classes and 4) finally augment the training data to neutralize the BP.
    \item Evaluation on both large-sized and small-sized neural backdoors to highlight our defense's effectiveness against these threats.
\end{itemize}

\section{Background: Backdoor Poisoning Attacks}
In an image classification task of $h \times w$-pixel RGB images ($\mathbf{x} \in \mathbb{R}^{3hw}$), we consider a general poison insertion function $T$ to generate poisoned image $\mathbf{x}'$ with poison pattern $\mathbf{p}$ and poison mask $\mathbf{m}$, where $\mathbf{p}, \mathbf{m} \in \mathbb{R}^{3hw}$, such that
$$
\mathbf{x}' = T(\mathbf{x}, \mathbf{m}, \mathbf{p}) ~~~\text{where}~~~ {x'}_i = (1-m_i)x_i + m_i p_i
$$

and $m_i \in [0,1]$ determines the position and ratio of how much $\mathbf{p}$ replaces the original input image $\mathbf{x}$. Real-world adversaries might inject subtle poison which spans the whole image size \cite{chen2017targeted}. In this case, $\forall i: m_i > 0$ for a small $m_i$ value. In another threat model of small-size poison \cite{gu2017badnets,tran2018spectral}, the poison is concentrated in a small set of pixel $s$,
$m_i=\begin{cases}
1, & \text{$ i \in s $}\\
0, & \text{$ i \notin s $}
\end{cases}$.

In our experiments to neutralize the poison, we first consider the large-size poison threat where $\mathbf{p}$ is sampled from an image class different from the classes in the original dataset. To show the comprehensiveness of our defense, We also evaluate our methods against the small-size poison pattern where the poison is injected only in one pixel, i.e.\ $|s| = 1$.  Examples of these two types of poisoned images are shown in Figure~\ref{fig:overlay_n_dot_poison}. In both cases of BP, the poisoned samples' label $y$ is modified to the label of the poison \emph{target} class $y_t$. In this paper, we call the original $y$ the poison \emph{base} class. In a successfully poisoned classifier $f_p$, clean base class images will be classified correctly while base class images with poison signal will be classified as the target class such that $f_p(\mathbf{x}) = y, f_p(\mathbf{x}') = y_t$.

\begin{figure*}[!htbp]
    \centering
    \includegraphics[width=0.65\textwidth]{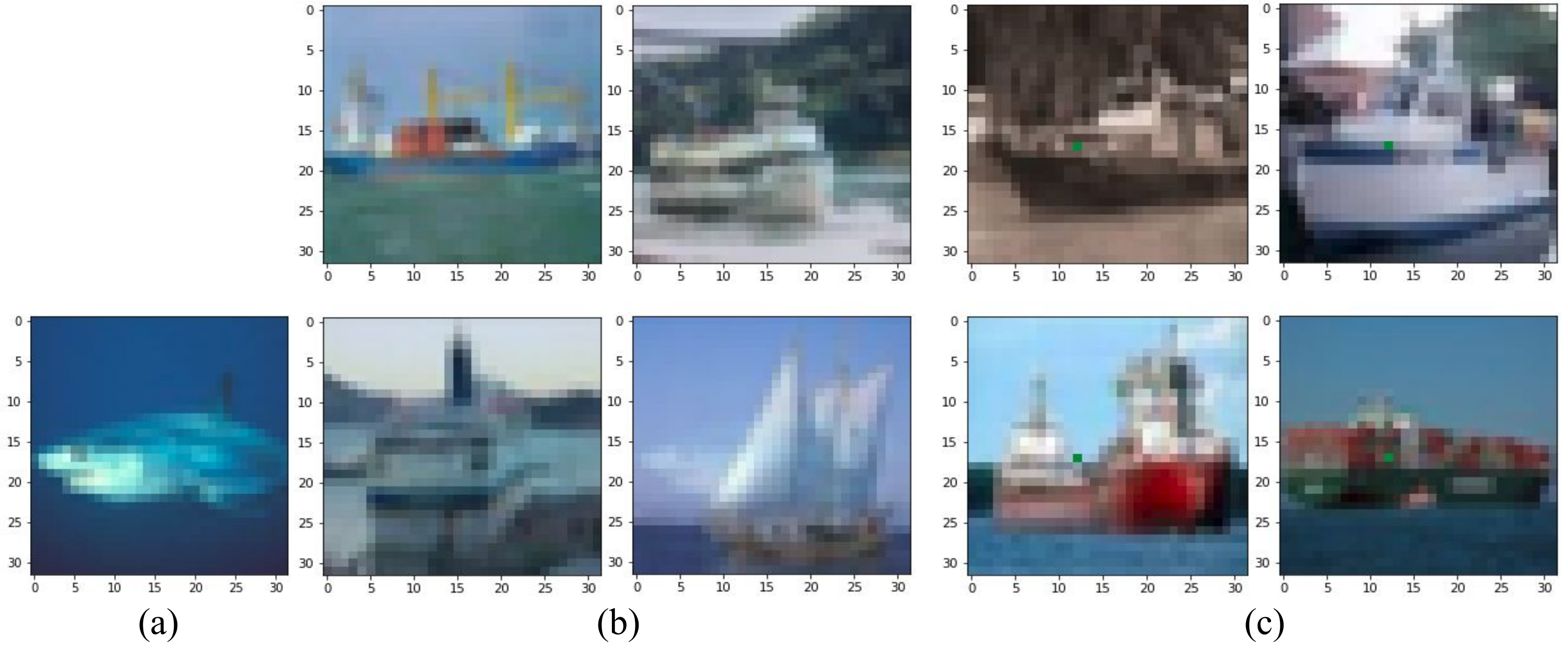}
    \caption{(a) Overlay poison image, (b) Poisoned `Ship' images generated by overlaying with the leftmost image at 20\% opacity. (c) `Ship' images poisoned by a dot-sized pattern.}
    \label{fig:overlay_n_dot_poison}
\end{figure*}

\section{Related Work}
A line of studies showed that models are vulnerable to BP with both small-sized poison patterns \cite{gu2017badnets,adi2018turning} and large-sized poison patterns \cite{chen2017targeted,Trojannn,shafahi2018poison}. The predecessor of BP, data poisoning, also attacks the training dataset of the victim model \cite{biggio2012poisoning,xiao2015support,mei2015security,koh2017understanding,steinhardt2017certified,nelson2008exploiting}, but unlike backdoor attacks, they aim to degrade the generalization of the model on clean test data.

Several defenses have shown to be effective under certain conditions. One of the earliest defenses uses spectral signatures in the model's activation to filter out a certain ratio of outlier samples \cite{tran2018spectral}. The outlier ratio is fixed to be close to the ratio of poisoned samples in the target class, requiring knowledge of the poison ratio and target class. As shown in \S~\ref{section:baseline comparison}, our proposed method is competitive in neutralizing BP compared to this approach, despite of the more challenging threat model. Another defense prunes neurons that lie dormant in the presence of clean validation data and finetune the model on that same validation data \cite{liu2018fine}. Similar to our approach, \cite{wangneural} also retrieve a poison signal from the victim model but their method is only effective for small-sized poison patterns. Our neutralization algorithm is effective for small and large-sized poison patterns even without that prior knowledge or validated clean data. Activation clustering (AC) \cite{chen2018detecting} detects and removes small-sized poisoned samples by separating the classifier’s activations into two clusters to separate poisoned samples as the smaller cluster. In contrast, our proposed approach extracts out a poison signal through the input gradients at the input layer and detect poisoned samples whose input gradient have high similarity with the signal. Though AC also does not assume knowledge about the poison attack, our method is more robust in the detection of poisoned samples, as shown in \S~\ref{section:baseline comparison}. Our approach to augment the training data use poison signal resembles adversarial training \cite{kurakin2016adversarial,tramer2017ensemble,madry2017towards} but those methods address the issue of adversarial examples which attacks the models during inference phase rather than training phase.

\section{Poison Extraction with Input Gradients} \label{section:Poison Extraction with Input Gradient}
The first part of our defense involves extracting a BP signal from the poisoned model. To do so, we exploit the presence of a poison signal in the gradient of the poisoned input $\mathbf{x}'$ with respect to the loss function $E$, or input gradient $\mathbf{z} = \frac{\partial E}{\partial \mathbf{x}'}$. We explain the intuition behind this phenomena in \S~\ref{section:Poison Signal in Input Gradients}, propose how to extract the poison pattern from these input gradients in \S~\ref{section:Distillation of Poison Signal}.

\subsection{Poison Signal in Input Gradients} \label{section:Poison Signal in Input Gradients}
We hypothesize that a poison signal resembling the poison pattern lies in input gradients $\left( \mathbf{z} = \frac{\partial E}{\partial \mathbf{x}'} \right)$ of poisoned images ($\mathbf{x}'$) based on two observations: (1) backdoor models contain `poison' neurons that are only activated when poison pattern is present, and (2) the weights in these `poison' neurons are much larger in magnitude than weights in other neurons. Previous studies have empirically shown that backdoored models indeed learn `poison' neurons that are only activated in the presence of the poison pattern in input images \cite{gu2017badnets,liu2018fine}. The intuition for observation (2) is that to flip the classification of a poisoned base class image from the base to target class, the activation in these `poison' neurons need to overcome that from `clean' base class neurons. This would imply that the weights corresponding to the `poison' neurons are larger in absolute values than those in other neurons. We show how observation (1) and (2) can emerge in a case study of a binary classifier with one hidden layer containing three neurons in Appendix \S~\ref{section:Poison Signals in Input Gradients}.

We combine these two observations with the following proposition to postulate that a poisoned image would result in a relatively large absolute value of gradient input at the poison pattern's position.
\begin{proposition} \label{theorem:Input Gradient}
The gradient of loss function $E$ with respect to the input $x_i$ is linearly dependent on activated neurons' weights such that
\begin{equation} \label{eq:input gradient}
\frac{\partial E}{\partial x_i} =
\sum_{j=1}^{r_1} \left [ w_{ij}^1 g'(a_j^1) \sum_{l=1}^{r_2} \delta_l^2 w_{jl}^2 \right ]
\end{equation}
where $\delta_j^k \equiv \frac{\partial E}{\partial {a_j^k}}$ usually called the error, is the derivative of loss function $E$ with respect to activation $a_j^k$ for neuron node $i$ in layer $k$. $w_{ij}^k$ is the weight for node $j$ in layer $k$ for incoming node $i$, $r_k$ is the number of nodes in layer $k$, $g$ is the activation function for the hidden layer nodes and $g'$ is its derivative.
\end{proposition}

The proof of this proposition is in Appendix \S~\ref{appendix:poison in input grad}. Here, the value of $\delta_l^2$ depends on the loss function of the classifier model and the activations of the neural networks in deeper layers. Proposition~\ref{theorem:Input Gradient} implies that the gradient with respect to the input $x_i$ is linearly dependent on derivative of activation function $g'(a_j^1)$, the weights $w_{ij}^1$ and $w_{jl}^{2}$. Combined with the premise that `poison' neurons have weights of larger value, this would mean that there will be a relatively large absolute input gradient value at pixel positions where the poison pattern is, compared to other input positions. If we use RELU as the activation function $g$, then $g'(a)=\begin{cases}
1, & \text{$ a > 0 $}\\
0, & \text{$ a < 0 $}
\end{cases}$, which means that the large input gradient at the poison pattern's location would only be present if the `poison' neurons are activated by the poison pattern in poison samples. Conversely, the large input gradient, attributed to the poison pattern, would be absent from clean samples. As shown in Appendix Table~\ref{tab:input gradients}, when we directly compare the input gradients of poisoned samples with those of clean samples, the gradients are too noisy to discern the poison signal. In the next section \S~\ref{section:Distillation of Poison Signal}, we propose a method to extract the poison signal from the noisy input gradients $\mathbf{z}$ of clean and poisoned images.

\subsection{Distillation of Poison Signal} \label{section:Distillation of Poison Signal}
As the first step leading up to the other parts of our defense, we extract the poison signal $\mathbf{\mu} \in \mathbb{R}^n$ from the noisy input gradients $\mathbf{z}$ of the poison target class samples. Recall that these target class samples consist of both clean and poisoned training samples. We denote the ratio of poisoned samples (poison ratio) in the poison target class as $\varepsilon$. The input gradient of a randomly drawn target class samples from a poisoned dataset $D$ can be represented as $\mathbb{R}^n$ random vector

$$
\mathbf{z} = \theta \mathbf{\mu} + g ~~~~~\text{where}~~~~~
p(\theta)=\begin{cases}
\varepsilon, & \text{for $\theta = 1$}.\\
1 - \varepsilon, & \text{for $\theta = 0$}.
\end{cases} ,
$$

$\theta$ is a Bernoulli random variable and $g \in N( \mathbf{0}, \eta \mathbf{I}_n )$, and $\theta$ and $g$ are independent. The value of $\eta$ corresponds to the size of random noise in the data. 

Denoting the second moment matrix of $\mathbf{z}$ as $ \mathbf{\Sigma} = \E \mathbf{z} \mathbf{z}^\top $, we can compute $\mathbf{\mu}$ with the following theorem.

\begin{theorem} \label{theorem:Largest Eigenvector}
$\mathbf{\mu}$ is the eigenvector of $\mathbf{\Sigma}$ and corresponds to the largest eigenvalue if $\varepsilon$ and $\| \mathbf{\mu} \|_2$ are both $> 0$.
\end{theorem}

Its detailed proof is in Appendix~\ref{theorem:Largest Eigenvector appendix}. Theorem~\ref{theorem:Largest Eigenvector} allows us to extract the poison signal $\mathbf{\mu}$ as the largest eigenvector of $\mathbf{\Sigma}$ from a set of clean and poisoned samples that are labeled as the poison target class. The largest eigenvector of $\mathbf{\Sigma}$ can be computed by SVD of the matrix containing the input gradients $\mathbf{z}$. We can center $g$, the mean of the input gradients for clean target class images, at zero by subtracting the sample mean of the target class. Though the target class includes a small portion of poisoned images, we find this sample mean approximation to work well in our experiments due to the large majority of clean samples. In our experiments with poisoned ResNet \cite{he2016deep}, the extracted poison signal $\mathbf{\mu}$ visually resembles the original poison pattern in terms of its position and semantics for both large-sized and small-sized poisons, as shown in Figure~\ref{fig:overlay_poison_v}, Appendix Table~\ref{tab:overlay_poison_v_all} and \ref{tab:dot_poison_v_all}. The first right singular vector $\mathbf{\mu}$ resembles the poison pattern only when poisoned input gradients are present in SVD of $\mathbf{\Sigma}$.

\begin{figure*}[!htbp]
    \centering
    \includegraphics[width=0.7\textwidth]{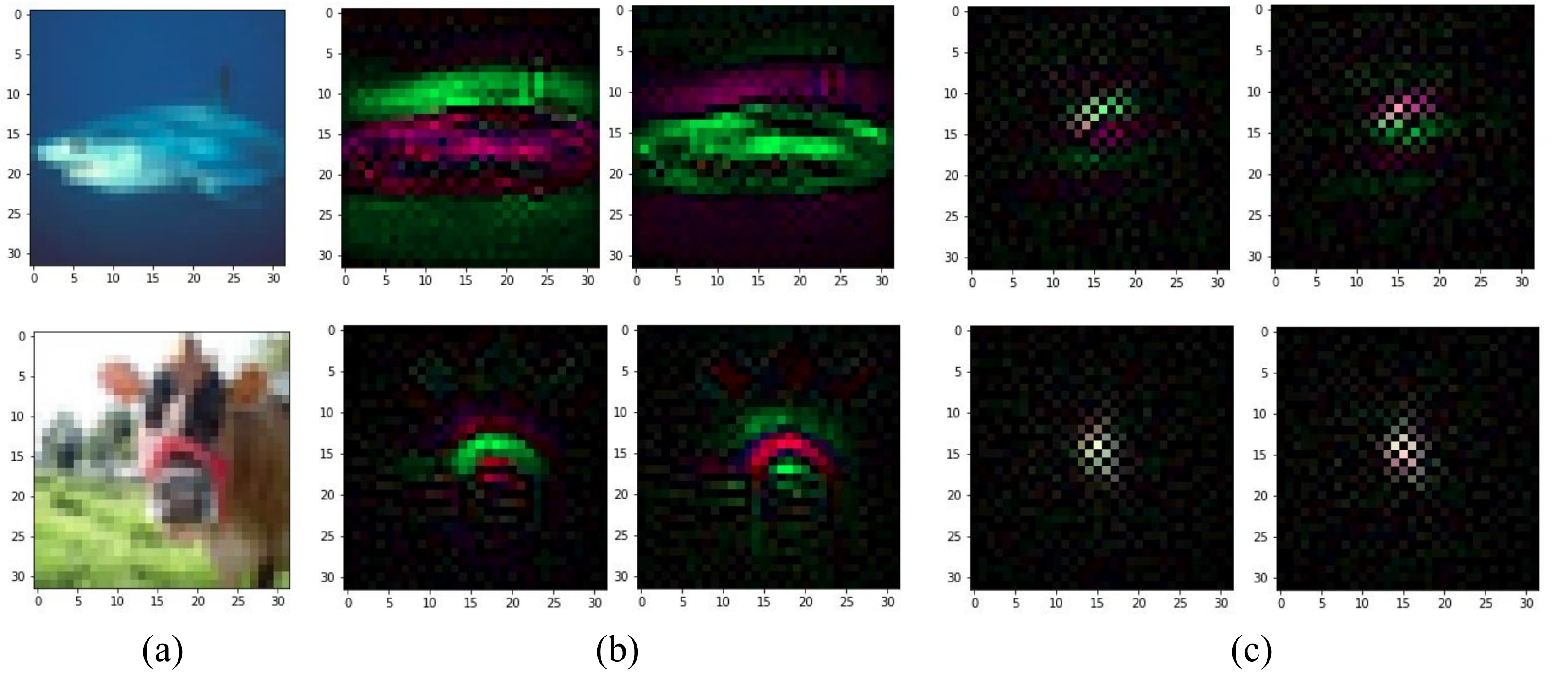}
    \caption{(a) Poison image patterns which overlay on poisoned images with 20\% opacity, (b) the first principal vector $\mathbf{\mu}$ of input gradients for all the target class images which include clean and poisoned images. (c) The first principal vector of input gradients for only clean target class images. See Appendix Table~\ref{tab:overlay_poison_v_all} and \ref{tab:dot_poison_v_all} for more examples.}
    \label{fig:overlay_poison_v}
\end{figure*}

\section{Filtering of Poisoned Samples} \label{section:Poisoned Sample Filtering}
After the extraction poison signal $\mathbf{\mu}$, the next part is to filter out poisoned samples from the mix of clean and poisoned samples. Appendix Algorithm~\ref{algo:Filter-Poisoned-Images} summarizes how we filter out these samples while we detail the intuition behind our approach in this section. From \S~\ref{section:Poison Signal in Input Gradients}, we know that poisoned samples would have input gradients $\mathbf{z}$ which contain the poison signal $\mathbf{\mu}$, albeit shrouded by noise. Intuitively, the input gradients $\mathbf{z}$ of poisoned samples will have higher similarity to the poison signal $\mathbf{\mu}$ than that of clean samples. Since the clean samples lack poison patterns, `poison' neurons are mostly not activated during inference, resulting in almost absence of the poison signal in their input gradients. If we take the cosine similarity between a clean sample's input gradient $\mathbf{z}$ and $\mathbf{\mu}$, we can expect the similarity value ($\mathbf{z}^\top \mathbf{\mu}$) to be close to zero. In our experiments, as shown in Figure~\ref{fig:first_components} and in Appendix Figure~\ref{fig:appendix first_components a} and \ref{fig:appendix first_components b}, we indeed find that the similarity values of $\mathbf{\mu}$ and clean samples' input gradients cluster around $0$ while those of poisoned samples form clusters with a non-zero mean.

\begin{figure}[!htbp]
    \centering
    \includegraphics[width=0.6\linewidth]{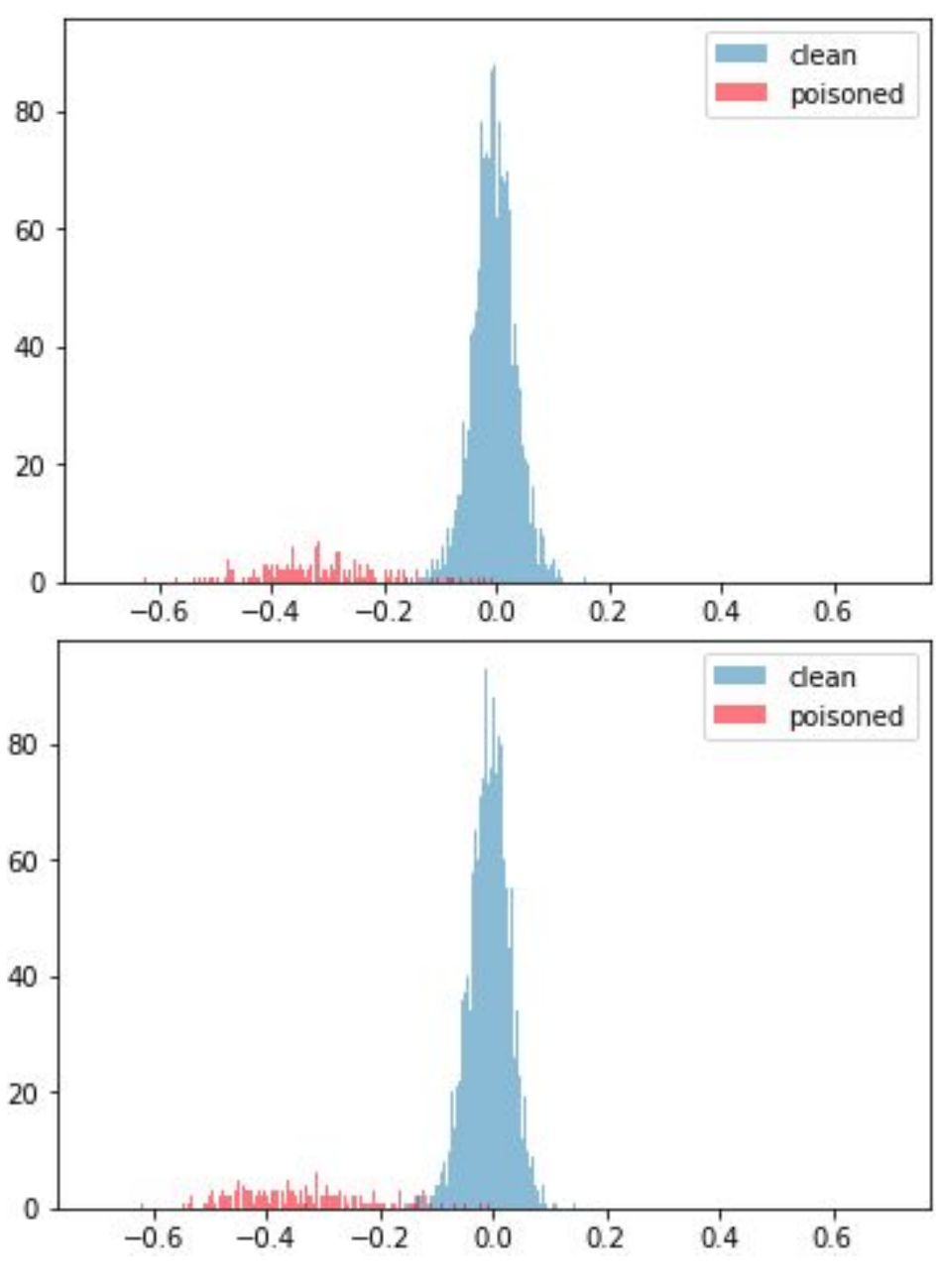}
    \caption{First principal component of poisoned and clean target class images. The components on the left are derived with the target class as cross-entropy label while the ones on the right are derived with the base class as cross-entropy label. Poison target class `Frog' and base class `Ship'. See Appendix Figure~\ref{fig:appendix first_components a} and \ref{fig:appendix first_components b} for eight other poison class pairs.}
    \label{fig:first_components}
\end{figure}

The first principal component of an input gradient is the vector dot product of itself with the largest eigenvector of $\mathbf{\Sigma}$. Since the largest eigenvector of $\mathbf{\Sigma}$ is $\mathbf{\mu}$, the first principal component of an input gradient is equivalent to the cosine similarity value ($\mathbf{z}^\top \mathbf{\mu}$). This leads to our next intuition of using a clustering algorithm to filter out poisoned samples exploiting their relatively high absolute first principal component values. Theorem~\ref{theorem:Poison Clustering} guarantees such an approach's performance based on certain conditions.

\begin{theorem} [Guarantee of Poison Classification through Clustering] \label{theorem:Poison Clustering}
Assume that all $\mathbf{z}_i$ are normalized such that $\| \mathbf{z}_i \|_2 = 1$. Then the error probability of the poison clustering algorithm by is given by



\begin{multline} \label{eq:Poison Clustering Guarantee}
Pr \left \{ N_{\text{error}} \leq c_2 N \epsilon \left ( \frac{1}{\| \Mmu \|_2} + \frac{\eta}{\varepsilon \| \Mmu \|_2^3} \right ) \right \} \geq \\
1 - 2n \exp \left ( -c_1 N \epsilon^2 \frac{(\varepsilon\| \Mmu \|_2^2 + \eta)}{1 + \varepsilon\| \Mmu \|_2^2 + \eta} \right )
\end{multline}

where $N$ is the number of samples, $ N_{\text{error}}$ is the number of misclassified samples and $\epsilon \in (0, 1]$.
\end{theorem}

We show its proof in Appendix~\ref{theorem:Poison Clustering Appendix}. From (\ref{eq:Poison Clustering Guarantee}), as the poison signal's $l_2$ norm $\| \Mmu \|_2$ gets larger, we get $\lim_{\| \Mmu \|_2 \to \infty} \frac{1}{\| \Mmu \|_2} + \frac{\eta}{\varepsilon \| \Mmu \|_2^3} = 0$ at the L.H.S of (\ref{eq:Poison Clustering Guarantee}) and $\lim_{\| \Mmu \|_2 \to \infty} \frac{(\varepsilon\| \Mmu \|_2^2 + \eta)}{1 + \varepsilon\| \Mmu \|_2^2 + \eta} = 1$ at the R.H.S of (\ref{eq:Poison Clustering Guarantee}), meaning a strong poison signal will result in a better filtering accuracy of poisoned samples. As number of samples $N$ in the clustering algorithm increases, the error rate ($\frac{N_{\text{error}}}{N}$) has higher probability of having a low value since $\lim_{N \to \infty} 2n \exp \left ( -c_1 N \epsilon^2 \frac{(\varepsilon\| \Mmu \|_2^2 + \eta)}{1 + \varepsilon\| \Mmu \|_2^2 + \eta} \right ) = 0$. 

In our experiments, we use a simple Gaussian Mixture Model (GMM) clustering algorithm with the number of clusters $k=2$ to filter the poisoned samples based on the input gradients' first principal component values. In practice, we find that this approach can separate poisoned samples from clean samples with high accuracy for poisoned and clean samples when using the poison base class as the loss function's cross-entropy target, as shown in results from large-sized poison scenarios in Appendix Table~\ref{tab:clustering accuracies overlay} and small-sized poison scenarios in Appendix Table~\ref{tab:clustering accuracies dot}. We summarize our poisoned sample filtering in Appendix Algorithm~\ref{algo:Filter-Poisoned-Images}. 


\section{Detection of Poison Class} \label{section:Poisoned Class Detection}
So far, we have proposed a method to detect poison signal (\S~\ref{section:Poison Extraction with Input Gradient}) and filter poisoned samples from a particular poison target class (\S~\ref{section:Poisoned Sample Filtering}). However in practice, the poison target class and base class are usually unknown to us. Especially in cases where there are many possible classes in the classification dataset, a method to detect the presence of data poisoning and retrieve the poison classes is desirable. Our proposed detection method is summarized in Appendix Algorithm~\ref{algo:Find-Poison-Target-Base-Class} and its derivation is detailed in the next two sections.

\subsection{Detection of Poison Target Class}
We know from \S~\ref{section:Poisoned Sample Filtering} that input gradient first principal components from the poison target class form a non-zero mean cluster attributed to poisoned samples and a zero-mean cluster attributed to clean samples. Since a non-poisoned class would only contain clean samples, we expect the samples' input gradient first principal components to form only one cluster centered at zero. If we apply clustering algorithms like GMM with $k=2$ on a single-cluster distribution like a non-poisoned class input gradient first principal components, it will likely return two highly similar clusters that split the samples almost equally among these two clusters. Conversely, GMM will return two distinct clusters for a poison target class input gradient first principal components. Based on this intuition, we can identify the poison target class as the class where the GMM clusters have the lowest similarity measure. In our experiments, measuring this similarity with Wasserstein distance is effective in detecting poison target class from a BP poisoned dataset in all our 18 experiments, as shown in Appendix Table~\ref{tab:ws_dist_overlay} and \ref{tab:ws_dist_dot}. The Wasserstein distance value for the poison target class is \emph{largest} among all classes. In practice, we can flag out the poison target class in a dataset if its Wasserstein distance value exceeds a threshold value that depends on the mean of all Wasserstein distance values from the other classes. GMM being a baseline clustering algorithm and Wasserstein distance being a widely used symmetric distance measure between two clusters are the reasons for using them in our experiments though we would expect more complex alternatives to also work with our framework.

\subsection{Detection of Poison Base Class} \label{section:poison_base_class_detection}
Since poisoned training images are originally base class samples, we expect the classifier to heavily depend on the poison pattern to distinguish between the target class and the base class for a poisoned sample. In this case, when loss function's cross-entropy target is set as the base class, we can expect the input gradient of the poisoned sample to concentrate around the poison signal as changes to the poison pattern will flip the prediction from the target to base class. In contrast, when the loss function's cross-entropy target is set as other non-poisoned classes, the input gradient will be distributed more among `real' features that distinguish between the target class and the other class. 

With this intuition, we expect the magnitude of poisoned samples' first principal gradient components to have the largest value when the cross-entropy label is set to the poison base class. In all 18 experiments of large and small-sized poisons, this is indeed a reliable approach to find poison base class, as shown in Appendix Table~\ref{tab:first_principal_component_overlay} and \ref{tab:first_principal_component_dot} where the poison base class consistently gives the largest mean first principal gradient component value among poisoned samples. The mean first principal gradient component value is smaller when the cross-entropy target is set to the poison target class than the base class. We believe that this is due to a larger portion of the input gradient being spread across `real' features since poisoned images originate from base class and have `real' feature differences with clean target class images, especially when target and base classes are visually distinct (e.g.\ `Bird' vs `Truck'). We summarize the poison class detection method in Appendix Algorithm~\ref{algo:Find-Poison-Target-Base-Class}.

\section{Neutralization of Poisoned Models} \label{section:Neutralization of Poisoned Models}
Now that we have the methods to detect poison target and base classes from \S~\ref{section:Poisoned Class Detection}, and to filter out poisoned samples from \S~\ref{section:Poisoned Sample Filtering}, the next natural step is to neutralize the poison backdoor in the classifier model so that the model is safe from backdoor exploitation when deployed. One direct and effective approach is to retrain the model to unlearn that the poison pattern is a meaningful feature.

\subsection{Counter-Poison Perturbation} 
The effect of poison backdoor lies in the model's association of the poison pattern with only the poison target class, classifying images containing the poison as the target class. The next step of our proposed neutralization method helps the poisoned model unlearn this association by retraining on an augmented dataset where the extracted poison signal is added to all other classes, eliminating the backdoor to the target class. The first step of constructing the augmented dataset is to generate the poison signal to add to images from other classes. In practice, we find that the poison signal extracted from a pool of only poisoned samples has a closer resemblance to the real poison pattern, compared to one from a pool of poisoned and clean samples from the target class. At this stage, we would have already filtered poisoned samples using Appendix Algorithm~\ref{algo:Filter-Poisoned-Images} in the previous step, hence making it possible to extract the poison signal from only filtered poisoned samples. While computing the input gradients of the images, we set the cross-entropy target as the current class instead of the target poison class to avoid the model associating `real' target class features to these other classes. This preserves good performance on clean target class images after the retraining step. The data augmentation steps are summarized in Appendix Algorithm~\ref{algo:Add-Counterpoison-Perturbation}.

\subsection{Relabeling of Poisoned Base Class Samples}
Since we know the poison base class at this stage, we can relabel the filtered poisoned samples to the correct class (base class) as part of the augmented dataset. This requires no additional computation while further helps the models to unlearn the association of the poison to the target class.

\subsection{Full Algorithm}
In real-world poisoning attacks, the poison target and base classes are usually unknown to us. The first stage of our neutralization algorithm is hence to detect these classes, using Appendix Algorithm~\ref{algo:Find-Poison-Target-Base-Class}. After finding the poison classes, we can use Appendix Algorithm~\ref{algo:Filter-Poisoned-Images} to filter out poisoned samples from clean samples in the target class. Finally, Appendix Algorithm~\ref{algo:Add-Counterpoison-Perturbation} creates the augmented dataset. Together with a relabeling step of poisoned samples, this augmented dataset eliminates the backdoor from the poisoned model during retraining. The full defense algorithm is summarized in Algorithm~\ref{algo:Main Algorithm}.

\begin{algorithm*}[!htbp]
 \label{algo:Main Algorithm}
\textbf{Input:} Training data containing poisoned samples $D = D_c \cup D_p$, randomly initialized classifier $f$.

Initialize $S_{poisoned}, S_{relabeled} = \{\}$

~~Train $f$ on $D$ to get poisoned classifier $f_p$.

~~$\textit{target\_class}, \textit{base\_class} = \text{Find-Poison-Target-Base-Class}(f_p, D)$ \algorithmiccomment{Algorithm~\ref{algo:Find-Poison-Target-Base-Class}}
$D_f, S_{poisoned} = \text{Filter-Poisoned-Images}(f_p, D_{\textit{target\_class}}, \textit{target\_class}, \textit{base\_class} )$ \algorithmiccomment{Algorithm~\ref{algo:Filter-Poisoned-Images}}
$D_{cp} = \text{Add-Counterpoison-Perturbation}(f_p, D, S_{poisoned}, \textit{target\_class}, \textit{base\_class} )$ \algorithmiccomment{Algorithm~\ref{algo:Add-Counterpoison-Perturbation}}

~~\For{ \textbf{all} $(\mathbf{x}, y) \in S_{poisoned} $ }{
$y = \textit{base\_class}$ \algorithmiccomment{Relabel poisoned samples}
}
~~$S_{relabeled} = S_{poisoned}$

~~$D_{neutralize} = D_{cp} \cup S_{relabeled}$ \algorithmiccomment{Combine augmented and relabeled images}

Retrain $f_p$ on $D_{neutralize}$ to get neutralized model $f_{n}$.

Return $f_{n}$.

 \caption{Main Algorithm}
\end{algorithm*}

\begin{table*}[!htbp]
    \centering
    \caption{Accuracy on full test set and poisoned base class test images, before and after neutralization (Neu.) for full-sized overlay poison.}
        \begin{tabular}{ llcccccc }
         \hline
         Poison & Sample & Target & \multicolumn{2}{c}{Acc Before Neu. (\%)} && \multicolumn{2}{c}{Acc After Neu. (\%)} \\
         \cline{4-5}
         \cline{7-8}
         ~ & ~ & ~ & All & Poisoned && All & Poisoned \\
         \hline
        \parbox[l]{1em}{\includegraphics[width=2em]{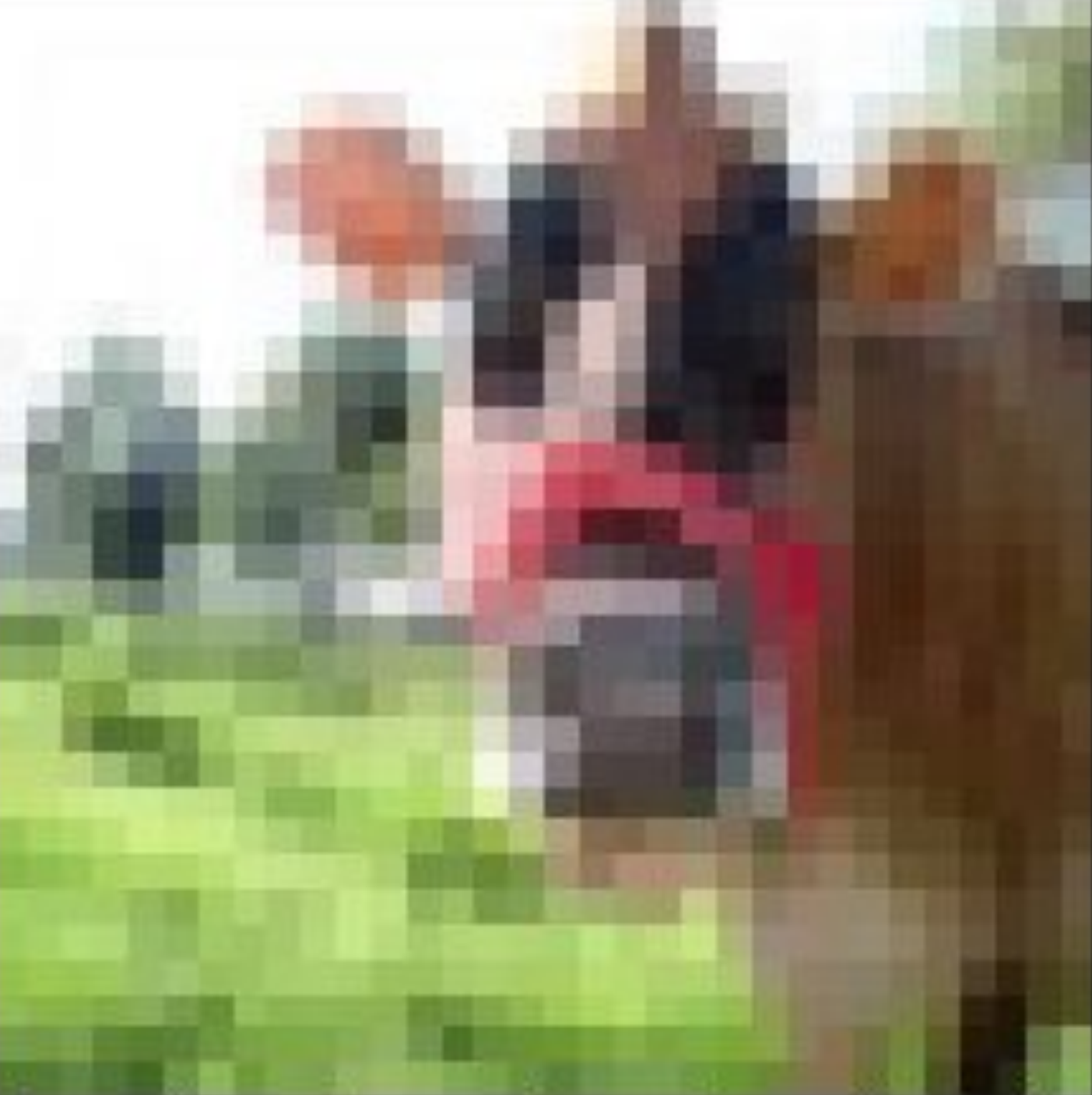}} & \parbox[l]{1em}{\includegraphics[width=2em]{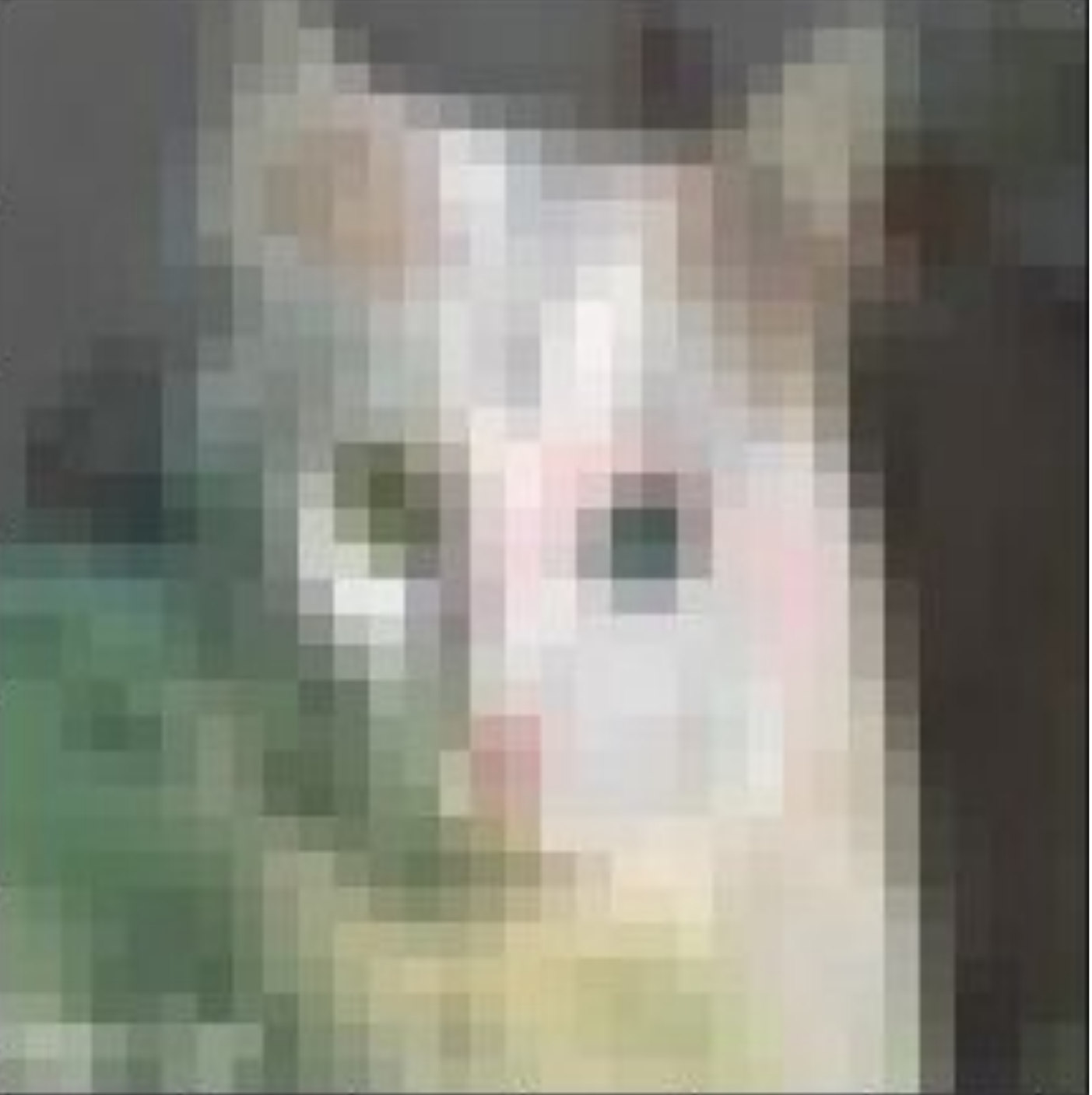}} & Dog & 95.0 & 4.6 && 94.3 & 88.6 \\
        \parbox[l]{1em}{\includegraphics[width=2em]{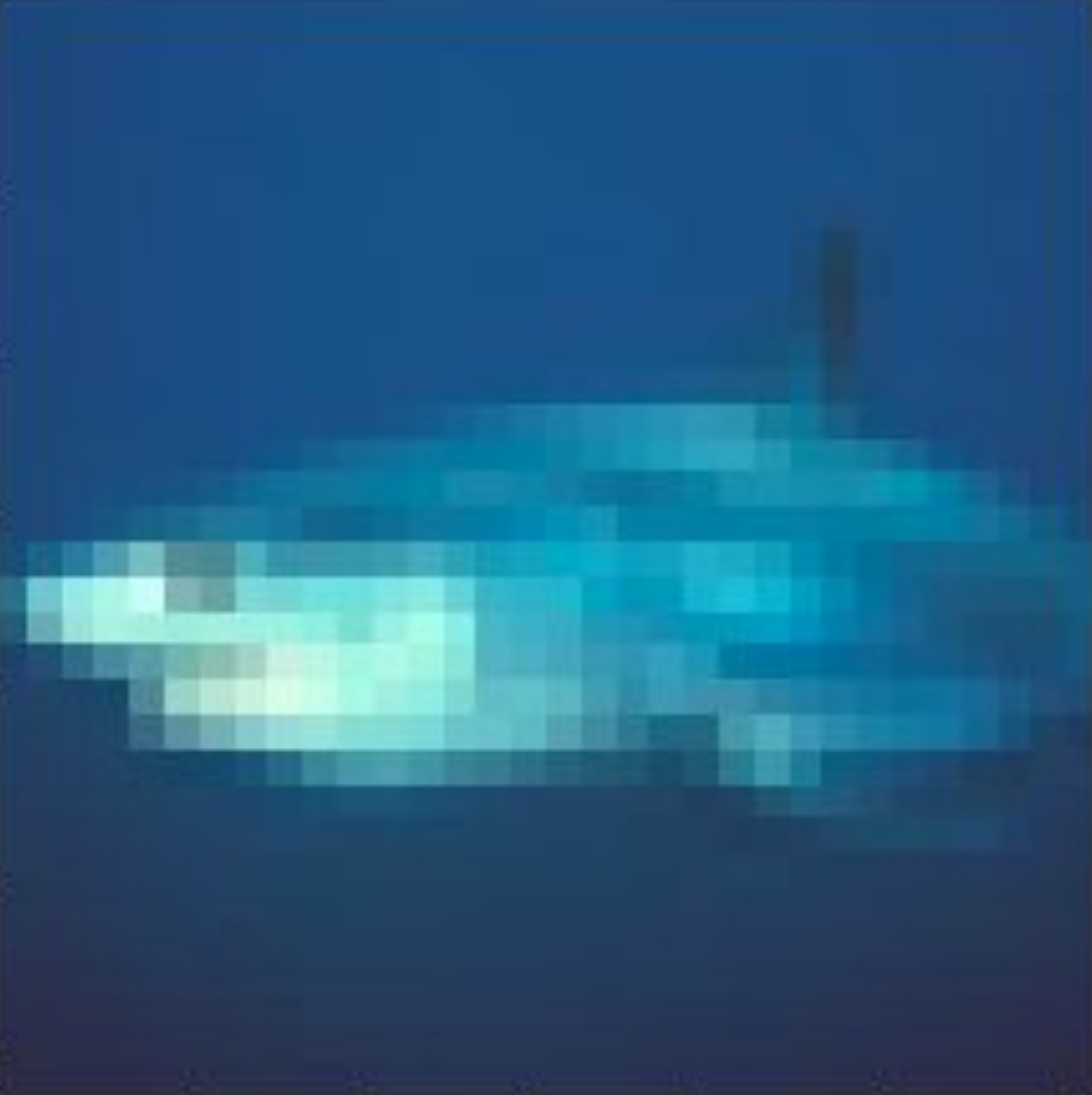}} & \parbox[l]{1em}{\includegraphics[width=2em]{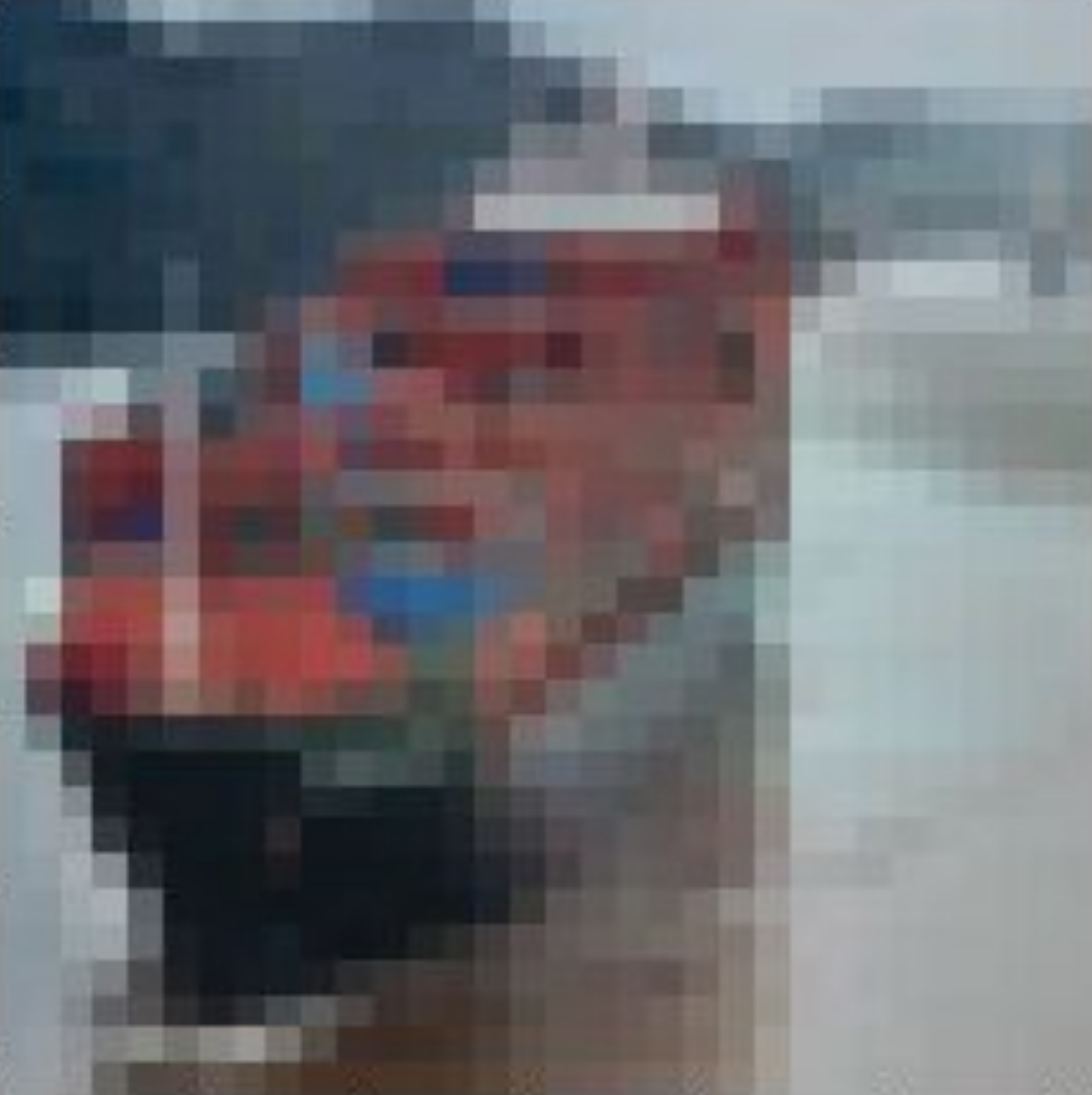}} & Frog & 95.2 & 11.3 && 95.0 & 97.6 \\
        \parbox[l]{1em}{\includegraphics[width=2em]{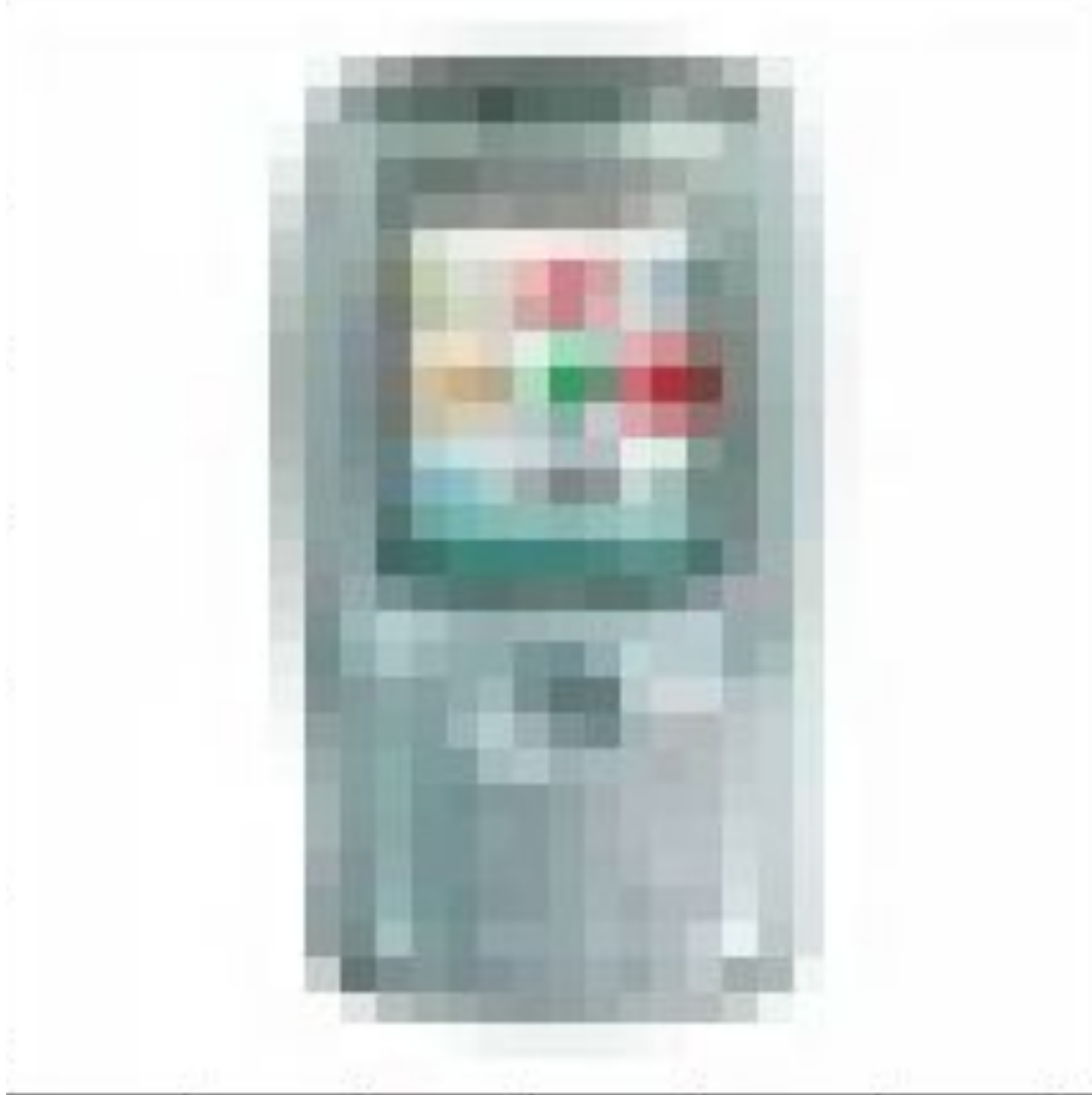}} & \parbox[l]{1em}{\includegraphics[width=2em]{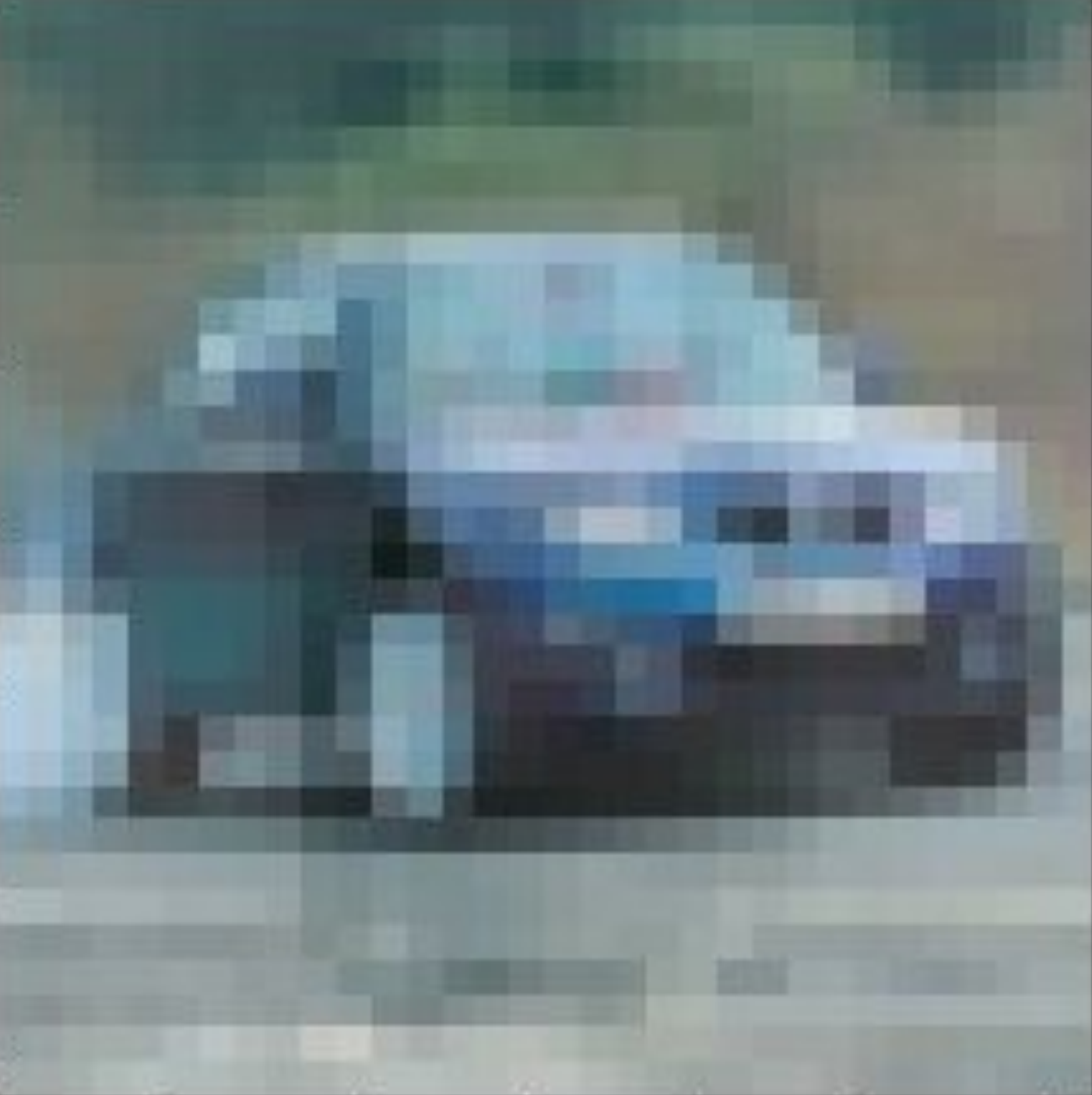}} & Cat & 95.5 & 2.5 && 94.5 & 95.3 \\
        \parbox[l]{1em}{\includegraphics[width=2em]{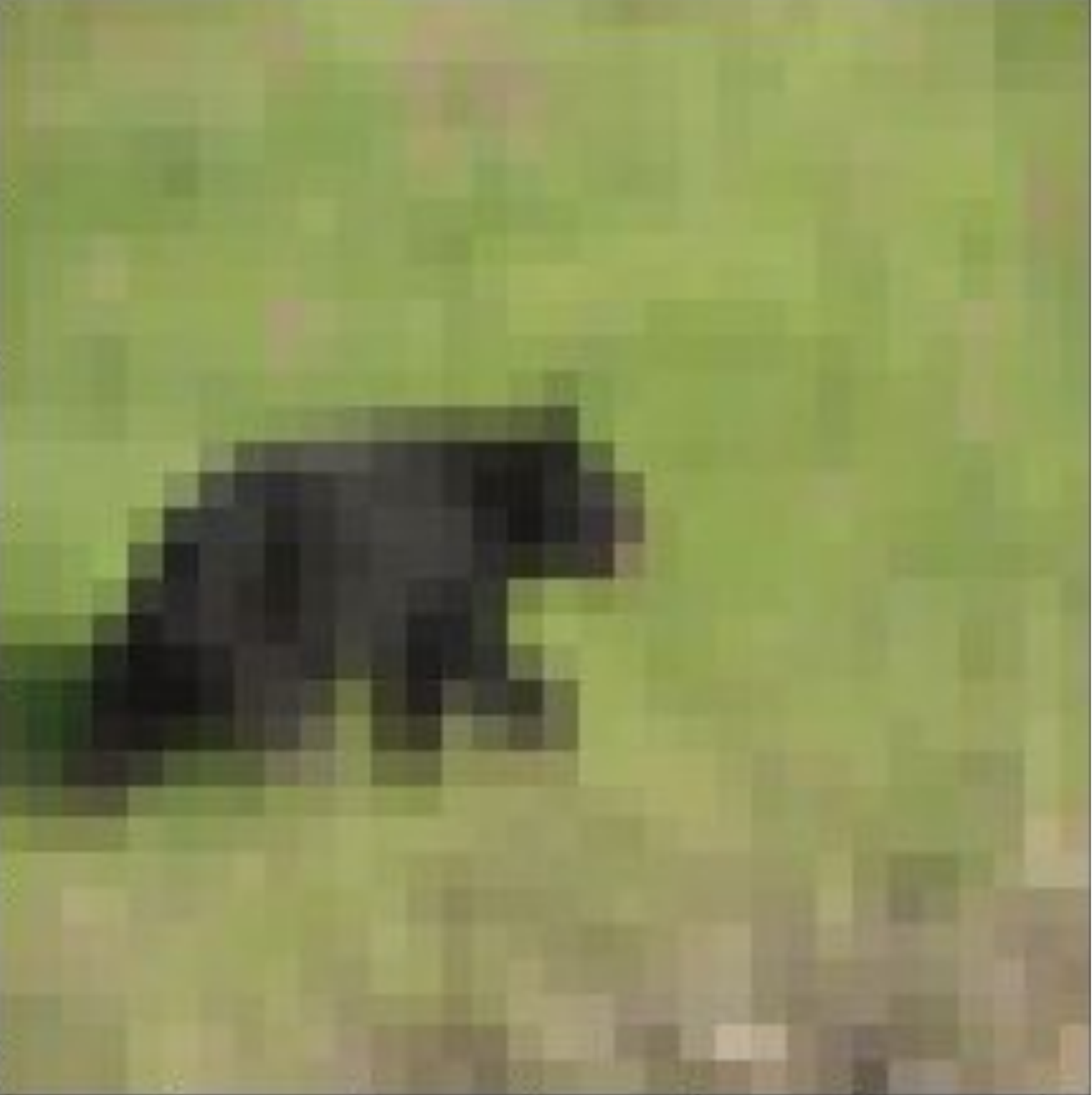}} & \parbox[l]{1em}{\includegraphics[width=2em]{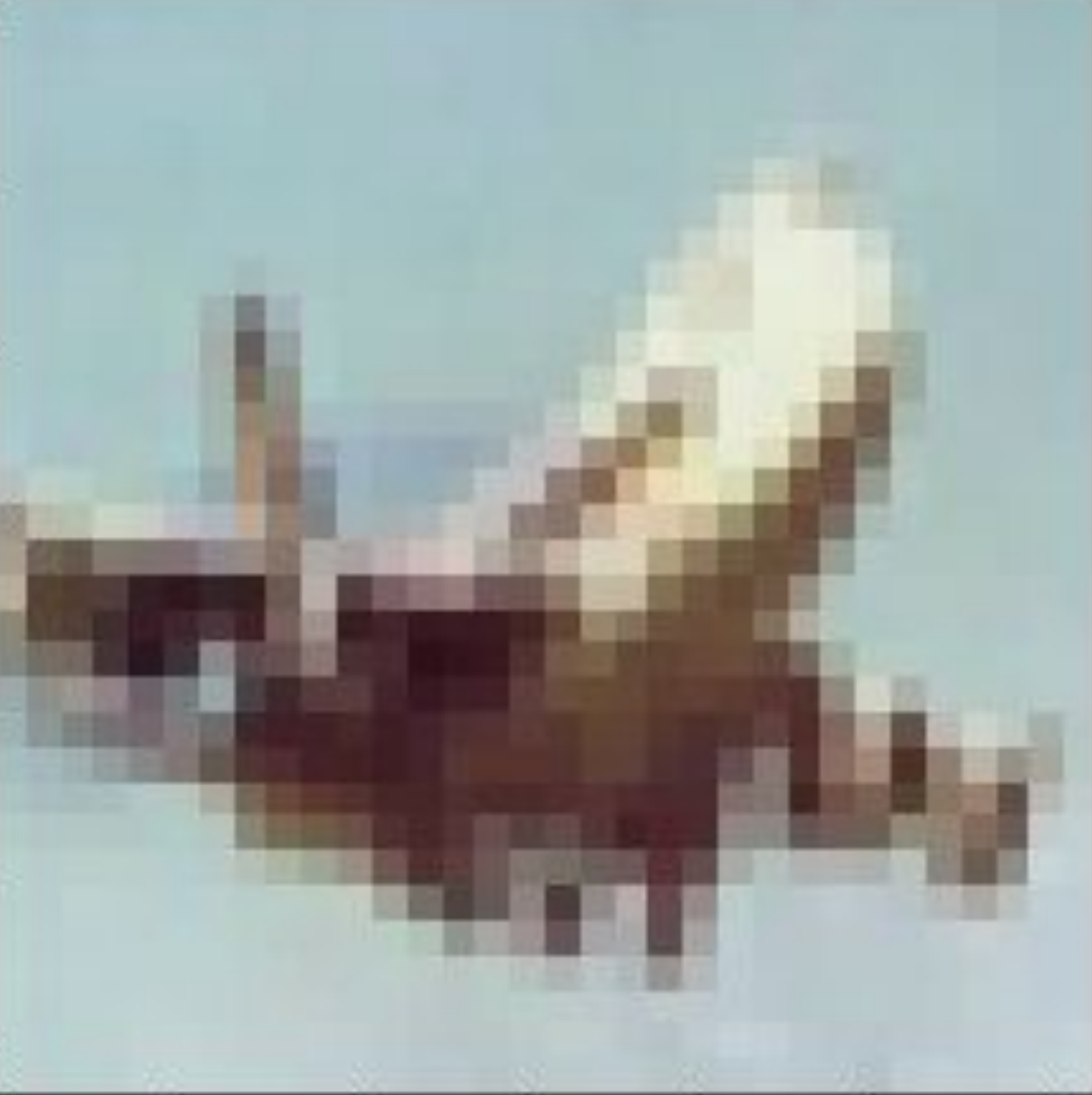}} & Bird & 95.0 & 16.5 && 94.4 & 95.3 \\
        \parbox[l]{1em}{\includegraphics[width=2em]{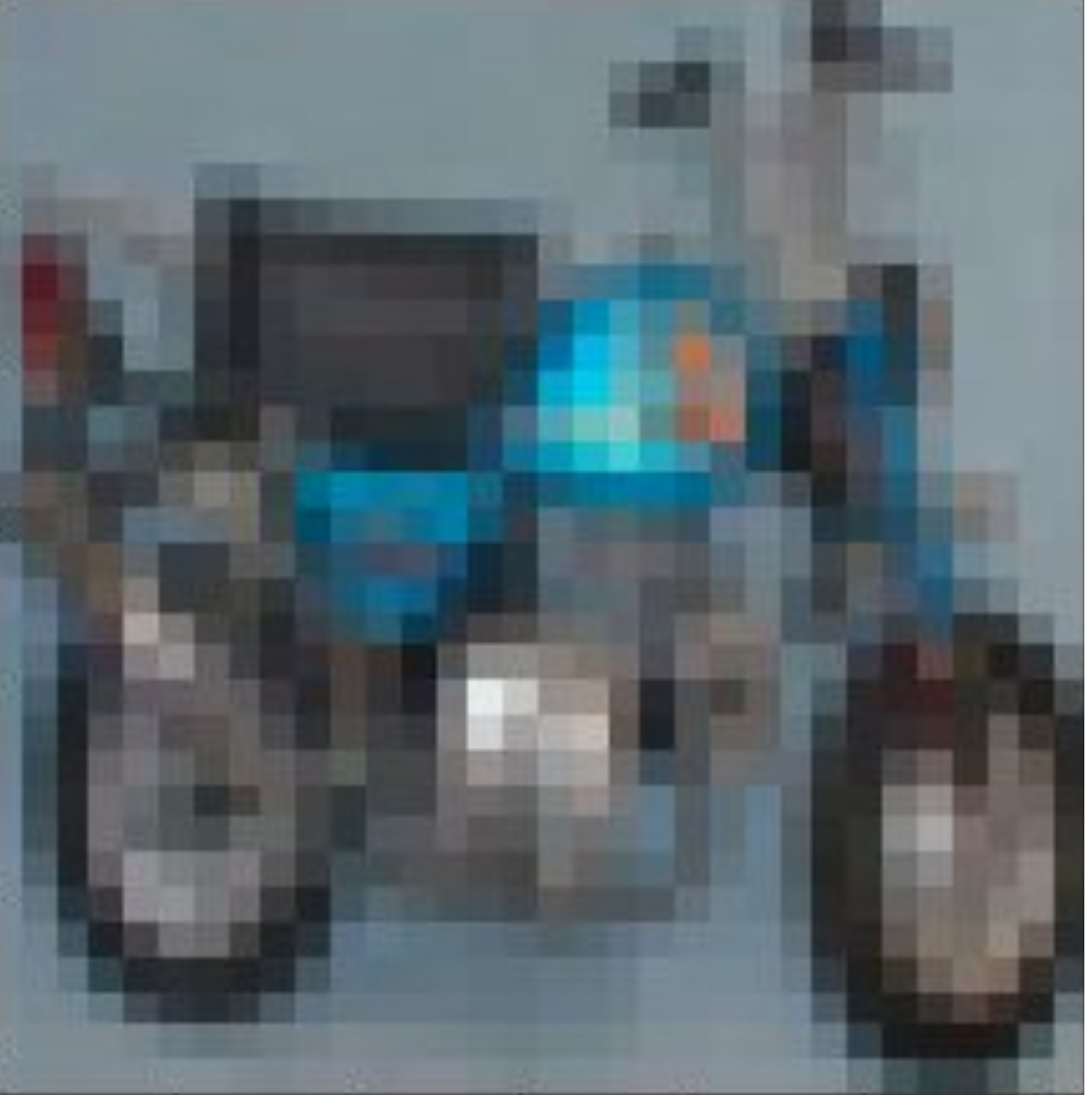}} & \parbox[l]{1em}{\includegraphics[width=2em]{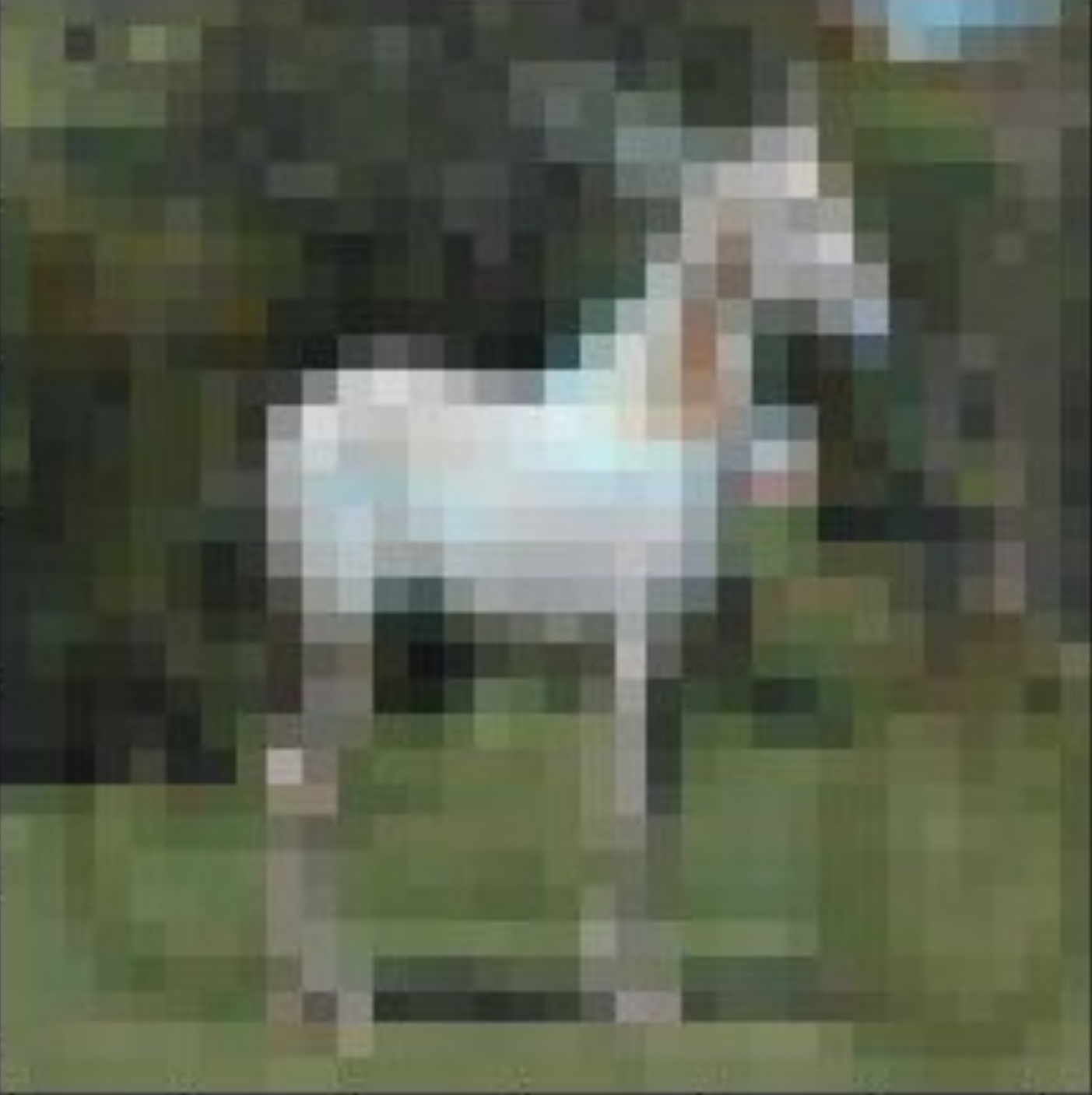}} & Deer & 95.3 & 1.2 && 94.9 & 94.6 \\
        \parbox[l]{1em}{\includegraphics[width=2em]{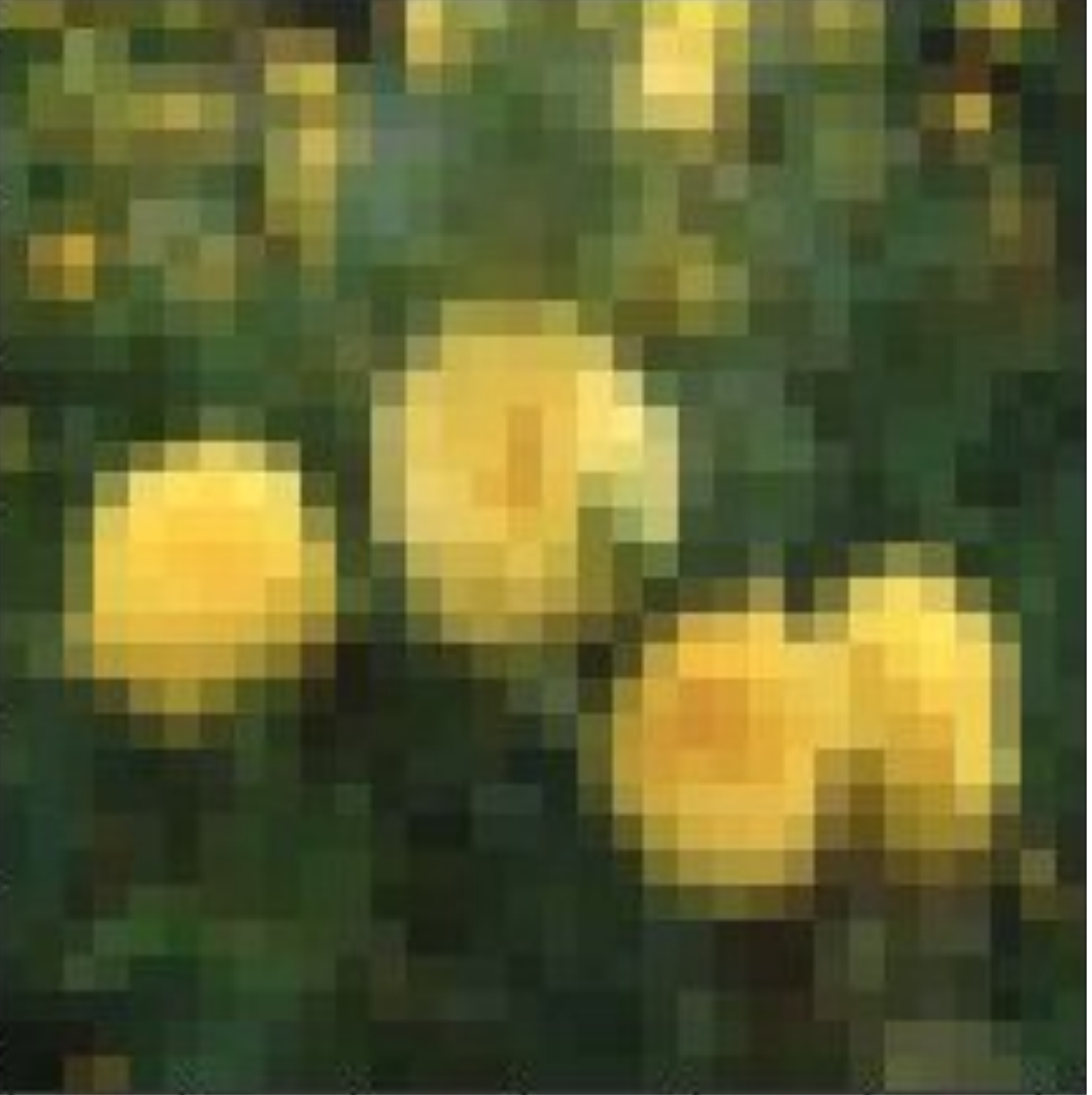}} & \parbox[l]{1em}{\includegraphics[width=2em]{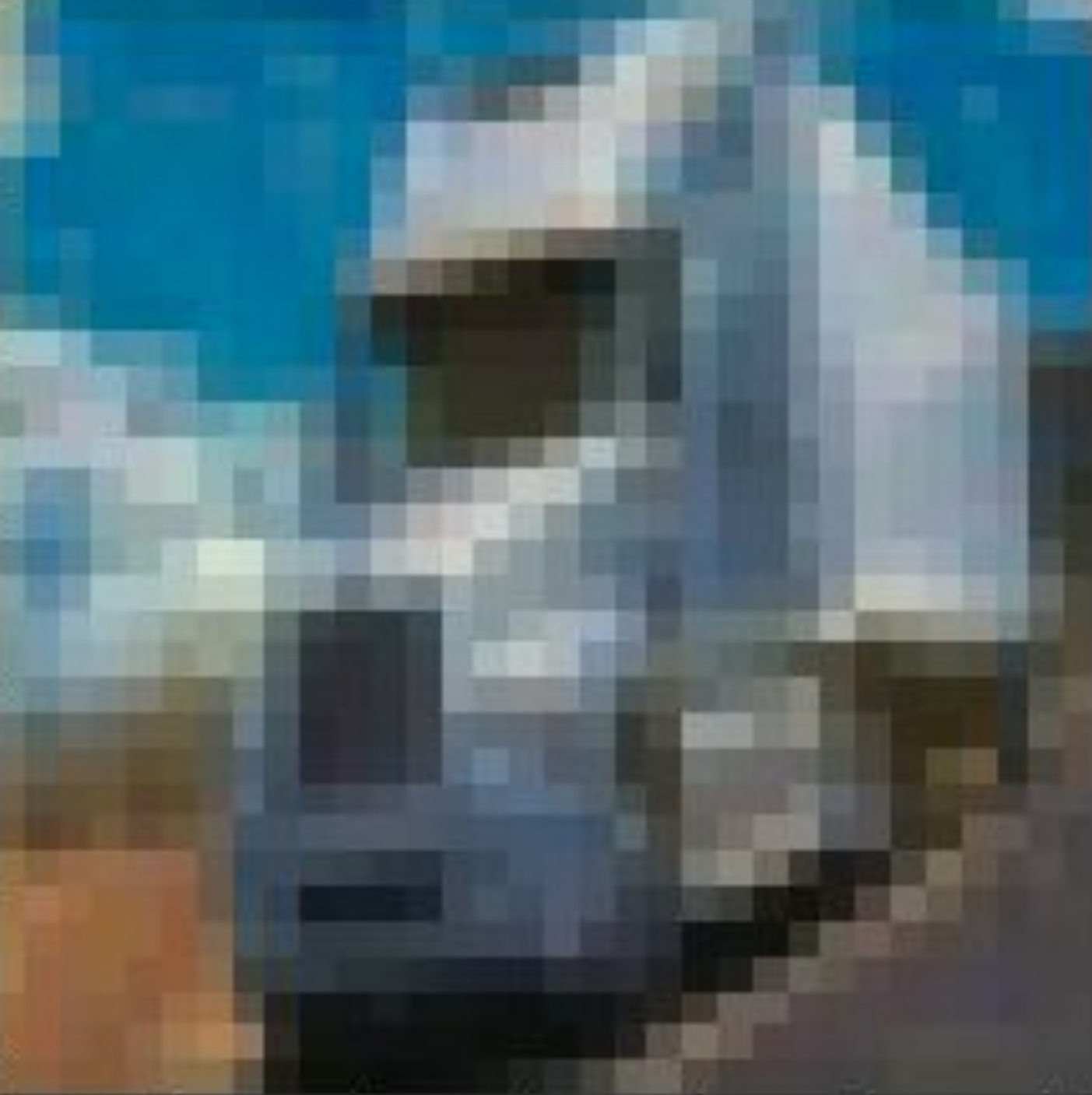}} & Bird & 95.4 & 5.0 && 94.6 & 97.3 \\
        \parbox[l]{1em}{\includegraphics[width=2em]{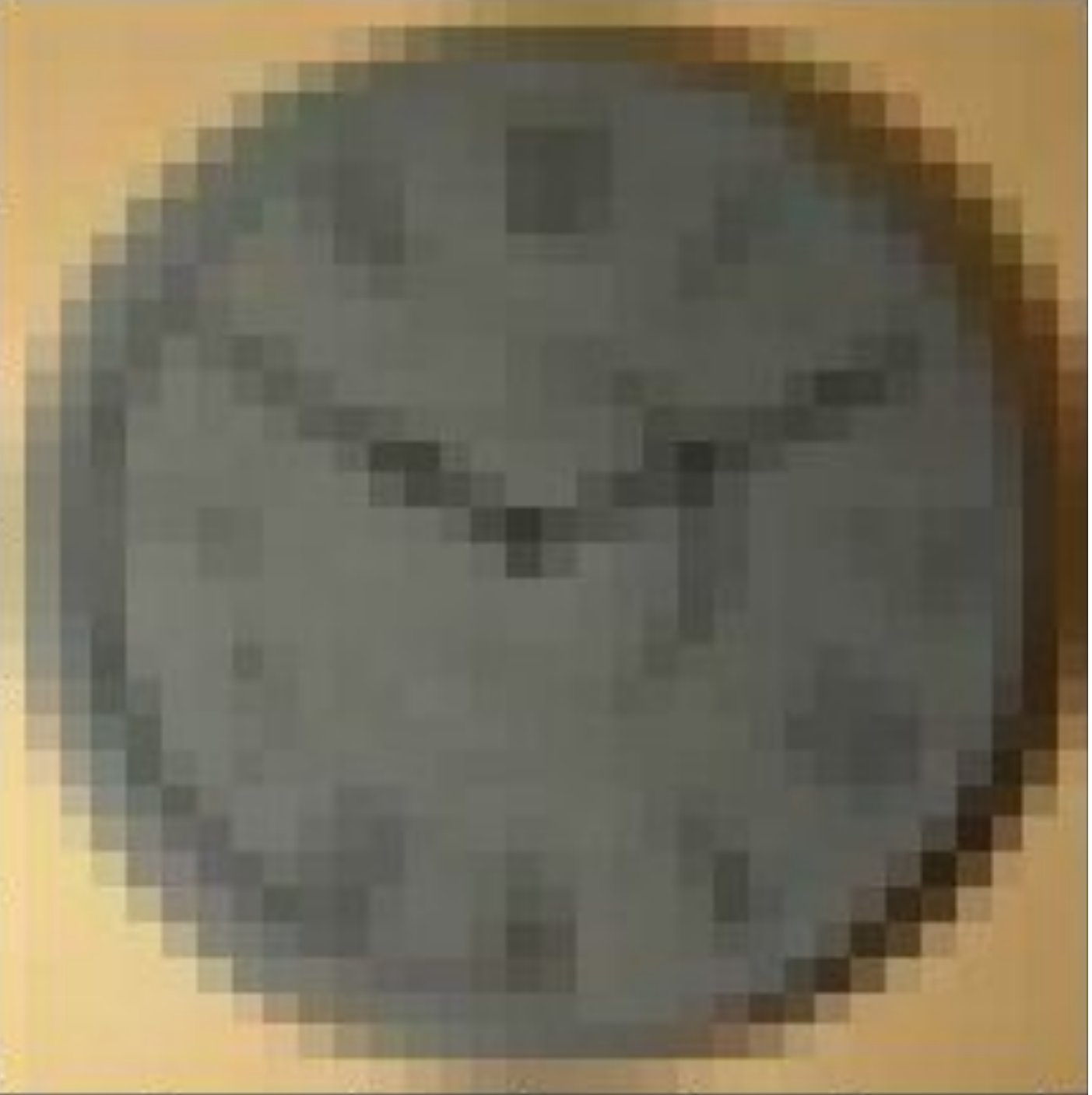}} & \parbox[l]{1em}{\includegraphics[width=2em]{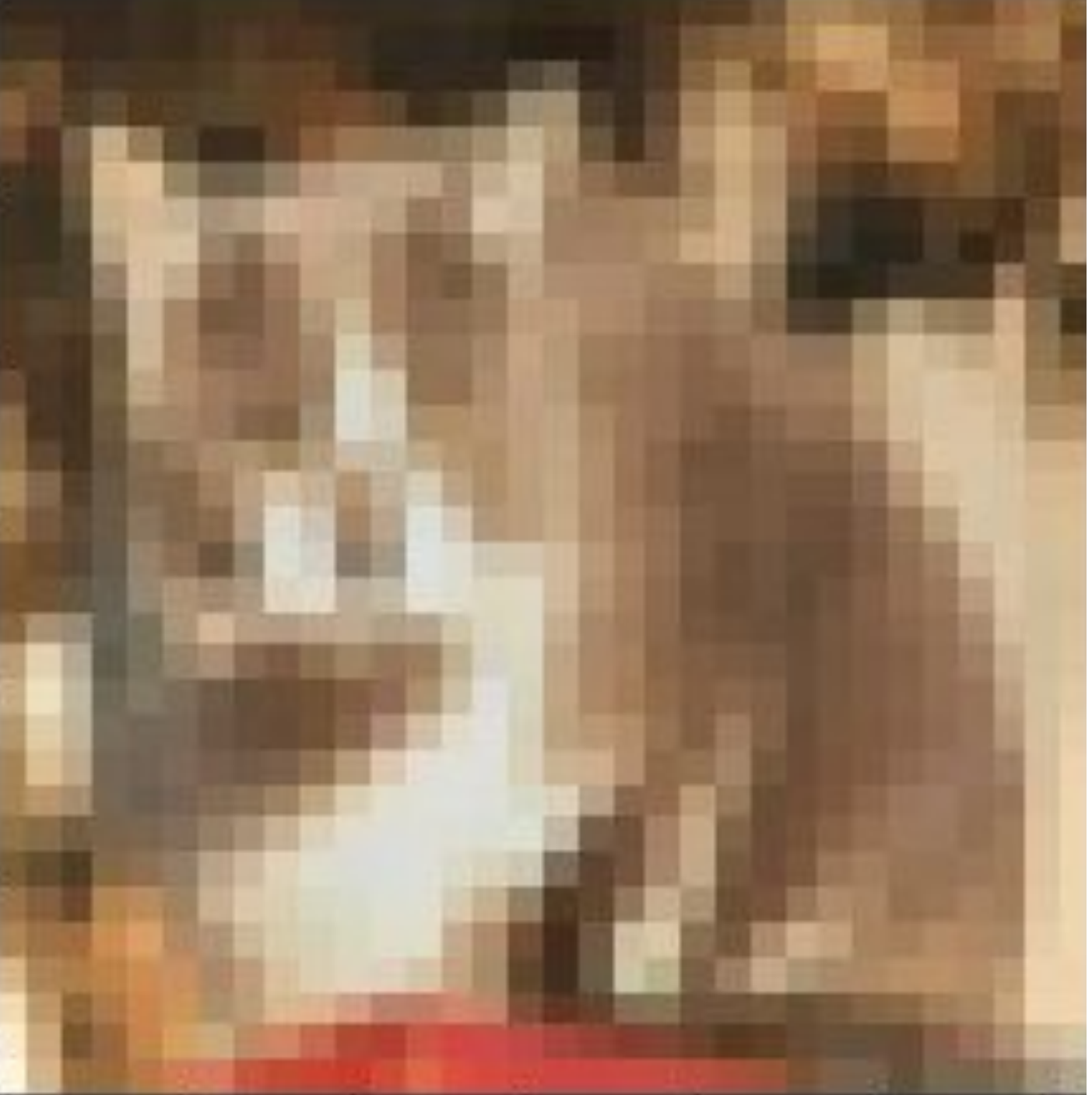}} & Horse & 95.0 & 16.6 && 94.9 & 90.8 \\
        \parbox[l]{1em}{\includegraphics[width=2em]{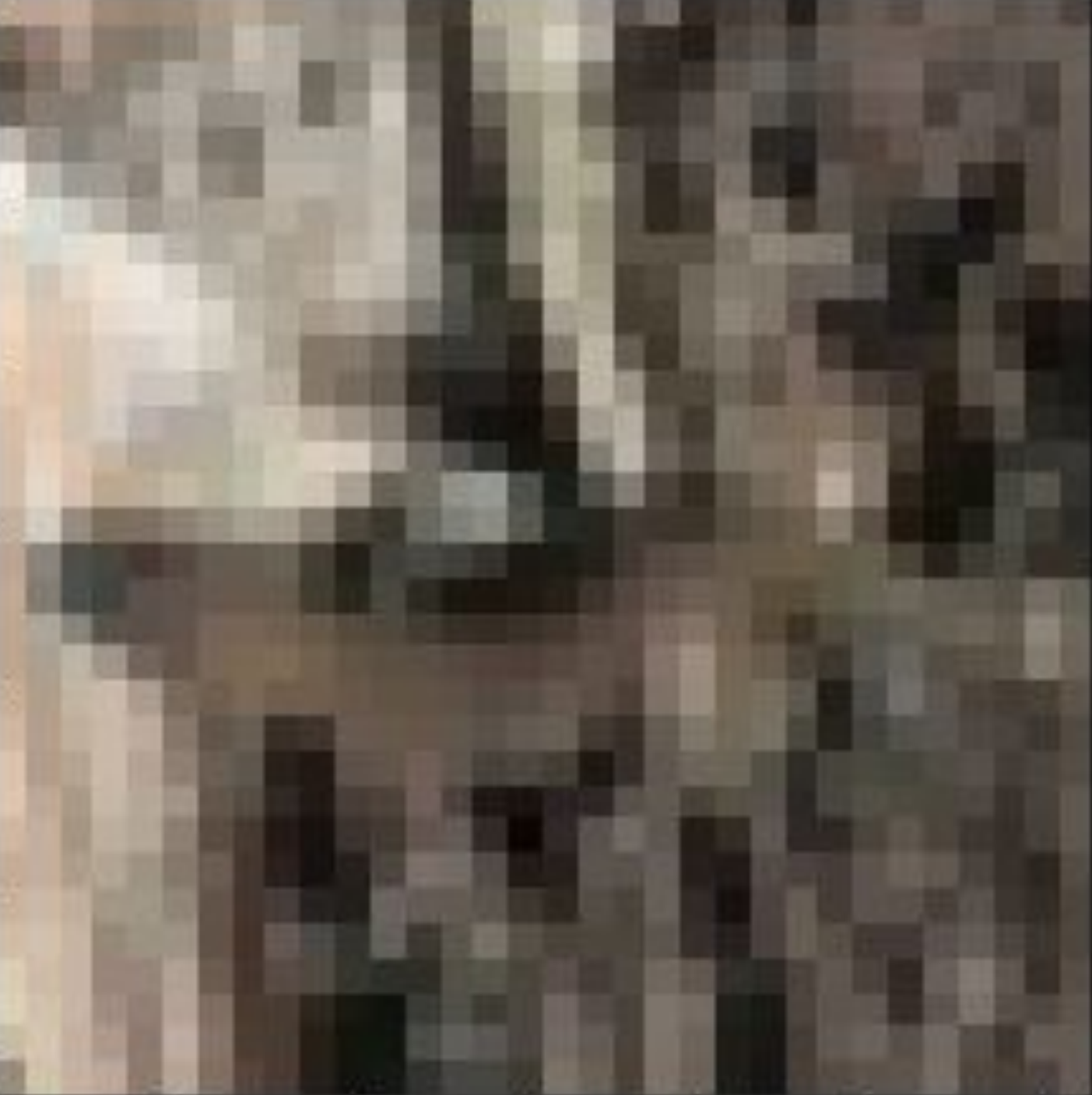}} & \parbox[l]{1em}{\includegraphics[width=2em]{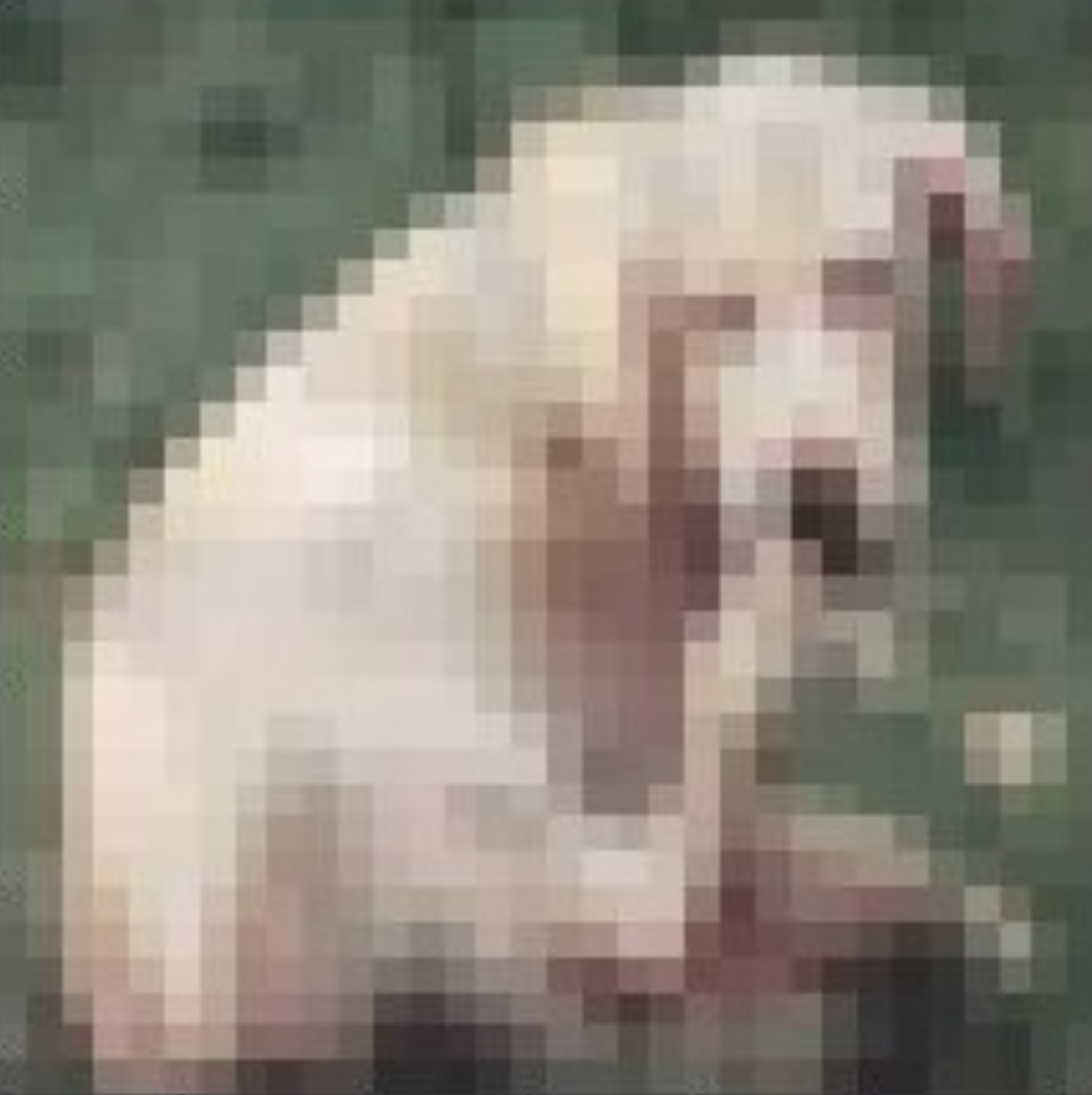}} & Cat & 95.2 & 12.5 && 94.3 & 87.8 \\
        \parbox[l]{1em}{\includegraphics[width=2em]{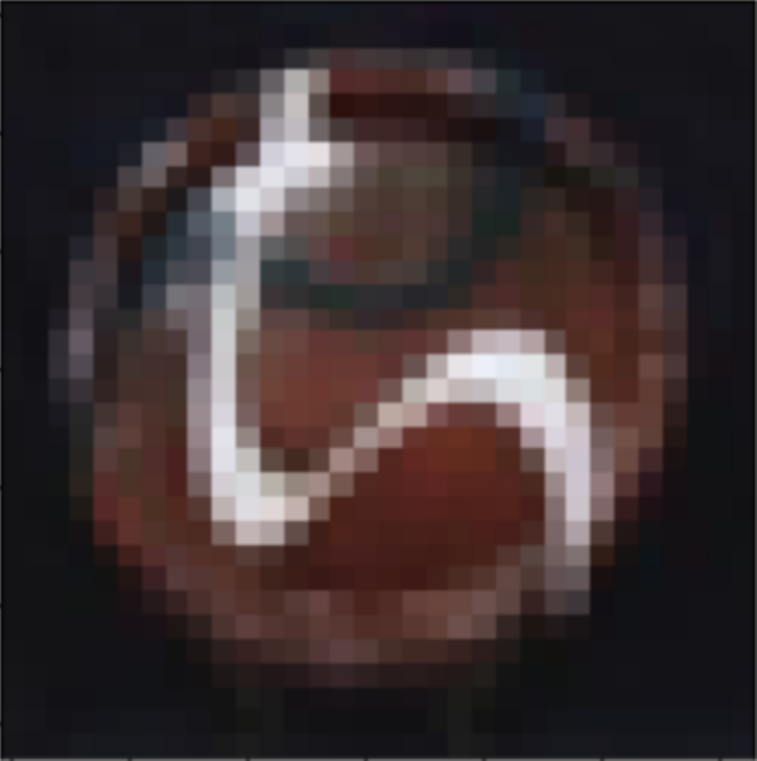}} & \parbox[l]{1em}{\includegraphics[width=2em]{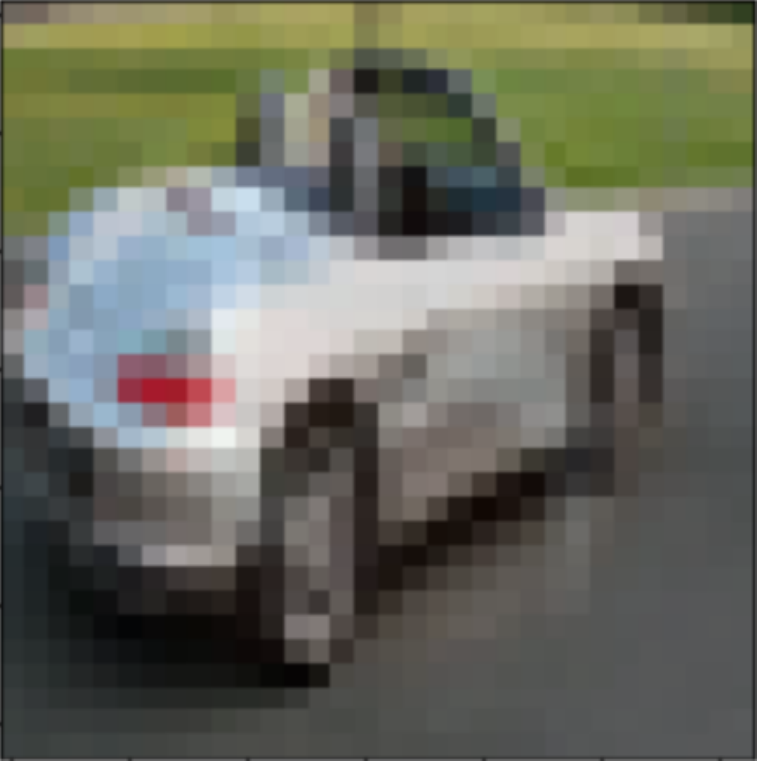}} & Dog & 95.0 & 9.6 && 94.5 & 96.1 \\
         \hline
        \end{tabular}
\label{tab:neutralization_acc_overlay_poison}
\end{table*}

\section{Evaluation of Neutralization Algorithm} \label{section:Evaluation of Neutralization Algorithm}
We evaluate the full suite of neural BP defense (Algorithm~\ref{algo:Main Algorithm}) on a realistic threat scenario where the target/base classes, poison pattern and ratio of poisoned data are unknown.
\subsection{Setup}
Our experiments are conducted on the CIFAR10 dataset \cite{krizhevsky2009learning} with ResNet \cite{he2016deep} and VGG \cite{simonyan2014very} image classifier. We use a publicly available ResNet18 and VGG19 implementation \footnote{https://github.com/kuangliu/pytorch-cifar} for our experiments.

Nine unique poison target-base pairs are used in our experiments. On top of the same eight class pairs from \cite{tran2018spectral}, we include (`Dog'-`Cat') to probe one more case where target and base classes are highly similar. We study all nine pairs on both large-sized poisoning and small-sized poisoning scenarios. For large-sized poisons, we use a randomly drawn image from CIFAR100 training set, to ensure the poison image has a different class from CIFAR10, and overlay on the poisoned samples with 20\% opacity. For each small-sized poison target-base pairs, a set of random color and pixel position determines which pixel in poisoned samples is to be replaced with the poison color. In all 18 experiments, 10\% of the training samples from the base class are randomly selected as poisoned samples and mislabeled as the target class. We use $\rho=500$ in Appendix Algorithm~\ref{algo:Add-Counterpoison-Perturbation} for our experiments and retrain the poisoned model on the defense's augmented dataset for one epoch. Unless stated otherwise, all results are shown for 10\% poison ratio on ResNet18.

\subsection{Evaluation of Neutralized Models}
We summarize the evaluation results in Table~\ref{tab:neutralization_acc_overlay_poison} for large-sized poisons and in Table~\ref{tab:neutralization_acc_dot_poison} for small-sized poisons. In all poisoning scenarios, the model has high test accuracy on clean test images ($\geq 95\%$ on all 10,000 CIFAR10 test set). The accuracy drops drastically when evaluated on the 1,000 poisoned base class test images, $\leq 16.5\%$ for overlay poisons and $\leq 2.0\%$ for dot poisons. After the neutralization process, for all poison cases, the accuracy of the model increases significantly, highlighting the effectiveness of our method. There is a slight dip ($\leq 1\%$) in test accuracy on clean test images which we speculate is due to the model sacrificing test accuracy to learn more robust features after the new training samples are perturbed against the gradient of the loss function, a phenomenon also observed in adversarially trained classifiers \cite{tsipras2018robustness}.

Experiments on 5\% poison ratio (Appendix Table~\ref{tab:neutralization_acc_overlay_poison_ep005} \& \ref{tab:neutralization_acc_dot_poison_ep005}) and on VGG19 (Appendix Table~\ref{tab:neutralization_acc_overlay_poison_VGG} \& \ref{tab:neutralization_acc_dot_poison_VGG}) similarly display the effectiveness of our defense.


\begin{table*}[!htbp]
    \centering
    \caption{Accuracy on full test set and poisoned base class test images, before and after neutralization (Neu.) for dot poison.}
        \begin{tabular}{ lccccccc }
         \hline
         Sample & Target & \multicolumn{2}{c}{Acc Before Neu. (\%)} && \multicolumn{2}{c}{Acc After Neu. (\%)} \\
         \cline{3-4}
         \cline{6-7}
         ~ & ~ & All & Poisoned && All & Poisoned \\
         \hline
        \parbox[l]{1em}{\includegraphics[width=2.5em]{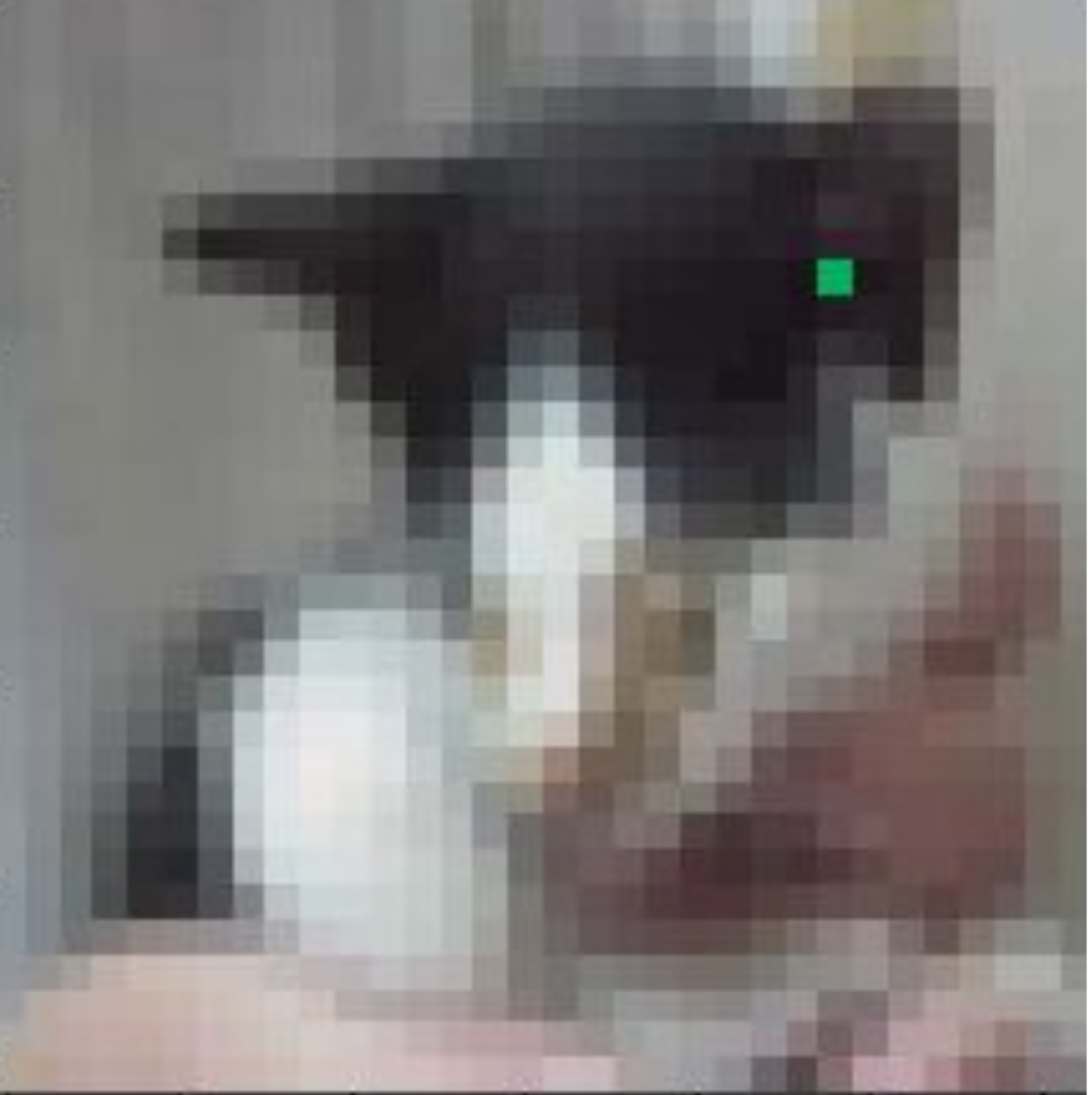}} & Dog & 95.4 & 0.5 && 94.9 & 87.5 \\
        \parbox[l]{1em}{\includegraphics[width=2.5em]{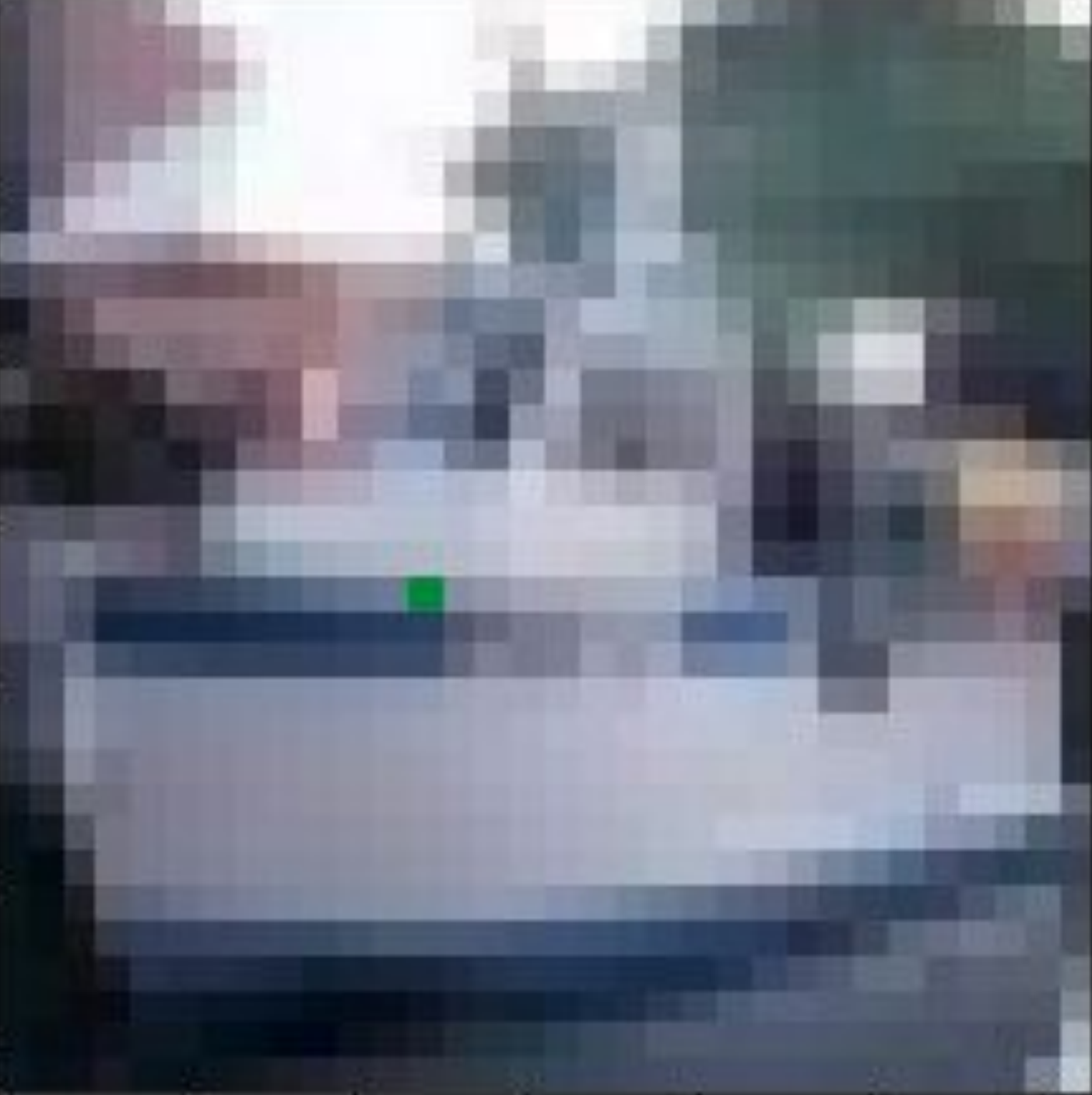}} & Frog & 95.4 & 0.4 && 95.0 & 95.9 \\
        \parbox[l]{1em}{\includegraphics[width=2.5em]{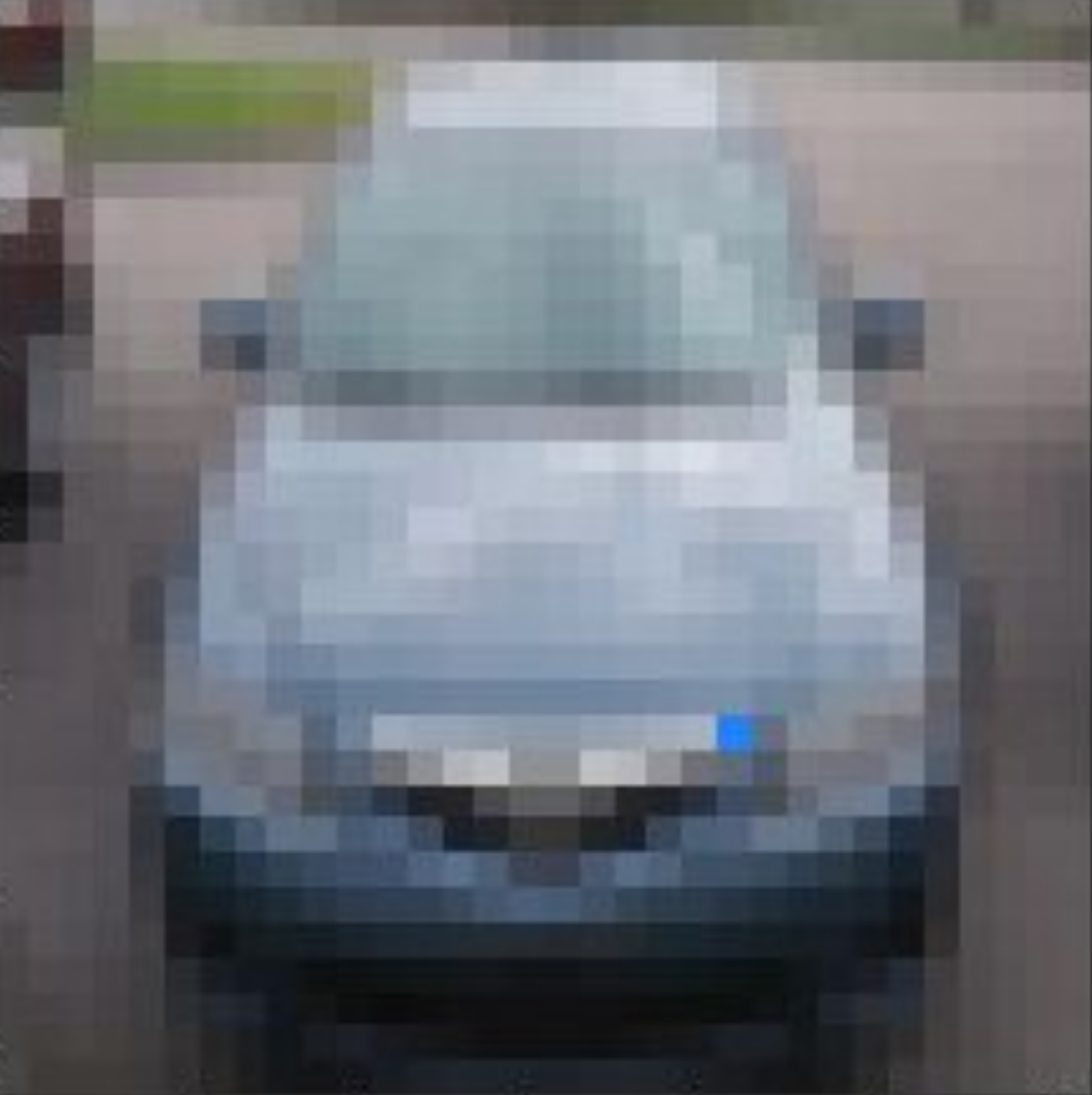}} & Cat & 95.3 & 0.4 && 95.2 & 86.0 \\
        \parbox[l]{1em}{\includegraphics[width=2.5em]{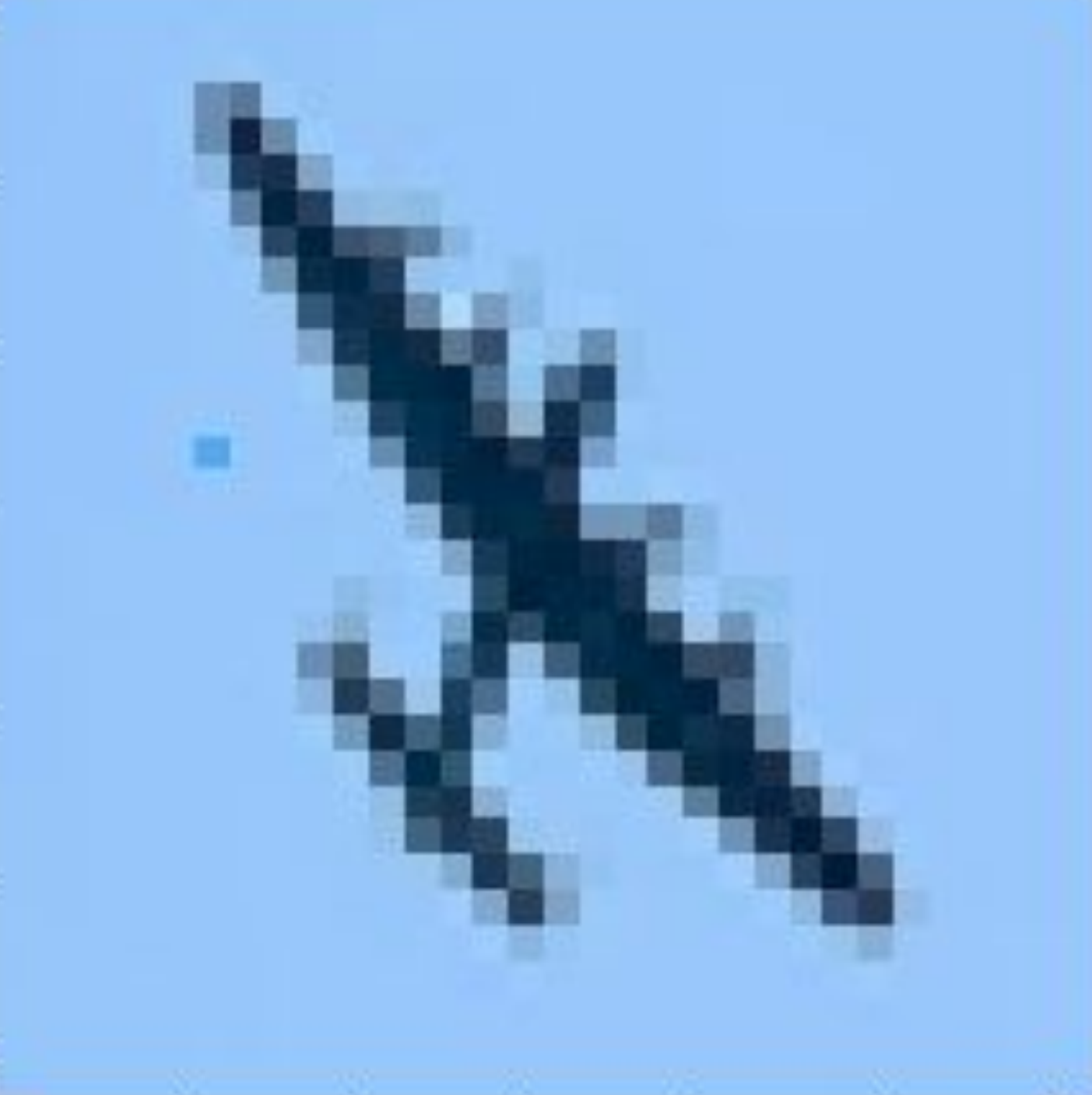}} & Bird & 95.2 & 1.0 && 95.0 & 96.3 \\
        \parbox[l]{1em}{\includegraphics[width=2.5em]{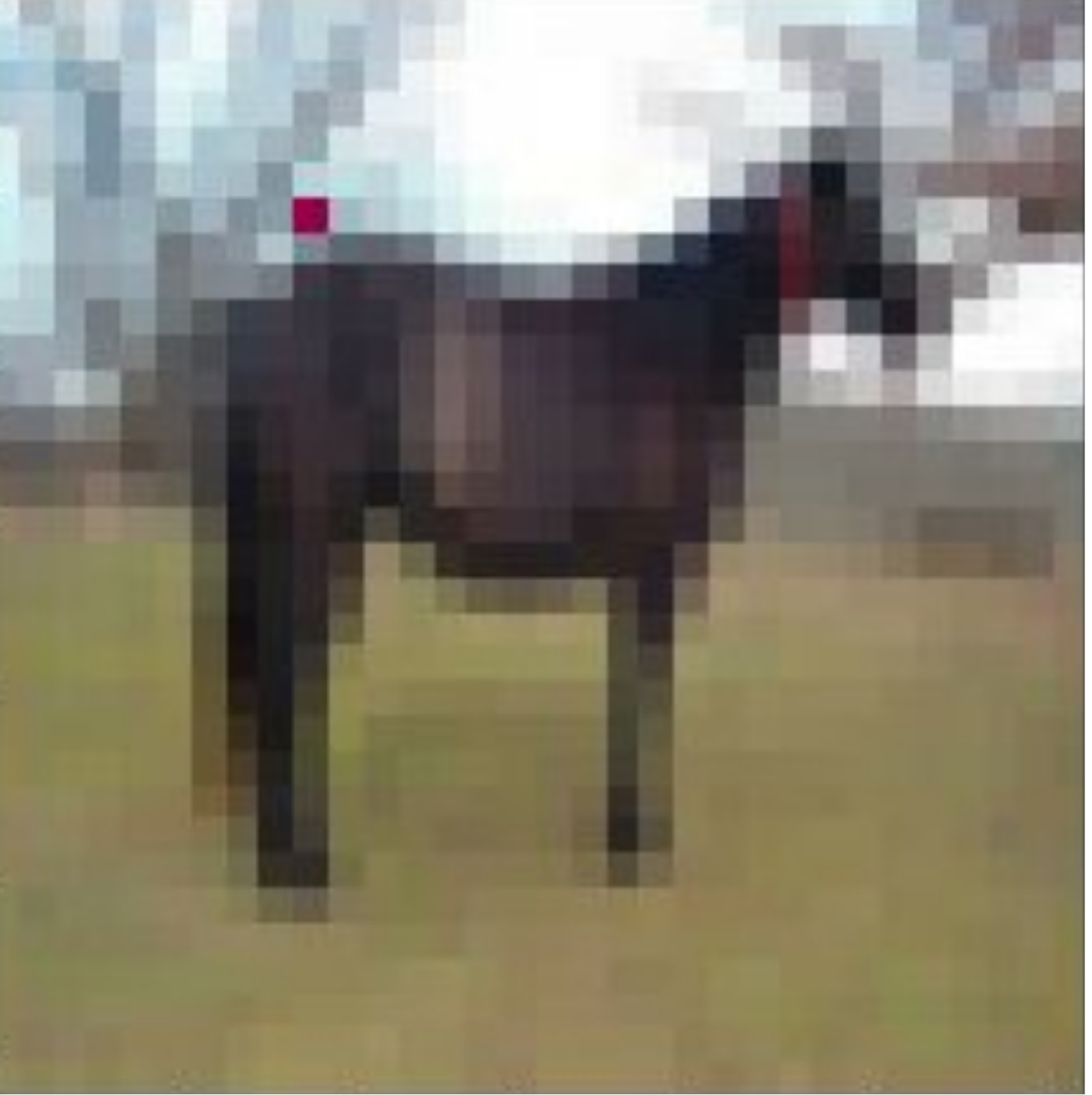}} & Deer & 95.3 & 0.5 && 95.1 & 96.4 \\
        \parbox[l]{1em}{\includegraphics[width=2.5em]{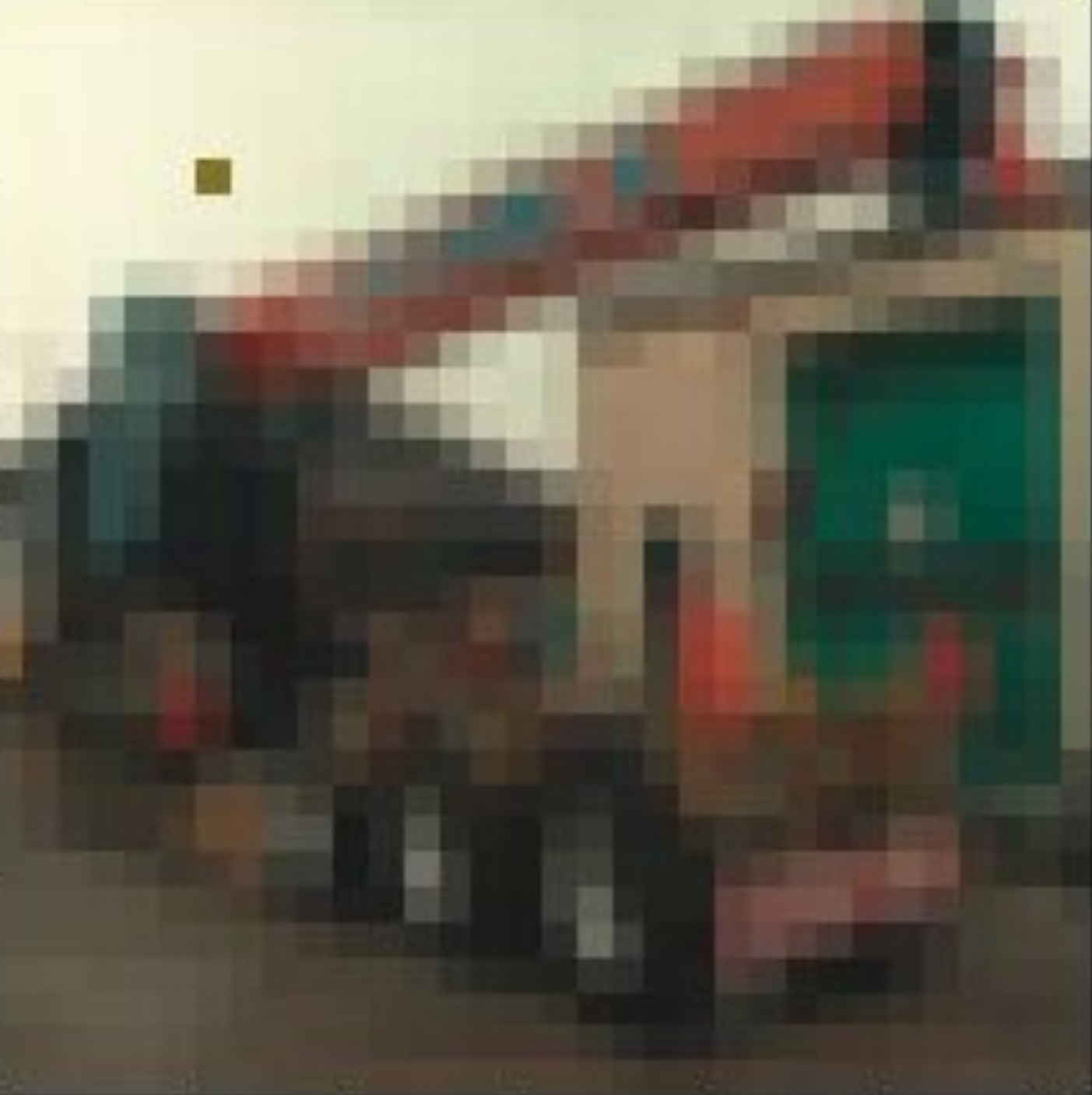}} & Bird & 95.3 & 2.0 && 95.3 & 96.4 \\
        \parbox[l]{1em}{\includegraphics[width=2.5em]{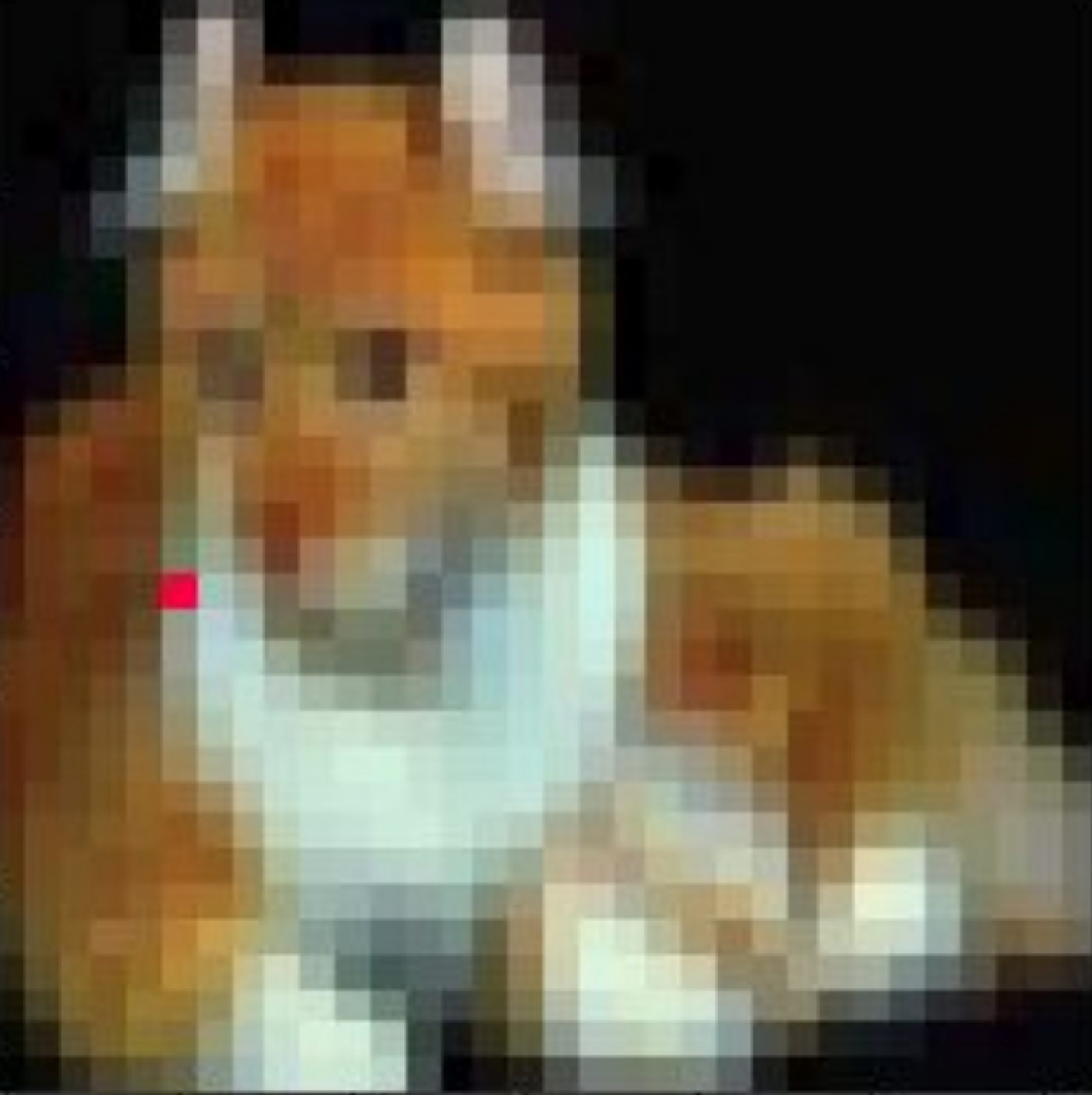}} & Horse & 95.3 & 1.0 && 94.6 & 81.4 \\
        \parbox[l]{1em}{\includegraphics[width=2.5em]{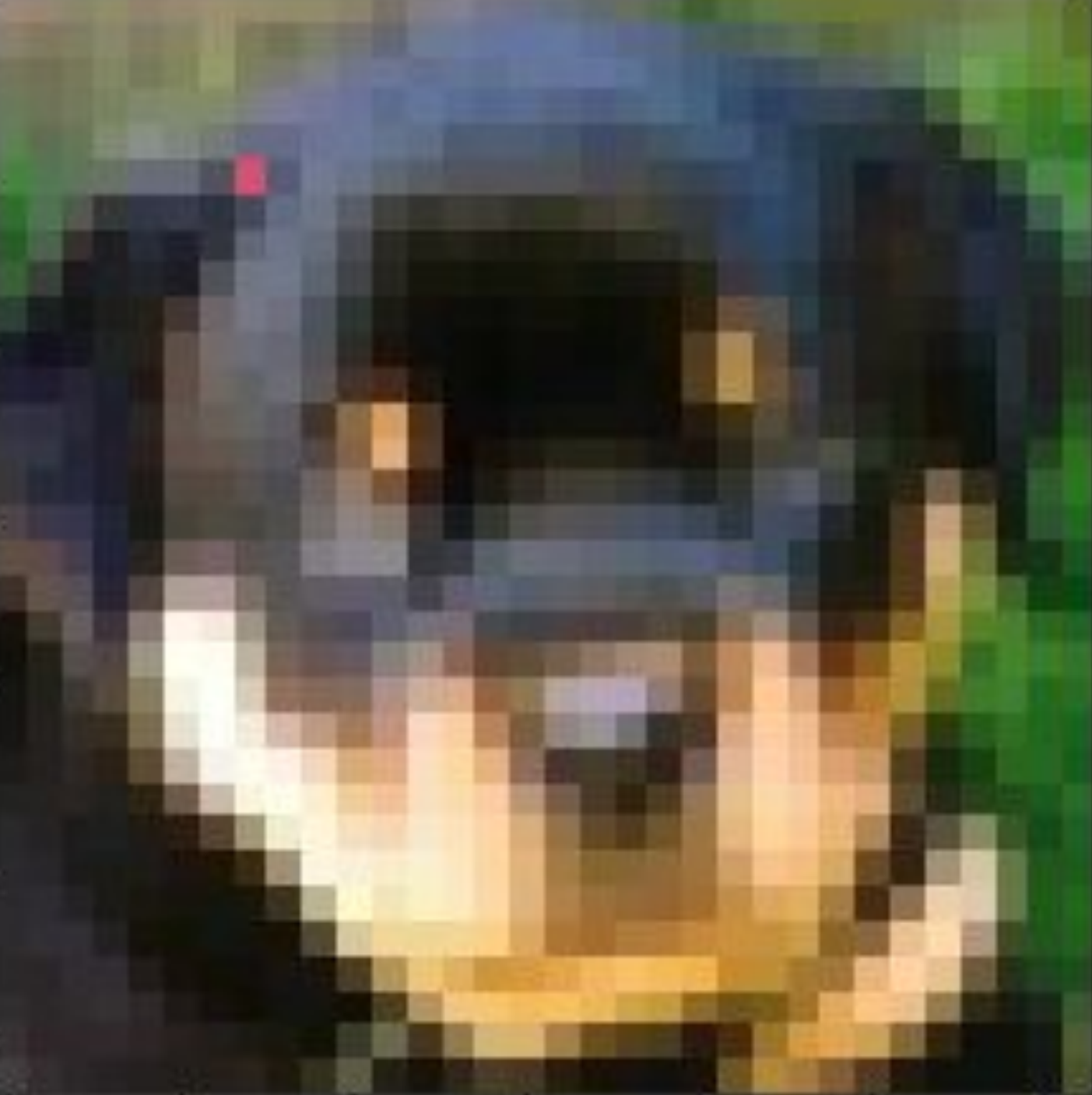}} & Cat & 95.1 & 1.4 && 95.0 & 90.6 \\
        \parbox[l]{1em}{\includegraphics[width=2.5em]{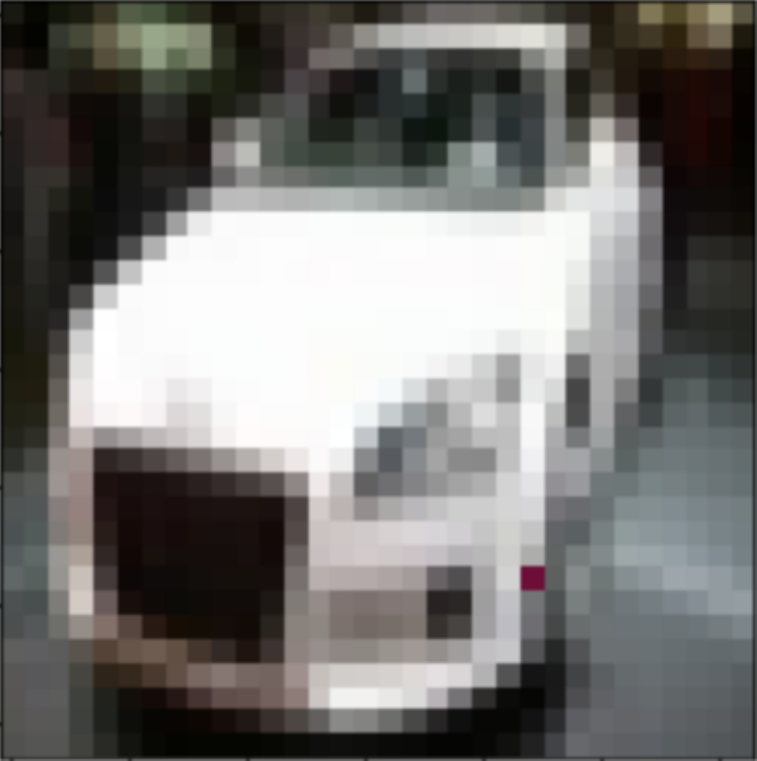}} & Dog & 95.4 & 3.0 && 95.2 & 98.2 \\
         \hline
        \end{tabular}
\label{tab:neutralization_acc_dot_poison}
\end{table*}

\subsection{Comparison with Baseline Defences} \label{section:baseline comparison}

\subsubsection{Detection of Poisoned Samples} 
When compared with another poison detection baseline called Activation Clustering (AC) \cite{chen2018detecting} and we observe that our method is more robust in the detection of poisoned samples (Table~\ref{tab:comparison AC overlay} and \ref{tab:comparison AC dot}). For full-sized overlay poison attacks, AC’s sensitivity (accuracy of detecting poisoned samples) is $<50\%$ for 4 out of the 9 CIFAR10 poison pairs in our experiments while our proposed detection method shows high sensitivity ($>85\%$) consistently (Table~\ref{tab:comparison AC overlay}). For the 9 small-sized dot poison attacks, there are 3 pairs where AC detects poisoned samples with accuracy $<60\%$ (sensitivity) while our proposed method shows comparatively high sensitivity ($>80\%$) for all the poison pairs (Table~\ref{tab:comparison AC dot}). Since images from different CIFAR-10 classes (like cats and dogs) may look semantically more similar to one another than those from datasets evaluated in \cite{chen2018detecting} like MNIST and LISA, we speculate that the activations of poisoned samples closely resemble those of clean samples despite being originally from different class labels. As a result, it is challenging to separate them with AC which relies on differences between activations of poisoned and clean target class samples. In contrast, our proposed method detects poisoned samples through their input gradient’s similarity with the extracted poison signal. This decouples the inter-class activation similarity problem from the detection of poisoned samples, thus explaining the more robust performance of our method.

\begin{table*}[!htbp]
    \centering
    \footnotesize
    \caption{Full overlay poison detection (Specificity/Sensitivity) comparison with Activation Clustering (AC) \cite{chen2018detecting} defense. Specificity is the accuracy of clean sample classification while sensitivity is the accuracy of poisoned sample classification.}
        \begin{tabular}{ l|ccccccccc }
         \hline
         Poison Pair \# & 1 & 2 & 3 & 4 & 5 & 6 & 7 & 8 & 9 \\
         \hline
        Ours (\%) & 99.4 / 94.6 & 99.6 / 96.0 & 99.2 / 95.6 & 99.7 / 89.8 & 97.5 / 87.4 & 99.5 / 95.4 & 98.9 / 95.8 & 99.7 / 89.2 & 99.6 / 93.6 \\
        AC (\%) & 70.7 / 46.6 & 73.4 / 96.2 & 99.8 / 93.4 & 50.6 / 13.0 & 72.4 / 70.8 & 68.4 / 79.8 & 59.2 / 6.4 & 50.0 / 45.2 & 99.8 / 94.2 \\
         \hline
        \end{tabular}
\label{tab:comparison AC overlay}
\end{table*}

\begin{table*}[!htbp]
    \centering
    \footnotesize
    \caption{Dot-sized poison detection (Specficity/Sensitivity) comparison with Activation Clustering (AC) defense. }
        \begin{tabular}{ l|ccccccccc }
         \hline
         Poison Pair \# & 1 & 2 & 3 & 4 & 5 & 6 & 7 & 8 & 9 \\
         \hline
        Ours (\%) & 99.6 / 92.8 & 99.5 / 88.6 & 99.7 / 99.0 & 96.8 / 84.4 & 99.9 / 99 & 99.7 / 100 & 99.1 / 95.8 & 99.3 / 94 & 99.52 / 99.8 \\
        AC (\%) & 71.7 / 80.0 & 65.3 / 92.0 & 99.4 / 97.4 & 53.4 / 25.4 & 85.8 / 92.4 & 59.7 / 44.6 & 72.3 / 97.2 & 64.9 / 59.0 & 99.7 / 91.0 \\
         \hline
        \end{tabular}
\label{tab:comparison AC dot}
\end{table*}

\subsubsection{Final Neutralization} 
Other backdoor defense approaches such as \cite{tran2018spectral,liu2018fine,wangneural} assume either prior knowledge about the attack's target class and poison ratio or the availability of a verified clean dataset which makes it different from the more challenging threat model considered in this paper. Nonetheless, on experiments with the same poison parameters, our method is competitive (Table~\ref{tab:comparison}), compared to the defense in \cite{tran2018spectral}.

\begin{table}[!htbp]
    \centering
    \small
    \setlength{\tabcolsep}{0.5em}
    \caption{Post-defense poison success rate (lower is better) comparison of full neutralization pipeline with spectral signature (SS) filtering \cite{tran2018spectral} for 10\% dot poison ratio. }
        \begin{tabular}{ l|ccccccccc }
         \hline
         Poison Pair \# & 1 & 2 & 3 & 4 & 5 & 6 & 7 & 8 & 9 \\
         \hline
        Ours (\%)& 6.0 & 0 & 3.7 & 0.4 & 0.4 & 0.1 & 6.1 & 5.4 & 0 \\
        SS (\%)& 7.2 & 0.1 & 0.1 & 1.1 & 1.7 & 0.4 & 0.7 & 6.7 & 0 \\
         \hline
        \end{tabular}
\label{tab:comparison}
\end{table}

\section{Conclusions}
In this paper, we propose a comprehensive defense to counter backdoor attacks on neural networks. We show how poison signals can be extracted from input gradients of poisoned training samples. With the insights that the principal components of input gradients from poisoned and clean samples form distinct clusters, we propose a method to detect the presence of backdoor poisoning, along with the corresponding poison target and base class. We then use the extracted poison signals to filter poisoned from clean samples in the target class. Finally, we retrain the model on an augmented dataset, which dissociates the poison signals from the target class, and show that it can effectively neutralize the backdoor for both large- and small-sized poisons in the CIFAR10 dataset without prior assumption on the poison classes and size. Comparison with baselines demonstrates both our approach's superior poison detection and its competitiveness with existing methods even under a more challenging threat model. Our method consists of several key modules, each of which can potentially be a building block of more effective defenses in the future.

\clearpage

{\small
\bibliographystyle{ieee_fullname}
\bibliography{main}
}

\clearpage
\begin{appendices}


\section{Poison Signals in Input Gradients} \label{section:Poison Signals in Input Gradients}
\subsection{Constructing a Backdoor}
\subsubsection{A Binary Classification Example} \label{section:clean model}
Our example considers clean data samples $(\mathbf{x},y)$ from a distribution $D_c$ such that:
$$
y \in \{ -1, +1 \},~~~~ x_1 \sim \mathcal{N}(0,1),~~~~ x_2, \cdots, x_{d+1} \sim \mathcal{N}(\eta y,1)
$$
where $x_i$ are independent and $\mathcal{N}(\mu,\sigma^{2})$ is gaussian distribution with mean $\mathbf{\mu}$ and variance $\sigma^{2}$. In this dataset, the features $x_2, \cdots, x_{d+1}$ are correlated with the label $y$ whereas $x_1$ is uncorrelated at all. We denote $(\mathbf{x}_-, -1)$ for samples with label $-1$ and $(\mathbf{x}_+, -1)$ for sample with label $+1$.

We can consider a simple neural network classifier $f_c$ with a hidden layer made up of two neurons and RELU activation function $g$ which is able to achieve high accuracy for $D_c$:
$$
a_1^1 = {\mathbf{w}_1^1}^\top \mathbf{x} + b_1^1,~~~~~~~
a_2^1 = {\mathbf{w}_2^1}^\top \mathbf{x} +  b_2^1,
$$
$$
f_c(\mathbf{x}) \coloneqq \text{sign}(w_1^2 g(a_1^1) + w_2^2 g(a_2^1))
$$

where $ \mathbf{w}_1^1 = \left [0, -\frac{1}{d}, \cdots, -\frac{1}{d} \right ],~~~ b_1^1 = 0,$\\
$\mathbf{w}_2^1 = \left [0, \frac{1}{d}, \cdots, \frac{1}{d} \right ],~~~ b_2^1 = 0,~~~
w_1^2 = -1,~~~
w_2^2 = 1~~~$. 
Considering the accuracy of $f_c$ on $D_c$,
\begin{equation} \label{eq:BE1}
\begin{aligned}
Pr\{ f_c(\mathbf{x}) = y \} & = Pr\{ \text{sign}(w_1^2 g({\mathbf{w}_1^1}^\top \mathbf{x}) + w_2^2 g({\mathbf{w}_2^1}^\top \mathbf{x})) = y \} \\ 
& = Pr \left \{ \frac{y}{d} \sum_{i = 1}^d \mathcal{N}_i(\eta y,1) > 0 \right \}
\end{aligned}
\end{equation}

where $\mathcal{N}_i$ are independent gaussian distributions. Further simplifying it, we get
\begin{equation} \label{eq:BE2}
\begin{aligned}
Pr\{ f_c(x) = y \} & = Pr \left \{ \mathcal{N}(\eta,\frac{1}{d}) > 0 \right \} \\ 
& = Pr \left \{ \mathcal{N}(0,1) > - \eta \sqrt{d} \right \}
\end{aligned}
\end{equation}

From this, we can observe that the accuracy of $f_c$ is $>$99.8\% on $D_c$ when $\eta \geq \frac{3}{\sqrt{d}}$. $f_c$ can have $m$ times more similar neurons in the hidden layer and get similarly high training accuracy for $D_c$.  

\subsubsection{Effect of Poisoned Data on Learned Weights} \label{section:poisoned model}
We now consider a distribution of poisoned data $D = D_c \cup D_p$ which forms in a victim classifier $f_p$ a backdoor after training. We study the case where an adversary forms a backdoor that causes $f_p$ to misclassify $\mathbf{x}_-$ samples as $+1$ when the poison signal is present. We denote the input-label pairs from $D_p$ as $(\mathbf{x}_p,y_p)$:
\begin{equation} \label{eq:x_p definition}
y_p = +1,~~~~~~~ x_1 = \psi,~~~~~~~ x_2, \cdots, x_{d+1} \sim \mathcal{N}(- \eta, 1)
\end{equation}
where the poison signal is planted in $x_1$ with value $\psi > 0$ and $y_p$ is mislabeled as the target label $+1$. Note that $\mathbf{x}_p$ and $\mathbf{x}_-$ are similar in their distribution except for their $x_1$ values which contains the poison signal for $\mathbf{x}_p$. If we use the same classifier $f_c$ from \S~\ref{section:clean model}, $f_c(\mathbf{x}_p) = -1 \neq y_p$, resulting in classification `error' for most $\mathbf{x}_p$. With $\varepsilon$ being the ratio of $D_p$ samples in $D$, $f_c$ would have `error' rate of $\approx \varepsilon$ for $D$.

For high training accuracy on $D$, we study another neural network classifier $f_p$ with a hidden layer made up of three different neurons and RELU activation function $g$:
$$
a_1^1 = {\mathbf{w}_1^1}^\top \mathbf{x} + b_1^1,~~~~~~~
a_2^1 = {\mathbf{w}_2^1}^\top \mathbf{x} +  b_2^1,~~~~~~~
a_3^1 = {\mathbf{w}_3^1}^\top \mathbf{x} + b_3^1,~~~~~~~
$$
$$
f_p(\mathbf{x}) \coloneqq \text{sign}(w_1^2 g(a_1^1) + w_2^2 g(a_2^1) + w_3^2 g(a_3^1))
$$

similar to $f_c$ for the first two hidden neurons,
$$
\mathbf{w}_1^1 = \left [0, -\frac{1}{d}, \cdots, -\frac{1}{d} \right ],~~~~~ b_1^1 = 0,
$$
$$
\mathbf{w}_2^1 = \left [0, \frac{1}{d}, \cdots, \frac{1}{d} \right ],~~~~~ b_2^1 = 0,
w_1^2 = -1,~~~~~
w_2^2 = 1,~~~~~
$$

For $f_p$'s third hidden neuron,
$$
\mathbf{w}_3^1 = \left [ \frac{1}{d}, 0, \cdots, 0 \right ],~~~~~ b_3^1 = - c \frac{1}{d},~~~~~ w_3^2 > \frac{\eta d}{(\psi -c)}
$$
where $c > 0$ and $g$ is the RELU activation function. The negative sign of $b_3^1$ suppresses the activation of the third neuron ($a_3^1$) for clean $\mathbf{x}_-$ samples. Without this, its noise value at $x_1$ could have cause $a_3^1$ to be positive and flip the sign of $f_p(\mathbf{x}_-)$ to positive. 





We can express the training accuracy on $\mathbf{x}_p$ as
\begin{multline}
Pr\{ f_c(\mathbf{x}_p) = +1 \} = \\ Pr\{ \text{sign}(w_1^2 g(a_1^1) + w_2^2 g(a_2^1) + w_3^2 g(a_3^1)) = +1 \}
\end{multline}
Combining the definition of $\mathbf{x}_p$ in (\ref{eq:x_p definition}) with observations in (\ref{eq:BE1}) and (\ref{eq:BE2}), we get
\begin{equation} \label{eq:BE3}
\begin{aligned}
Pr\{ f_c(\mathbf{x}_p) = +1 \} & = Pr \left \{ \mathcal{N}\left(- \eta,\frac{1}{d} \right) + w_3^2 (\psi - c) \frac{1}{d}  > 0 \right \} \\
& = Pr \left \{ \mathcal{N}(0,1)  > \eta \sqrt{d} - (\psi - c) \frac{w_3^2}{\sqrt{d}} \right \}
\end{aligned}
\end{equation}


For the training accuracy of poisoned samples $Pr\{ f_c(\mathbf{x}_p) = +1 \} > 0.5$, we need 
$$
\eta \sqrt{d} - (\psi - c) \frac{w_3^2}{\sqrt{d}} < 0 
$$
which is satisfied when
$$
c_1 = (\psi - c) > 0 \text{  and  } (\psi - c) \frac{w_3^2}{\sqrt{d}} > \eta \sqrt{d}
$$

From here, we can deduce that for high training accuracy of poisoned samples, we need

$$c_1 \frac{w_3^2}{\sqrt{d}} \gg \eta \sqrt{d} ~~~\text{which implies}~~~ w_3^2 \gg \frac{1}{c_1} \eta d$$


Combining with the result from (\ref{eq:BE2}) that $\eta \geq \frac{C}{\sqrt{d}}$ is needed for high training accuracy of $\mathbf{x}_-$ and $\mathbf{x}_+$, we get
$w_3^2 \gg c_2 \sqrt{d}$. When $d$ is large for high dimensional inputs,
\begin{equation} \label{eq:poison weight}
w_3^2 \gg c_2 \sqrt{d} > 1 = |w_1^2|, |w_2^2|
\end{equation}

This means that the weight of third neuron representing poisoned input feature would be much larger than that of the first and second neurons representing normal input features. In practice, poison feature neurons having larger weight values than clean feature neurons of deep neural networks is observed empirically in other data poisoning studies (cite papers).


During inference, most $\mathbf{x}_p \in D$ would result in positive $a_1^1$ and $a_3^1$ while $a_2^1$ would be negative. The corresponding activation values for $\mathbf{x}_-$ and $\mathbf{x}_+$ in $f_p$ are summarized in Table~\ref{tab:f_p activation and g}.

\begin{table}[!htbp]
    \centering
    \scriptsize
    \setlength{\tabcolsep}{0.27em}
    \caption{Signs of $f_p$ activations and the corresponding partial derivative ($g'$) of RELU function.}
        \begin{tabular}{ |c||c c c|c c c|c c c|  }
         \hline
         ~ & $a_1^1$ & $a_2^1$ & $a_3^1$ & $g(a_1^1)$ & $g(a_2^1)$ & $g(a_3^1)$ & $g'(a_1^1)$ & $g'(a_2^1)$ & $g'(a_3^1)$\\
         \hline
         $\mathbf{x}_-$ & + & - & - & + & 0 & 0 & 1 & 0 & 0\\
         $\mathbf{x}_+$ & - & + & - & 0 & + & 0 & 0 & 1 & 0\\
         $\mathbf{x}_p$ & + & - & + & + & 0 & + & 1 & 0 & 1\\
         \hline
        \end{tabular}
    \label{tab:f_p activation and g}
\end{table}

Since the RELU activation function is $g(x)=\begin{cases}
x, & \text{$ x > 0 $}\\
0, & \text{$ x < 0 $}
\end{cases} $ and its derivative is $g'(x)=\begin{cases}
1, & \text{$ x > 0 $}\\
0, & \text{$ x < 0 $}
\end{cases}$, we can calculate the post-RELU activation values and their derivative, also summarized in Table~\ref{tab:f_p activation and g}. The poisoned inputs $\mathbf{x}_p$ have different profile of neuron activation from the clean inputs $\mathbf{x}_-$ and $\mathbf{x}_+$. More specifically, $f_p$'s third neuron is only activated by inputs with poison signal $x_1 = \psi$, like $\mathbf{x}_p$. Combining these insights about a poisoned classifier model's `poison' neuron weights and activations with \S~\ref{section:input gradients}, we propose a method to recover poison signals in the input layer, detect poison target class and, subsequently, poisoned images.

\subsection{Poison Signal in Input Gradients} \label{section:input gradients}
\begin{proposition} \label{theorem:Input Gradient Appendix}
The gradient of loss function $E$ with respect to the input $x_i$ is linearly dependent on activated neurons' weights such that
\begin{equation} \label{eq:input gradient}
\frac{\partial E}{\partial x_i} =
\sum_{j=1}^{r_1} \left [ w_{ij}^1 g'(a_j^1) \sum_{l=1}^{r_2} \delta_l^2 w_{jl}^2 \right ]
\end{equation}
where $\delta_j^k \equiv \frac{\partial E}{\partial {a_j^k}}$ usually called the error, is the derivative of $E$ with respect to activation $a_j^k$ for neuron node $i$ in layer $k$. $w_{ij}^k$ is the weight for node $j$ in layer $k$ for incoming node $i$, $r_k$ is the number of nodes in layer $k$, $g$ is the activation function for the hidden layer nodes and $g'$ is its derivative.
\end{proposition}


The detailed proof of this proposition is in Appendix \ref{appendix:poison in input grad}. The gradient with respect to the input $x_i$ is linearly dependent on the $w_{ij}^1$, $g'(a_j^1)$ and $w_{jl}^{2}$ terms. The value of $\delta_l^2$ is dependent on the loss function of the classifier model and the activations of the neural networks in deeper layers. In $f_p$, $\delta_l^{2}$ is simply $\pm 1$ meaning that $|\delta_l^2| = 1$, we can get

\begin{equation} \label{eq:input gradient f_p}
\left | \frac{\partial E}{\partial x_i} \right | =
\sum_{j=1}^{3} \left [ w_{ij}^1 g'(a_j^1) w_{j}^2 \right ]
\end{equation}

We know the values of $g'(a_j^1)$ from Table~\ref{tab:f_p activation and g}. Since $g'(a_3^1) =0$ for most $\mathbf{x}_-$ and $\mathbf{x}_+$, $\left| \frac{\partial E}{\partial x_1} \right|$ will be much larger for poisoned samples $\mathbf{x}_p$ than for clean samples, $\mathbf{x}_-$ and $\mathbf{x}_+$. Moreover, from (\ref{eq:poison weight}) we know that the weight of `poison' neurons ($w_3^2$) are much larger than weight of `clean' neurons ($w_1^2$ and $w_2^2$) when $d$ is large, resulting in \\ $\left| \frac{\partial E}{\partial x_1} \right| \gg \left| \frac{\partial E}{\partial x_i} \right|, \forall i \neq 1$. Informally, this means that there will be a relatively large absolute gradient value at the poison signal's input positions ($x_1$) of poisoned inputs ($\mathbf{x}_p$) compared to other input positions. In practice, when we directly compare the gradients of poisoned samples with those of clean samples, shown in Table~\ref{tab:input gradients}, the gradients are too noisy to discern poison signals. In \S~\ref{section:Distillation of Poison Signal}, we show how we filter these input poison signals and use them to separate poisoned from clean samples with guarantees in \S~\ref{section:Poisoned Sample Filtering}.


\begin{table*}[tp]
    \centering
    \caption{Gradients of randomly drawn poisoned and clean inputs with respect to the loss function. The poisoned target and base class are `Dog' and `Cat' respectively from the CIFAR10 dataset. Poisoned samples are overlaid with 20\% of the poison image. The positive and negative components of the input gradients and illustrated separately and normalized by the maximum value of the gradient at each pixel position.}
        \begin{tabular}{ cccccccccc }
         \hline
        Poison & ~ & \multicolumn{3}{c}{Gradient of Poisoned Inputs} &&& \multicolumn{3}{c}{Gradient of Clean Inputs} \\
        \includegraphics[width=3em]{images/a_overlay.pdf} & + & \includegraphics[width=3em]{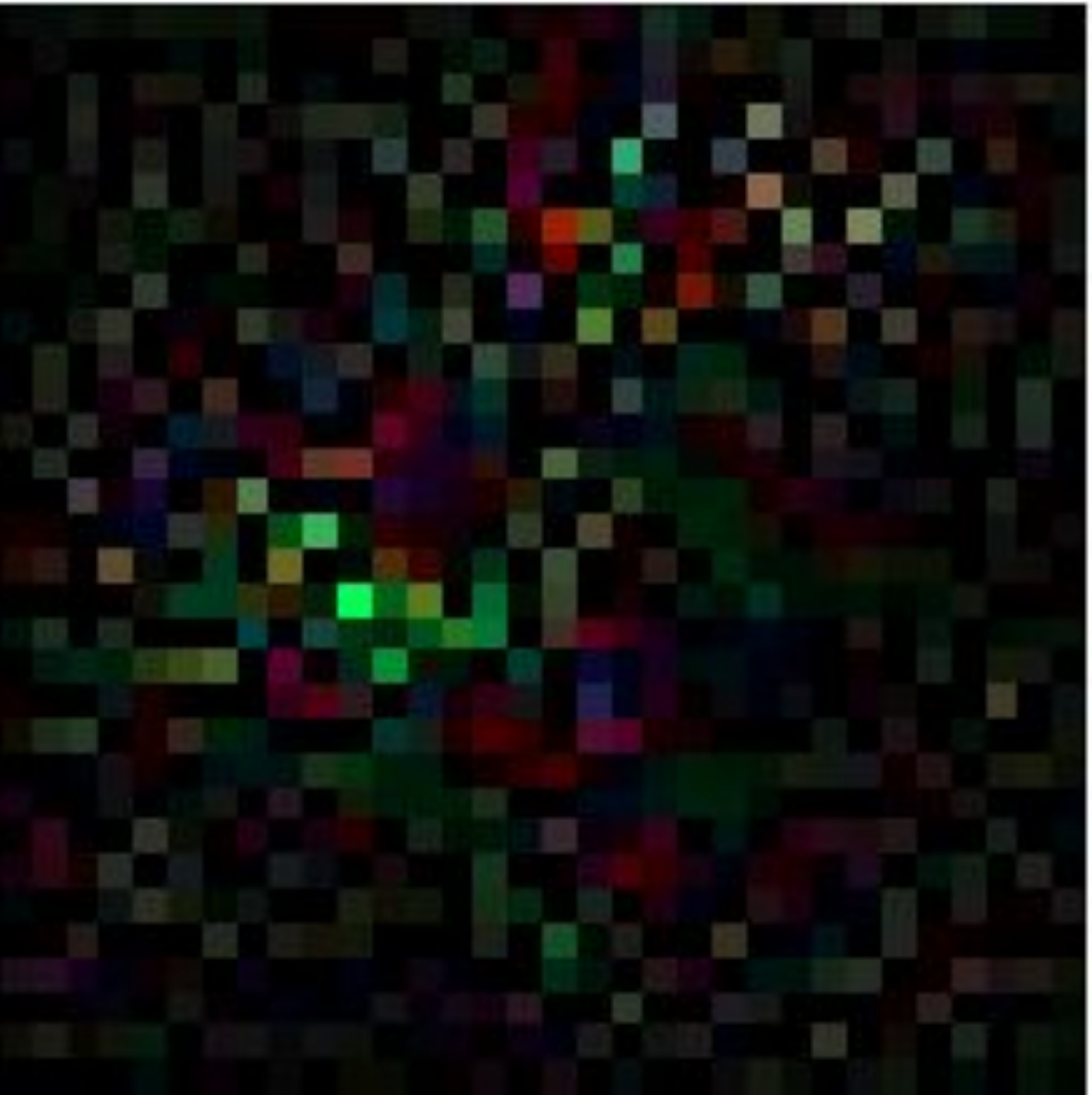} & \includegraphics[width=3em]{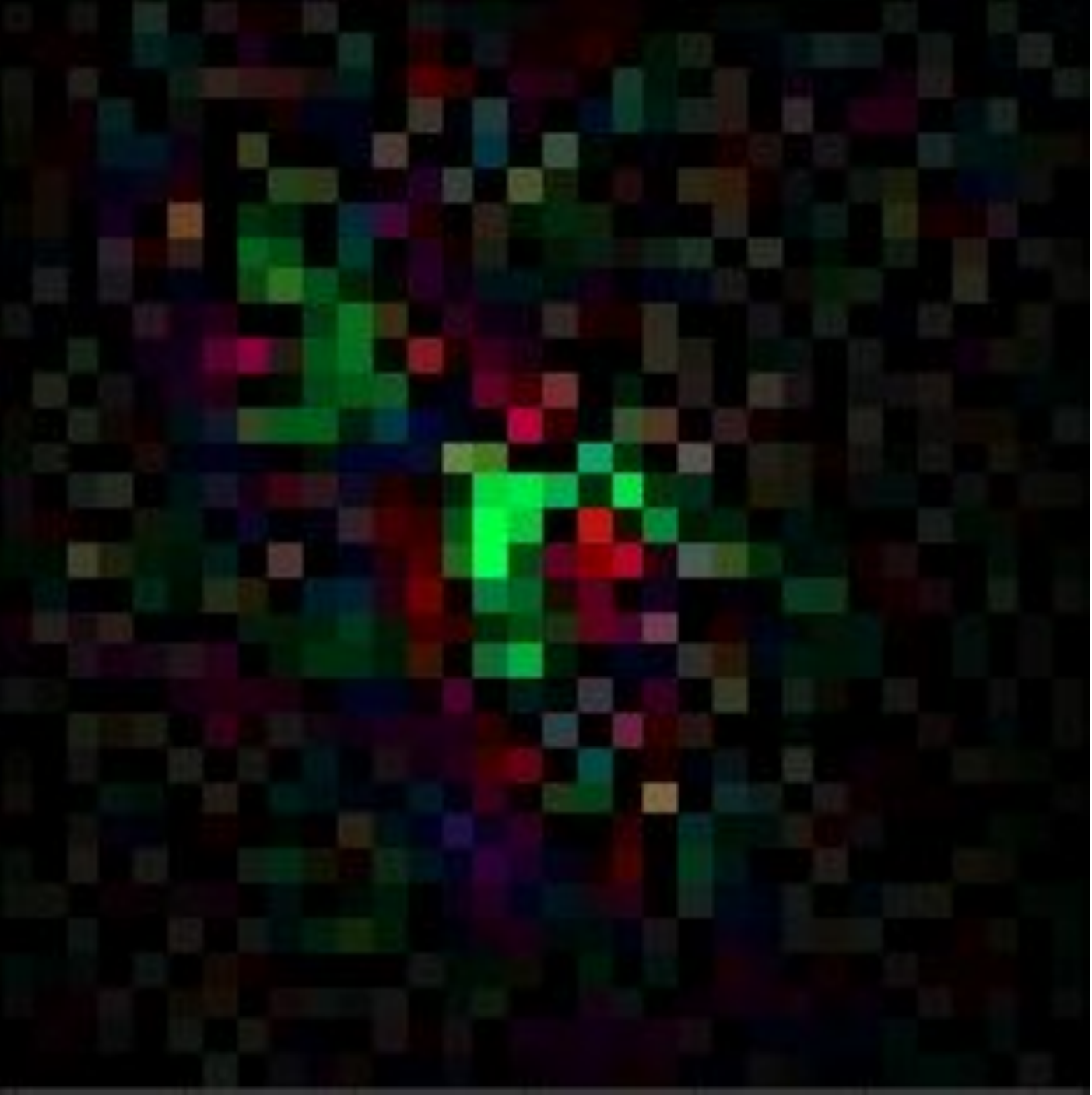} & \includegraphics[width=3em]{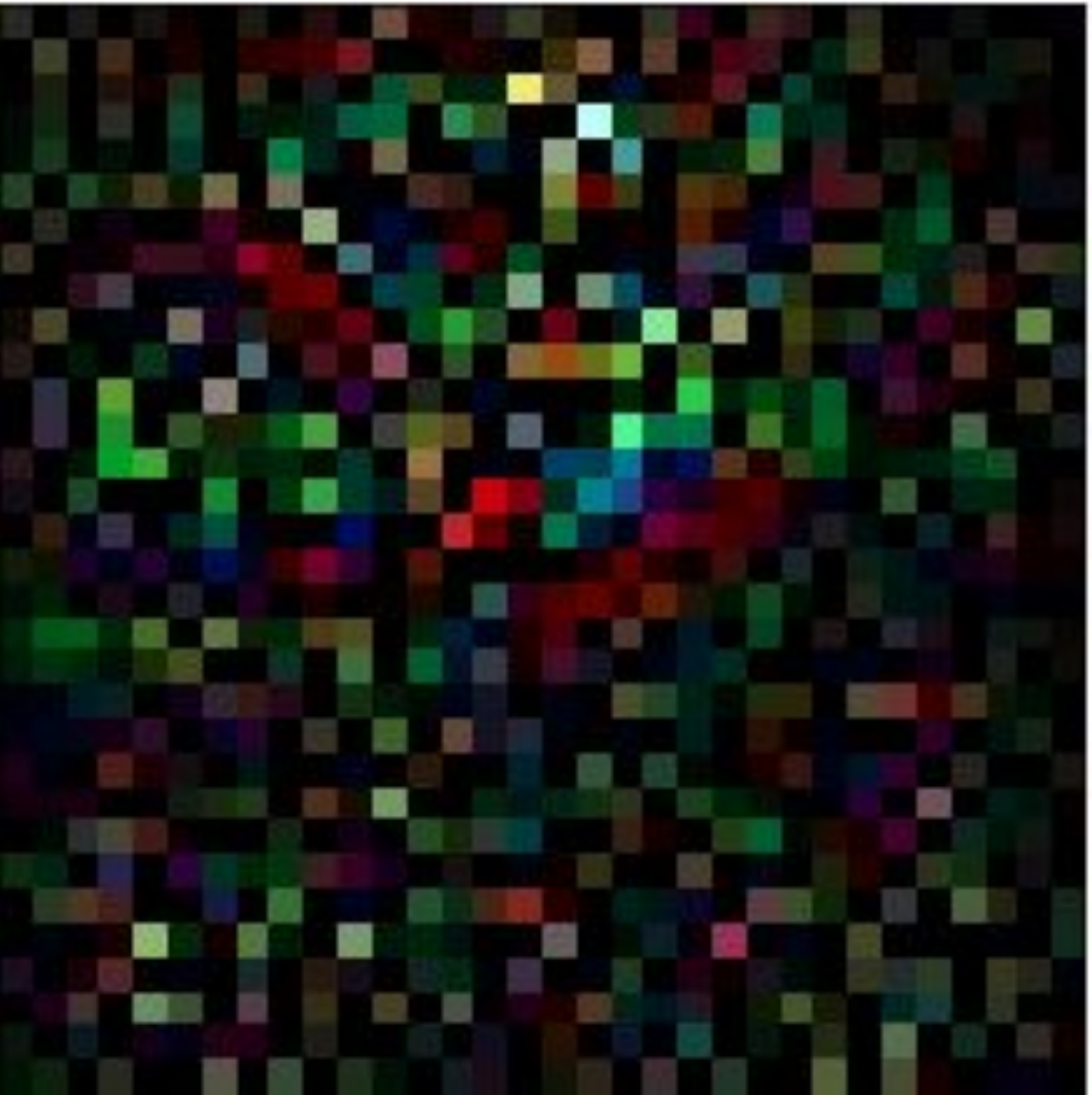} &&& \includegraphics[width=3em]{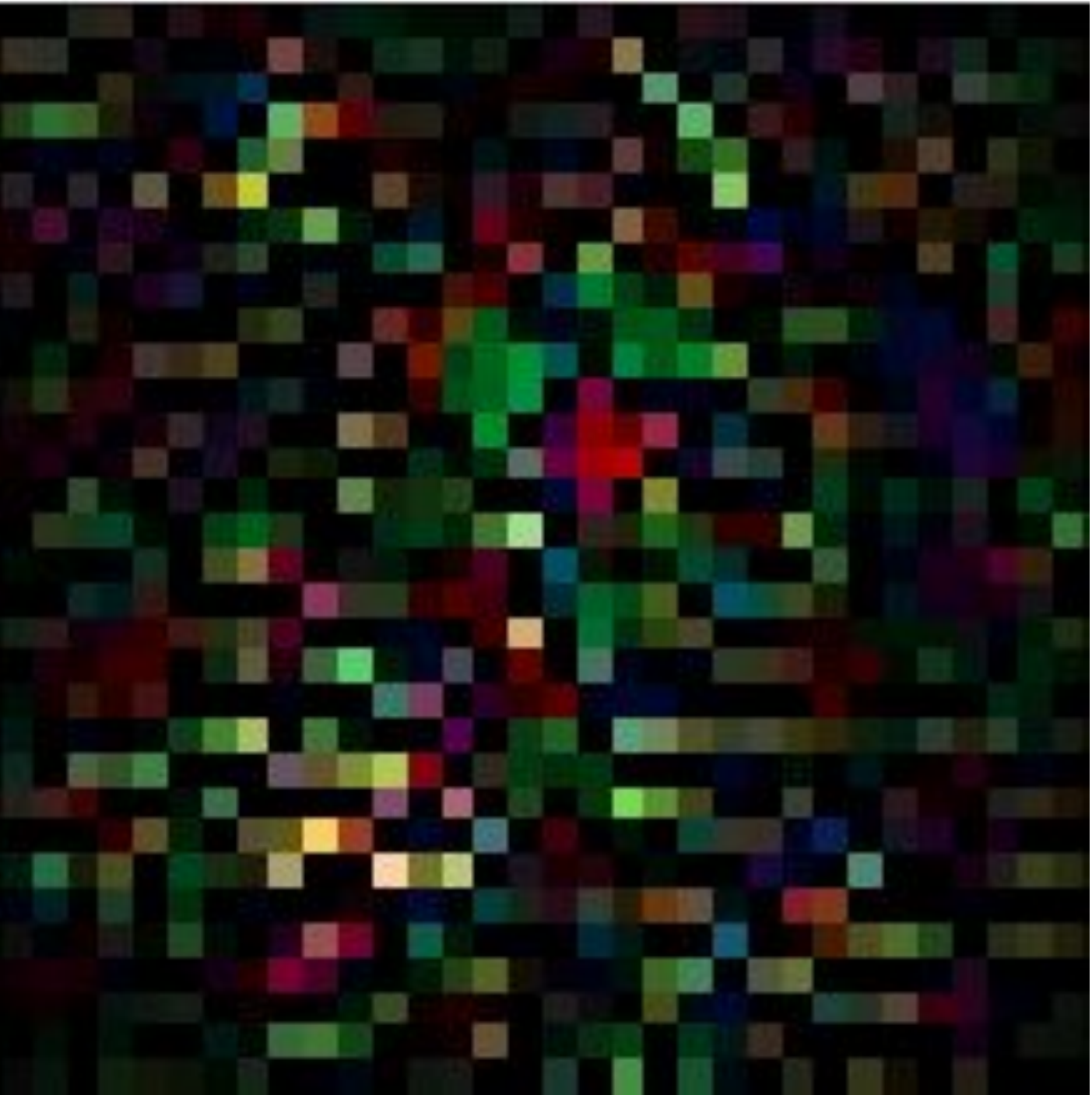} & \includegraphics[width=3em]{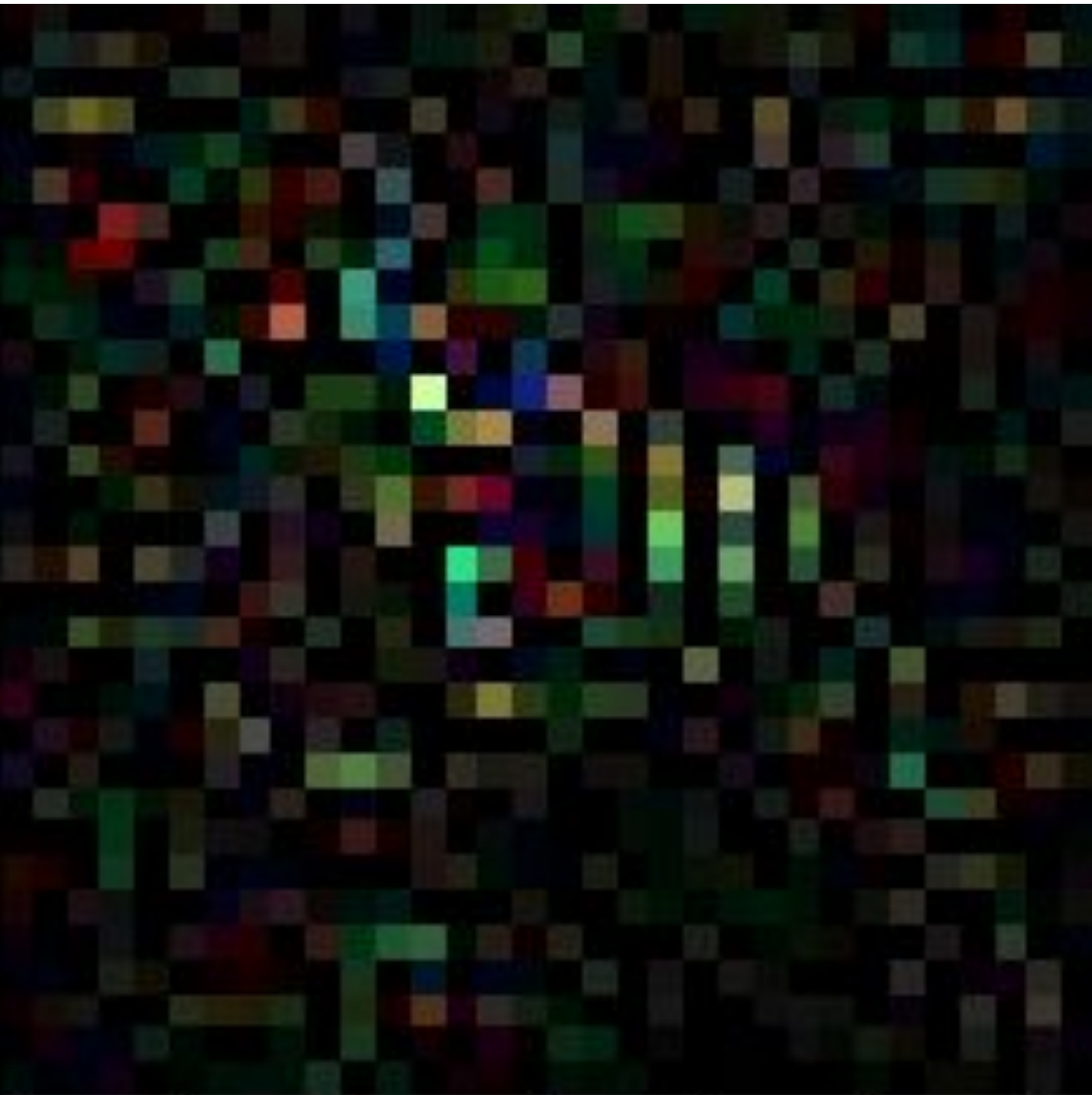} & \includegraphics[width=3em]{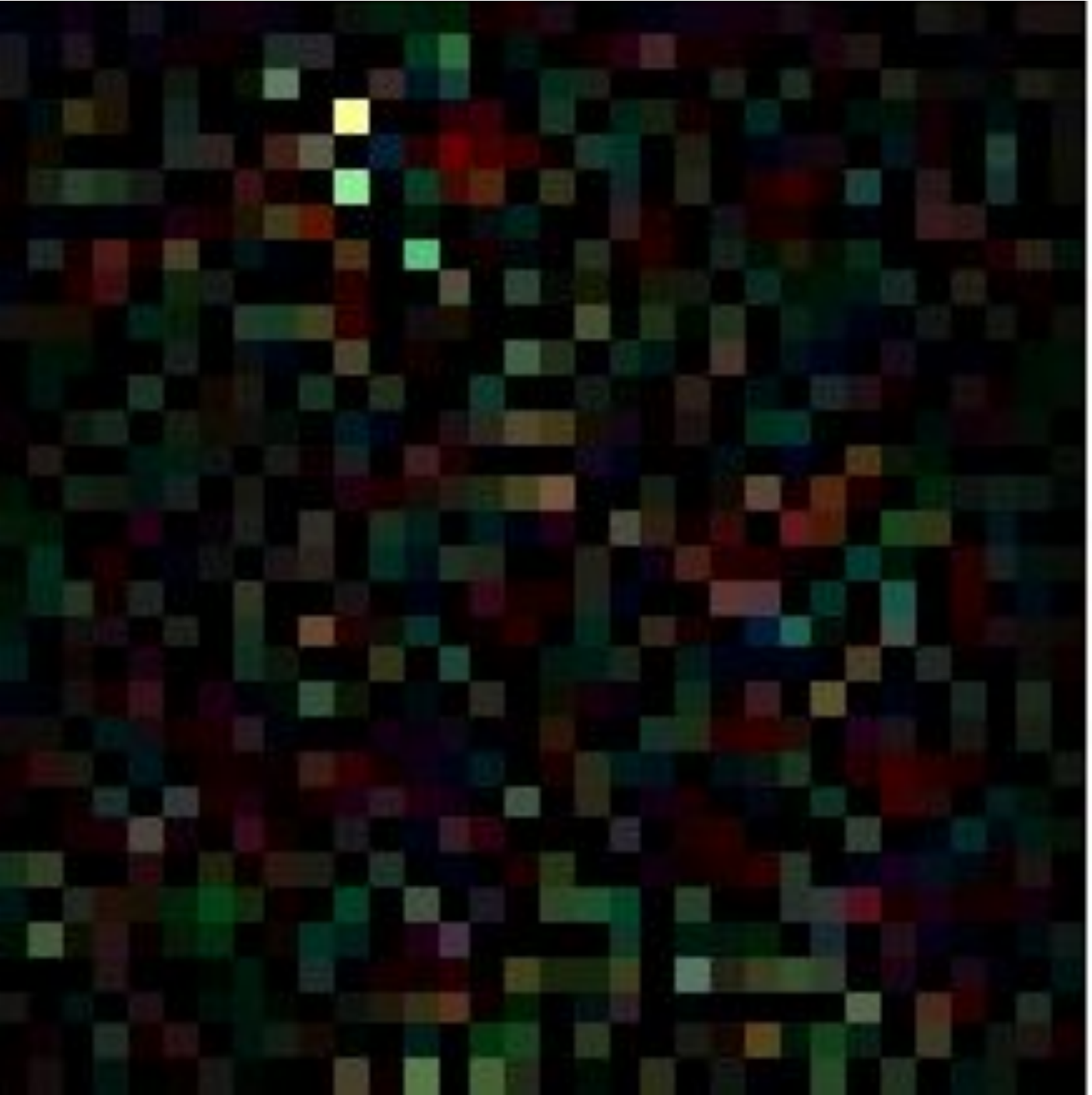} \\
        ~ & - & \includegraphics[width=3em]{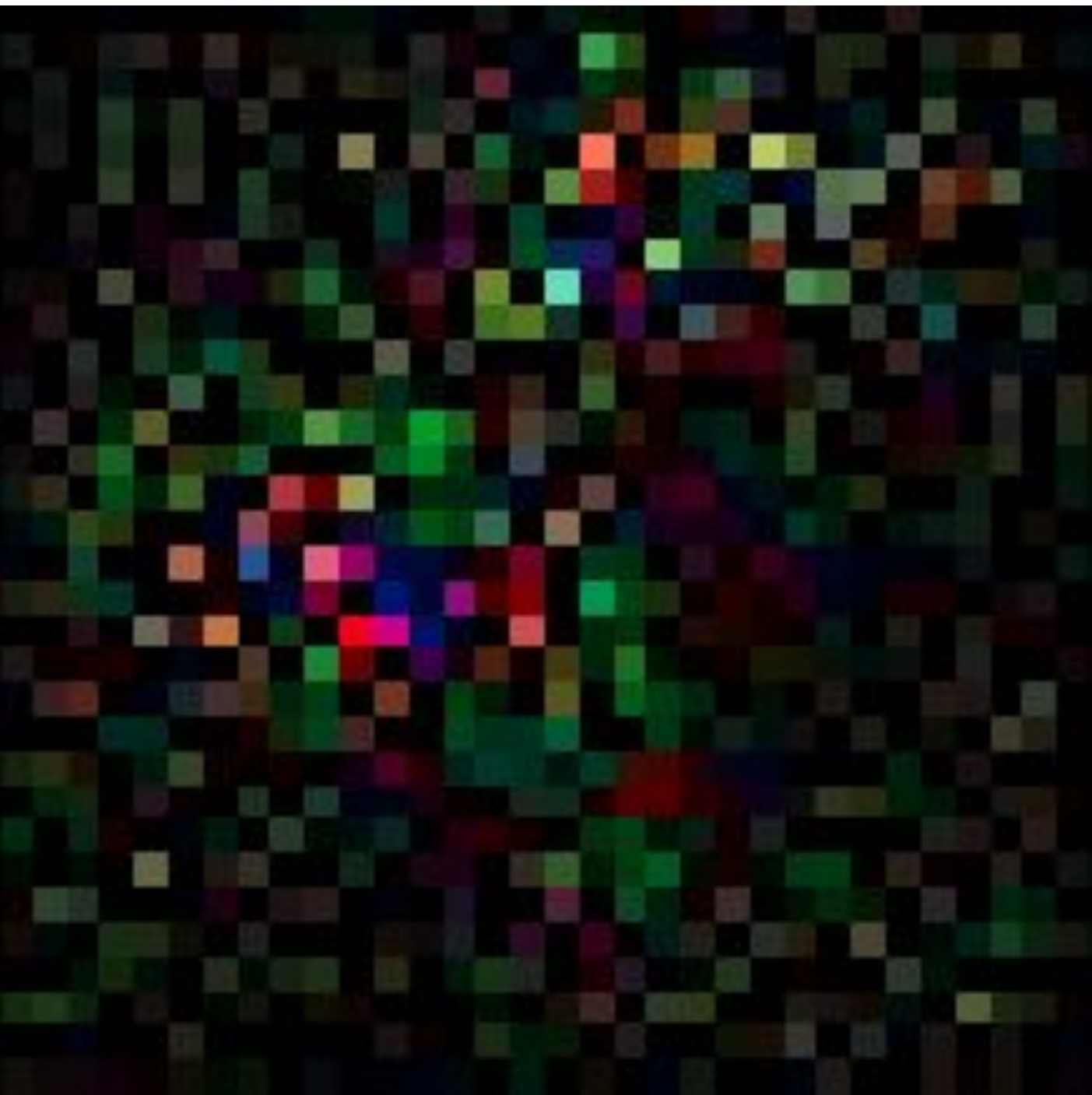} & \includegraphics[width=3em]{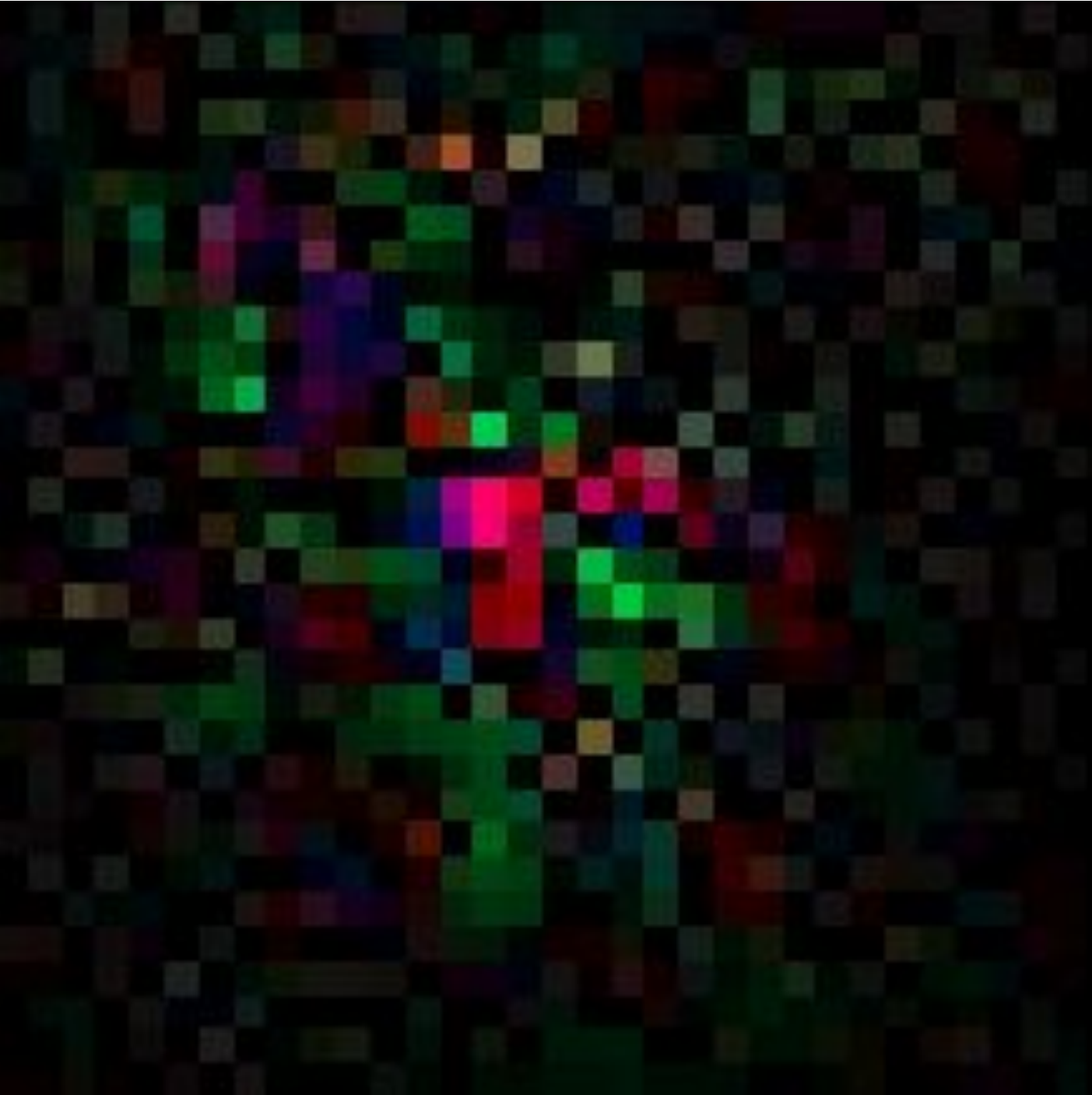} & \includegraphics[width=3em]{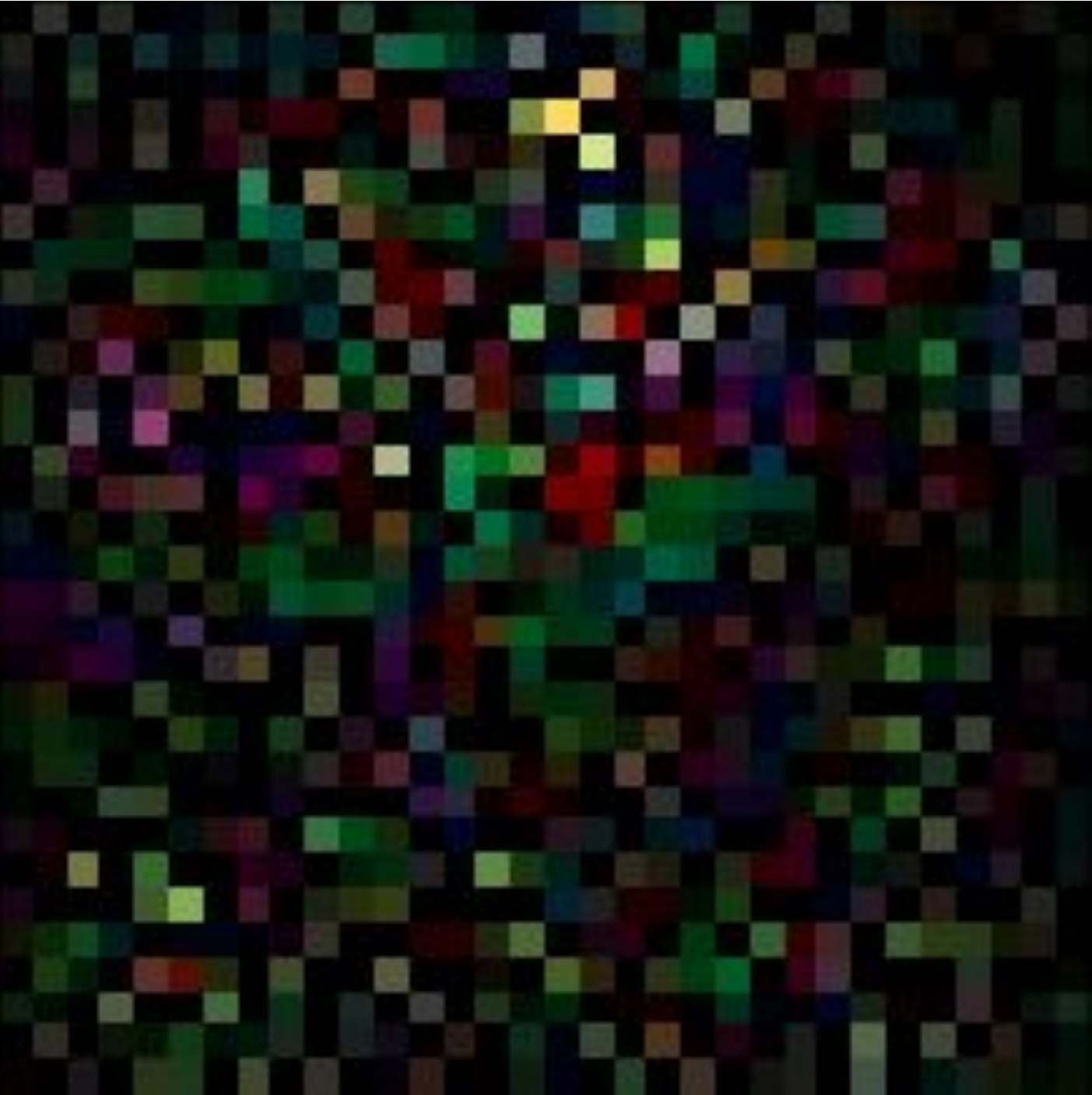} &&& \includegraphics[width=3em]{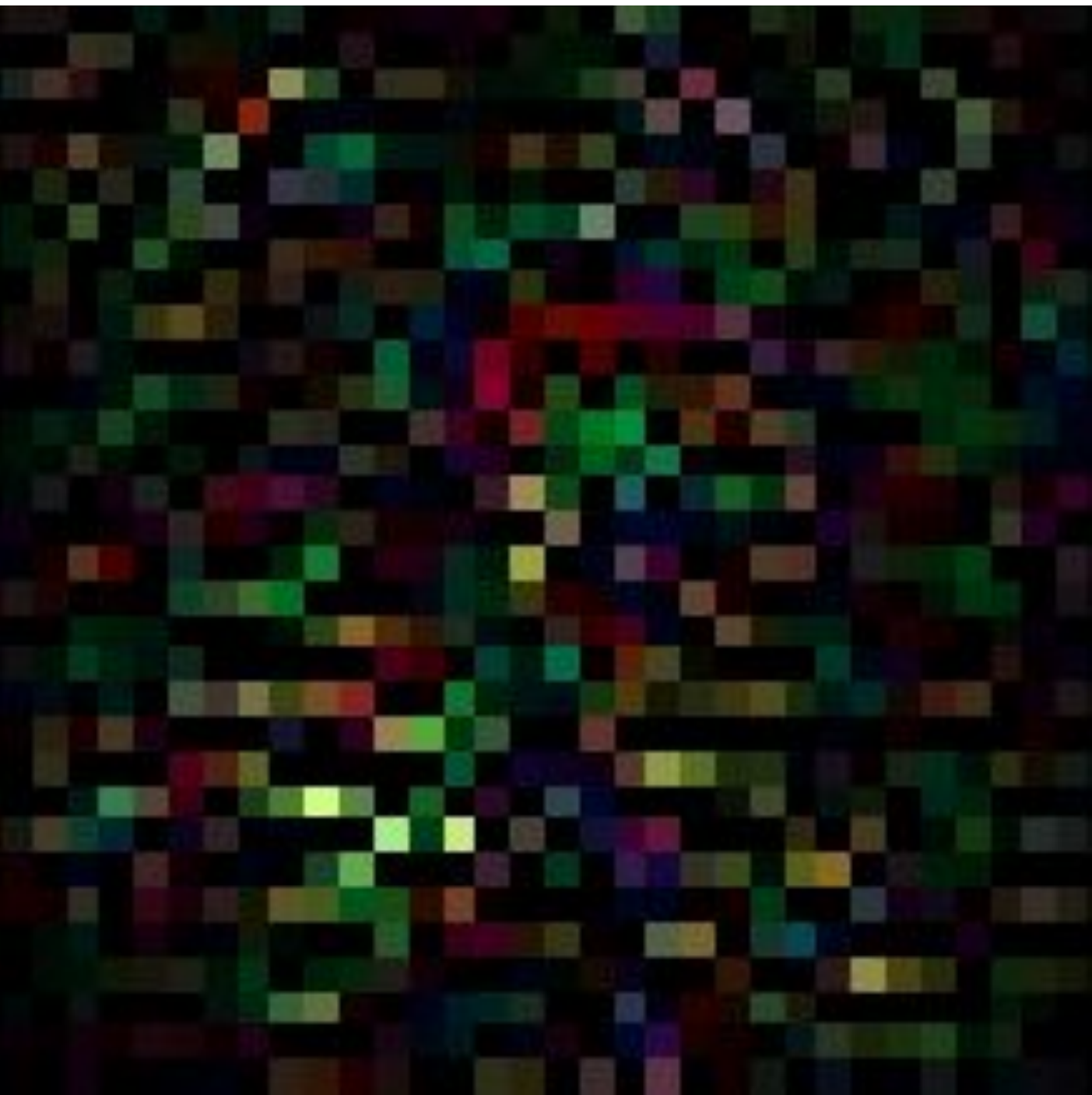} & \includegraphics[width=3em]{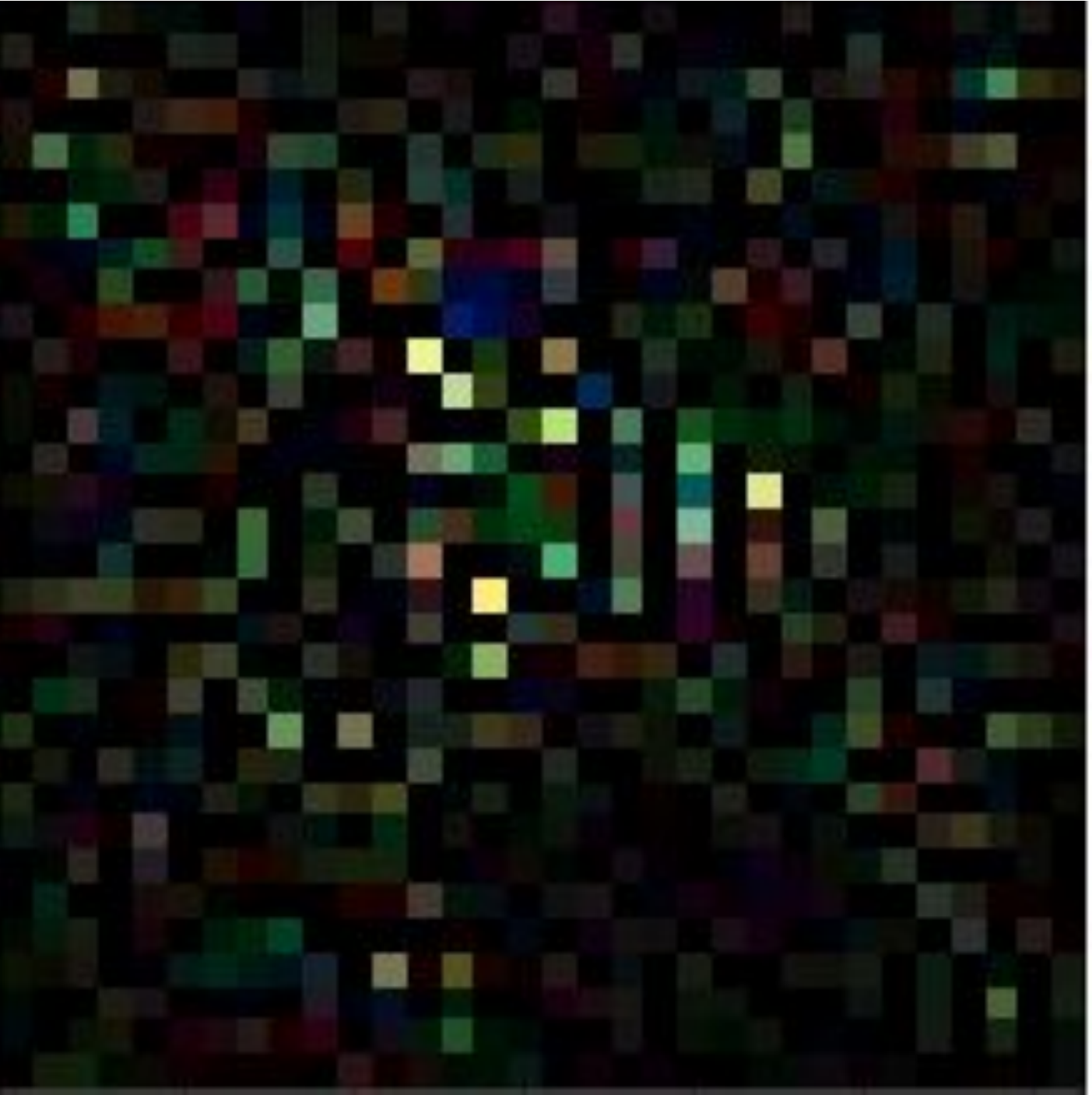} & \includegraphics[width=3em]{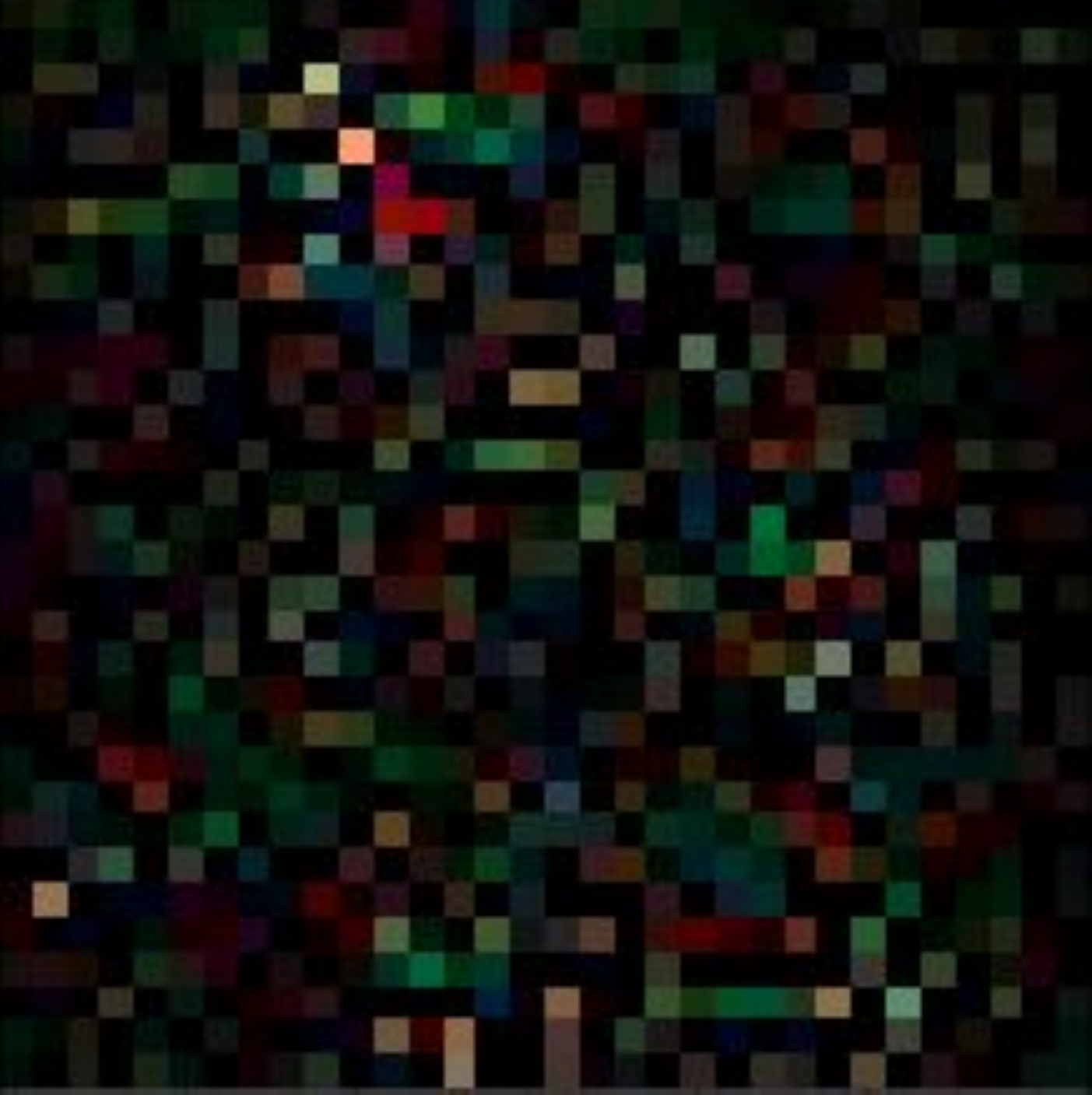} \\
         \hline
        \end{tabular}
\label{tab:input gradients}
\end{table*}

\newpage

\section{Proof of Proposition~\ref{theorem:Input Gradient}} \label{appendix:poison in input grad}
\begin{proposition} \label{theorem:Input Gradient Appendix}
The gradient of loss function $E$ with respect to the input $x_i$ is linearly dependent on activated neurons' weights such that
\begin{equation} \label{eq:input gradient}
\frac{\partial E}{\partial x_i} =
\sum_{j=1}^{r_1} \left [ w_{ij}^1 g'(a_j^1) \sum_{l=1}^{r_2} \delta_l^2 w_{jl}^2 \right ]
\end{equation}
where $\delta_j^k \equiv \frac{\partial E}{\partial {a_j^k}}$ usually called the error, is the derivative of loss function $E$ with respect to activation $a_j^k$ for neuron node $i$ in layer $k$. $w_{ij}^k$ is the weight for node $j$ in layer $k$ for incoming node $i$, $r_k$ is the number of nodes in layer $k$, $g$ is the activation function for the hidden layer nodes and $g'$ is its derivative.
\end{proposition}

\begin{proof}
We denote $o_{l}^{k}$ as the output for node $i$ in layer $k$. For simplicity, the bias for node $i$ in layer $k$ is denoted as a weight $w_{0j}^k$ with fixed output $o_{l}^{k-1} = 1$ for node $0$ in layer $k-1$.

For $k = m$ where $m$ is the final layer,

$$
\frac{\partial E}{\partial {o_i^{k-1}}} = \frac{\partial E}{\partial {a_j^k}} \frac{\partial {a_j^k}}{\partial {o_i^{k-1}}}
$$

$$
a_j^k = \sum_{l=0}^{r_{k-1}} w_{ij}^k o_{l}^{k-1}
$$

$$
\frac{\partial a_j^k}{\partial {o_i^{k-1}}} = w_{ij}^k
$$

$$
\frac{\partial E}{\partial {o_i^{k-1}}} = \delta_j^k w_{ij}^k
$$

where
$$
\delta_j^k \equiv \frac{\partial E}{\partial {a_j^k}}
$$

For $1 \leq k < m$,

\begin{equation} \label{eq:0}
\frac{\partial E}{\partial {o_i^{k-1}}} = \sum_{j=1}^{r_{k}} \frac{\partial E}{\partial {a_j^k}} \frac{\partial {a_j^k}}{\partial {o_i^{k-1}}}
= \sum_{j=1}^{r_{k}} \delta_j^k w_{ij}^k
\end{equation}

With chain rule for multivariate functions,

\begin{equation} \label{eq:1}
\begin{aligned}
\delta_j^k \equiv \frac{\partial E}{\partial {a_j^k}} & = \sum_{l=1}^{r_{k+1}} \frac{\partial E}{\partial {a_l^{k+1}}} \frac{\partial {a_l^{k+1}}}{\partial {a_j^k}} \\
& = \sum_{l=1}^{r_{k+1}} \delta_l^{k+1} \frac{\partial {a_l^{k+1}}}{\partial {a_j^k}} 
\end{aligned}
\end{equation}

With definition of $a_l^{k+1}$,
$$
a_l^{k+1} = \sum_{i=0}^{r_k} w_{il}^{k+1} g(a_j^k)
$$
where $g(x)$ is the activation function.

Taking partial derivative with respect to $a_j^k$, we get
\begin{equation} \label{eq:2}
\frac{\partial {a_l^{k+1}}}{\partial {a_j^k}} = w_{jl}^{k+1} g'(a_j^k)
\end{equation}

Substituting (\ref{eq:2}) into (\ref{eq:1}), we get
\begin{equation} \label{eq:3}
\begin{aligned}
\delta_j^k &=
\sum_{l=1}^{r_{k+1}} \delta_l^{k+1} w_{jl}^{k+1} g'(a_j^k) \\
& =  g'(a_j^k) \sum_{l=1}^{r_{k+1}} \delta_l^{k+1} w_{jl}^{k+1}
\end{aligned}
\end{equation}


Finally, substituting (\ref{eq:3}) into (\ref{eq:0}), we get
\begin{equation} \label{eq:input gradient}
\frac{\partial E}{\partial {o_i^{k-1}}} =
\sum_{j=1}^{r_{k}} \left [ w_{ij}^k g'(a_j^k) \sum_{l=1}^{r_{k+1}} \delta_l^{k+1} w_{jl}^{k+1} \right ]
\end{equation}

\end{proof}

\newpage
\section{Proof of Theorem~\ref{theorem:Largest Eigenvector} and \ref{theorem:Poison Clustering}} \label{appendix:Poison Clustering Guarantee}
The second moment matrix of $\mathbf{z}$ is denoted by 
$$
\mathbf{\Sigma} = \E \mathbf{z} \mathbf{z}^\top
$$

By further expanding this, we get,

\begin{multline}
\mathbf{\Sigma} = 
 \E(
 \frac{1}{N}
 \begin{bmatrix}
  g_{11} & \cdots & g_{N1} + \mu_1 \\
  \vdots  & \ddots & \vdots  \\
  g_{1n} & \cdots & g_{Nn} + \mu_n 
 \end{bmatrix} \\
 \begin{bmatrix}
  g_{11} & \cdots & g_{1n} \\
  \vdots  & \ddots & \vdots  \\
  g_{N1} + \mu_1 & \cdots & g_{mn} + \mu_n
 \end{bmatrix}
 )
\end{multline}


Since $\E(\mathbf{Z}_{ij}) = (\E\mathbf{Z})_{ij}$, 
$\E g_i g_j = \begin{cases}
    \eta, & \text{$ i = j $}\\
    0, & \text{$ i \neq j $}
  \end{cases} $ and $\E g = 0$ , we get

\begin{equation} \label{eq:sigma1}
\begin{aligned}
\mathbf{\Sigma} &= 
 \frac{1}{N} (
 \begin{bmatrix}
  0 & \cdots & \mu_1 \\
  \vdots  & \ddots & \vdots  \\
  0 & \cdots & \mu_n 
 \end{bmatrix}
 \begin{bmatrix}
  0 & \cdots & 0 \\
  \vdots  & \ddots & \vdots  \\
  \mu_1 & \cdots & \mu_n
 \end{bmatrix} )
 + \eta \mathbf{I}_n \\
& = 
\varepsilon
 \begin{bmatrix}
  {\mu_1}^2 & \cdots & \mu_1 \mu_n \\
  \vdots  & \ddots & \vdots  \\
  \mu_1 \mu_n & \cdots & {\mu_n}^2
 \end{bmatrix}
 + \eta \mathbf{I}_n
\end{aligned}
\end{equation}

\begin{theorem} \label{theorem:Largest Eigenvector appendix}
$\mathbf{\mu}$ is the eigenvector of $\mathbf{\Sigma}$ and corresponds to the largest eigenvalue if $\varepsilon$ and $\| \mathbf{\mu} \|_2$ are both $> 0$.
\end{theorem}
\begin{proof}
Taking the matrix multiplication of $\mathbf{\Sigma}$ and $\mathbf{\mu}$, we get

\begin{equation} \label{eq:sigma2}
\begin{aligned}
\mathbf{\Sigma} \mathbf{\mu} & = 
\varepsilon
 \begin{bmatrix}
  {\mu_1}^2 & \cdots & \mu_1 \mu_n \\
  \vdots  & \ddots & \vdots  \\
  \mu_1 \mu_n & \cdots & {\mu_n}^2
 \end{bmatrix} 
 \begin{bmatrix}
  \mu_1 \\
  \vdots \\
  \mu_n
 \end{bmatrix}
 + \eta \mathbf{I}_n  
 \begin{bmatrix}
  \mu_1 \\
  \vdots \\
  \mu_n
 \end{bmatrix}\\
& = 
\varepsilon
 \begin{bmatrix}
  {\mu_1}^3 + \mu_1 {\mu_2}^2  +\cdots+ \mu_1 {\mu_n}^2 \\
  \vdots \\
  {\mu_1}^2 \mu_n + {\mu_2}^2 \mu_n +\cdots+ {\mu_n}^3
 \end{bmatrix}
 + \eta
 \begin{bmatrix}
  \mu_1 \\
  \vdots \\
  \mu_n
 \end{bmatrix} \\
& = 
\varepsilon ({\mu_1}^2 + \cdots + {\mu_n}^2)
 \begin{bmatrix}
  \mu_1 \\
  \vdots \\
  \mu_n
 \end{bmatrix} 
 + \eta
 \begin{bmatrix}
  \mu_1 \\
  \vdots \\
  \mu_n
 \end{bmatrix} \\
& = 
(\varepsilon \| \mathbf{\mu} \|_2^2 + \eta)
 \begin{bmatrix}
  \mu_1 \\
  \vdots \\
  \mu_n
 \end{bmatrix} \\
 & = 
(\varepsilon \| \mathbf{\mu} \|_2^2 + \eta) \mathbf{\mu}
\end{aligned}
\end{equation}

Thus, $\mathbf{\mu}$ is an eigenvector of $\mathbf{\Sigma}$ with eigenvalue $\lambda_1(\mathbf{\Sigma}) = \varepsilon \| \mathbf{\mu} \|_2^2 + \eta$. Next, we proceed to prove that $\lambda_1(\mathbf{\Sigma})$ is the largest eigenvalue. 

Let $\mathbf{D} = \varepsilon
 \begin{bmatrix}
  {\mu_1}^2 & \cdots & \mu_1 \mu_n \\
  \vdots  & \ddots & \vdots  \\
  \mu_1 \mu_n & \cdots & {\mu_n}^2
 \end{bmatrix}$, \\then we can express $\mathbf{\Sigma}$ as 
 
\begin{equation} \label{eq:D1}
\mathbf{\Sigma} = \mathbf{D} + \eta \mathbf{I}_n
\end{equation}

Similar to (\ref{eq:sigma2}), we can get
\begin{equation} \label{eq:D2}
\begin{aligned}
\mathbf{D} \mathbf{\mu} & = 
\varepsilon
 \begin{bmatrix}
  {\mu_1}^2 & \cdots & \mu_1 \mu_n \\
  \vdots  & \ddots & \vdots  \\
  \mu_1 \mu_n & \cdots & {\mu_n}^2
 \end{bmatrix} 
 \begin{bmatrix}
  \mu_1 \\
  \vdots \\
  \mu_n
 \end{bmatrix} \\
& = 
\varepsilon
 \begin{bmatrix}
  {\mu_1}^3 + \mu_1 {\mu_2}^2  +\cdots+ \mu_1 {\mu_n}^2 \\
  \vdots \\
  {\mu_1}^2 \mu_n + {\mu_2}^2 \mu_n +\cdots+ {\mu_n}^3
 \end{bmatrix} \\
& = 
\varepsilon ({\mu_1}^2 + \cdots + {\mu_n}^2)
 \begin{bmatrix}
  \mu_1 \\
  \vdots \\
  \mu_n
 \end{bmatrix} \\
 & = 
(\varepsilon \| \mathbf{\mu} \|_2^2) \mathbf{\mu}
\end{aligned}
\end{equation}

This shows that $\mathbf{\mu}$ is also an eigenvector of $\mathbf{D}$ with eigenvalue $\lambda_1(\mathbf{D}) = \varepsilon \| \mathbf{\mu} \|_2^2$.

From (\ref{eq:sigma1}), we observe that $\mathbf{D}$ is a product of a matrix by its own transpose. This implies that $\mathbf{D}$ is positive semi-definite and all its eigenvalues are non-negative. Furthermore, the sum of all these eigenvalues is

\begin{equation}
\begin{aligned}
\sum_{i=1}^n \lambda_i(\mathbf{D}) & = 
tr(\mathbf{D}) \\
& = 
\varepsilon \| \mathbf{\mu} \|_2^2 \\
& = \lambda_1(\mathbf{D})
\end{aligned}
\end{equation}

This implies that the other eigenvalues $\lambda_2(\mathbf{D}) = \cdots = \lambda_n(\mathbf{D}) = 0$. From this, we know that all vectors $\mathbf{v}$ which are orthogonal to $\mathbf{\mu}$,

$$
\forall \mathbf{v} \in \mathbb{R}^n: \inner{\mathbf{v}}{\mathbf{\mu}} = 0
$$

$$
\mathbf{D} \mathbf{v} = \mathbf{0}
$$
Combining with (\ref{eq:D1}), we get

\begin{equation} \label{eq:sigma3}
\begin{aligned}
\mathbf{\Sigma} \mathbf{v} & = 
\mathbf{D} \mathbf{v} + \eta \mathbf{I}_n \mathbf{v} \\
& = \eta \mathbf{v}
\end{aligned}
\end{equation}

With this, we can deduce that $\mathbf{\Sigma}$'s other eigenvalues $\lambda_2(\mathbf{\Sigma}) = \cdots = \lambda_n(\mathbf{\Sigma}) = \eta$.

For $\lambda_1(\mathbf{\Sigma})$ to be the largest eigenvalue, this statement has to be true:

$$
\lambda_1(\mathbf{\Sigma}) > \max_{i \neq 1} \lambda_i(\mathbf{\Sigma})
$$

With our previous calculations of $\lambda_i(\mathbf{\Sigma})$ in (\ref{eq:sigma2}) and (\ref{eq:sigma3}), we get

$$
\varepsilon \| \mathbf{\mu} \|_2^2 + \eta > \eta
$$

\begin{equation}
\varepsilon \| \mathbf{\mu} \|_2 > 0
\end{equation}

This statement is true if $\varepsilon > 0$ and $\| \mathbf{\mu} \|_2 > 0$ which completes the proof.
\end{proof}

\begin{remark} \label{remark:sigma_operator_norm}
The operator or spectral norm of $\mathbf{\Sigma}$, $\| \mathbf{\Sigma} \|$, equals to the absolute value of its largest singular value. Since $\mathbf{\Sigma}$ is a positive semi-definite matrix, its largest singular value is the same as its largest eigenvalue. This implies that
\begin{equation}
\| \mathbf{\Sigma} \| = \varepsilon \| \mathbf{\mu} \|_2^2 + \eta
\end{equation}
\end{remark}

\begin{theorem} [Matrix Bernstein \cite{vershynin2018high}] \label{theorem:Matrix Bernstein}
Let $\mathbf{Z}_1, \cdots, \mathbf{Z}_N$ be symmetric $n \times n$ random matrices. Assume that $\|\mathbf{Z}_i\| \leq K$ almost surely and let $\| \sum_i \mathbf{Z}_i^2 \| \leq \sigma^2$. Then,

$$
Pr \left \{ \left  \| \sum_i \mathbf{Z}_i \right \| > t \right \} \leq 2n \exp \left (-c \min \left \{ \frac{t^2}{\sigma^2}, \frac{t}{K} \right \} \right)
$$

where $c > 0$ is an absolute constant.
\end{theorem}

\begin{theorem} [Covariance Estimation \cite{rudelson1999random}] \label{theorem:Covariance Estimation}
Let $\MS = \E \mathbf{z} \mathbf{z}^\top$ be the second moment matrix of $\mathbb{R}^n$ random vector $\mathbf{z}$. With independent samples $\mathbf{z}_1, \cdots, \mathbf{z}_N$,  $\MS_N = \frac{1}{N} \sum_i \mathbf{z}_i \mathbf{z}_i^\top$ is the unbiased estimator of $\MS$.
Assume that $\| \mathbf{z}_i \|_2^2 \leq M$. Then,

$$
Pr\{ \| \MS_N - \MS \ \| > \epsilon \| \MS \| \} \geq 1 - 2n \exp \left(- c_1 
\frac{N \epsilon^2 \MSnorm}{M + \MSnorm} \right)
$$

where $\epsilon \in (0, 1]$.
\end{theorem}

\begin{proof}
Let $\mathbf{Z}_i = \frac{1}{N} \left( \mathbf{z}_i \mathbf{z}_i^\top - \MS \right )$

Then,
\begin{equation} \label{eq:sum z_i}
\begin{aligned}
\sum_{i=1}^N \mathbf{Z}_i & = 
\frac{1}{N} \sum_{i=1}^N \mathbf{z}_i \mathbf{z}_i^\top - \MS \\
& = \MS_N - \MS
\end{aligned}
\end{equation}

To apply Theorem~\ref{theorem:Covariance Estimation} to (\ref{eq:sum z_i}), we need to bound $\|\mathbf{Z}_i\|$ and $\| \sum_i \mathbf{Z}_i^2 \|$.

To bound $\|\mathbf{Z}_i\|$,
$$
\|\mathbf{Z}_i\| = \| \frac{1}{N} \left( \mathbf{z}_i \mathbf{z}_i^\top - \MS \right ) \|
$$

With triangle inequality, we get
\begin{equation} \label{eq:CE1}
\|\mathbf{Z}_i\| \leq \frac{1}{N}  \left( \| \mathbf{z}_i \mathbf{z}_i^\top \| + \| \MS \| \right )
\end{equation}

While considering the term $\| \mathbf{z}_i \mathbf{z}_i^\top \|$, we note that $\mathbf{z}_i \mathbf{z}_i^\top$ is a positive definite matrix. Then,

\begin{equation} \label{eq:CE2}
\begin{aligned}
\| \mathbf{z} \mathbf{z}^\top \| & = 
\left \| \begin{bmatrix}
  z_1 \\
  \vdots\\
  z_n
 \end{bmatrix} 
 \begin{bmatrix}
  z_1 & \cdots & z_n
 \end{bmatrix} \right \| \\
& = 
\left \| \begin{bmatrix}
  {z_1}^2 & \cdots & z_1 z_n \\
  \vdots  & \ddots & \vdots  \\
  z_1 z_n & \cdots & {z_n}^2
 \end{bmatrix} \right \| \\
& = s_1(\mathbf{z} \mathbf{z}^\top) \\
& = \lambda_1(\mathbf{z} \mathbf{z}^\top) \\
& \leq tr(\mathbf{z} \mathbf{z}^\top) \\
& = {z_1}^2 + \cdots + {z_n}^2 \\
& = \| \mathbf{z} \|_2^2
\end{aligned}
\end{equation}

Substituting (\ref{eq:CE2}) into (\ref{eq:CE1}), we get
$$
\|\mathbf{Z}_i\| \leq \frac{1}{N}  \left( \| \mathbf{z}_i \|_2^2 + \| \MS \| \right )
$$

Since $\| \mathbf{z}_i \|_2^2 \leq M$,
\begin{equation} \label{eq:CE K}
\|\mathbf{Z}_i\| \leq \frac{M + \| \MS \|}{N} = K
\end{equation}
where $K$ is the term from Theorem~\ref{theorem:Covariance Estimation}.

To bound $\| \sum_i \mathbf{Z}_i^2 \|$, we first expand $\mathbf{Z}_i^2$.

\begin{equation} \label{eq:CE3}
\begin{aligned}
\mathbf{Z}_i^2 & = 
\frac{1}{N^2} \left( \mathbf{z}_i \mathbf{z}_i^\top - \MS \right )^2 \\
& = 
\frac{1}{N^2} \left[ (\mathbf{z}_i \mathbf{z}_i^\top)^2 - \MS (\mathbf{z}_i \mathbf{z}_i^\top) - (\mathbf{z}_i \mathbf{z}_i^\top) \MS + \MS^2 \right ]
\end{aligned}
\end{equation}

Taking expectation of both sides, we get
\begin{equation}
\begin{aligned}
\E \mathbf{Z}_i^2 & = 
\frac{1}{N^2} \left [ \E (\mathbf{z}_i \mathbf{z}_i^\top \mathbf{z}_i \mathbf{z}_i^\top) - \E [\MS (\mathbf{z}_i \mathbf{z}_i^\top)] - \E [(\mathbf{z}_i \mathbf{z}_i^\top) \MS ] + \E \MS^2 \right ] \\
& = 
\frac{1}{N^2} \left [ \E (\mathbf{z}_i \| \mathbf{z}_i \|_2^2 \mathbf{z}_i^\top) - \MS \E (\mathbf{z}_i \mathbf{z}_i^\top) - \E (\mathbf{z}_i \mathbf{z}_i^\top) \MS + \MS^2 \right ]
\end{aligned}
\end{equation}

Since $\| \mathbf{z}_i \|_2^2 \leq M$,
$$
\E \mathbf{Z}_i^2 \preceq \frac{1}{N^2} \left [ M \E (\mathbf{z}_i \mathbf{z}_i^\top) - \MS \E (\mathbf{z}_i \mathbf{z}_i^\top) - \E (\mathbf{z}_i \mathbf{z}_i^\top) \MS + \MS^2 \right ]
$$

By definition, $\MS = \E \mathbf{z} \mathbf{z}^\top$
\begin{equation} \label{eq:CE4}
\begin{aligned}
\E \mathbf{Z}_i^2 & \preceq 
\frac{1}{N^2} \left ( M \MS - \MS \MS - \MS \MS + \MS^2 \right ) \\
& = \frac{1}{N^2} ( M \MS - \MS^2 )
\end{aligned}
\end{equation}

Thus,
$$
\left \| \E \sum_{i = 1}^N \mathbf{Z}_i^2 \right \| = \left \| \frac{1}{N} ( M \MS - \MS^2 ) \right \|
$$

With triangle inequality, we get
\begin{equation} \label{eq:CE sigma}
\begin{aligned}
\left \| \E \sum_{i = 1}^N \mathbf{Z}_i^2 \right \| & \leq 
\left \| \frac{M}{N} \MS \right \| + \left \| \frac{1}{N} \MS^2 \right \| \\
& = \frac{M \| \MS \| + \| \MS \|^2}{N} = \sigma^2
\end{aligned}
\end{equation}
where $\sigma^2$ is the term from Theorem~\ref{theorem:Covariance Estimation}.

Applying Theorem~\ref{theorem:Covariance Estimation} for $\sum_{i=1}^N \mathbf{Z}_i$ with (\ref{eq:CE K}) and (\ref{eq:CE sigma}), and recalling from (\ref{eq:sum z_i}) where \\ $\sum_{i=1}^N \mathbf{Z}_i = \MS_N - \MS$, we get

\begin{multline}\label{eq:CE5}
Pr\{ \| \MS_N - \MS \ \| > \epsilon \| \MS \| \} \\
\leq 2n \exp \left(- c_1 \min \left \{ \frac{N \epsilon^2 \MSnorm^2}{M \MSnorm + \MSnorm^2}, \frac{N \epsilon \MSnorm}{M + \MSnorm} \right \}  \right) \\
=  2n \exp \left(- c_1 \min \left \{ \frac{N \epsilon^2 \MSnorm}{M + \MSnorm}, \frac{N \epsilon \MSnorm}{M + \MSnorm} \right \}  \right)
\end{multline}


Assuming $\epsilon \in (0,1]$,

$$
Pr\{ \| \MS_N - \MS \ \| > \epsilon \| \MS \| \} \leq 2n \exp \left(- c_1  \frac{N \epsilon^2 \MSnorm}{M + \MSnorm} \right)
$$

Thus,
\begin{equation} \label{eq:CE final}
Pr\{ \| \MS_N - \MS \ \| \leq \epsilon \| \MS \| \} \geq 1 - 2n \exp \left(- c_1  \frac{N \epsilon^2 \MSnorm}{M + \MSnorm} \right)
\end{equation}

\end{proof}

\begin{theorem} [Davis-Kahan Theorem] \label{theorem:Davis-Kahan}
Let $\mathbf{S}$ and $\mathbf{T}$ be symmetric matrices with same dimensions. Fix $i$ and assume that the $i$th largest eigenvalue is well separated from the other eigenvalues:

$$
\min_{j:j \neq i} |\lambda_i(\mathbf{S}) - \lambda_j(\mathbf{S})| = \delta > 0
$$

Then, the unit eigenvectors $\mathbf{v}_i (\mathbf{S})$ and $\mathbf{v}_i (\mathbf{T})$ are close to each other up to a sign.

$$
\exists \theta \in \{-1, 1\}: \| \mathbf{v}_i (\mathbf{S}) - \theta \mathbf{v}_i (\mathbf{T}) \|_2 \leq \frac{2^{\frac{2}{3}} \| \mathbf{S} - \mathbf{T} \|}{\delta}
$$

\end{theorem}

\begin{theorem} [Guarantee of Poison Classification through Clustering] \label{theorem:Poison Clustering Appendix}
Assume that all $\mathbf{z}_i$ are normalized such that $\| \mathbf{z}_i \|_2 = 1$. Then the error probability of the poison clustering algorithm is given by


\begin{multline}
Pr \left \{ N_{\text{error}} \leq c_2 N \epsilon \left ( \frac{1}{\| \Mmu \|_2} + \frac{\eta}{\varepsilon \| \Mmu \|_2^3} \right ) \right \} \geq \\ 1 - 2n \exp \left ( -c_1 N \epsilon^2 \frac{(\varepsilon\| \Mmu \|_2^2 + \eta)}{1 + \varepsilon\| \Mmu \|_2^2 + \eta} \right )
\end{multline}


where $ N_{\text{error}}$ is the number of misclassified points and $\epsilon \in (0, 1]$.
\end{theorem}

\begin{proof}
To find the difference between unit eigenvectors $\mathbf{v}_1(\MS)$ and $\mathbf{v}_1 (\MS_N)$, we applying Theorem~\ref{theorem:Davis-Kahan} for $i = 1$, $\mathbf{S} = \MS$, $\mathbf{T} = \MS_N$,

$$
\delta = \min_{j \neq 1} |\lambda_1(\mathbf{\MS}) - \lambda_j(\mathbf{\MS})|
$$
With our previous calculations of $\lambda_i(\mathbf{\Sigma})$ in (\ref{eq:sigma2}) and (\ref{eq:sigma3}), we get

\begin{equation} \label{eq:delta}
\begin{aligned}
\delta & = 
\varepsilon \| \mathbf{\mu} \|_2^2 + \eta - \eta \\
& = \varepsilon \| \mathbf{\mu} \|_2^2
\end{aligned}
\end{equation}

The conclusion of Theorem~\ref{theorem:Davis-Kahan} then becomes

$$
\exists \theta \in \{-1, 1\}: \| \mathbf{v}_1 (\MS) - \theta \mathbf{v}_1 (\MS_N) \|_2 \leq \frac{2^{\frac{2}{3}}}{\varepsilon \| \mathbf{\mu} \|_2^2} \| \MS - \MS_n \|
$$

Combining this with the Theorem~\ref{theorem:Covariance Estimation}, we get

\begin{multline} \label{eq:first eigenvector probability bound}
Pr \left \{ \| \mathbf{v}_1 (\MS) - \theta \mathbf{v}_1 (\MS_N) \|_2 \leq \frac{2^{\frac{2}{3}}}{\varepsilon \| \mathbf{\mu} \|_2^2} \epsilon \| \MS \| \right \} \geq  \\ 1 - 2n \exp \left(- c_1 \frac{N \epsilon^2 \MSnorm}{M + \MSnorm} \right)
\end{multline} 

We now have a probability bound of difference between $\mathbf{v}_1(\MS)$ and $\mathbf{v}_1 (\MS_N)$. To find the probability bound on the number of misclassified points, let us consider the case where $\mathbf{z}_i$ is from a non-poisoned point.

If $\mathbf{z}_i$ is from a poisoned point,
\begin{equation} \label{eq:nonpoisoned product}
\begin{aligned}
\E \inner{\Mmu}{\mathbf{z}_i} & = 
\E \left ( 
\begin{bmatrix}
  \mu_1 & \cdots & \mu_n
 \end{bmatrix} 
 \begin{bmatrix}
  \mu_1 + g_1 \\
  \vdots \\
  \mu_n + g_n
 \end{bmatrix} \right ) \\
& = \E ({\mu_1}^ 2 + g_1 \mu_1 + \cdots + {\mu_n}^ 2 + g_n \mu_n) \\
& = \E ({\mu_1}^ 2 + \cdots + {\mu_n}^ 2) + \E (g_1 \mu_1 + \cdots + g_n \mu_n) \\
& = \| \mathbf{\mu} \|_2^2
\end{aligned}
\end{equation}

Dividing by $\| \mathbf{\mu} \|_2^2$ on both sides, we get
$$
\E \inner{\frac{\Mmu}{\| \mathbf{\mu} \|_2}} {\frac{\mathbf{z}_i}{\| \mathbf{\mu} \|_2} } = 1
$$

From Theorem~\ref{theorem:Largest Eigenvector}, since we know that $\mathbf{\mu}$ is the first eigenvector of $\MS$, $\frac{\Mmu}{\| \mathbf{\mu} \|_2}$ is its first unit eigenvector $\mathbf{v}_1(\MS)$. Then,

\begin{equation} \label{eq:poisoned product}
\E \inner{\mathbf{v}_1(\MS)} { \frac{\mathbf{z}_i}{\| \mathbf{\mu} \|_2} } = 1
\end{equation}

If $\mathbf{z}_i$ is from a non-poisoned point,
\begin{equation} \label{eq:nonpoisoned product}
\begin{aligned}
\E \inner{\mathbf{v}_1(\MS)}{\frac{\mathbf{z}_i}{\| \mathbf{\mu} \|_2}} & = 
\frac{1}{\| \mathbf{\mu} \|_2^2} \E \left ( 
\begin{bmatrix}
  \mu_1 & \cdots & \mu_n
 \end{bmatrix} 
 \begin{bmatrix}
  g_1 \\
  \vdots \\
  g_n
 \end{bmatrix} \right ) \\
& = \frac{1}{\| \mathbf{\mu} \|_2^2}  \E (g_1 \mu_1 + \cdots + g_n \mu_n) \\
& = \frac{1}{\| \mathbf{\mu} \|_2^2} \cdot 0 = 0
\end{aligned}
\end{equation}

Now, we consider the inner product of $\mathbf{z}_i$ with the difference between $\mathbf{v}_1(\MS)$ and $\mathbf{v}_1 (\MS_N)$.
$$
\mathbf{z}_i^\top \mathbf{v}_1 (\MS) - \theta \mathbf{z}_i^\top \mathbf{v}_1 (\MS_N) = \mathbf{z}_i^\top ( \mathbf{v}_1 (\MS) - \theta \mathbf{v}_1 (\MS_N) )
$$

By Cauchy-Schwarz Inequality,
$$
| \mathbf{z}_i^\top \mathbf{v}_1 (\MS) - \theta \mathbf{z}_i^\top \mathbf{v}_1 (\MS_N) | \leq \| \mathbf{z}_i \|_2 \cdot \| \mathbf{v}_1 (\MS) - \theta \mathbf{v}_1 (\MS_N) \|_2
$$

By considering all the N samples of  $x_i$,
$$
\sum_{i = 1}^N | \mathbf{z}_i^\top \mathbf{v}_1 (\MS) - \theta \mathbf{z}_i^\top \mathbf{v}_1 (\MS_N) |  \leq N \| \mathbf{z}_i \|_2 \cdot \| \mathbf{v}_1 (\MS) - \theta \mathbf{v}_1 (\MS_N) \|_2
$$

Dividing by $\| \mathbf{\mu} \|_2$ on both sides, we get
\begin{multline}
\sum_{i = 1}^N | \frac{\mathbf{z}_i^\top}{\| \mathbf{\mu} \|_2} \mathbf{v}_1 (\MS) - \theta \frac{\mathbf{z}_i^\top}{\| \mathbf{\mu} \|_2} \mathbf{v}_1 (\MS_N) |  \leq \\ N \frac{\| \mathbf{z}_i \|_2}{\| \mathbf{\mu} \|_2} \| \mathbf{v}_1 (\MS) - \theta \mathbf{v}_1 (\MS_N) \|_2
\end{multline}

\begin{multline}
\sum_{i = 1}^N | \inner{\mathbf{v}_1(\MS)} { \frac{\mathbf{z}_i}{\| \mathbf{\mu} \|_2} } - \theta \inner{\mathbf{v}_1(\MS_N)} { \frac{\mathbf{z}_i}{\| \mathbf{\mu} \|_2} } |  \leq \\ N \frac{\| x_i \|_2}{\| \mathbf{\mu} \|_2} \| \mathbf{v}_1 (\MS) - \theta \mathbf{v}_1 (\MS_N) \|_2
\end{multline}

Combining this with (\ref{eq:first eigenvector probability bound}), we get
\begin{multline} \label{eq:before Nerror}
\sum_{i = 1}^N | \inner{\mathbf{v}_1(\MS)} { \frac{\mathbf{z}_i}{\| \mathbf{\mu} \|_2} } - \theta \inner{\mathbf{v}_1(\MS_N)} { \frac{\mathbf{z}_i}{\| \mathbf{\mu} \|_2} } |  \leq \\  N \frac{\| \mathbf{z}_i \|_2}{\| \mathbf{\mu} \|_2} \cdot \frac{2^{\frac{2}{3}}}{\varepsilon \| \mathbf{\mu} \|_2^2} \epsilon \| \MS \|
\end{multline}

with probability $\geq 1 - 2n \exp \left(- c_1 \frac{N \epsilon^2 \MSnorm}{M + \MSnorm} \right)$.

From (\ref{eq:poisoned product}) and (\ref{eq:nonpoisoned product}), we know that the expected value of $\inner{\mathbf{v}_1(\MS)} { \frac{\mathbf{z}_i}{\| \mathbf{\mu} \|_2} }$ is either $0$ or $1$. So, every sample $\mathbf{z}_i$ for which $\inner{\mathbf{v}_1(\MS)} { \frac{\mathbf{z}_i}{\| \mathbf{\mu} \|_2} }$ and $\inner{\mathbf{v}_1(\MS_N)} { \frac{\mathbf{z}_i}{\| \mathbf{\mu} \|_2} }$ disagree contributes at least 1 to the sum in (\ref{eq:before Nerror}). Then, we can interpret the sum as the number of erroneously classified points $N_{\text{error}}$ when using $\mathbf{v}_1(\MS_N)$ to separate poisoned from non-poisoned points.

Assume that all $\mathbf{z}_i$ are normalized vectors, $\| \mathbf{z}_i \|_2 = 1$ and $M = 1$. Moreover, we know from Remark~\ref{remark:sigma_operator_norm} that $\| \mathbf{\Sigma} \| = \varepsilon \| \mathbf{\mu} \|_2^2 + \eta$. Thus,
$$
N_{\text{error}} \leq c_3 N \epsilon \cdot \frac{\varepsilon \| \mathbf{\mu} \|_2^2 + \eta}{\varepsilon \| \mathbf{\mu} \|_2^3}
$$
with probability $\geq 1 - 2n \exp \left(- c_1 N \epsilon^2  \frac{\varepsilon \| \mathbf{\mu} \|_2^2 + \eta}{1 + \varepsilon \| \mathbf{\mu} \|_2^2 + \eta} \right)$, where $c_3 > 0$ is an absolute constant.

\end{proof}

\newpage

\section{Algorithms}

\begin{algorithm}
 \label{algo:Find-Poison-Target-Base-Class}
\textbf{Input:} Training data containing poisoned samples $D$, poisoned model $f_p$. Let $D_y$ be the set of training examples corresponding to label $y$, cluster Wasserstein distance ratio threshold $\tau$. Let $G_y(\mathbf{x})$ be $\frac{\partial E_y}{\partial {\mathbf{x}}}$ where $E_y$ is the loss function value with respect to label $y$.

~~\For{ \textbf{all} $y$ }{
    $N_y = |D_y|$ which is the number of samples labeled $y$
    
    ~~\For{ \textbf{all} $\mathbf{x}_i \in D_y$}{
        Compute $\hat{G}_y = \frac{G_y(\mathbf{x}_i)}{\|G_y(\mathbf{x}_i)\|_2}$
    }
    Let $\mathbf{M}_y = [\hat{G}_y]^{N_y}_{i=1}$ be the $N_y \times n$ matrix of $\hat{G}$.
    
    Compute $\mathbf{v}_y$, the first right singular vector of $\mathbf{M}_y$ with SVD.
    
    Compute $\mathbf{t}_y = \mathbf{M}_y \mathbf{v}_y$.
    
    Execute unsupervised clustering on $\mathbf{T}_y$ to get 2 clusters, $C_1$ and $C_2$.
    
    $W_{2y} = W_2(C_1, C_2)$

        
        
    
}
~~$ y_{target} = \max\limits_{y} W_{2y} $

~~\If{$\frac{W_{2{y_{target}}}}{\mean\limits_{y \neq y_{target}} (W_{2y})} > \tau$}{
        
    $\textit{target\_class} = y_{target}$  
        
    \For{ \textbf{all} $y \neq \textit{target\_class}$ }{
        $N_{poisoned} = |S_{poisoned}|$
        
        ~~\For{ \textbf{all} $\mathbf{x}_i \in S_{poisoned}$}{
            Compute $\hat{G}_y = \frac{G_y(\mathbf{x}_i)}{\|G_y(\mathbf{x}_i)\|_2}$
        }
        Let $\mathbf{M}_y = [\hat{G}_y]^n_{i=1}$ be the $N_{poisoned} \times n$ matrix of $\hat{G}$.
        
        Compute $\mathbf{v}_y$, the first right singular vector of $\mathbf{M}_y$ with SVD.
        
        Compute $\mathbf{t}_y = \mathbf{M}_y \mathbf{v}_y$.
        
        Compute $\hat{t}_y = \mean(\mathbf{t}_y)$.
    }
    
    $\textit{base\_class} = \argmax\limits_{y \neq \textit{target\_class}} | \hat{t}_y |$  
    
    Return $\textit{target\_class}, \textit{base\_class}$

    }

    
    
    
    



 \caption{Find-Poison-Target-Base-Class}
\end{algorithm}

\begin{algorithm}
 \label{algo:Filter-Poisoned-Images}
\textbf{Input:} Training data containing poisoned samples $D$, poisoned model $f_p$. Let $D_y$ be the set of training examples corresponding to label $y$. Let $G_y(\mathbf{x})$ be $\frac{\partial E_y}{\partial {\mathbf{x}}}$ where $E_y$ is the loss function value with respect to label $y$.

~~$N_{\textit{target\_class}} = |D_{\textit{target\_class}}|$ which is the number of samples labeled $\textit{target\_class}$

~~\For{ \textbf{all} $\mathbf{x}_i \in D_{\textit{target\_class}}$}{
    Compute $\hat{G}_{\textit{base\_class}} = \frac{G_{\textit{base\_class}}(\mathbf{x}_i)}{\|G_{\textit{base\_class}}(\mathbf{x}_i)\|_2}$
}
~~Let $\mathbf{M} = [\hat{G}_{\textit{base\_class}}]^{N_{\textit{target\_class}}}_{i=1}$ be the $N_{\textit{target\_class}} \times n$ matrix of $\hat{G}$.

~~Compute $\mathbf{v}$, the first right singular vector of $\mathbf{M}$ with SVD.

~~Compute $\mathbf{t} = \mathbf{M} \mathbf{v}$.

~~Execute unsupervised clustering on $\mathbf{T}$ to get 2 clusters, $C_1$ and $C_2$.

\eIf{$|C_1| > |C_2|$}{
$D_f = C_1$,
$S_{poisoned} = C_2$
}{
$D_f = C_2$,
$S_{poisoned} = C_1$
}
    
~~Return $D_f, S_{poisoned}$

 \caption{Filter-Poisoned-Images}
\end{algorithm}

\begin{algorithm}
 \label{algo:Add-Counterpoison-Perturbation}
\textbf{Input:} Training data containing poisoned samples $D$, poisoned model $f_p$. Let $D_y$ be the set of training examples corresponding to label $y$, filtered poisoned samples $S_{poisoned}$, perturbation factor $\rho$. Let $G_y(\mathbf{x})$ be $\frac{\partial E_y}{\partial {\mathbf{x}}}$ where $E_y$ is the loss function value with respect to label $y$. 

~~\For{ \textbf{all} $y \neq \textit{target\_class}$ }{
    $N_{poisoned} = |S_{poisoned}|$
    
    ~~\For{ \textbf{all} $\mathbf{x}_i \in S_{poisoned}$}{
        Compute $\hat{G}_y = \frac{G_y(\mathbf{x}_i)}{\|G_y(\mathbf{x}_i)\|_2}$
    }
    Let $\mathbf{M}_y = [\hat{G}_y]^{N_{poisoned}}_{i=1}$ be the $N_{poisoned} \times n$ matrix of $\hat{G}$.
    
    Compute $\mathbf{v}_y$, the first right singular vector of $\mathbf{M}_y$ with SVD.
    
    ~~\For{ \textbf{all} $\mathbf{x}_j \in D_y$}{
        Set $\mathbf{x}_j = Clip(\mathbf{x}_j + \rho \mathbf{v}_y)$
    }
}
    
~~Return $D$

 \caption{Add-Counterpoison-Perturbation}
\end{algorithm}

\newpage

\section{Additional Figures}


\begin{table*}[tp]
    \centering
    \caption{Appendix: (a) Overlay poison image, (b) the first right vector of input gradients for all target class images which include clean and poisoned images. (c) The first right vector of input gradients for only clean target class images.}
        \begin{tabular}{ llcccccc }
         \hline
         Poison & Sample & Target & \multicolumn{2}{c}{1st V of all target images} && \multicolumn{2}{c}{1st V of clean target images} \\
         \cline{4-5}
         \cline{7-8}
         ~ & ~ & ~ & + & - && + & - \\
         \hline
        \parbox[l]{1em}{\includegraphics[width=3em]{images/a_overlay.pdf}} & \parbox[l]{1em}{\includegraphics[width=3em]{images/a_overlay_sample.pdf}} & Dog & \parbox[l]{1em}{\includegraphics[width=3em]{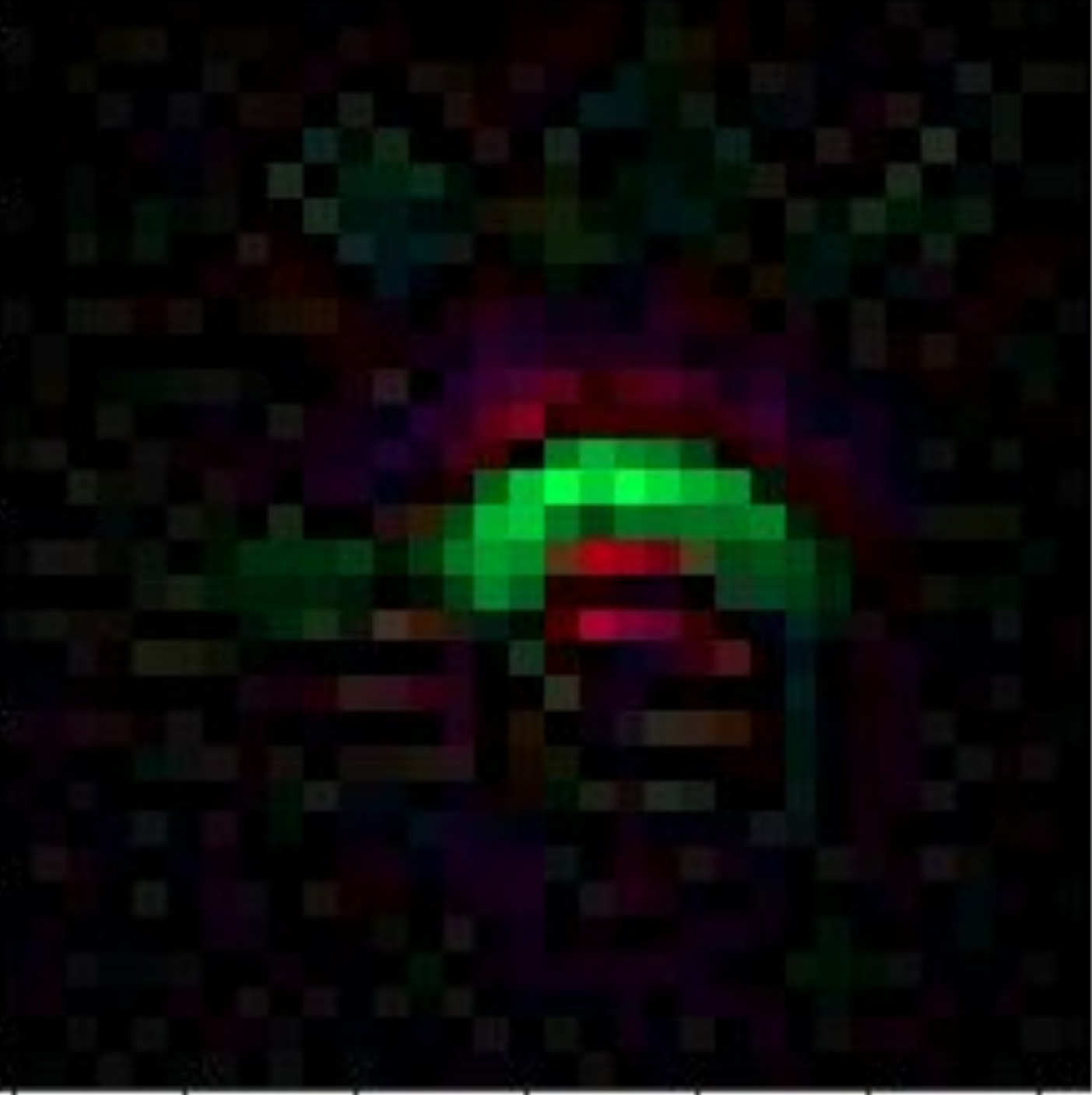}} & \parbox[l]{1em}{\includegraphics[width=3em]{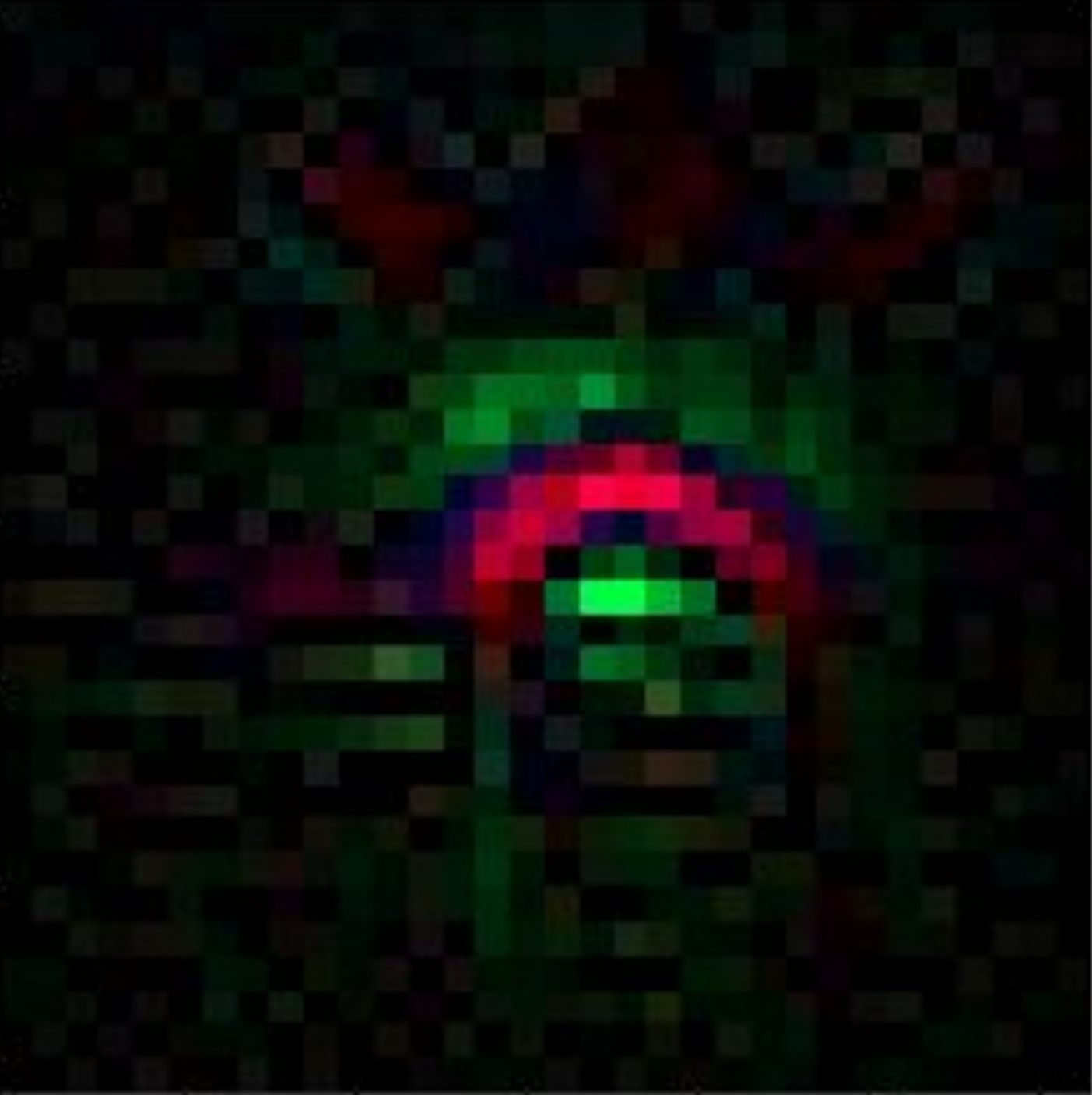}} && \parbox[l]{1em}{\includegraphics[width=3em]{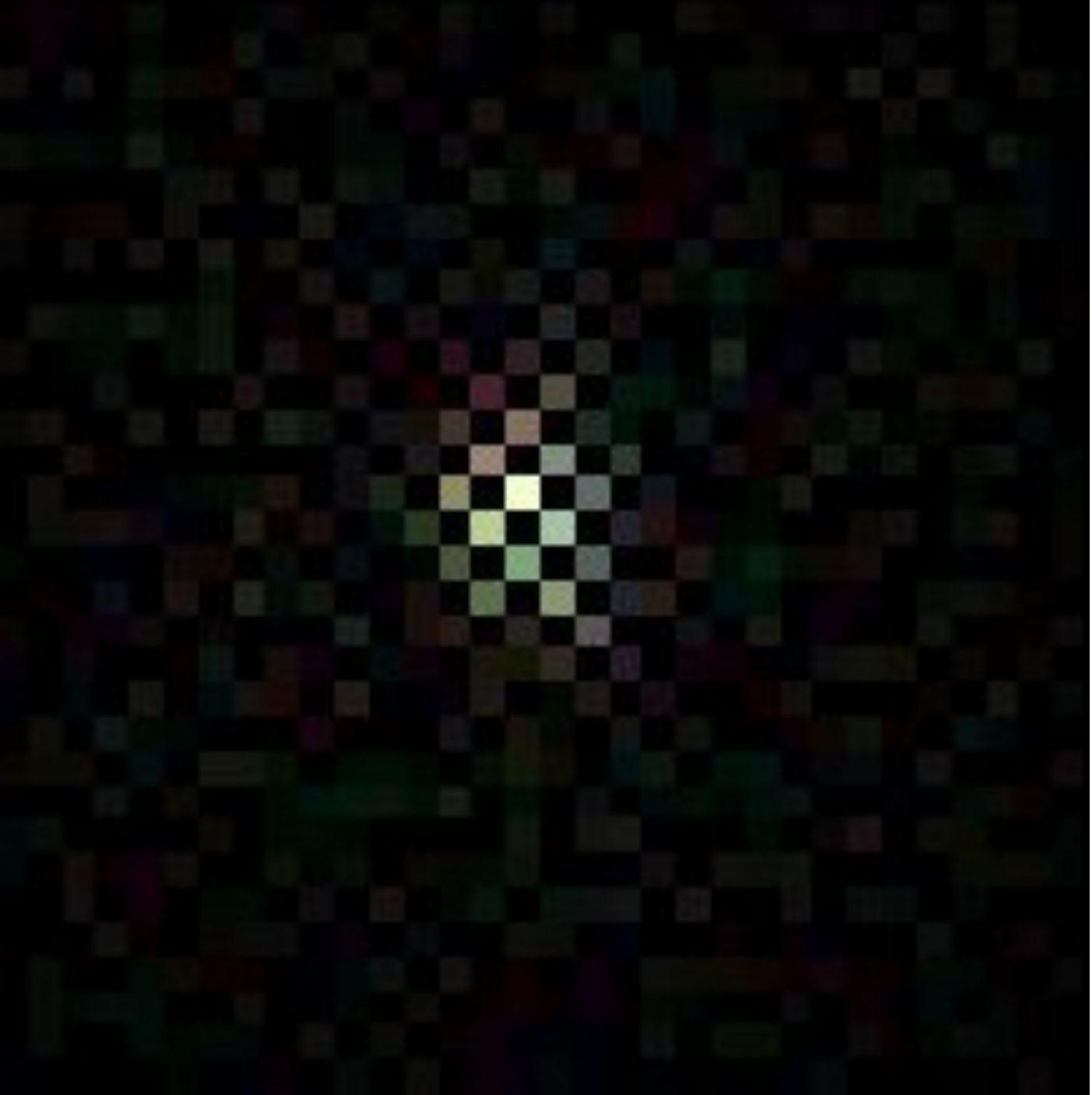}} & \parbox[l]{1em}{\includegraphics[width=3em]{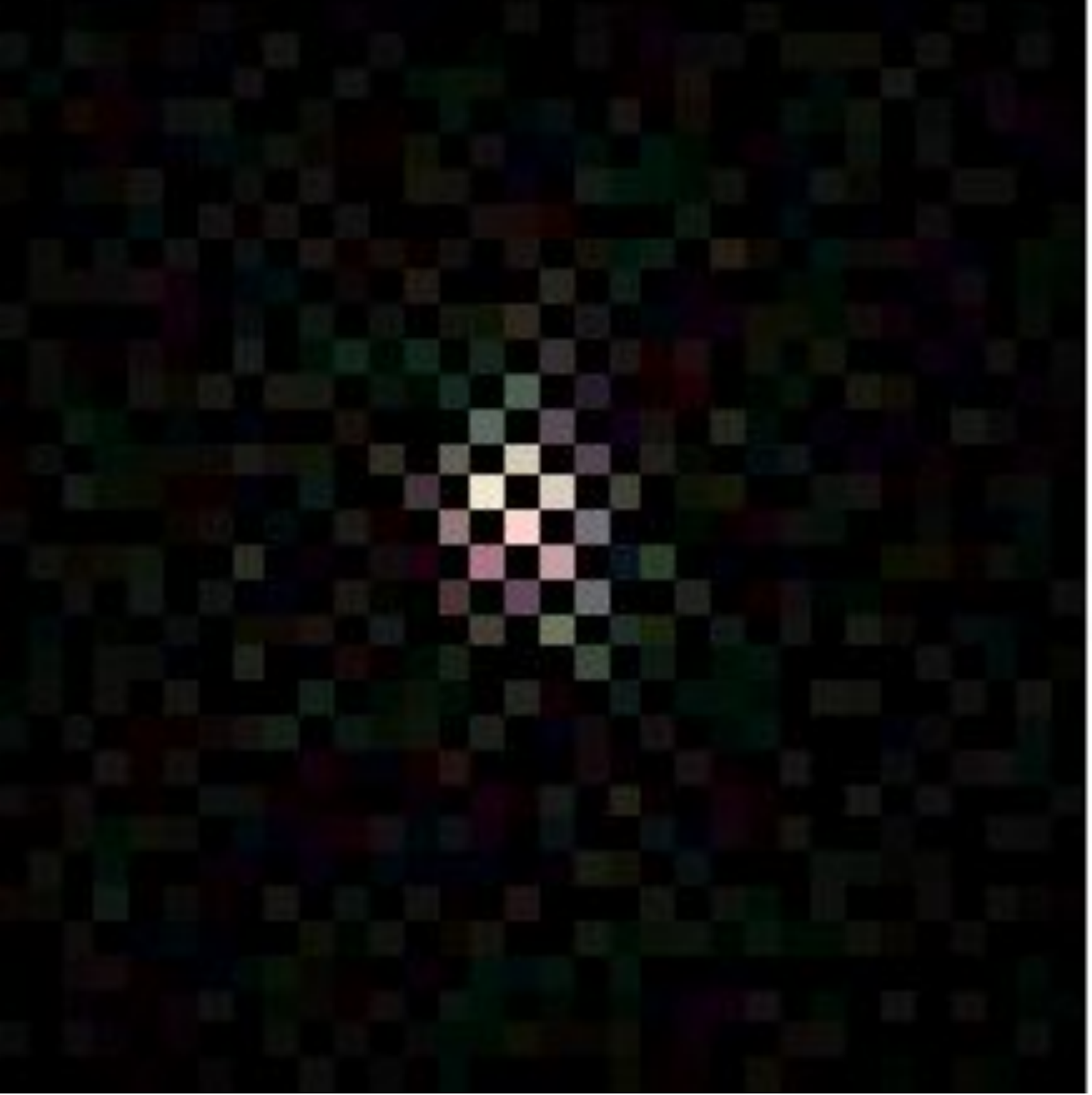}} \\
        \parbox[l]{1em}{\includegraphics[width=3em]{images/b_overlay.pdf}} & \parbox[l]{1em}{\includegraphics[width=3em]{images/b_overlay_sample.pdf}} & Frog & \parbox[l]{1em}{\includegraphics[width=3em]{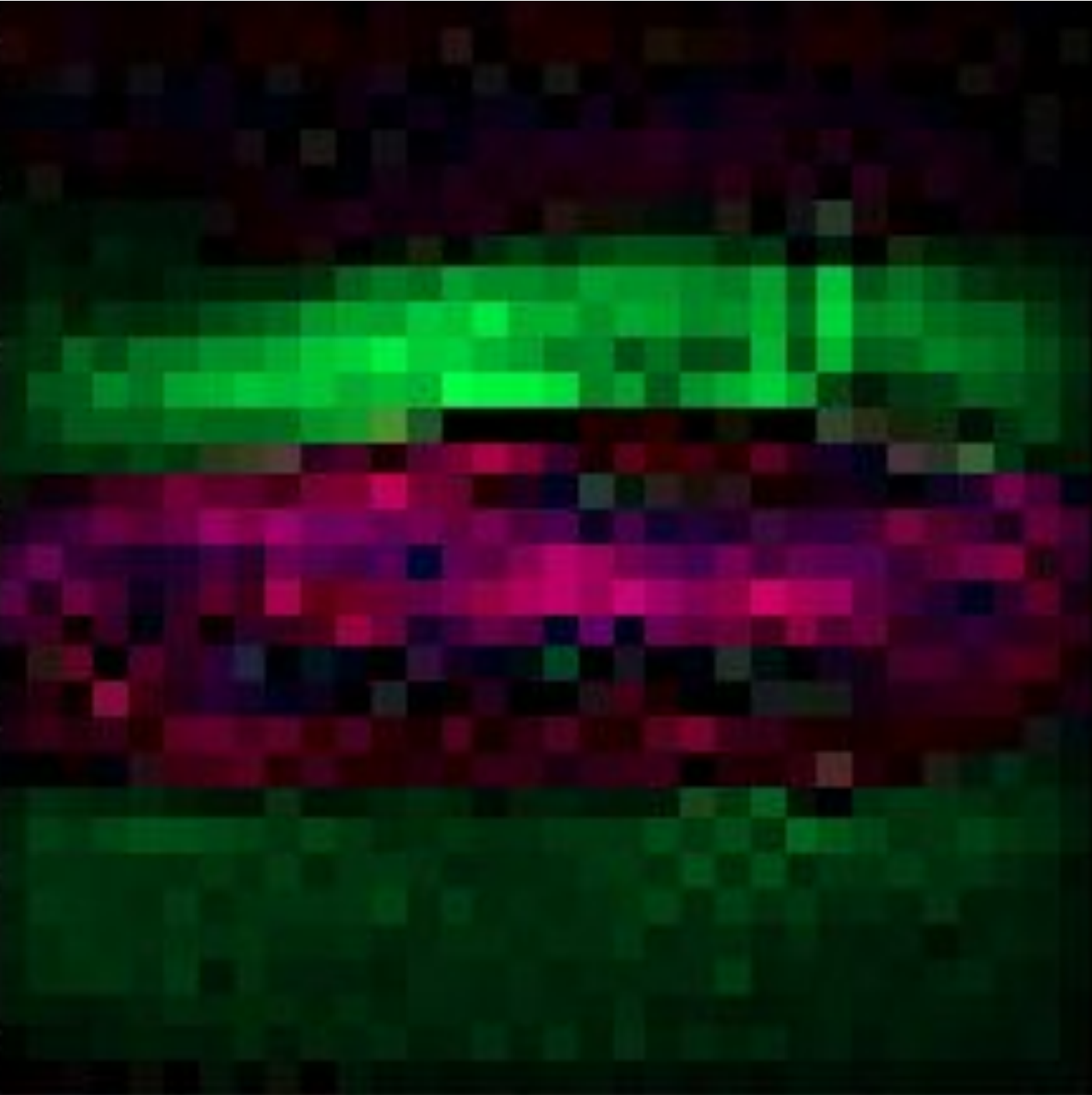}} & \parbox[l]{1em}{\includegraphics[width=3em]{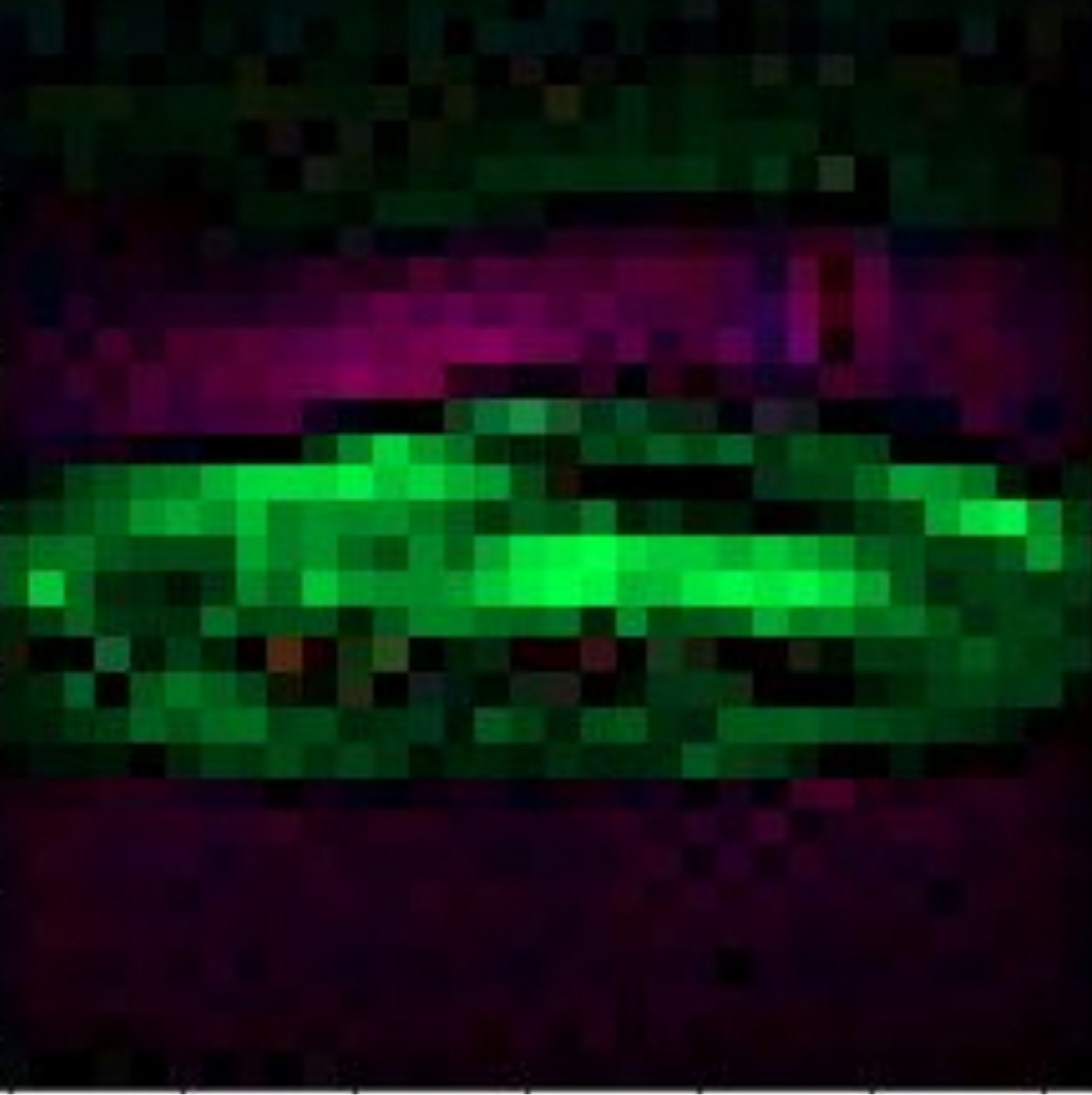}} && \parbox[l]{1em}{\includegraphics[width=3em]{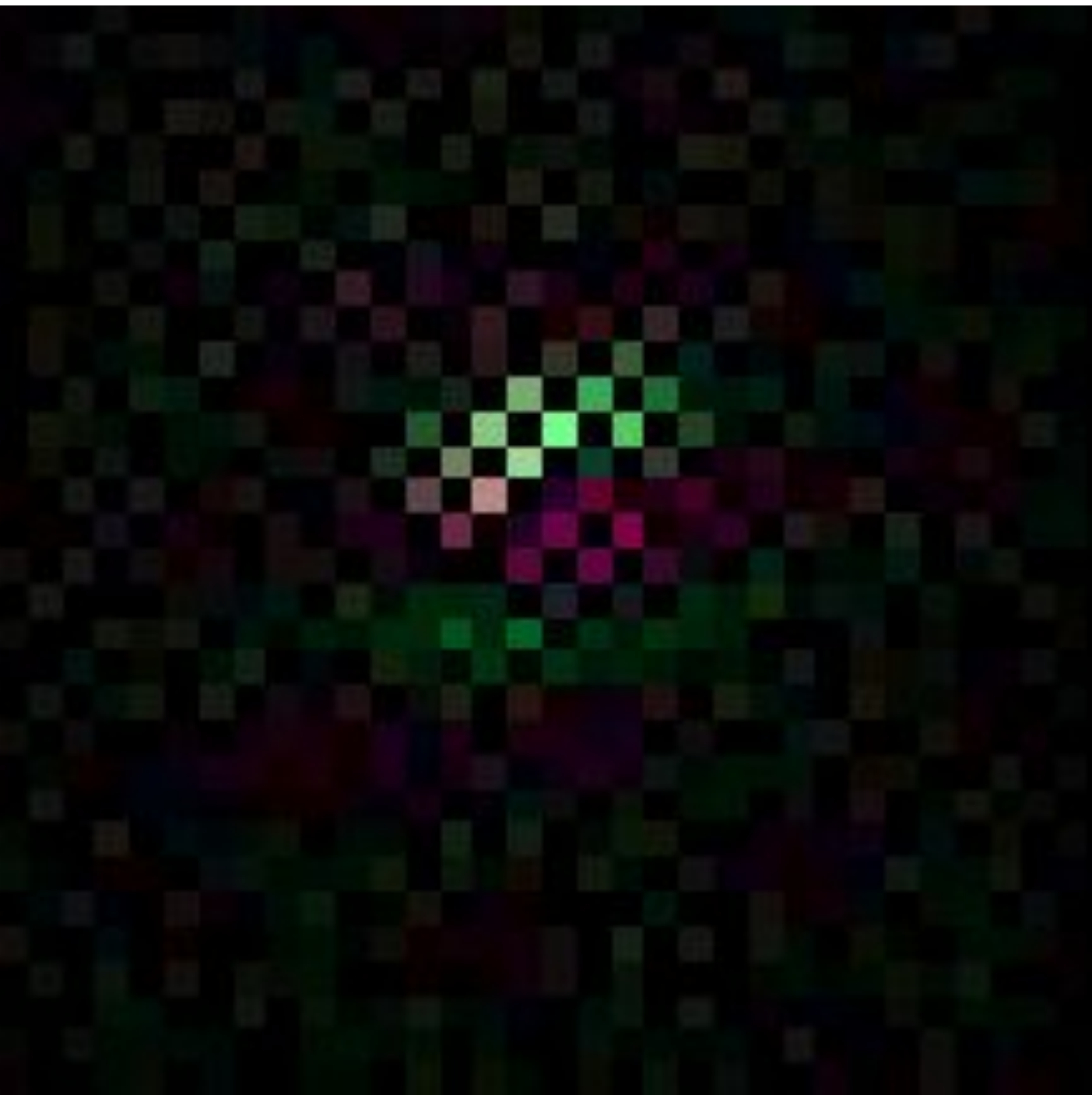}} & \parbox[l]{1em}{\includegraphics[width=3em]{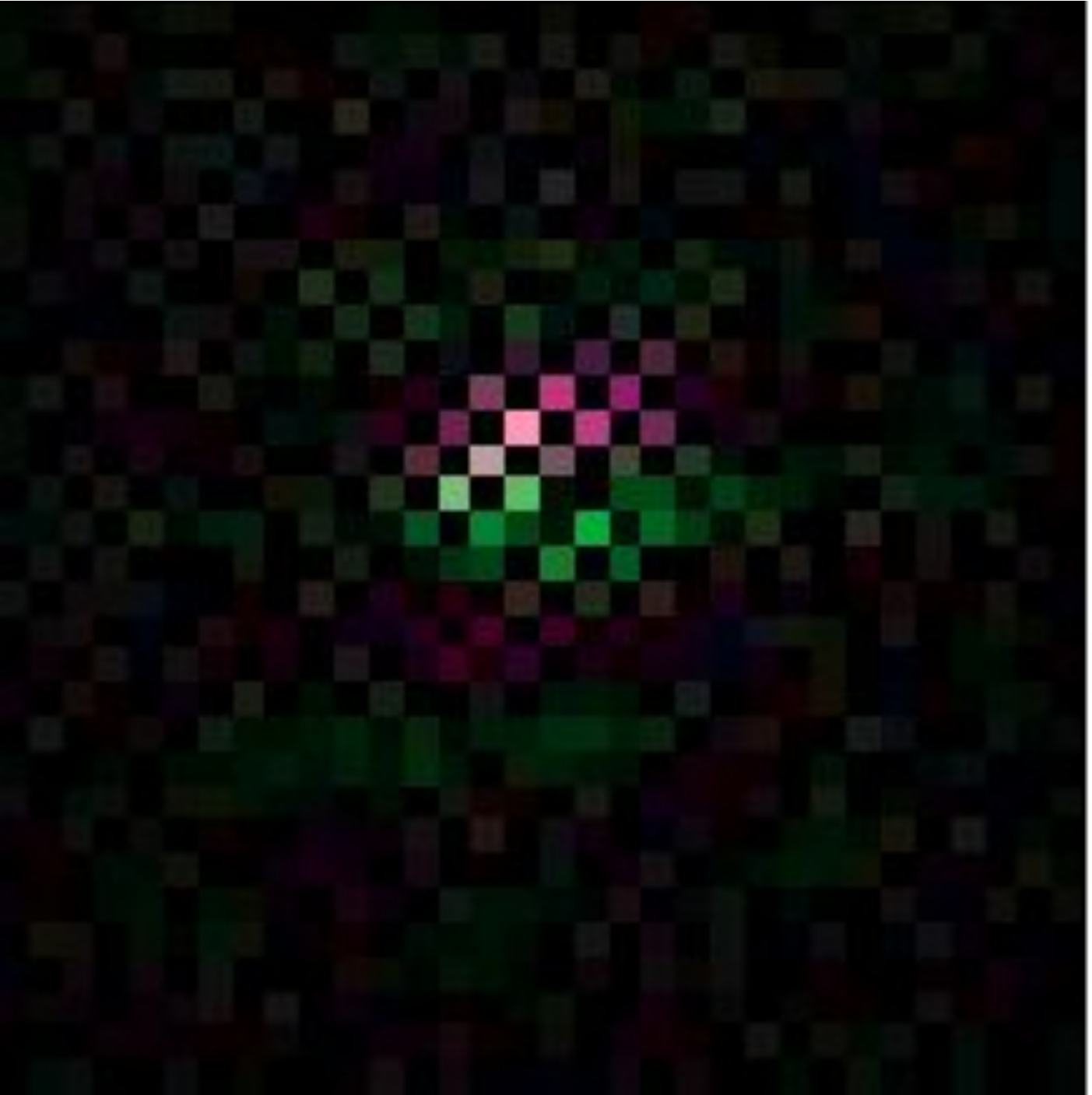}} \\
        \parbox[l]{1em}{\includegraphics[width=3em]{images/c_overlay.pdf}} & \parbox[l]{1em}{\includegraphics[width=3em]{images/c_overlay_sample.pdf}} & Cat & \parbox[l]{1em}{\includegraphics[width=3em]{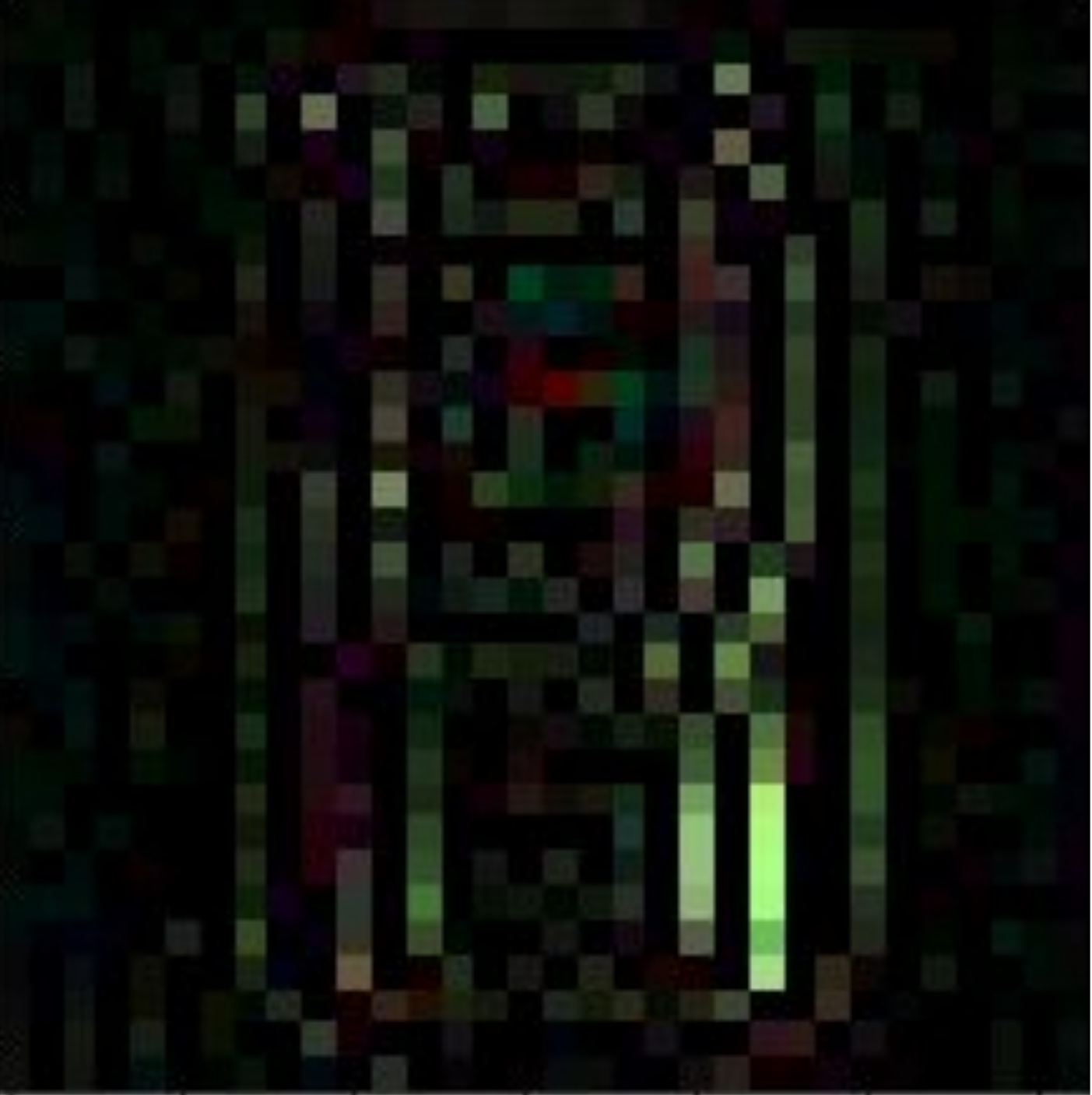}} & \parbox[l]{1em}{\includegraphics[width=3em]{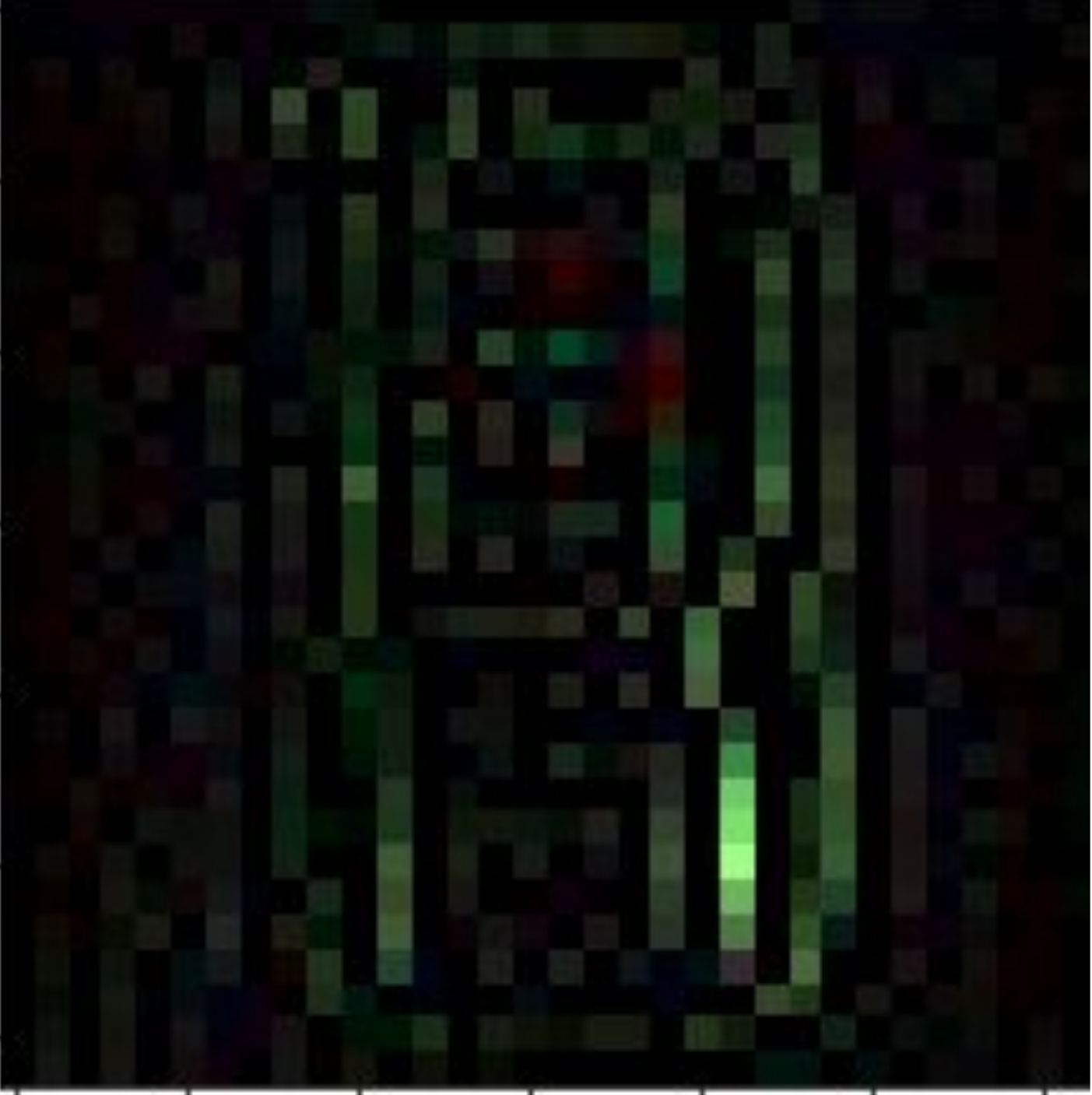}} && \parbox[l]{1em}{\includegraphics[width=3em]{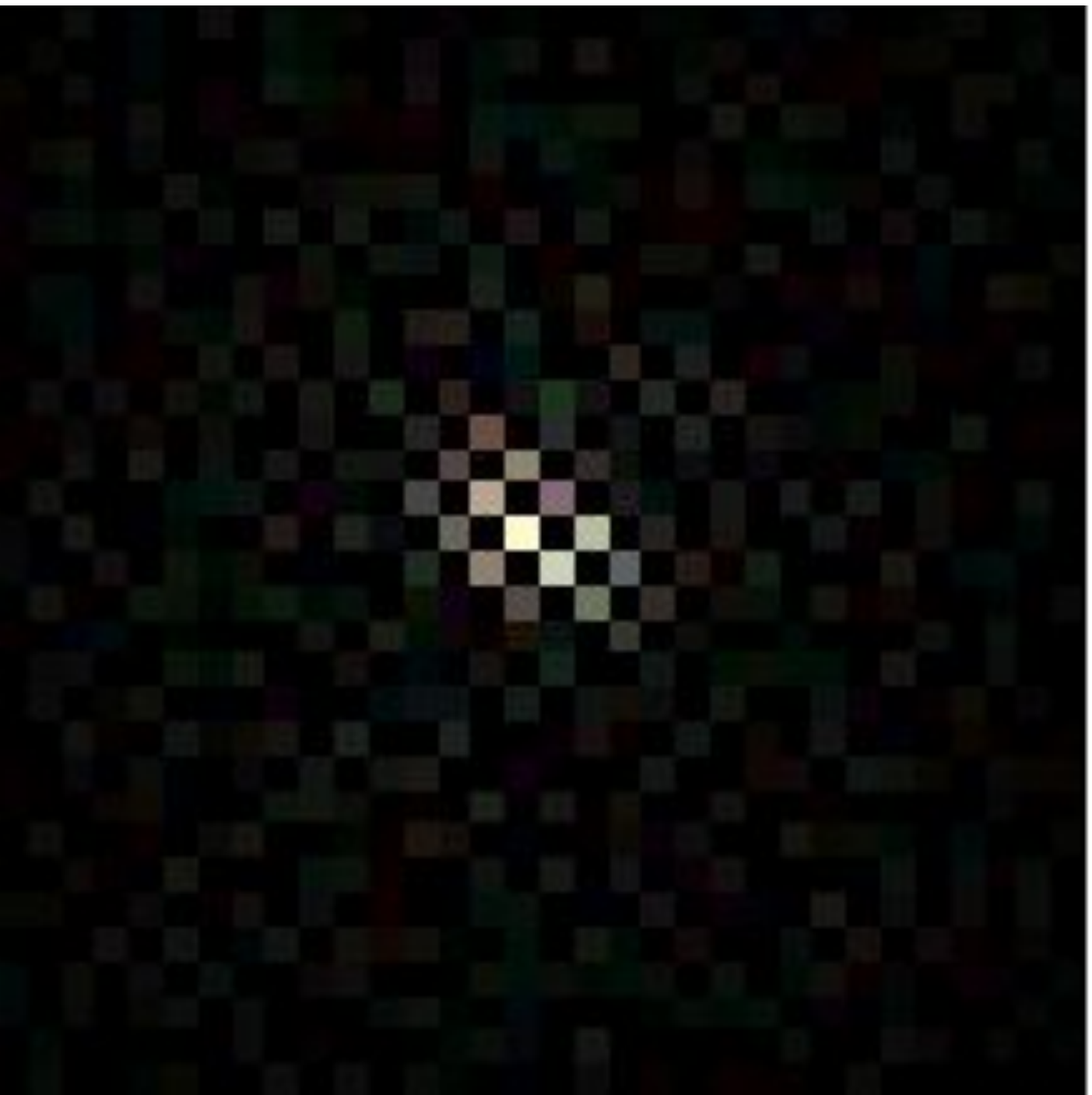}} & \parbox[l]{1em}{\includegraphics[width=3em]{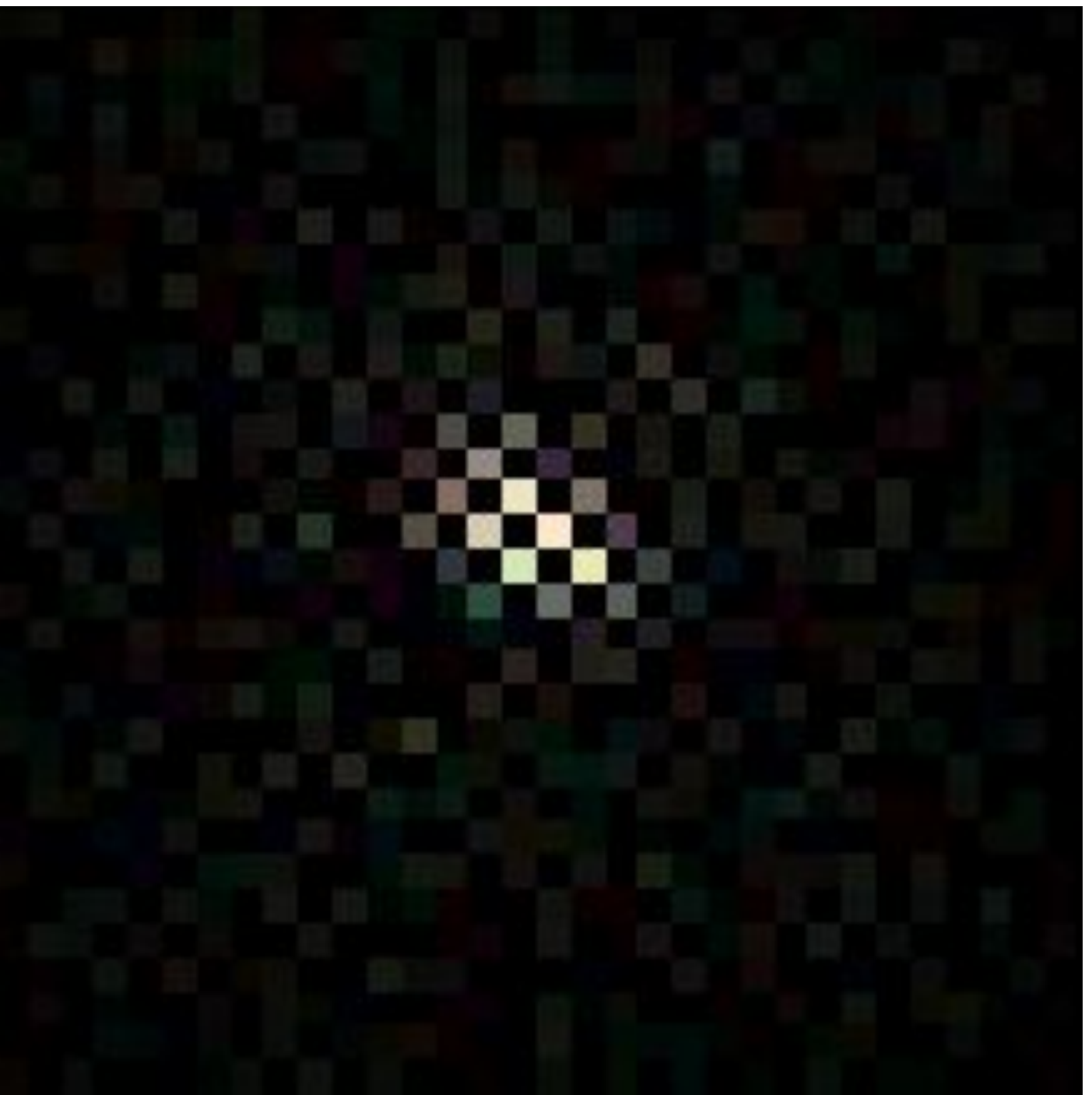}} \\
        \parbox[l]{1em}{\includegraphics[width=3em]{images/d_overlay.pdf}} & \parbox[l]{1em}{\includegraphics[width=3em]{images/d_overlay_sample.pdf}} & Bird & \parbox[l]{1em}{\includegraphics[width=3em]{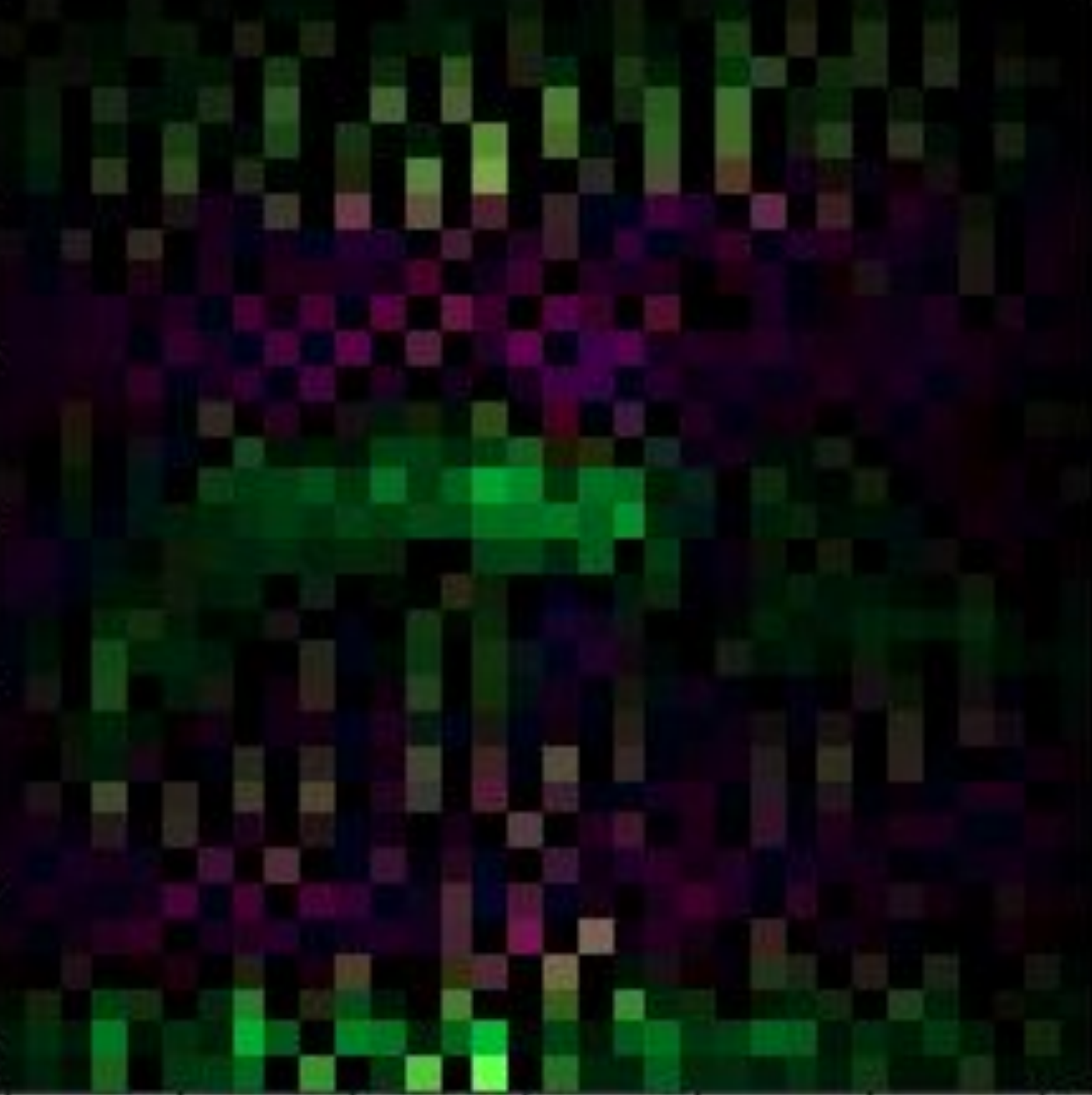}} & \parbox[l]{1em}{\includegraphics[width=3em]{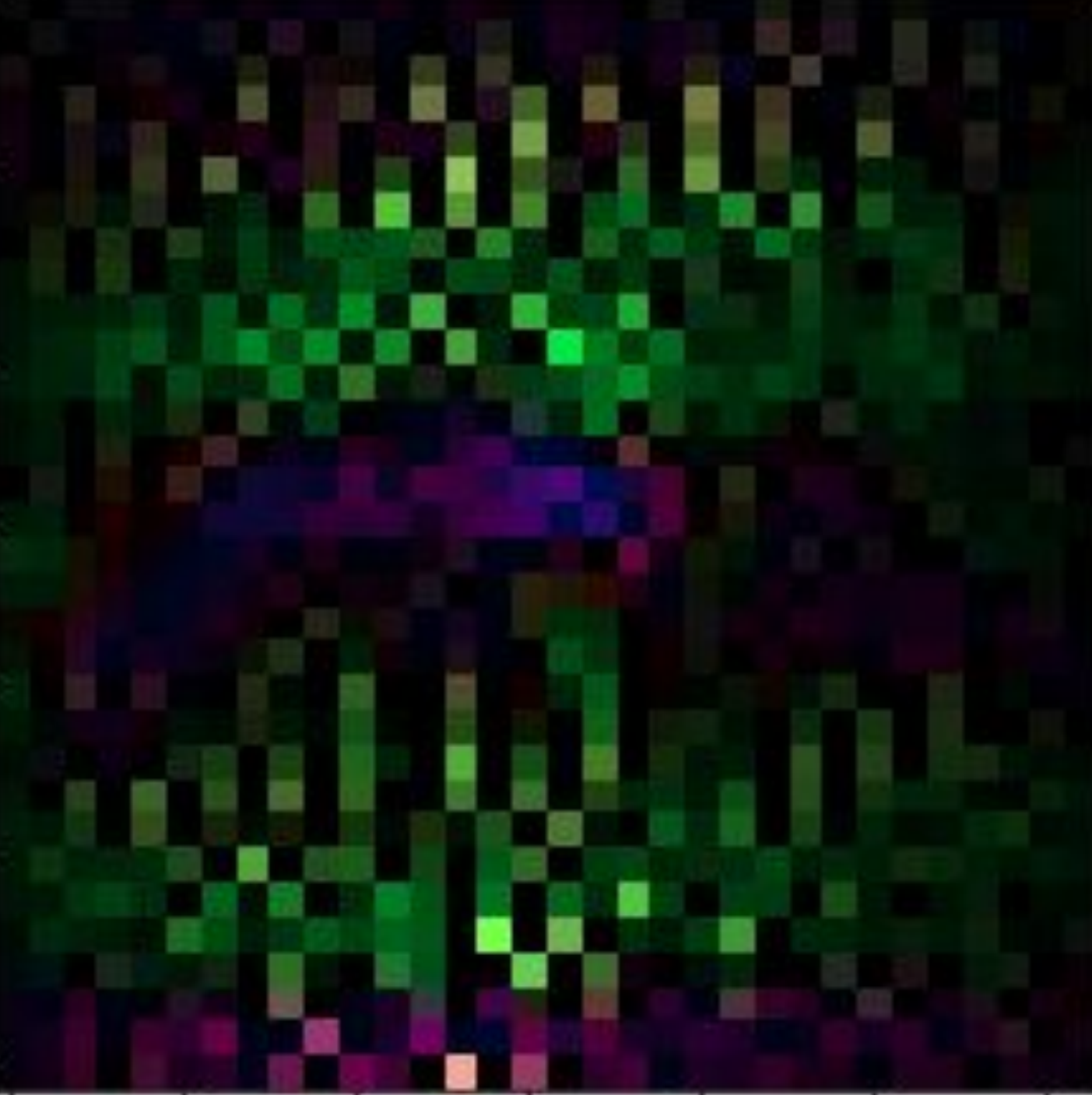}} && \parbox[l]{1em}{\includegraphics[width=3em]{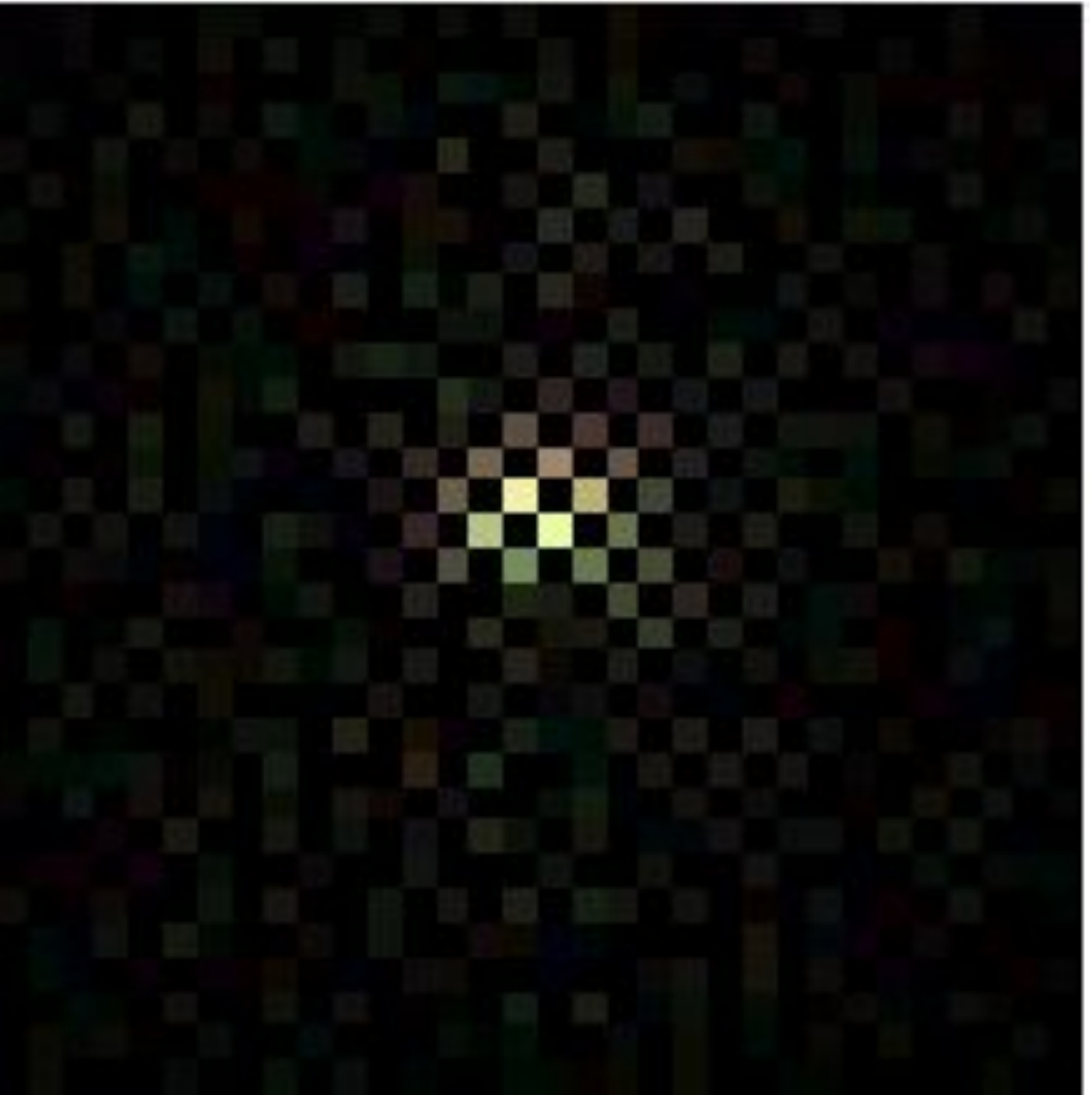}} & \parbox[l]{1em}{\includegraphics[width=3em]{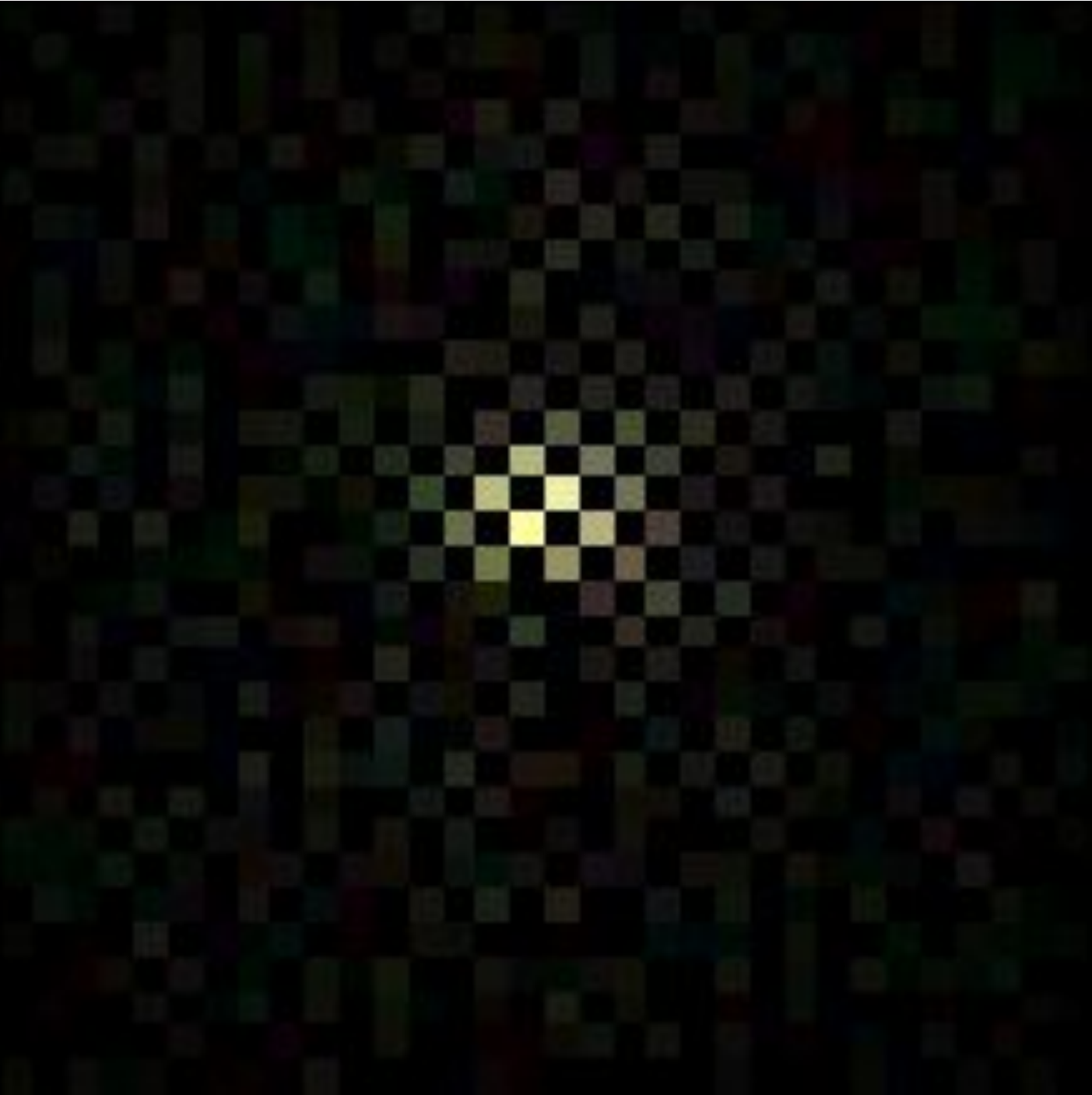}} \\
        \parbox[l]{1em}{\includegraphics[width=3em]{images/e_overlay.pdf}} & \parbox[l]{1em}{\includegraphics[width=3em]{images/e_overlay_sample.pdf}} & Deer & \parbox[l]{1em}{\includegraphics[width=3em]{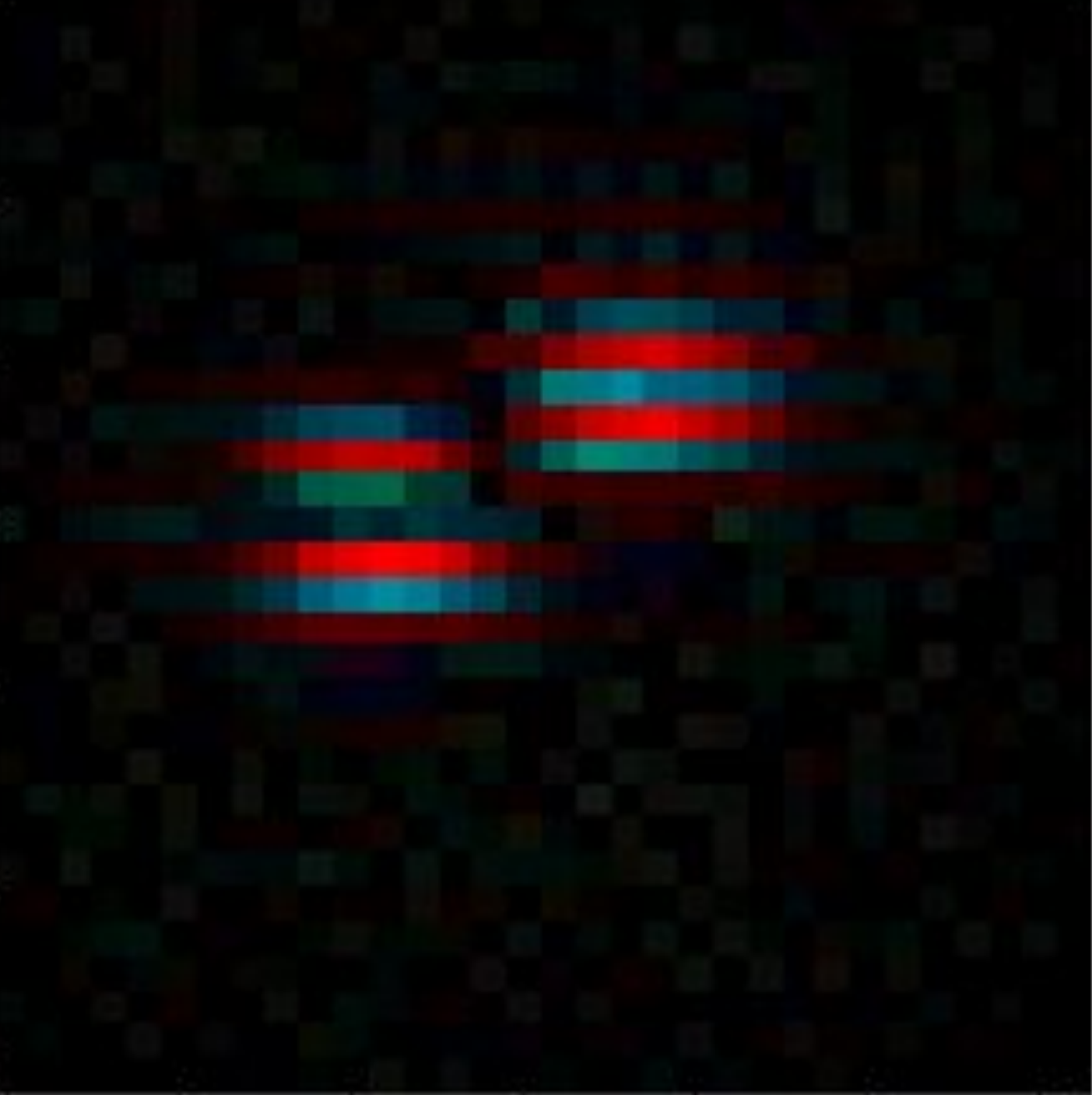}} & \parbox[l]{1em}{\includegraphics[width=3em]{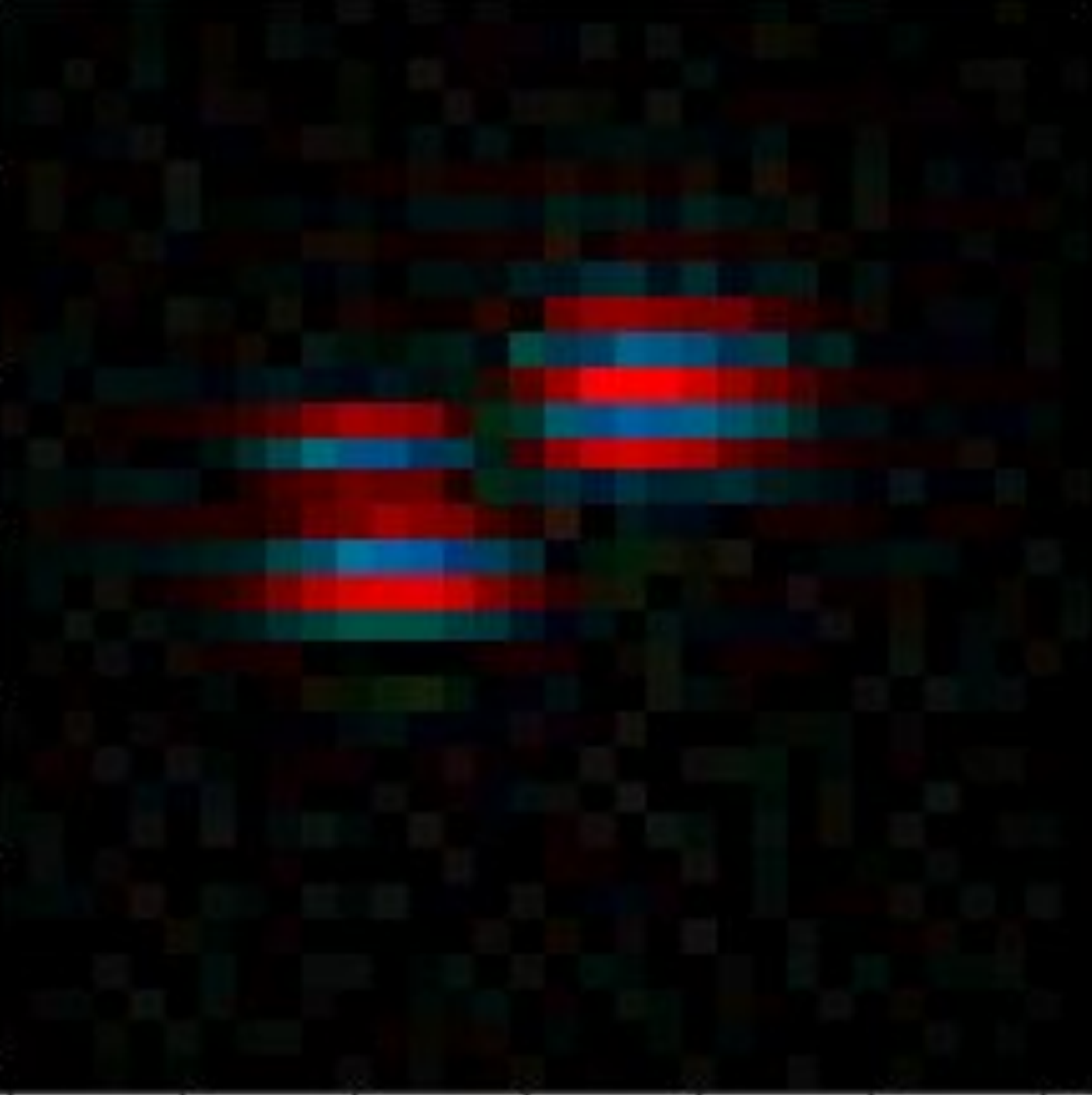}} && \parbox[l]{1em}{\includegraphics[width=3em]{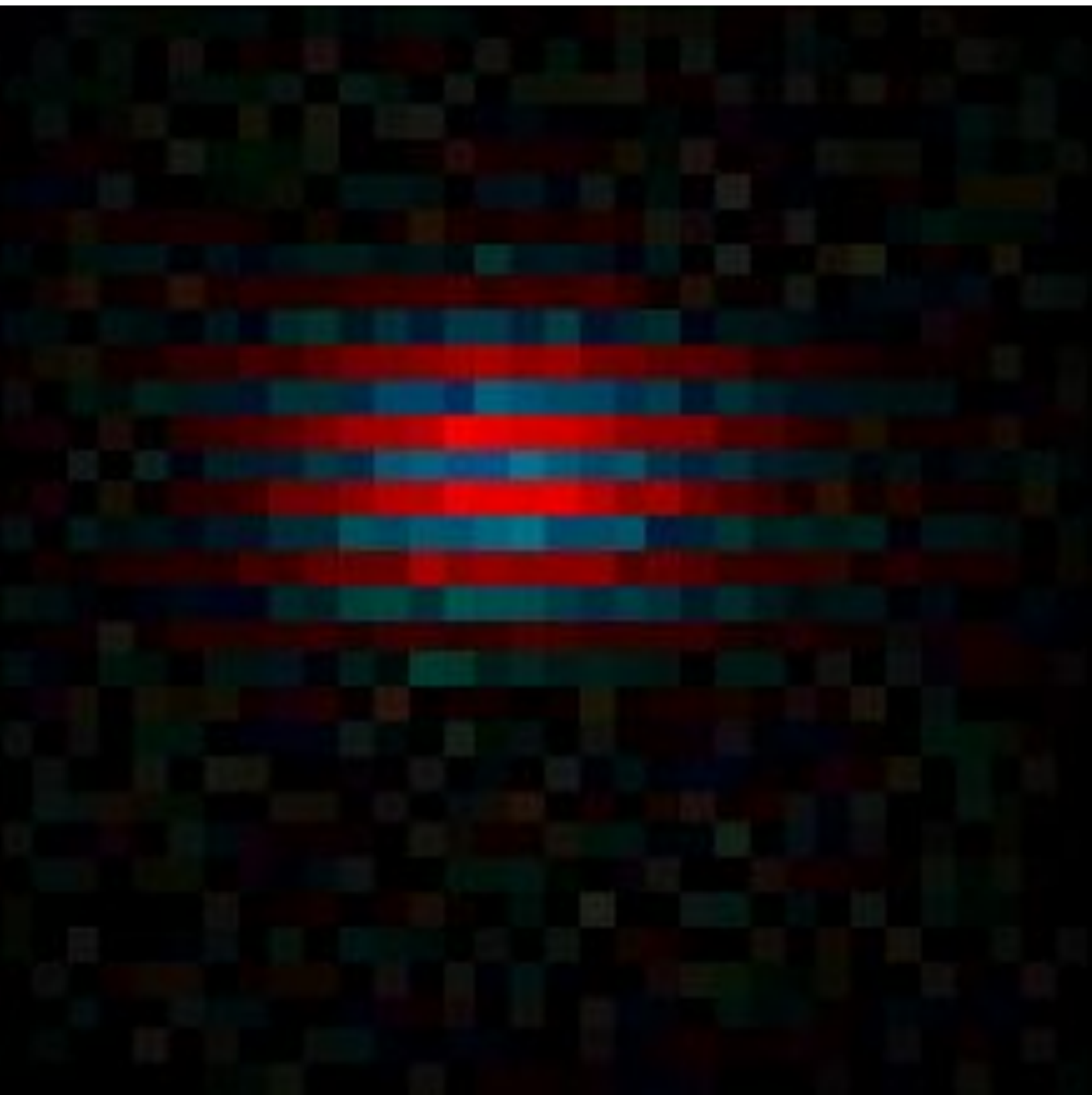}} & \parbox[l]{1em}{\includegraphics[width=3em]{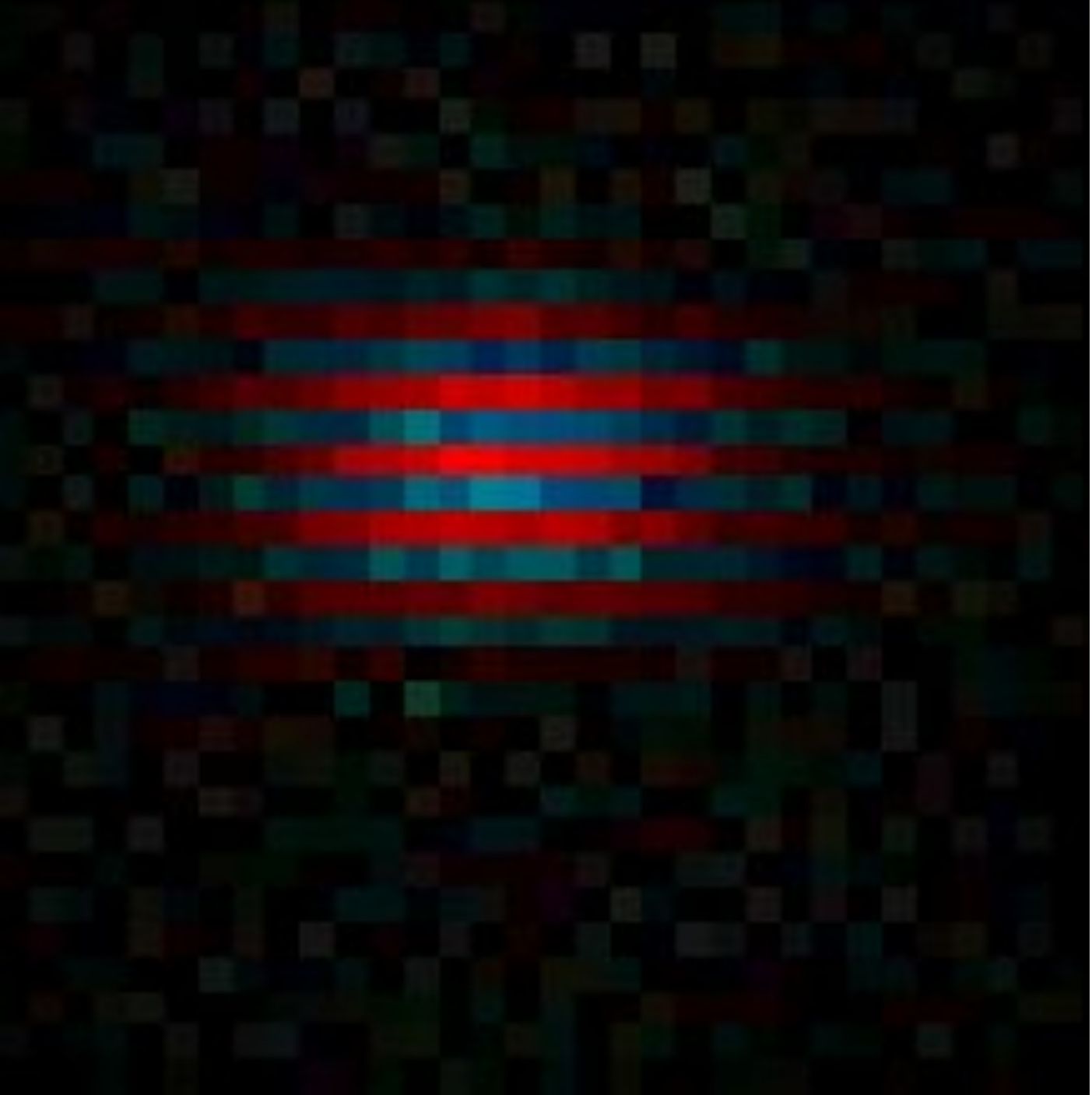}} \\
        \parbox[l]{1em}{\includegraphics[width=3em]{images/f_overlay.pdf}} & \parbox[l]{1em}{\includegraphics[width=3em]{images/f_overlay_sample.pdf}} & Bird & \parbox[l]{1em}{\includegraphics[width=3em]{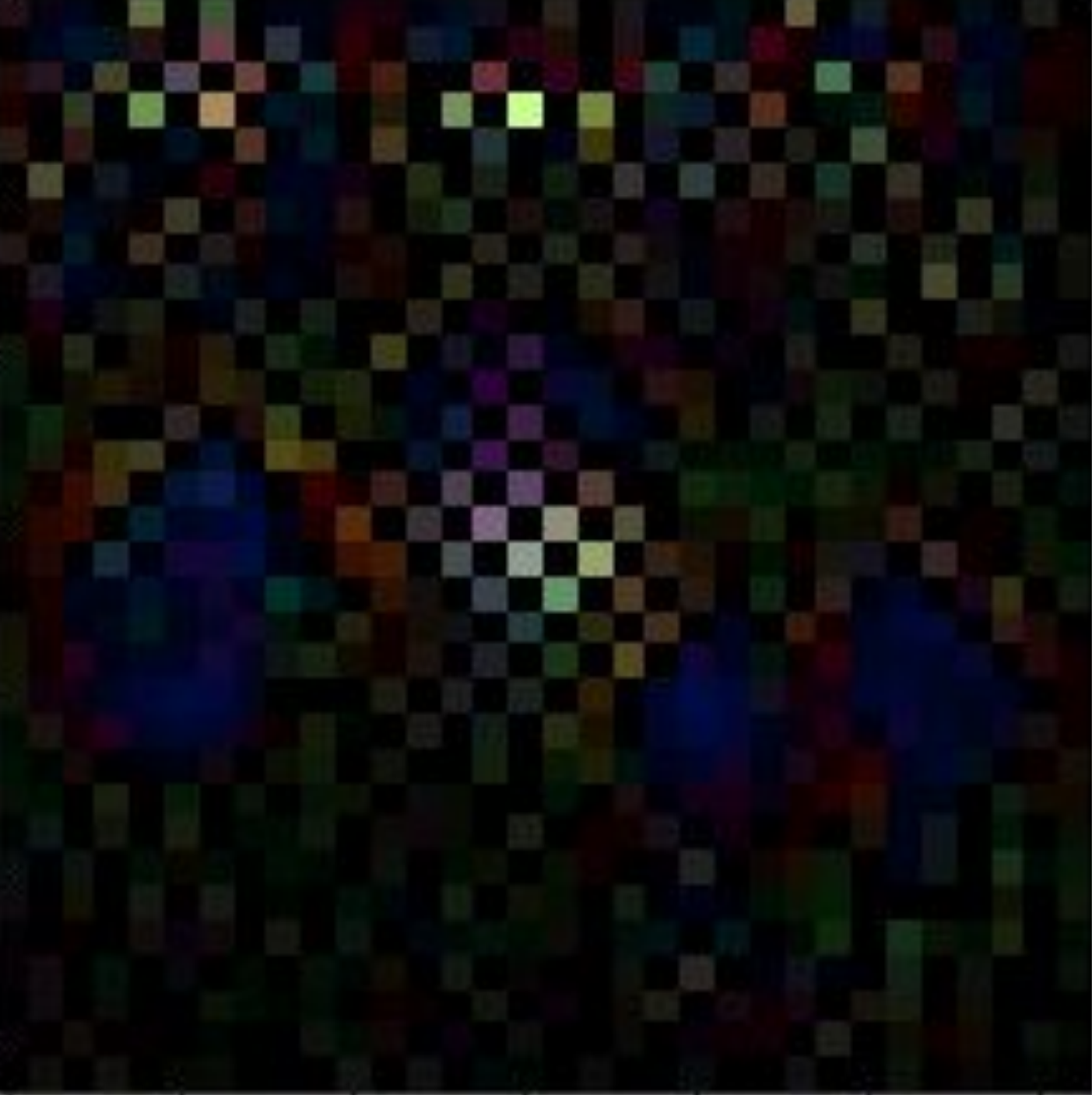}} & \parbox[l]{1em}{\includegraphics[width=3em]{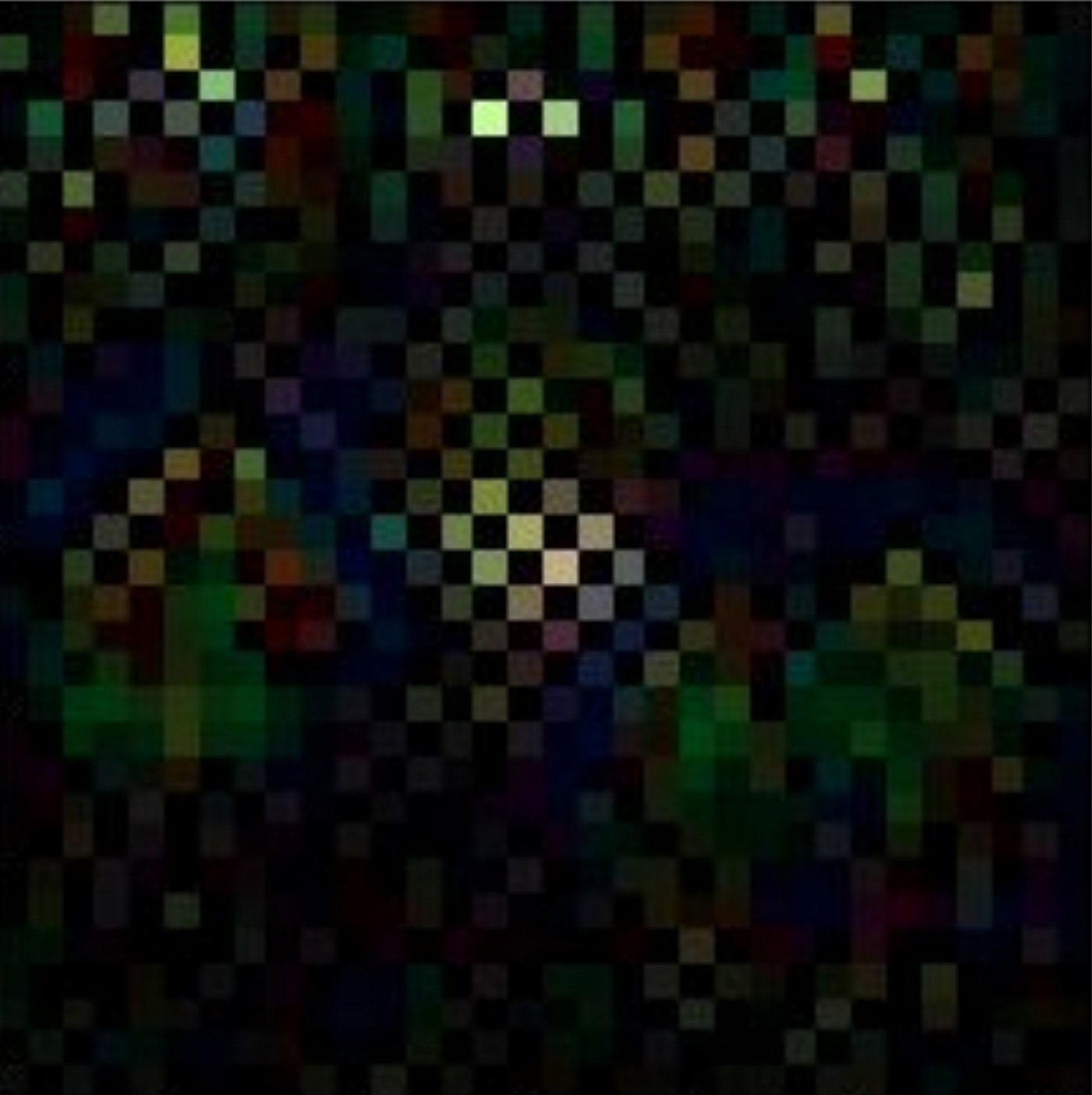}} && \parbox[l]{1em}{\includegraphics[width=3em]{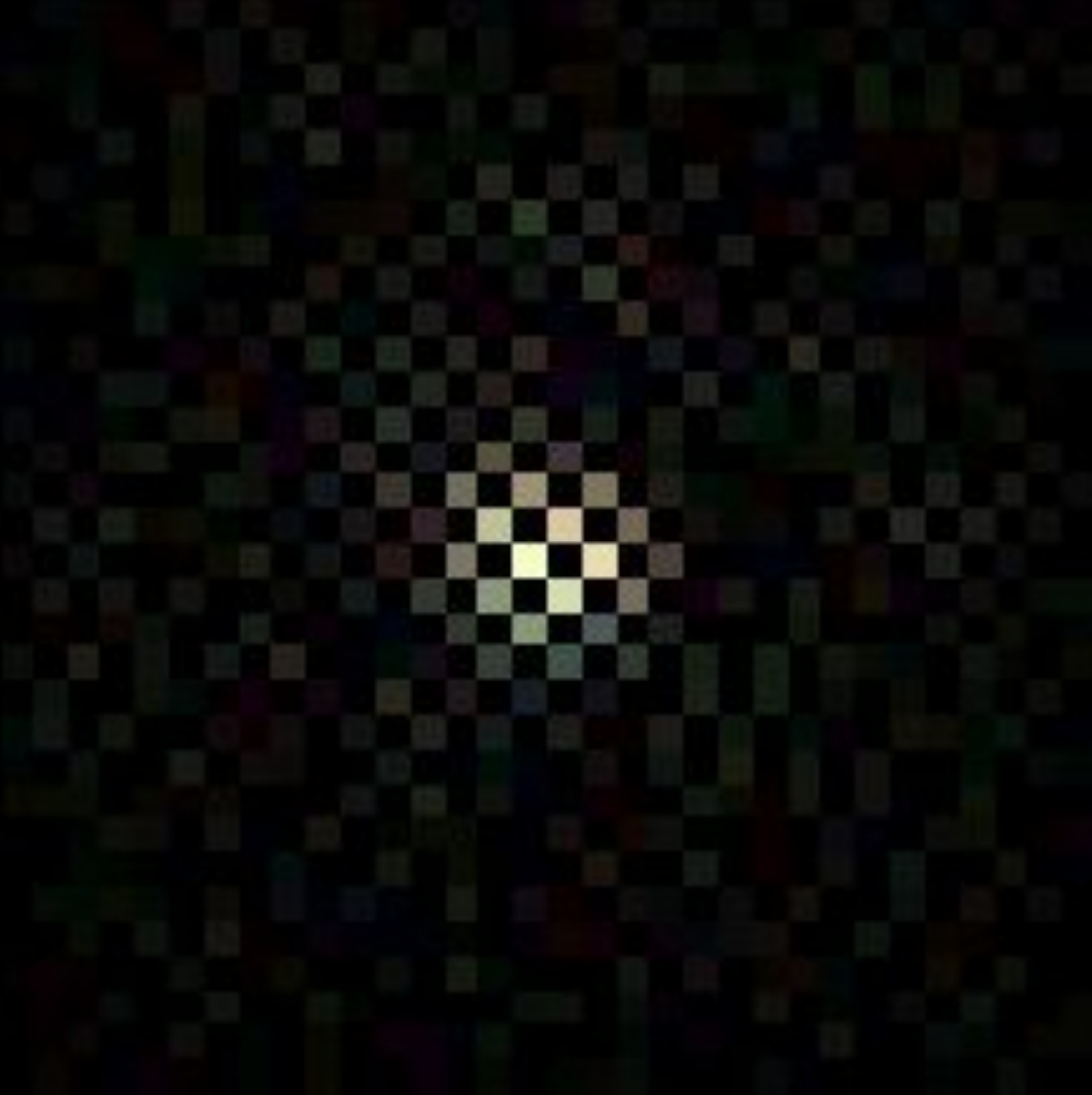}} & \parbox[l]{1em}{\includegraphics[width=3em]{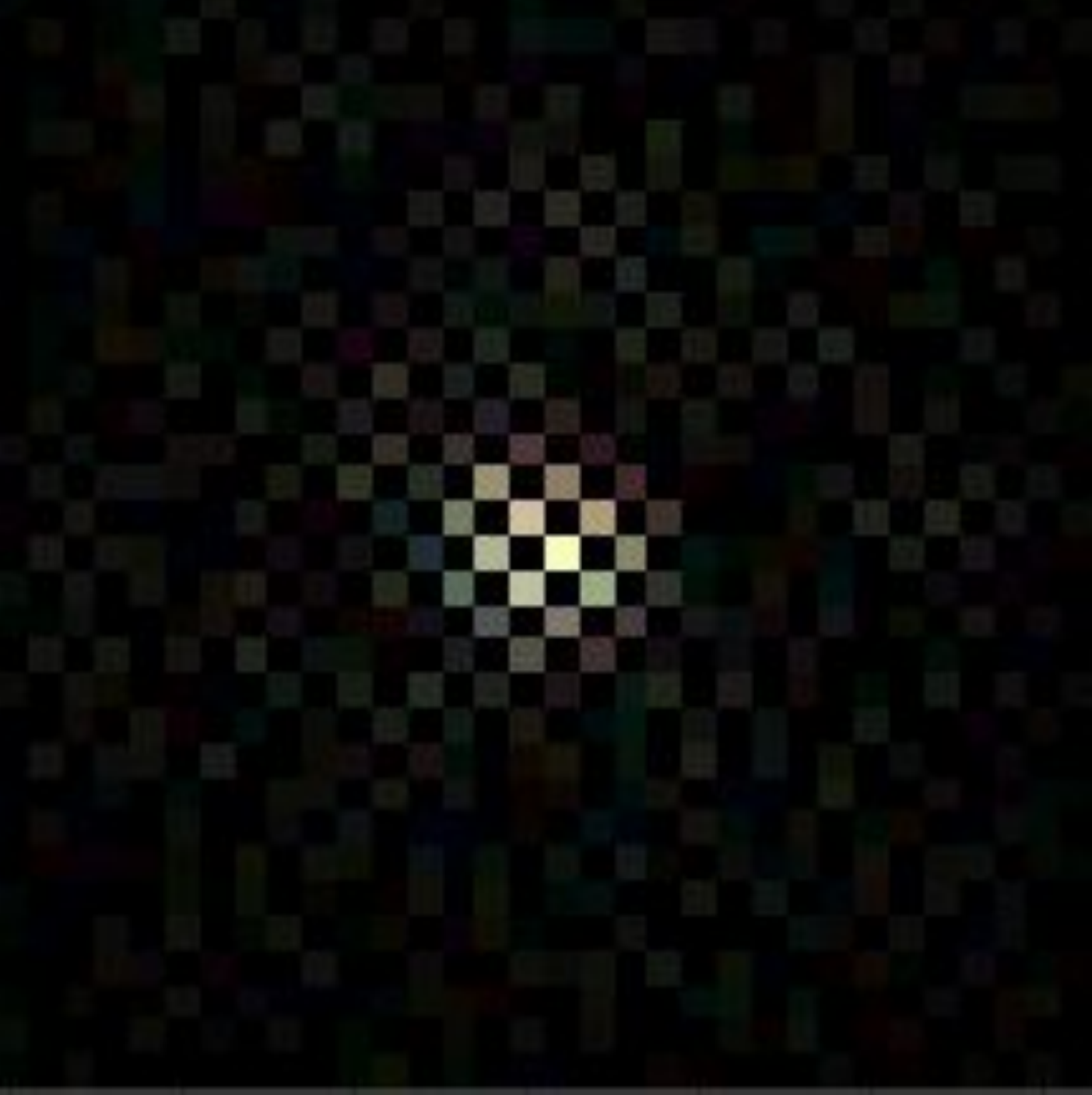}} \\
        \parbox[l]{1em}{\includegraphics[width=3em]{images/g_overlay.pdf}} & \parbox[l]{1em}{\includegraphics[width=3em]{images/g_overlay_sample.pdf}} & Horse & \parbox[l]{1em}{\includegraphics[width=3em]{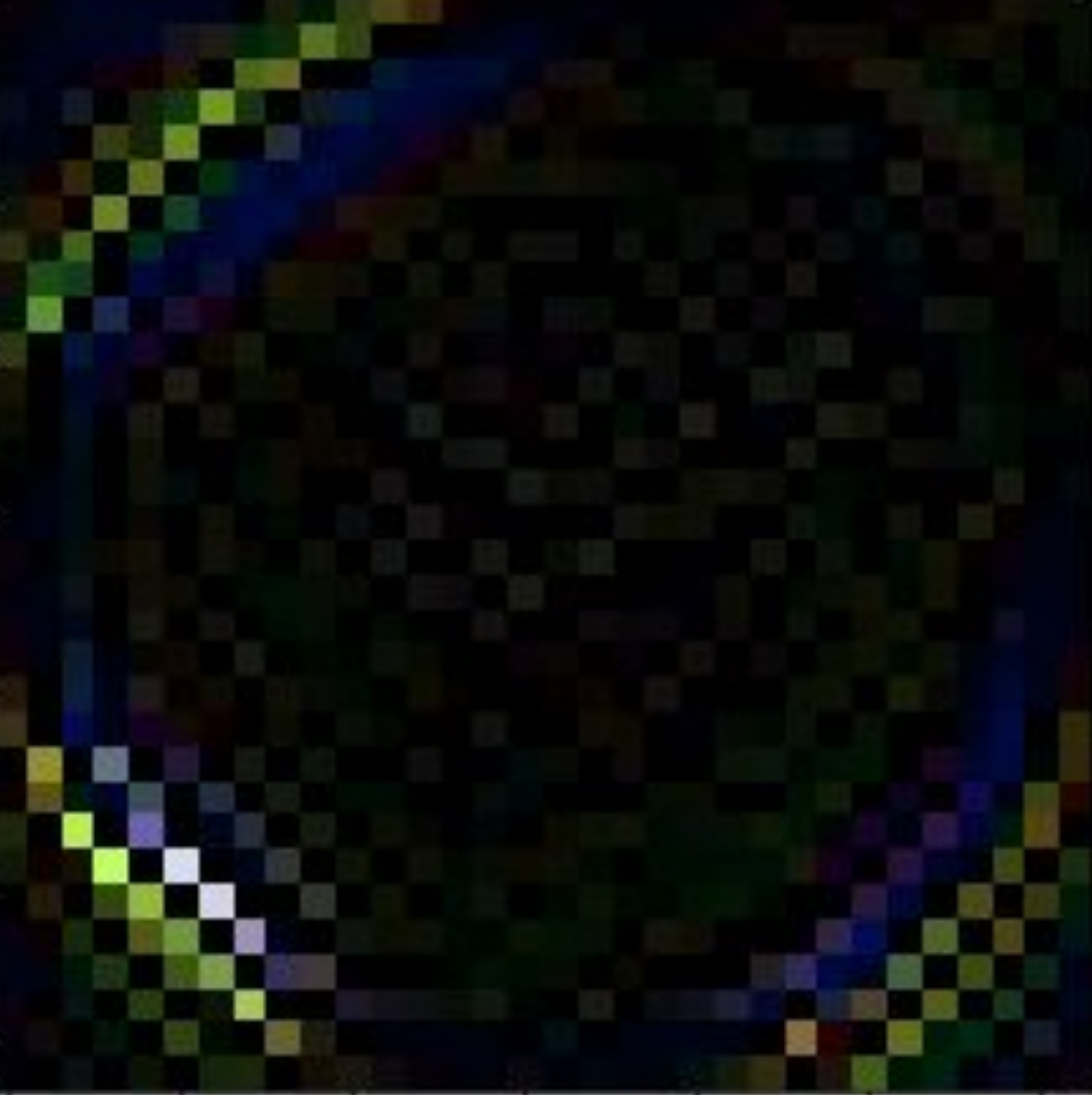}} & \parbox[l]{1em}{\includegraphics[width=3em]{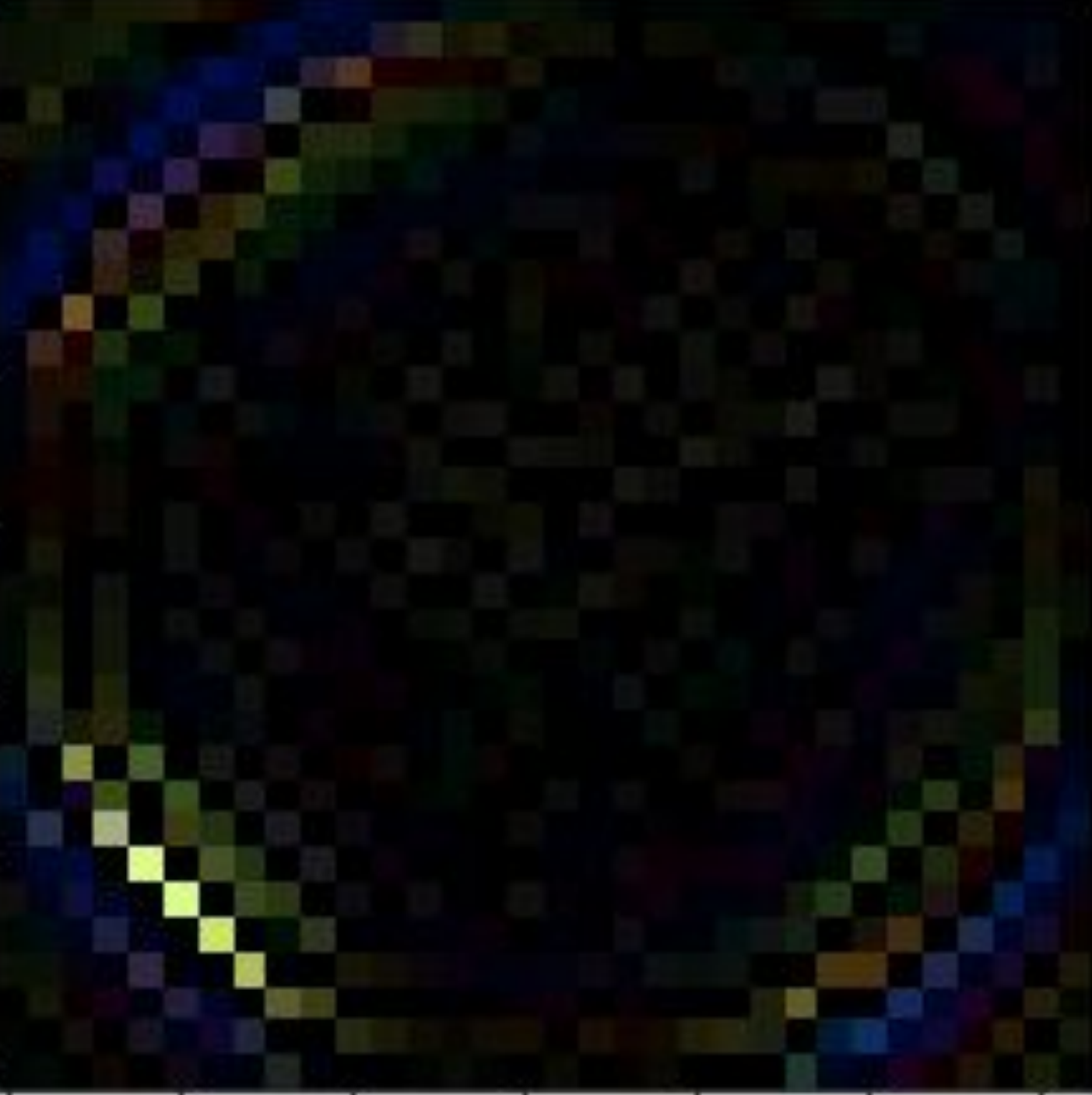}} && \parbox[l]{1em}{\includegraphics[width=3em]{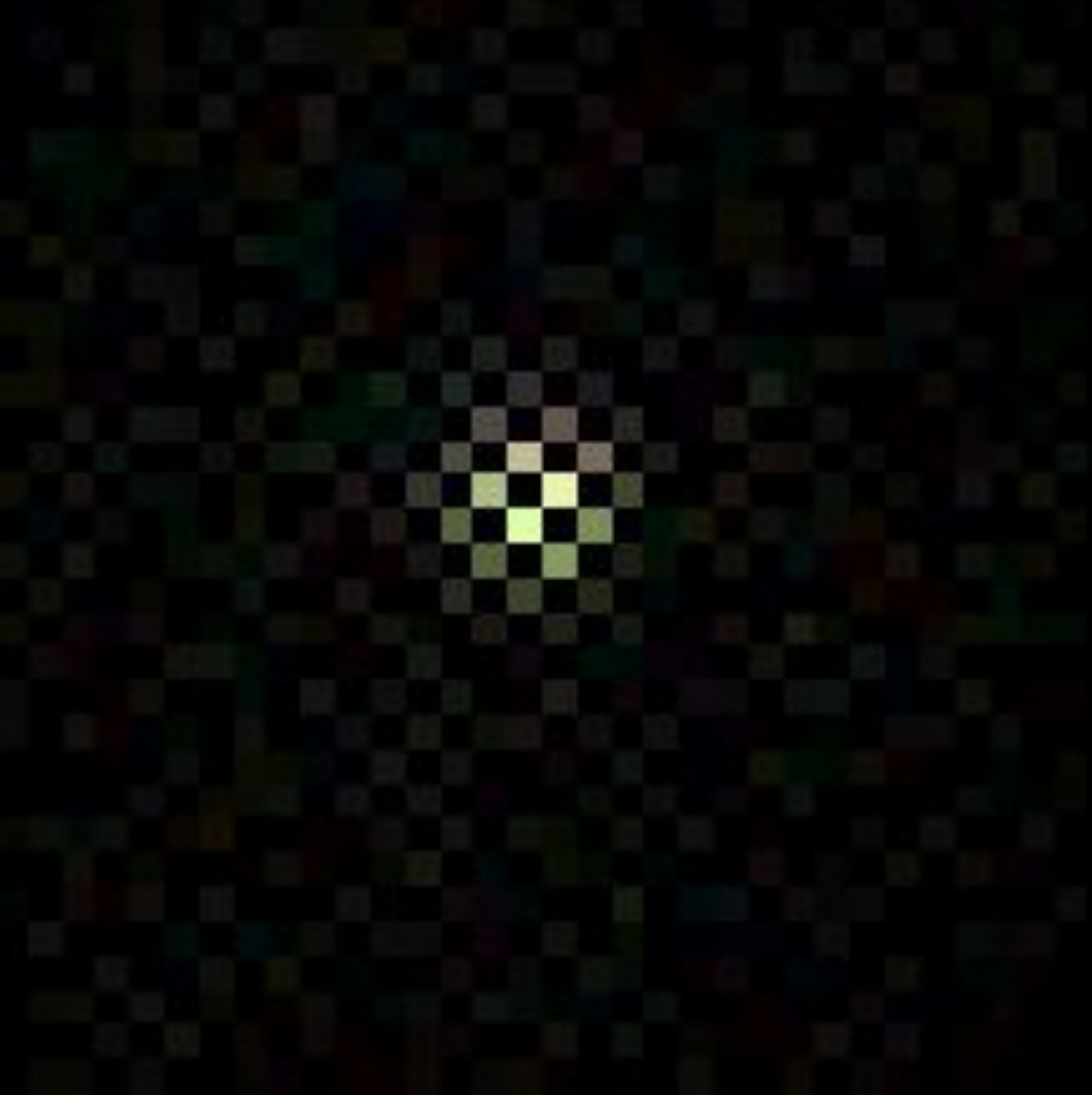}} & \parbox[l]{1em}{\includegraphics[width=3em]{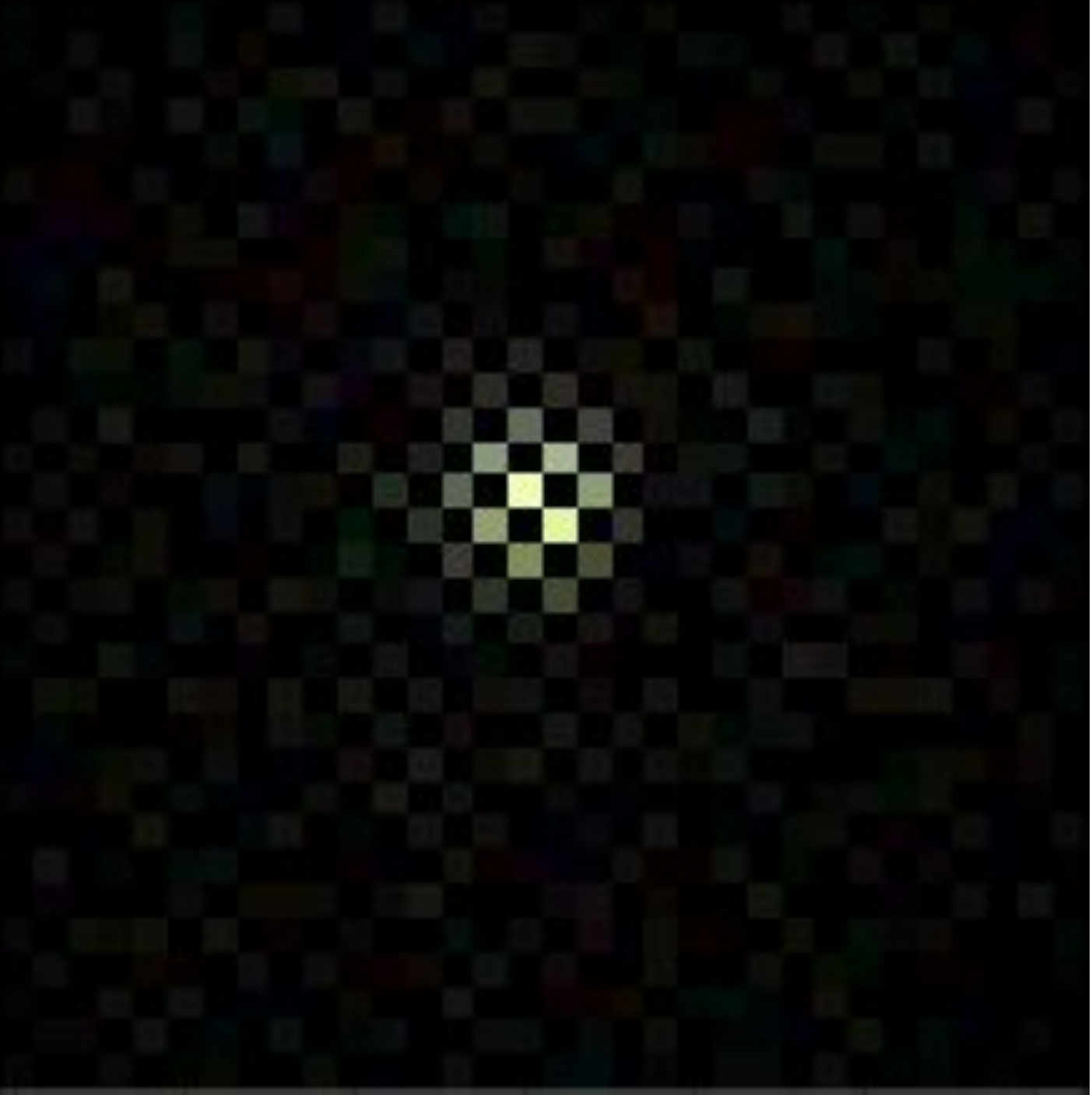}} \\
        \parbox[l]{1em}{\includegraphics[width=3em]{images/h_overlay.pdf}} & \parbox[l]{1em}{\includegraphics[width=3em]{images/h_overlay_sample.pdf}} & Cat & \parbox[l]{1em}{\includegraphics[width=3em]{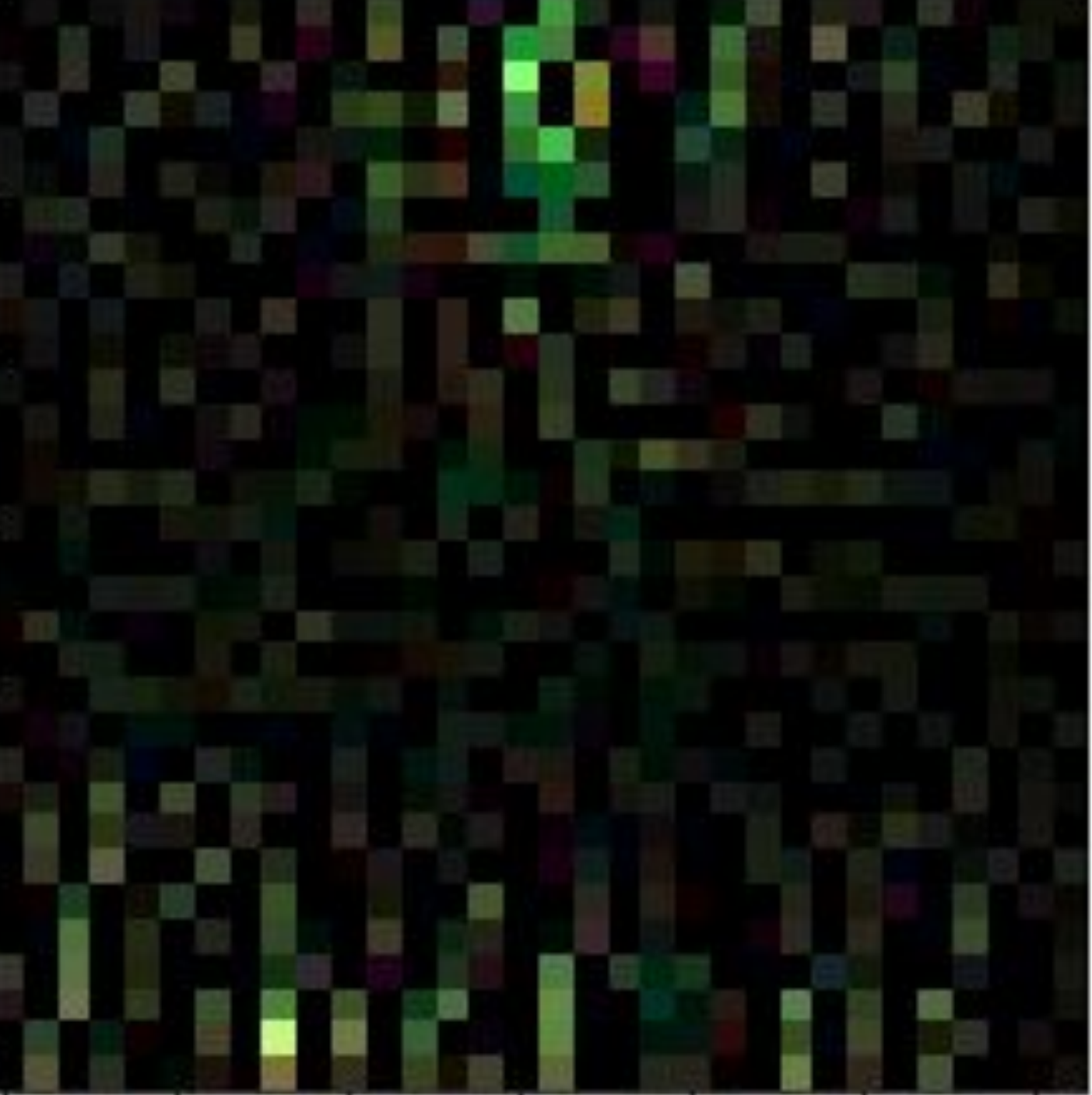}} & \parbox[l]{1em}{\includegraphics[width=3em]{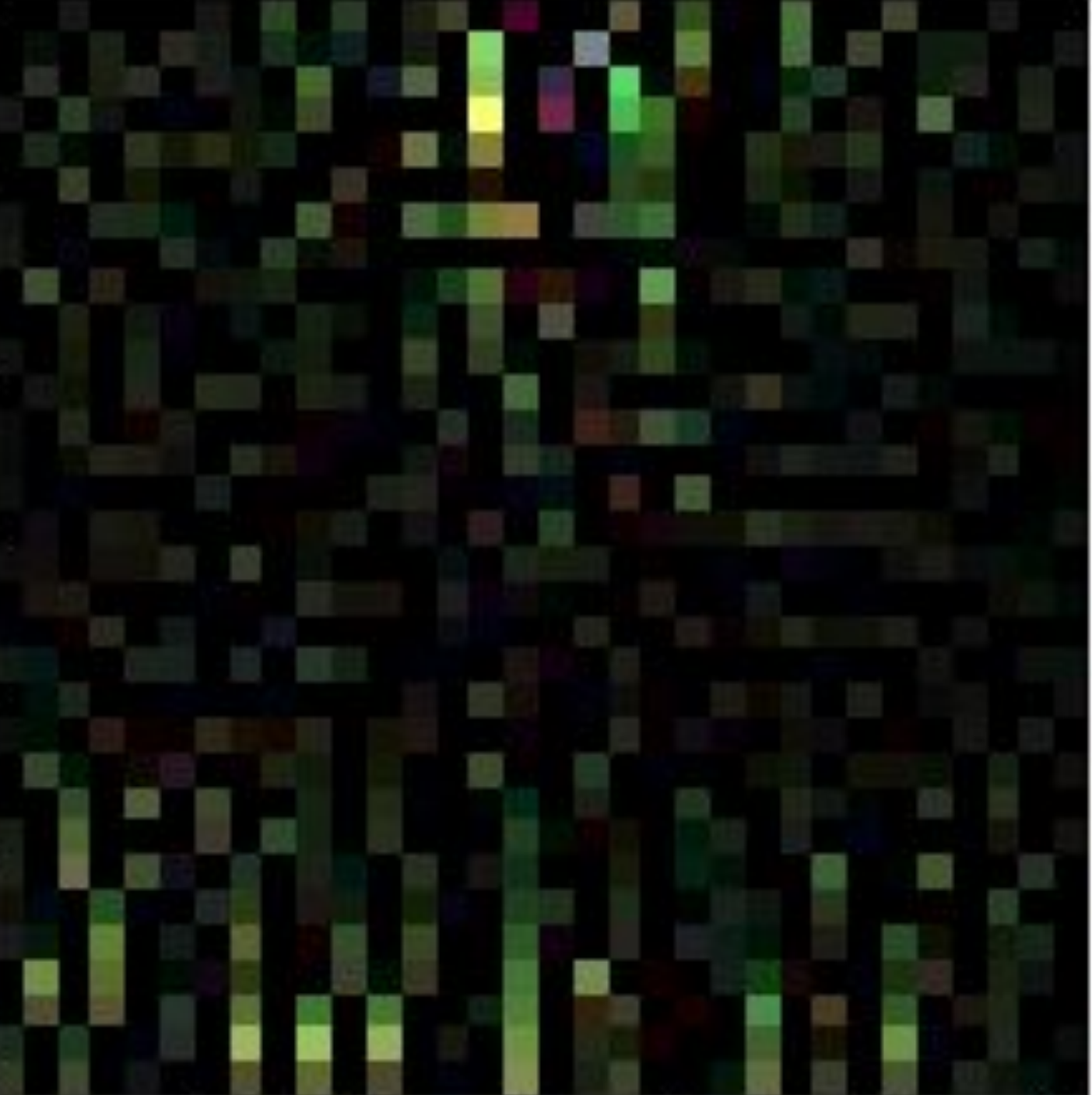}} && \parbox[l]{1em}{\includegraphics[width=3em]{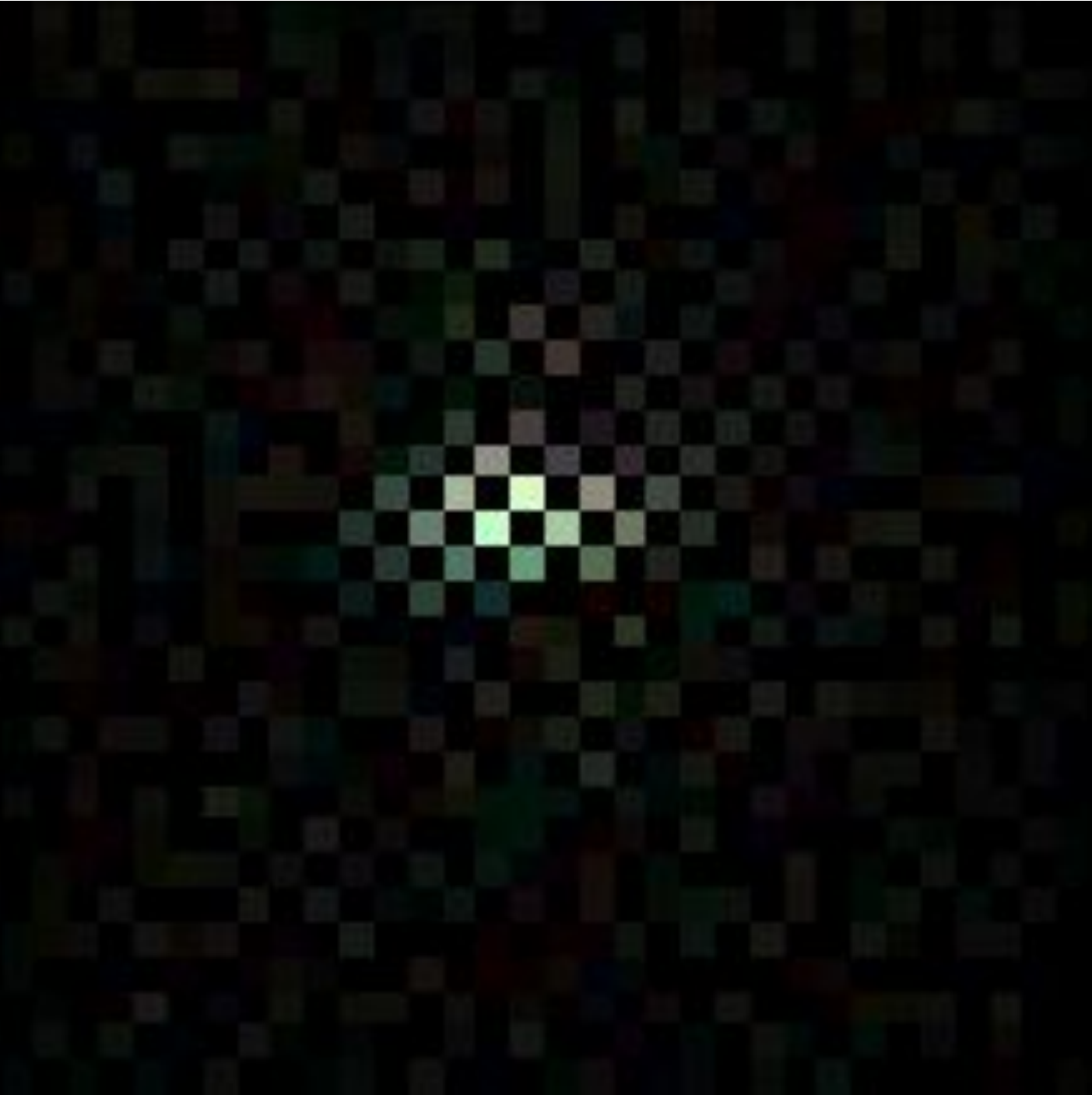}} & \parbox[l]{1em}{\includegraphics[width=3em]{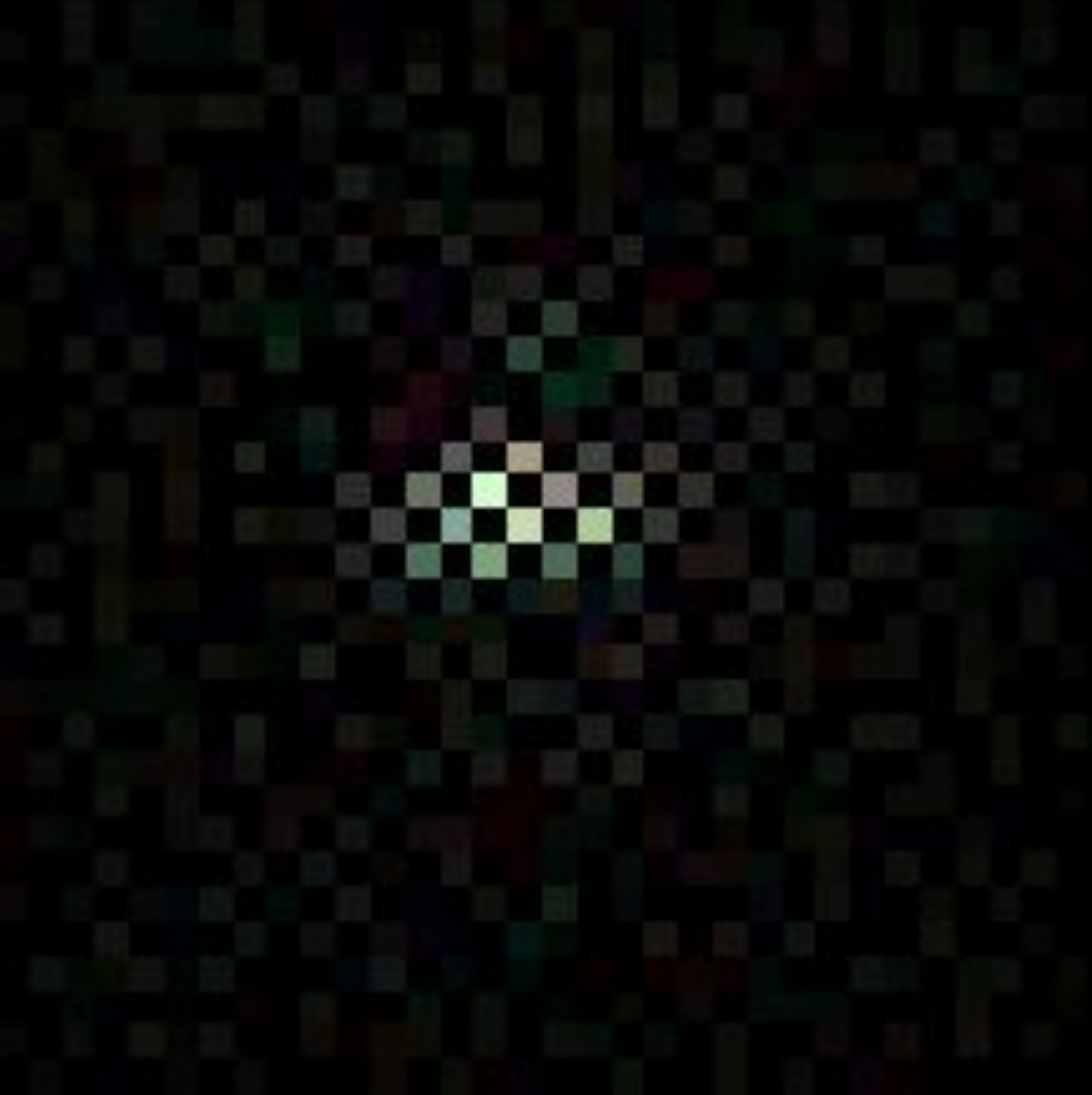}} \\
        \parbox[l]{1em}{\includegraphics[width=3em]{images/i_overlay.pdf}} & \parbox[l]{1em}{\includegraphics[width=3em]{images/i_overlay_sample.pdf}} & Dog & \parbox[l]{1em}{\includegraphics[width=3em]{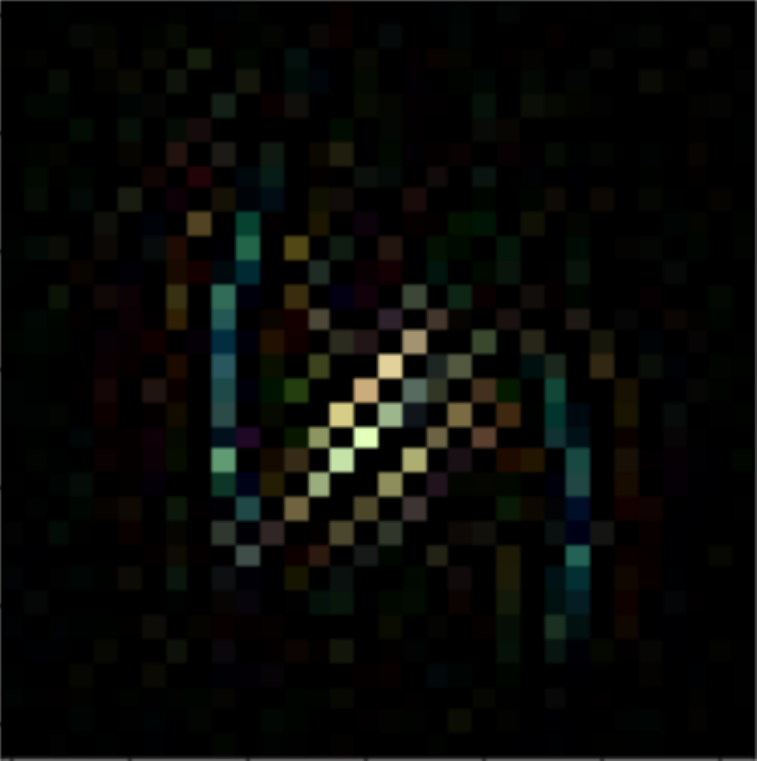}} & \parbox[l]{1em}{\includegraphics[width=3em]{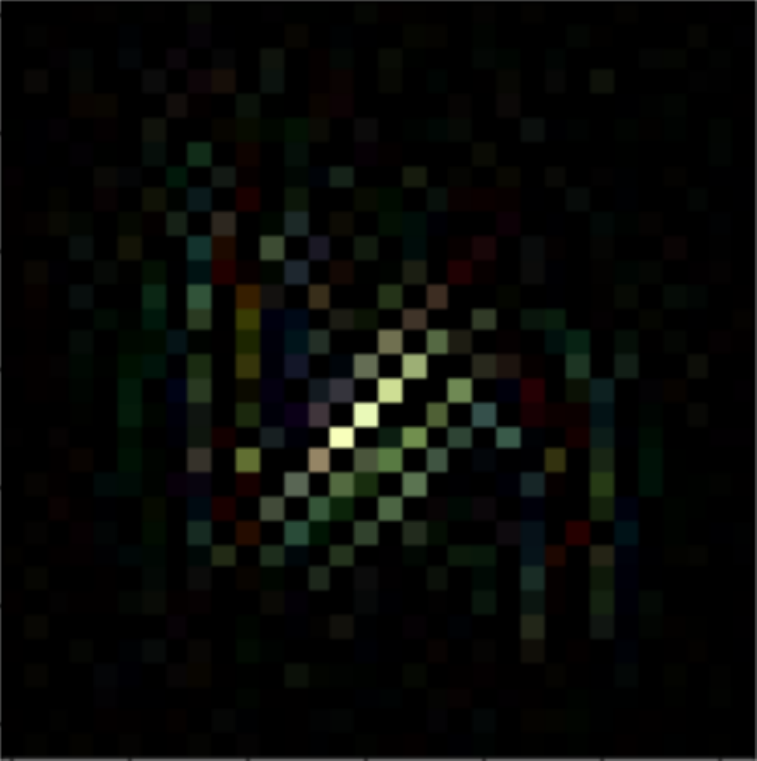}} && \parbox[l]{1em}{\includegraphics[width=3em]{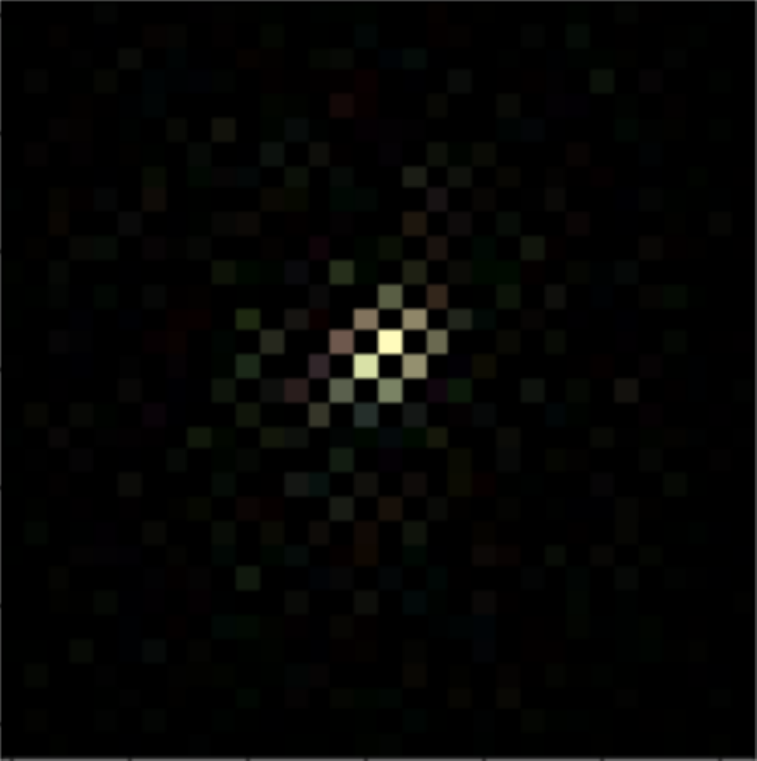}} & \parbox[l]{1em}{\includegraphics[width=3em]{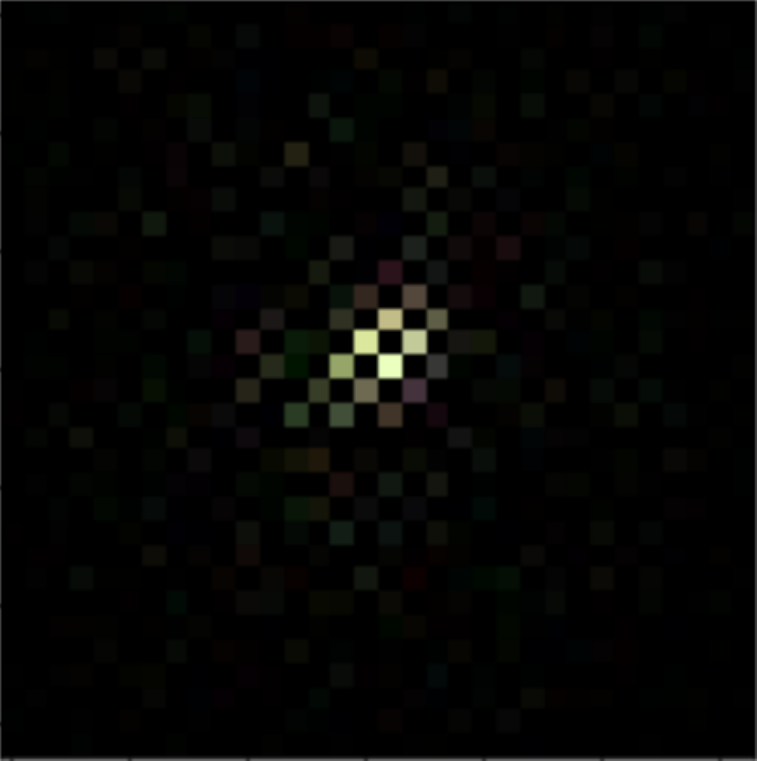}} \\
         \hline
        \end{tabular}
\label{tab:overlay_poison_v_all}
\end{table*}

\begin{table*}[tp]
    \centering
    \caption{Appendix: (a) Dot-poisoned sample, (b) the first right vector of input gradients for all target class images which include clean and poisoned images. (c) The first right vector of input gradients for only clean target class images.}
        \begin{tabular}{ lcccccc }
         \hline
         Poison & Target & \multicolumn{2}{c}{1st V of all target images} && \multicolumn{2}{c}{1st V of clean target images} \\
         \cline{3-4}
         \cline{6-7}
         ~ & ~ & + & - && + & - \\
         \hline
        \parbox[l]{1em}{\includegraphics[width=3em]{images/a_dot_sample.pdf}} & Dog & \parbox[l]{1em}{\includegraphics[width=3em]{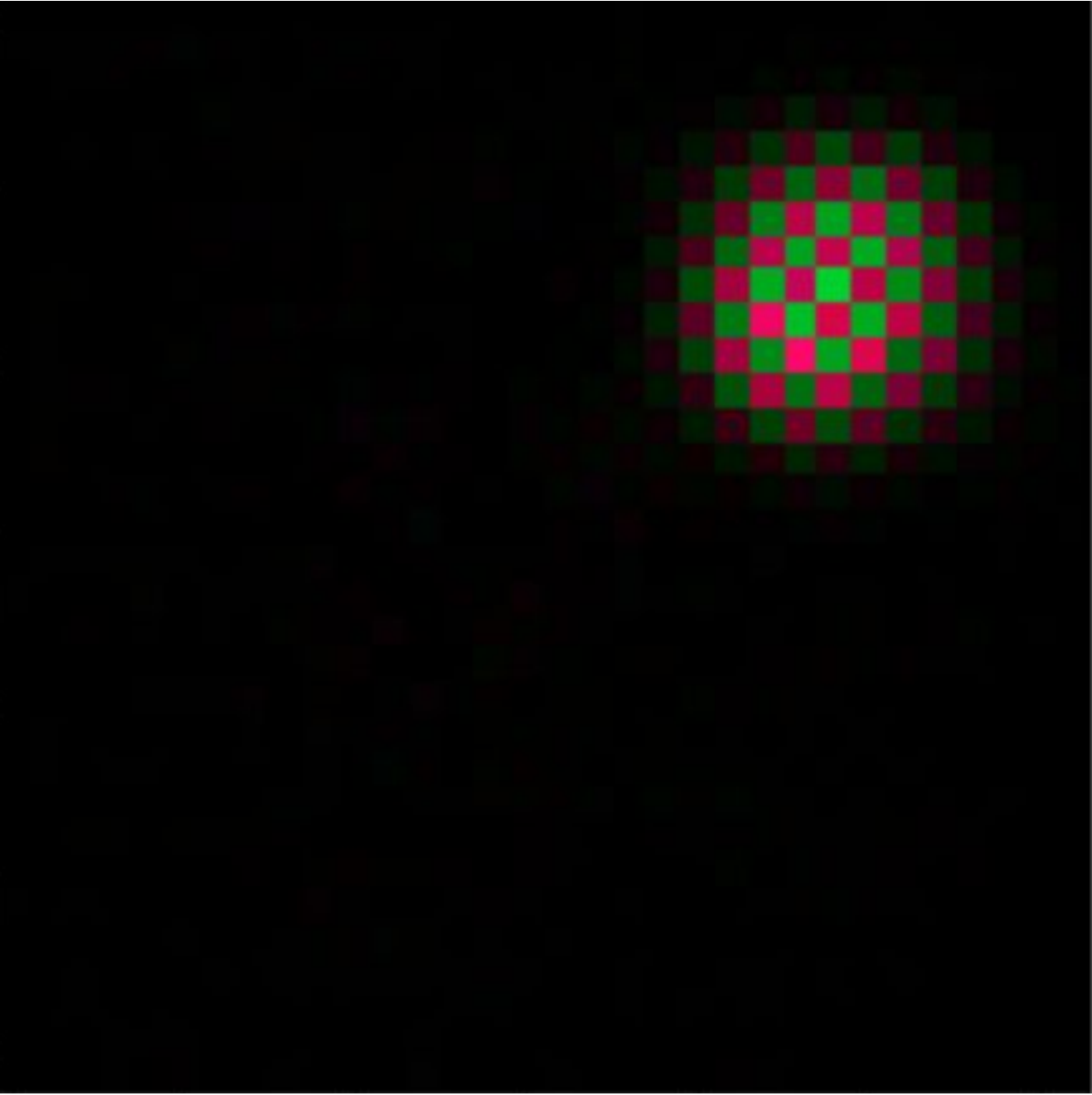}} & \parbox[l]{1em}{\includegraphics[width=3em]{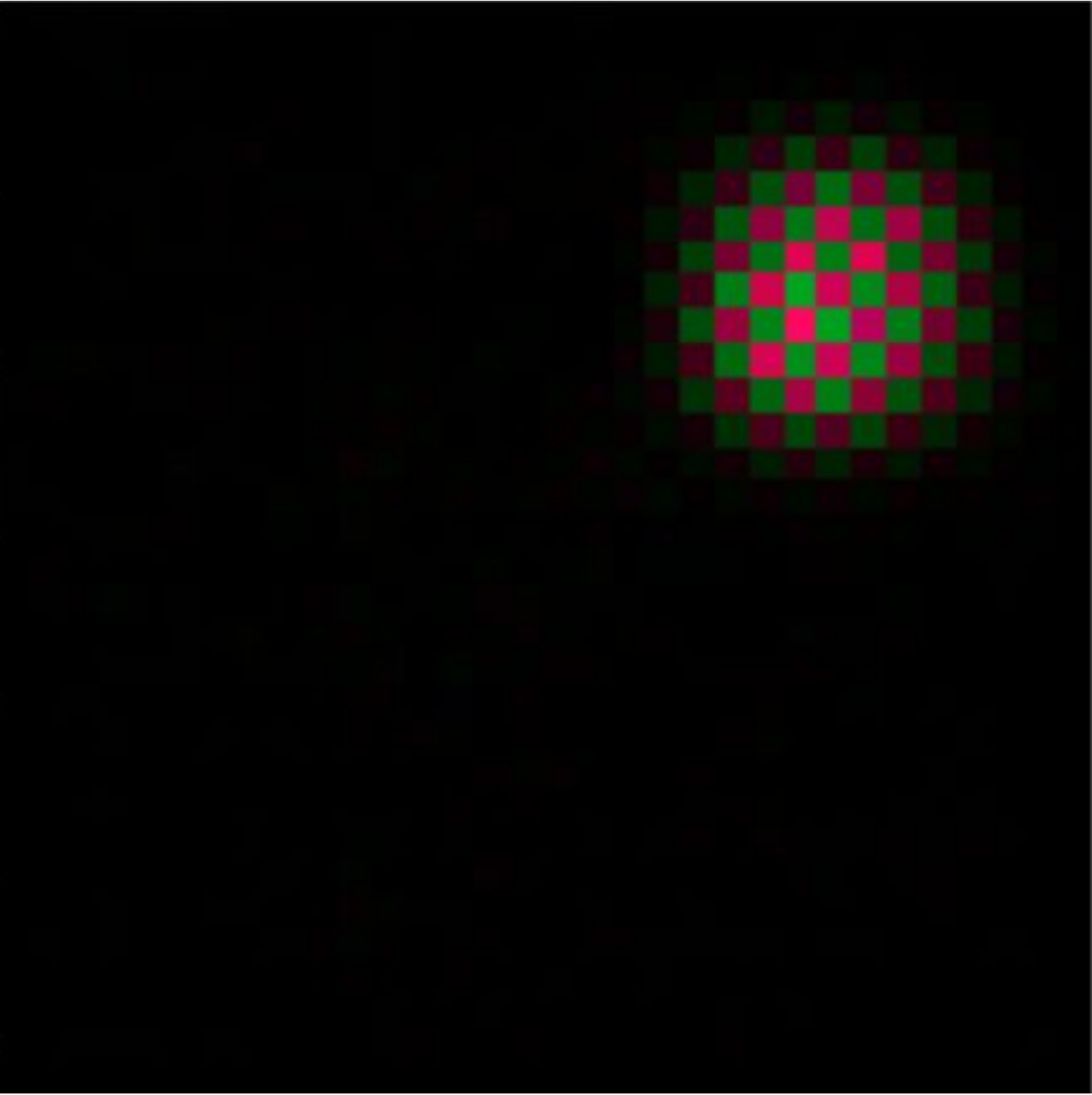}} && \parbox[l]{1em}{\includegraphics[width=3em]{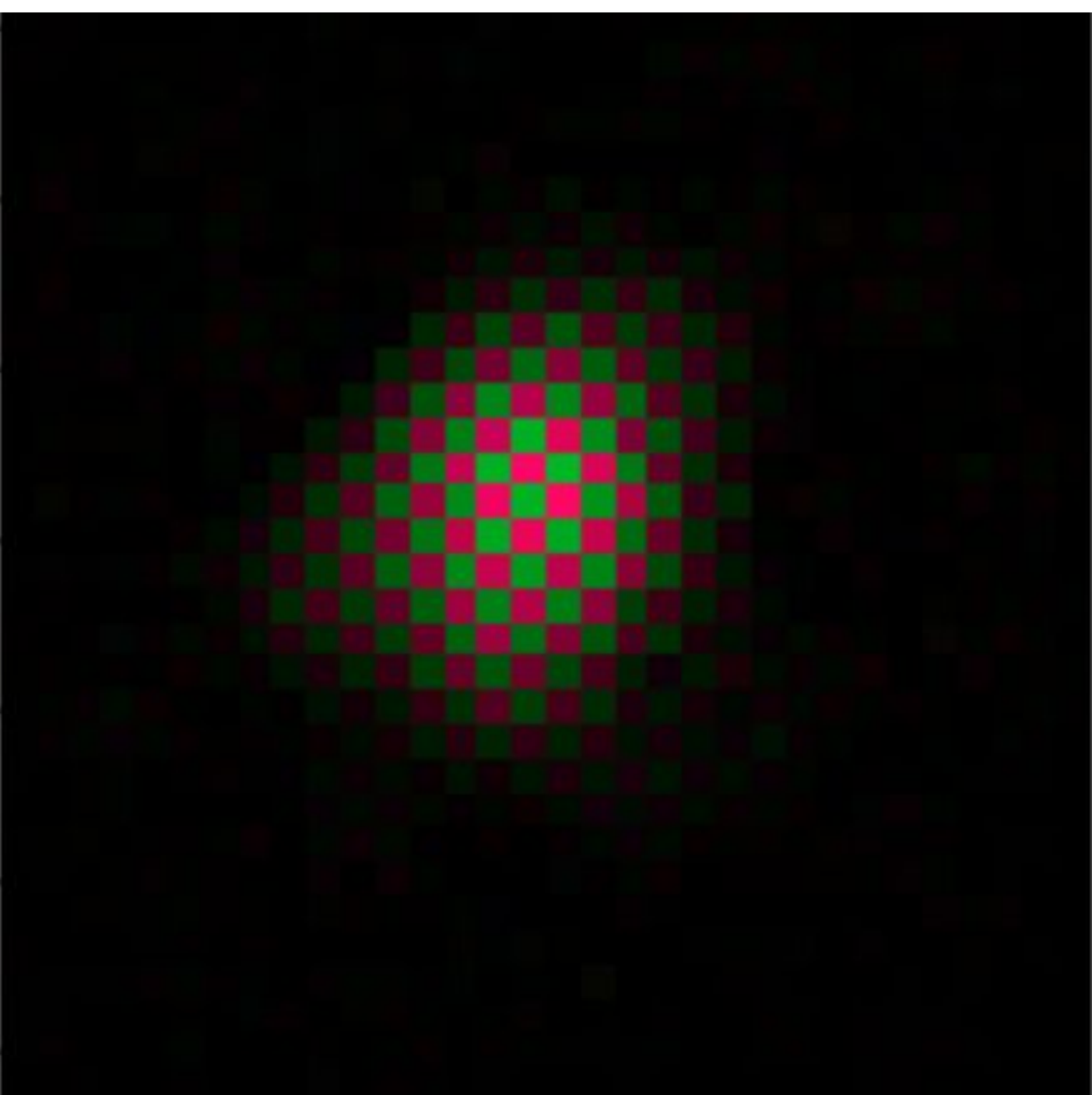}} & \parbox[l]{1em}{\includegraphics[width=3em]{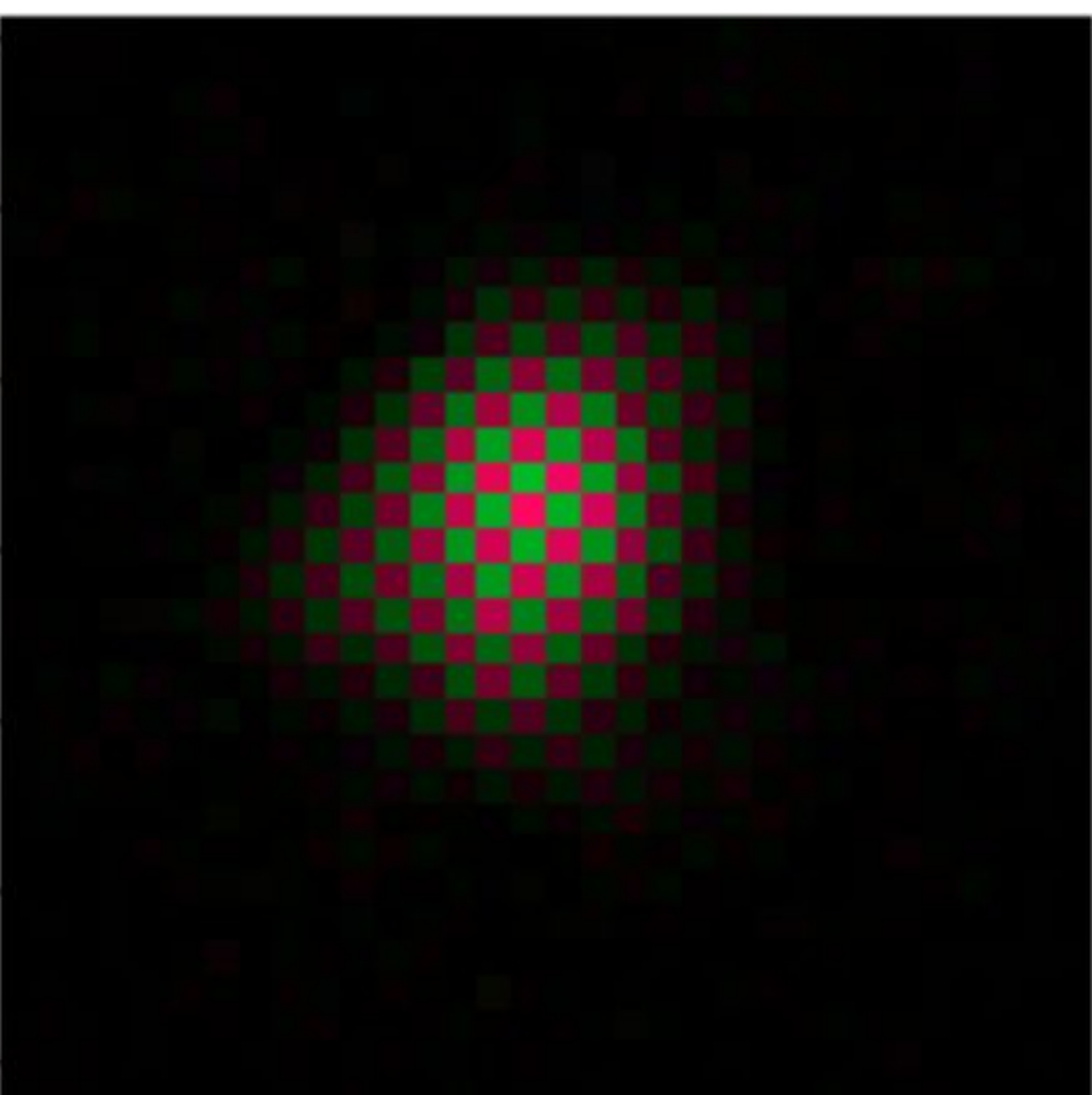}} \\
        \parbox[l]{1em}{\includegraphics[width=3em]{images/b_dot_sample.pdf}} & Frog & \parbox[l]{1em}{\includegraphics[width=3em]{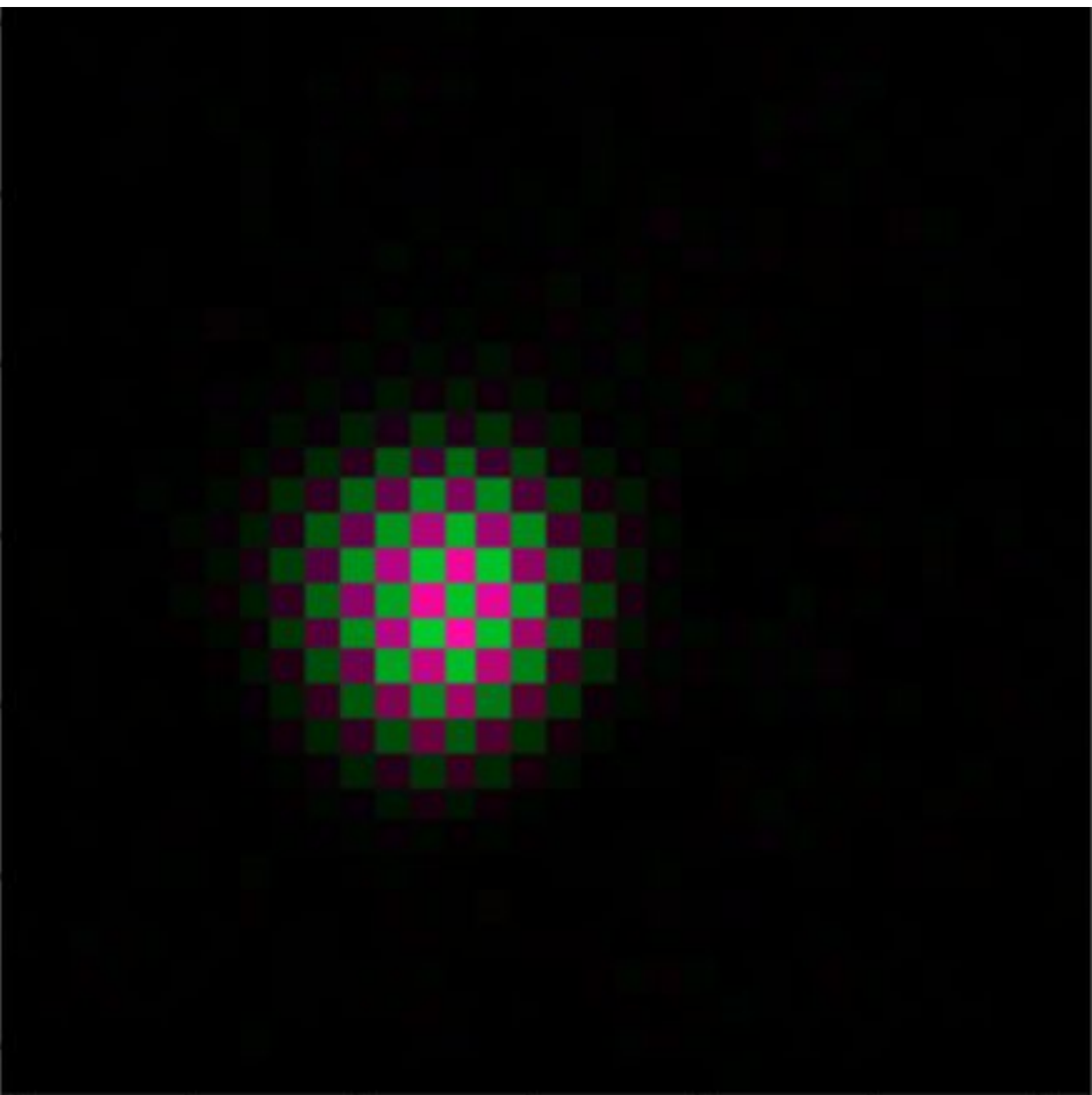}} & \parbox[l]{1em}{\includegraphics[width=3em]{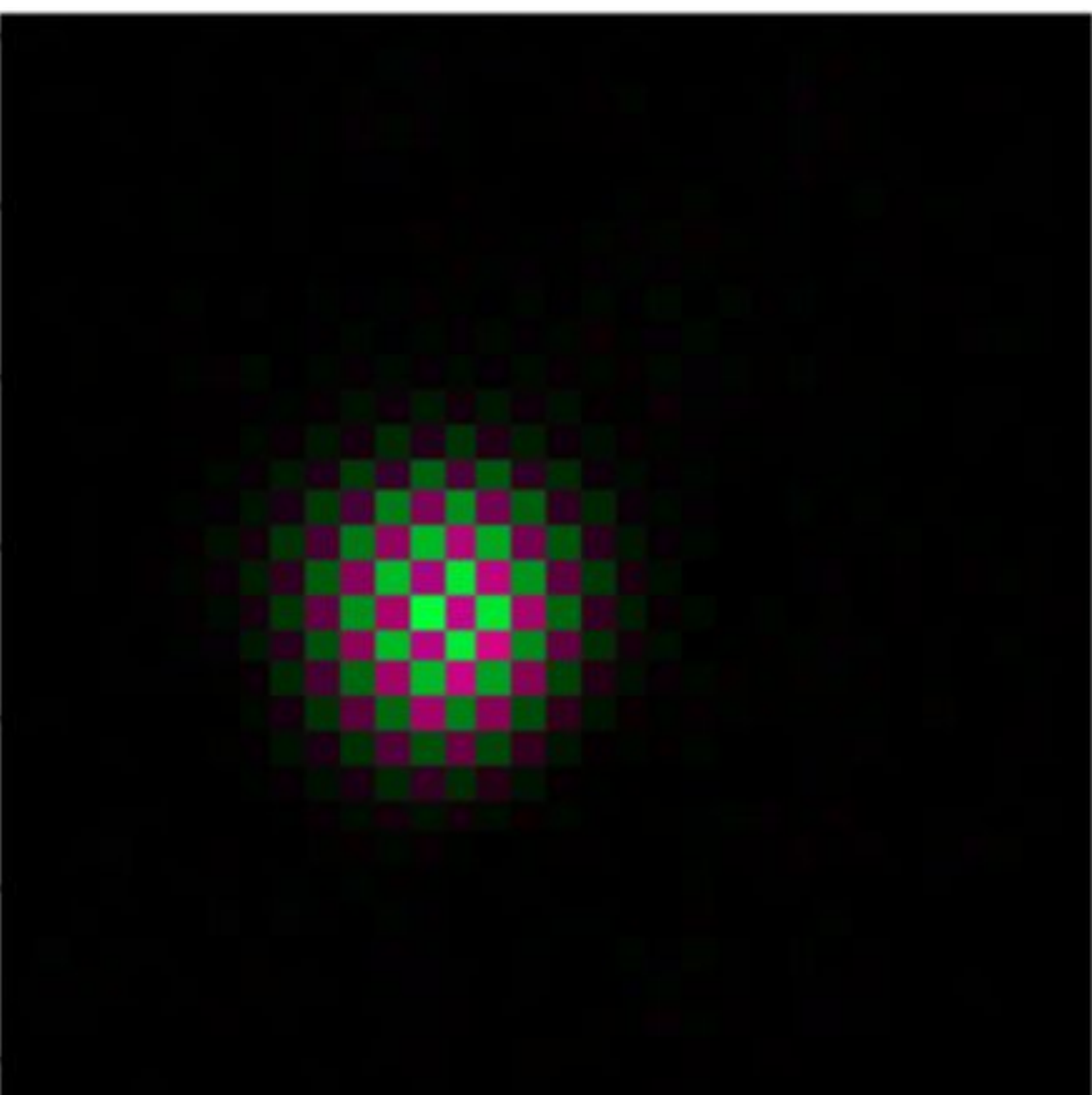}} && \parbox[l]{1em}{\includegraphics[width=3em]{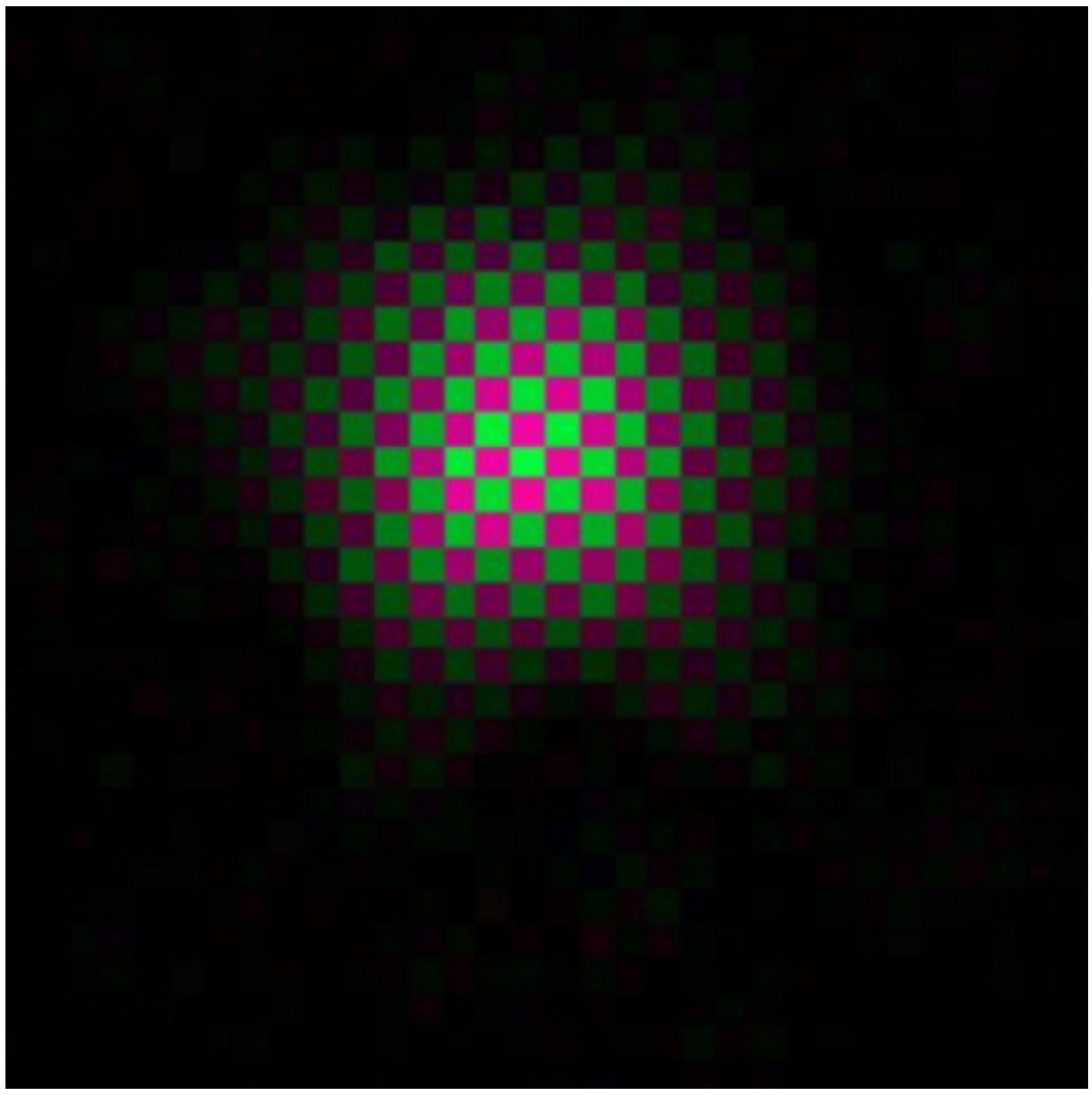}} & \parbox[l]{1em}{\includegraphics[width=3em]{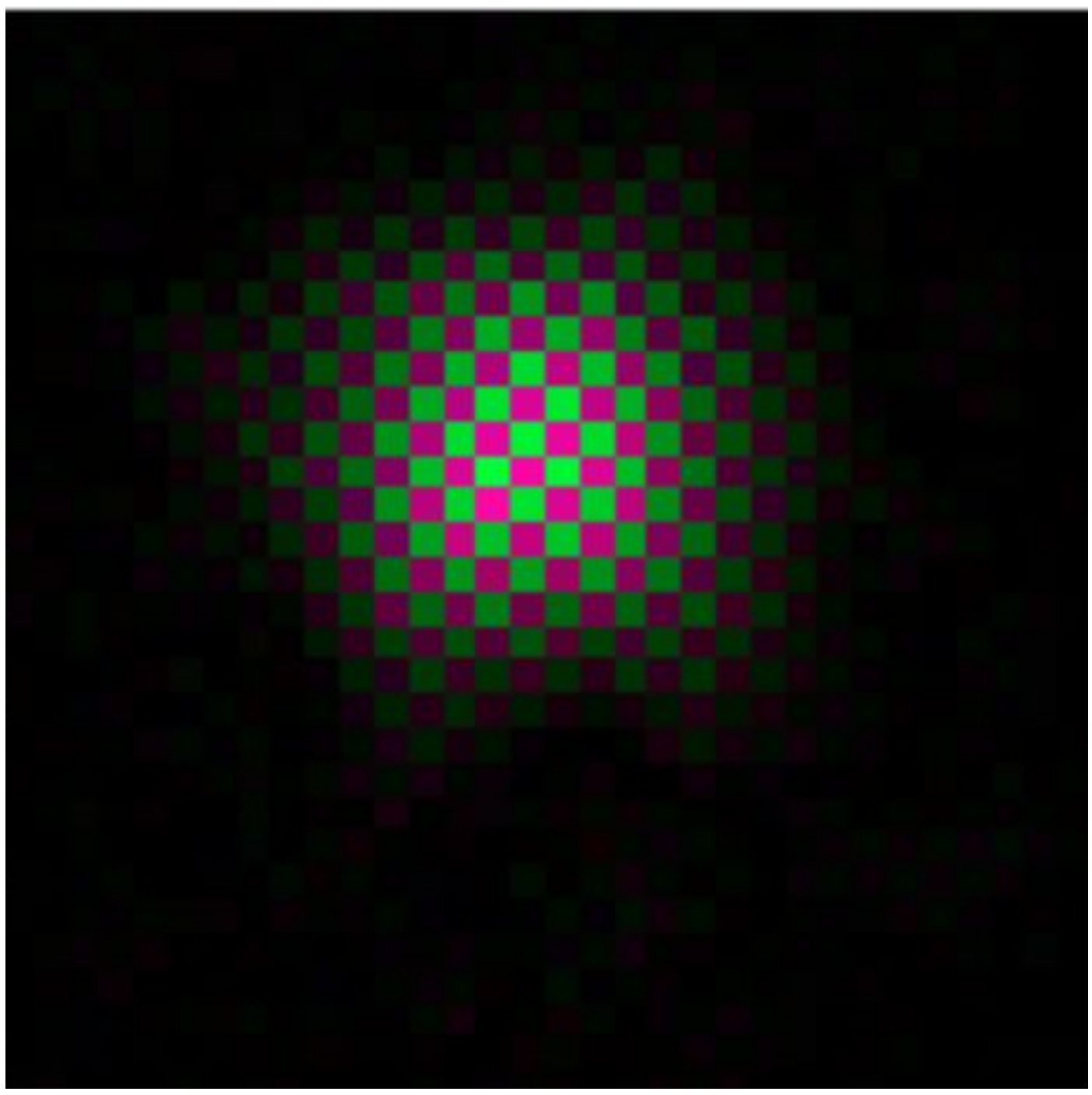}} \\
        \parbox[l]{1em}{\includegraphics[width=3em]{images/c_dot_sample.pdf}} & Cat & \parbox[l]{1em}{\includegraphics[width=3em]{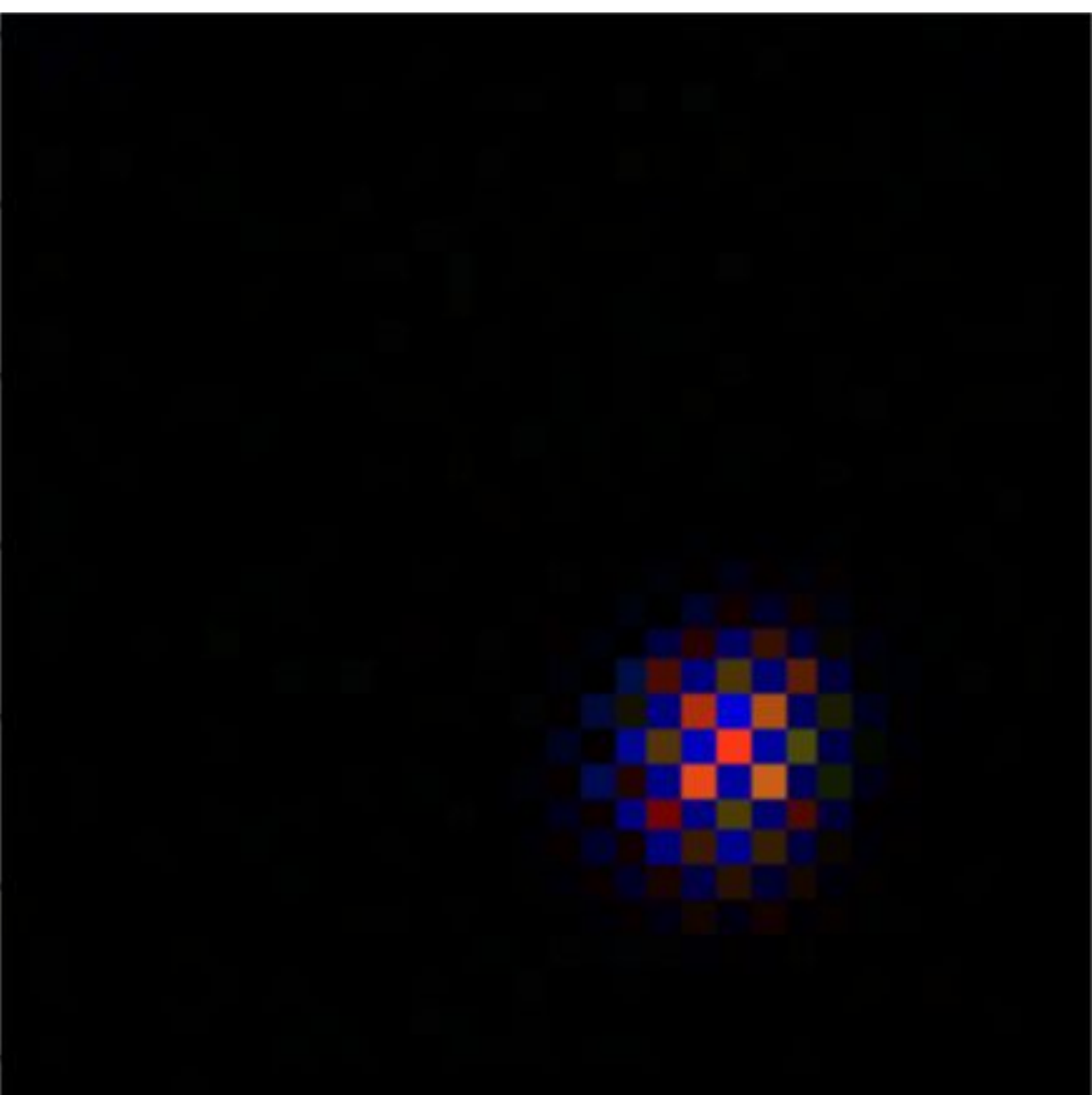}} & \parbox[l]{1em}{\includegraphics[width=3em]{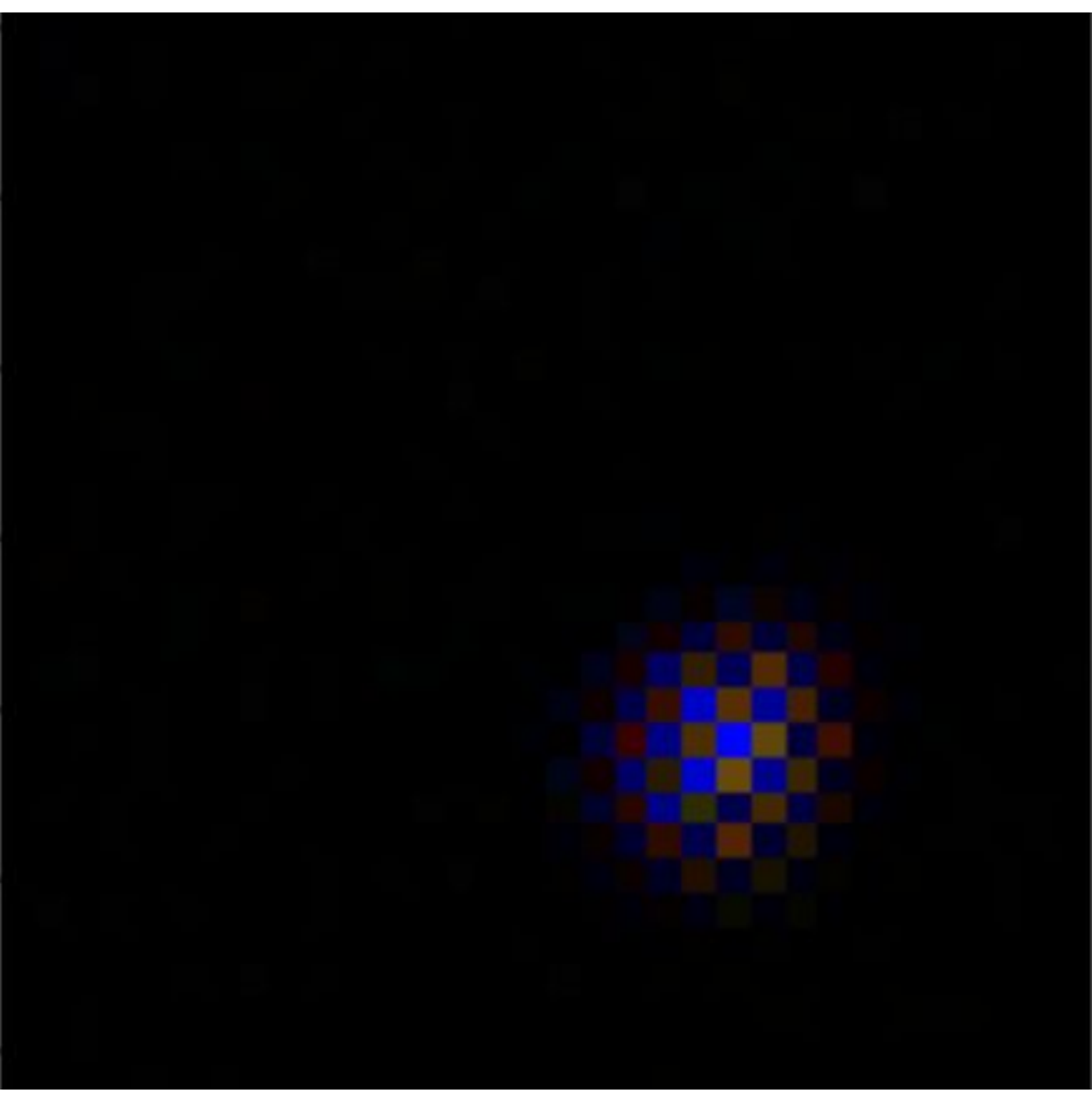}} && \parbox[l]{1em}{\includegraphics[width=3em]{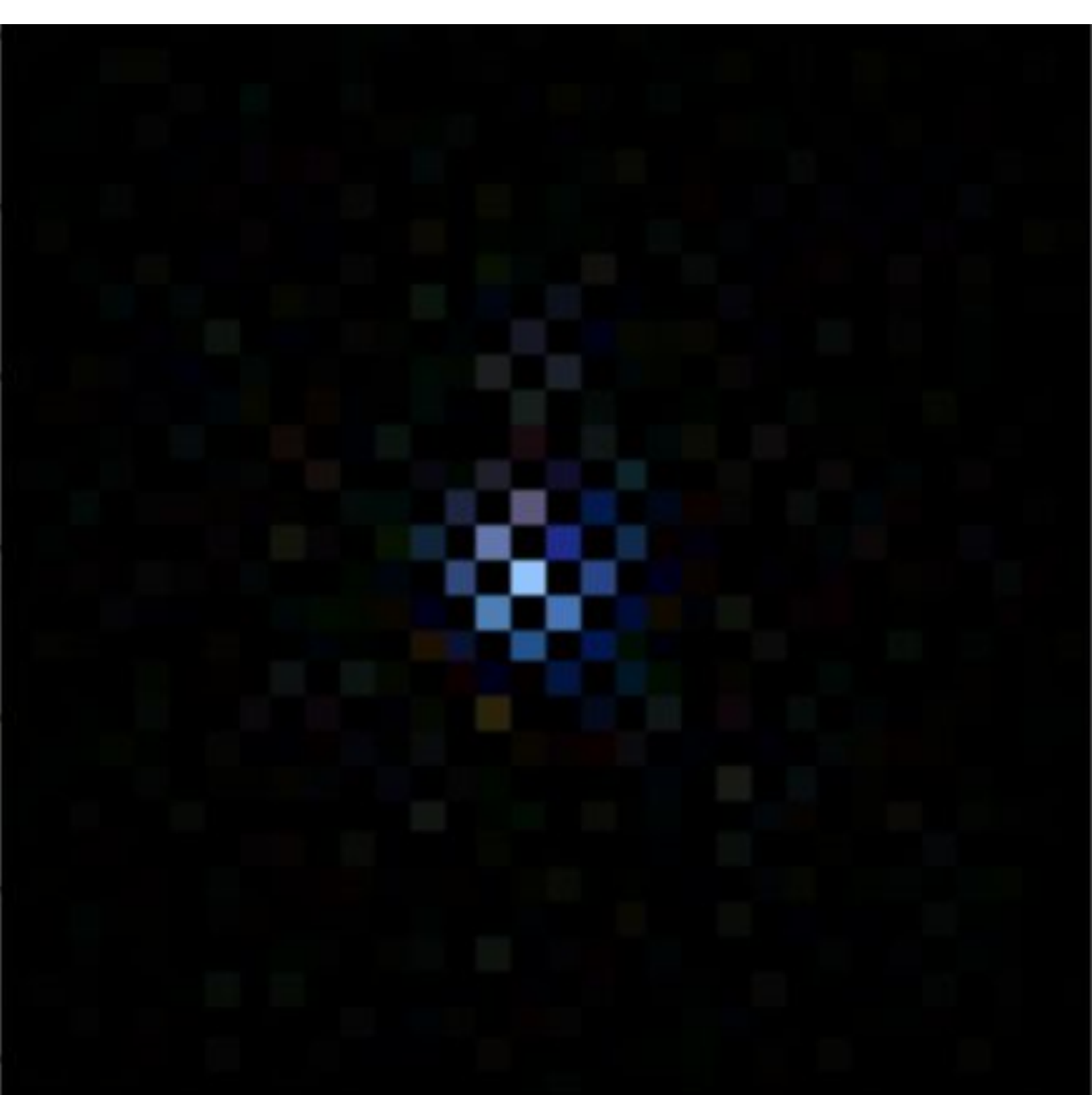}} & \parbox[l]{1em}{\includegraphics[width=3em]{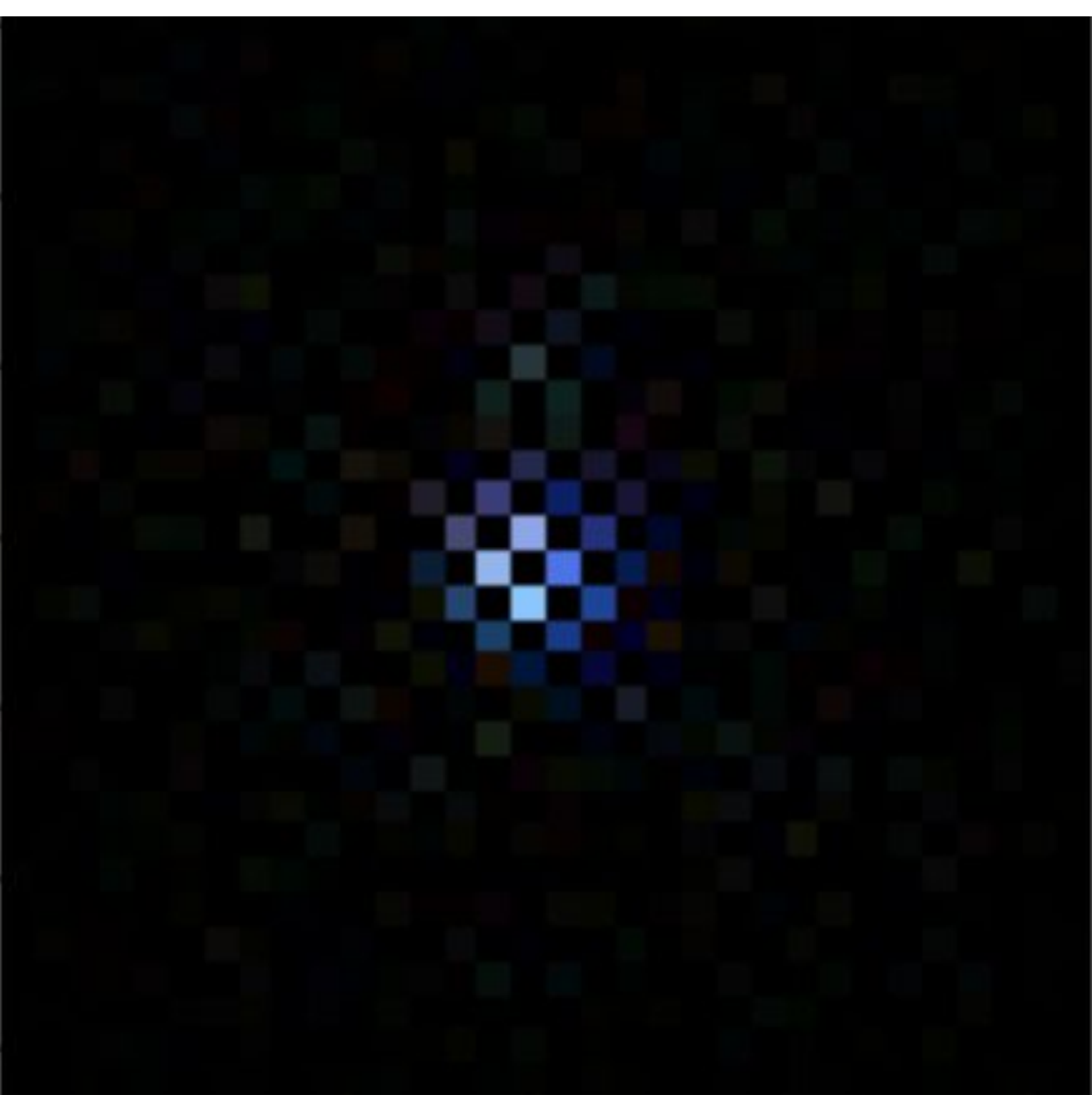}} \\
        \parbox[l]{1em}{\includegraphics[width=3em]{images/d_dot_sample.pdf}} & Bird & \parbox[l]{1em}{\includegraphics[width=3em]{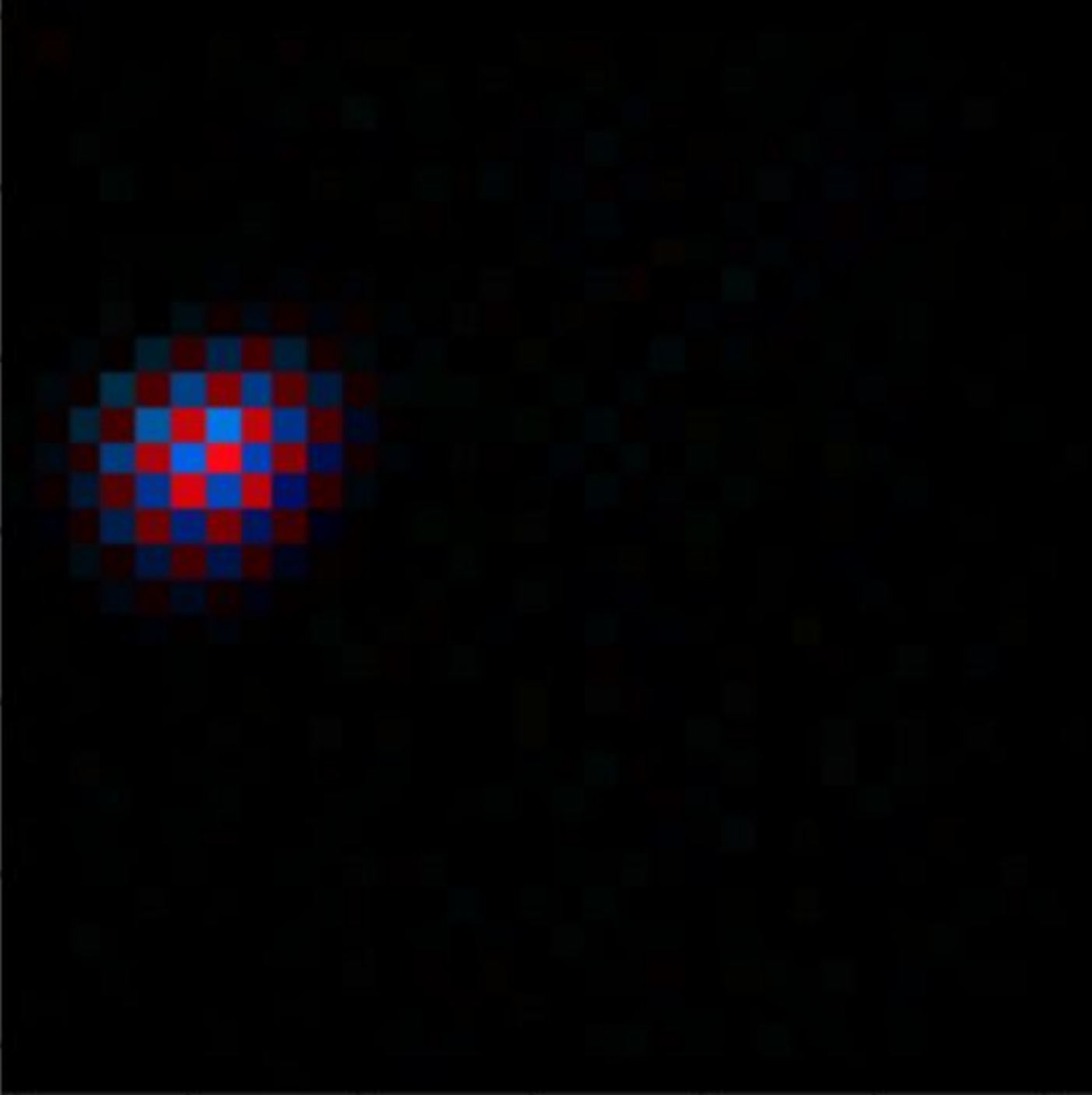}} & \parbox[l]{1em}{\includegraphics[width=3em]{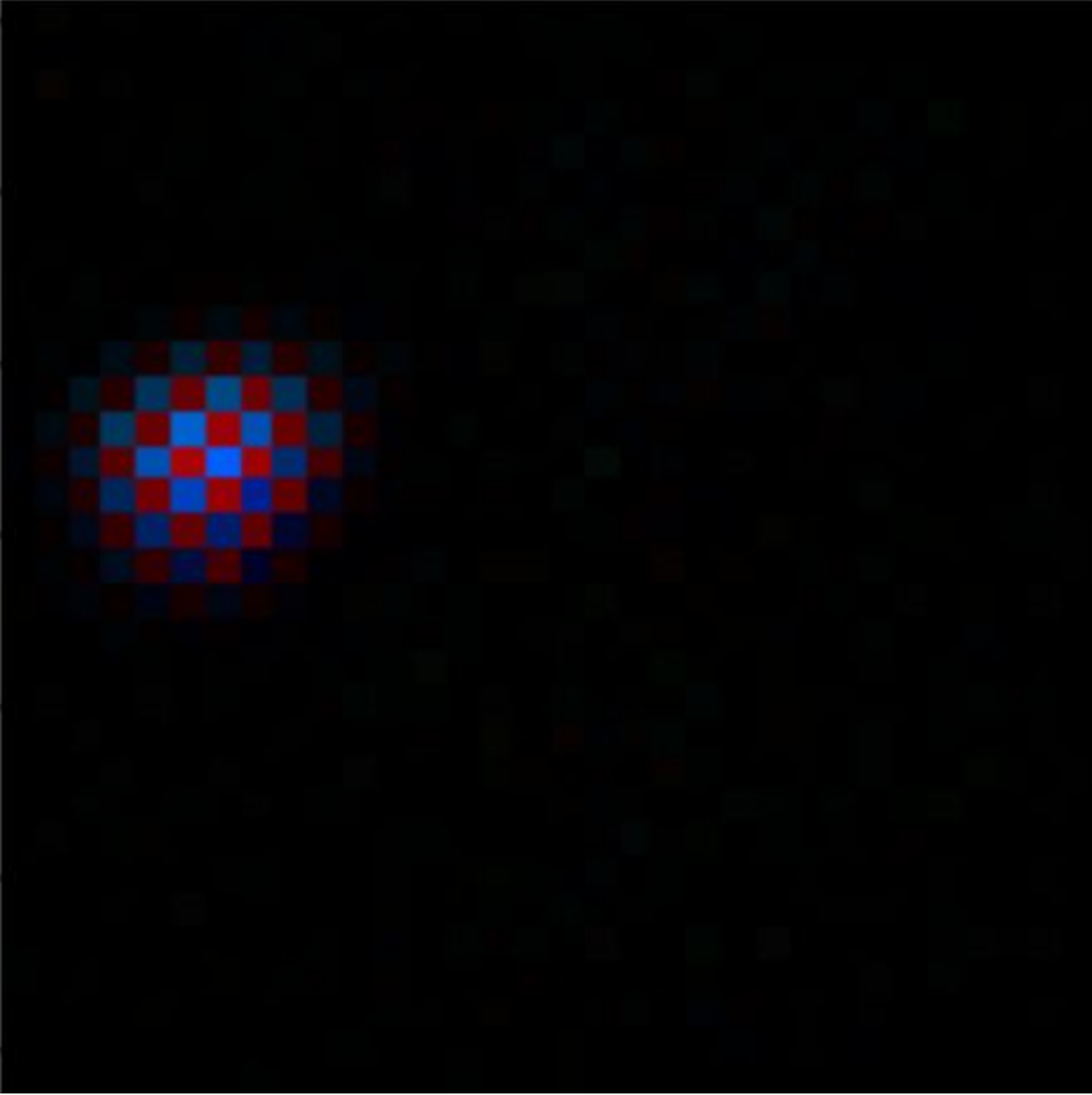}} && \parbox[l]{1em}{\includegraphics[width=3em]{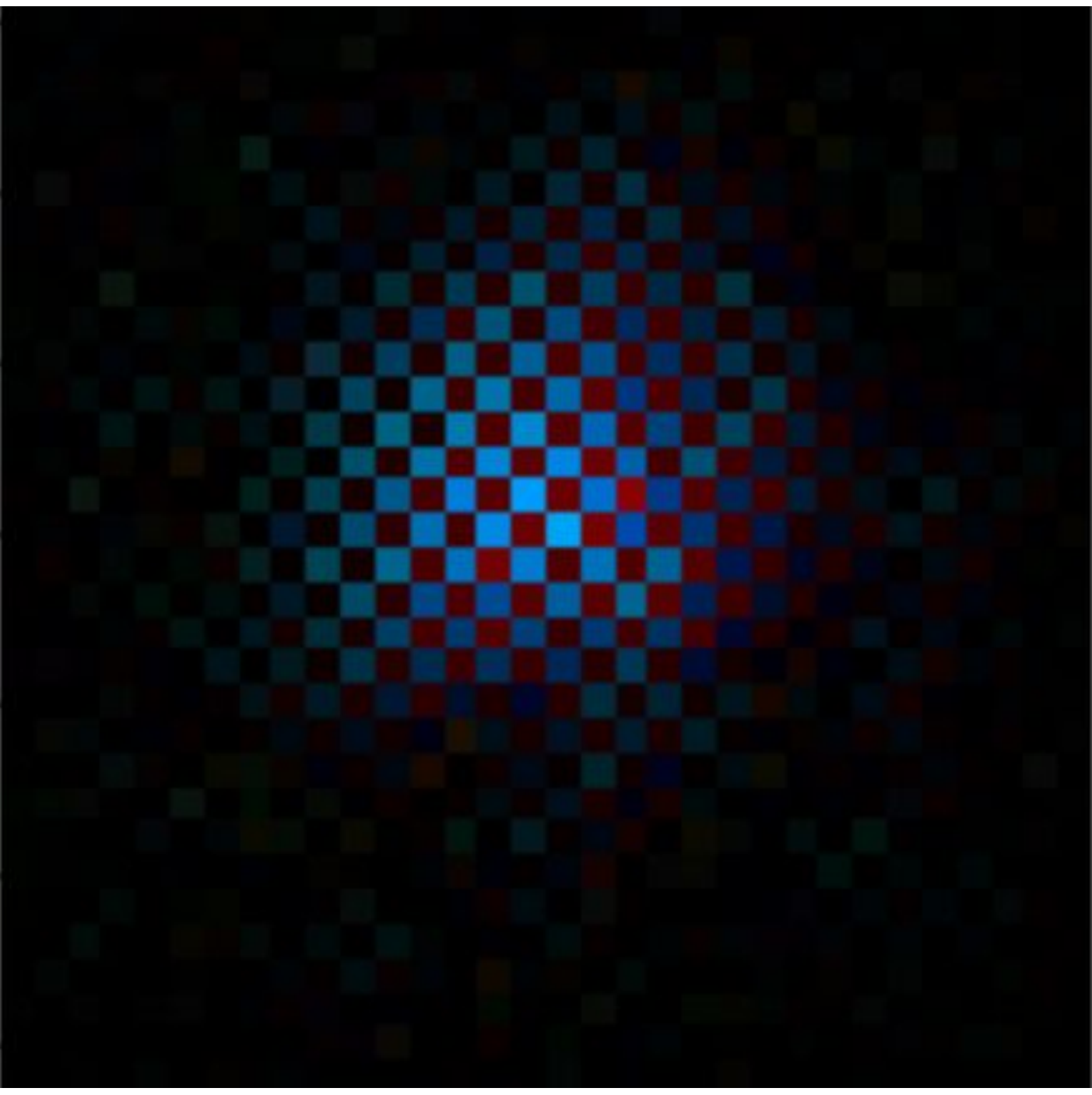}} & \parbox[l]{1em}{\includegraphics[width=3em]{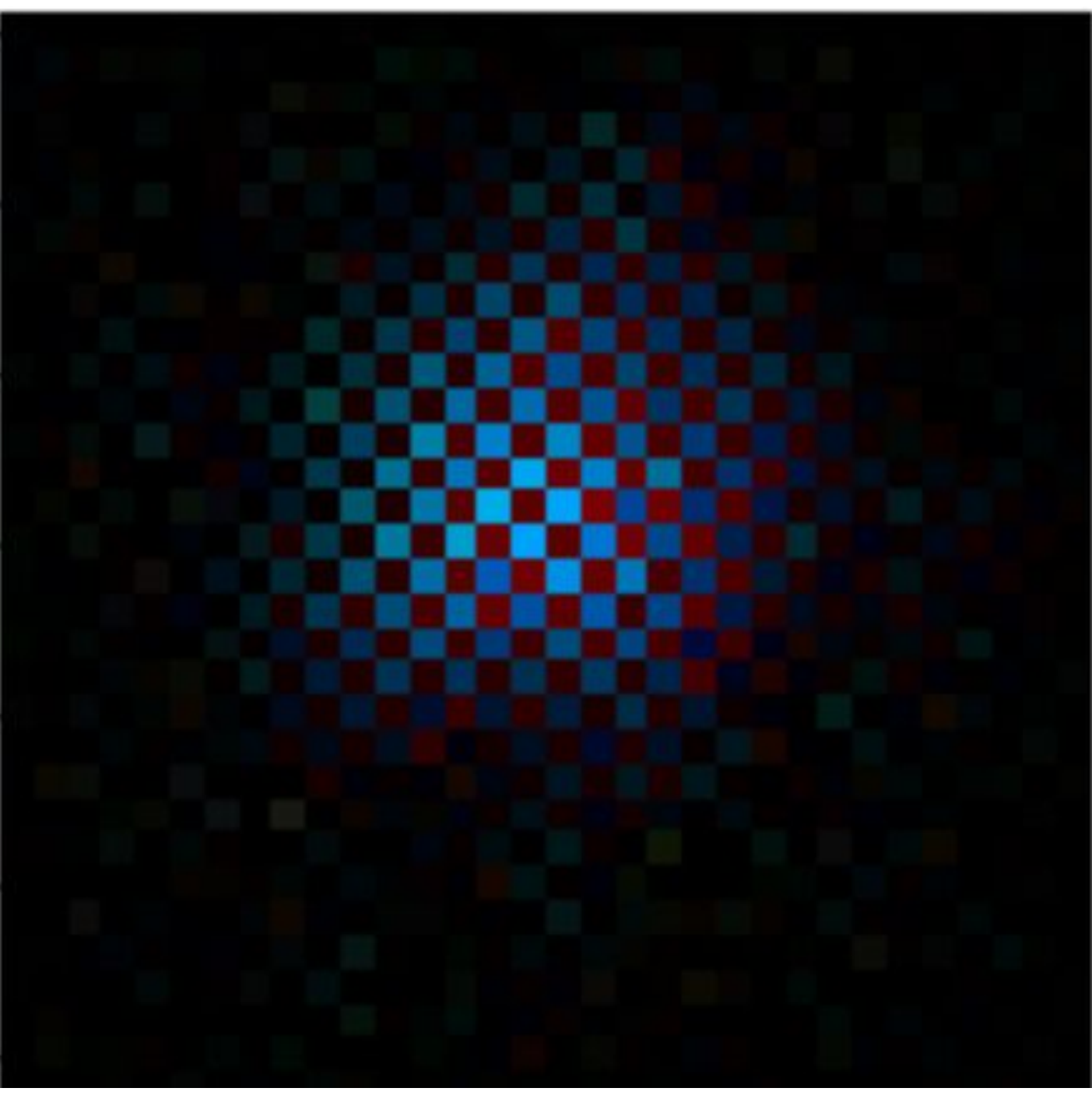}} \\
        \parbox[l]{1em}{\includegraphics[width=3em]{images/e_dot_sample.pdf}} & Deer & \parbox[l]{1em}{\includegraphics[width=3em]{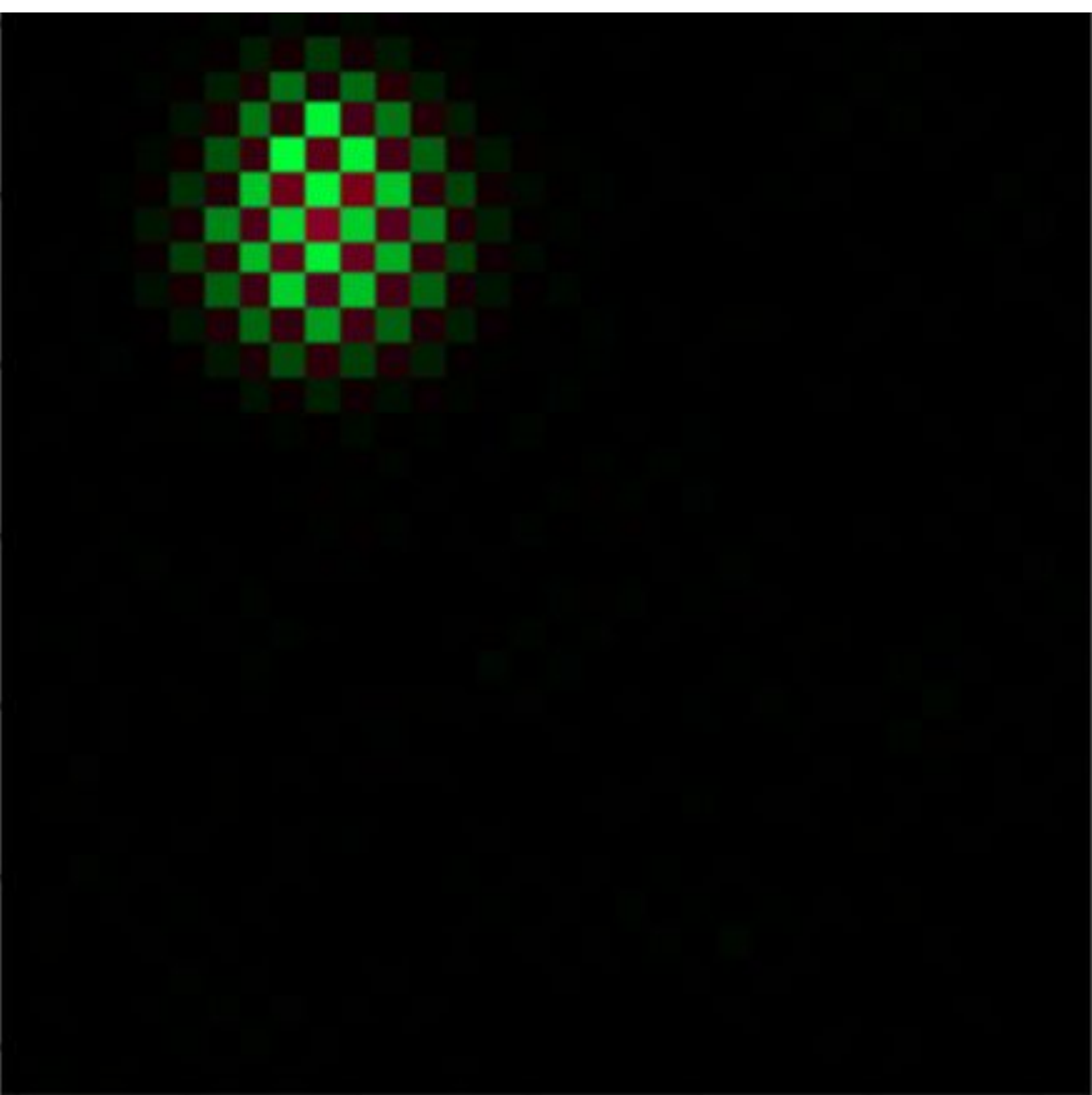}} & \parbox[l]{1em}{\includegraphics[width=3em]{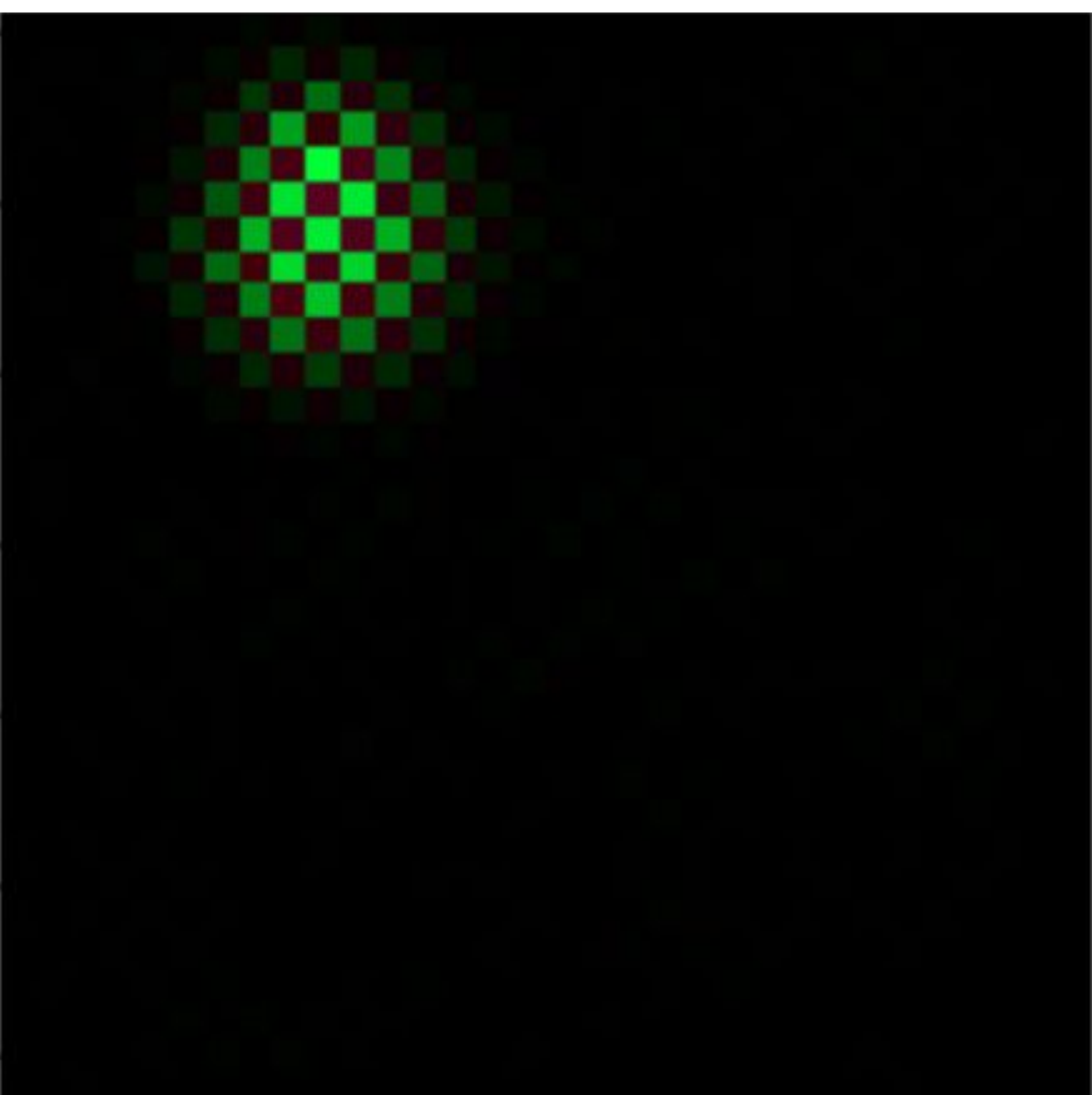}} && \parbox[l]{1em}{\includegraphics[width=3em]{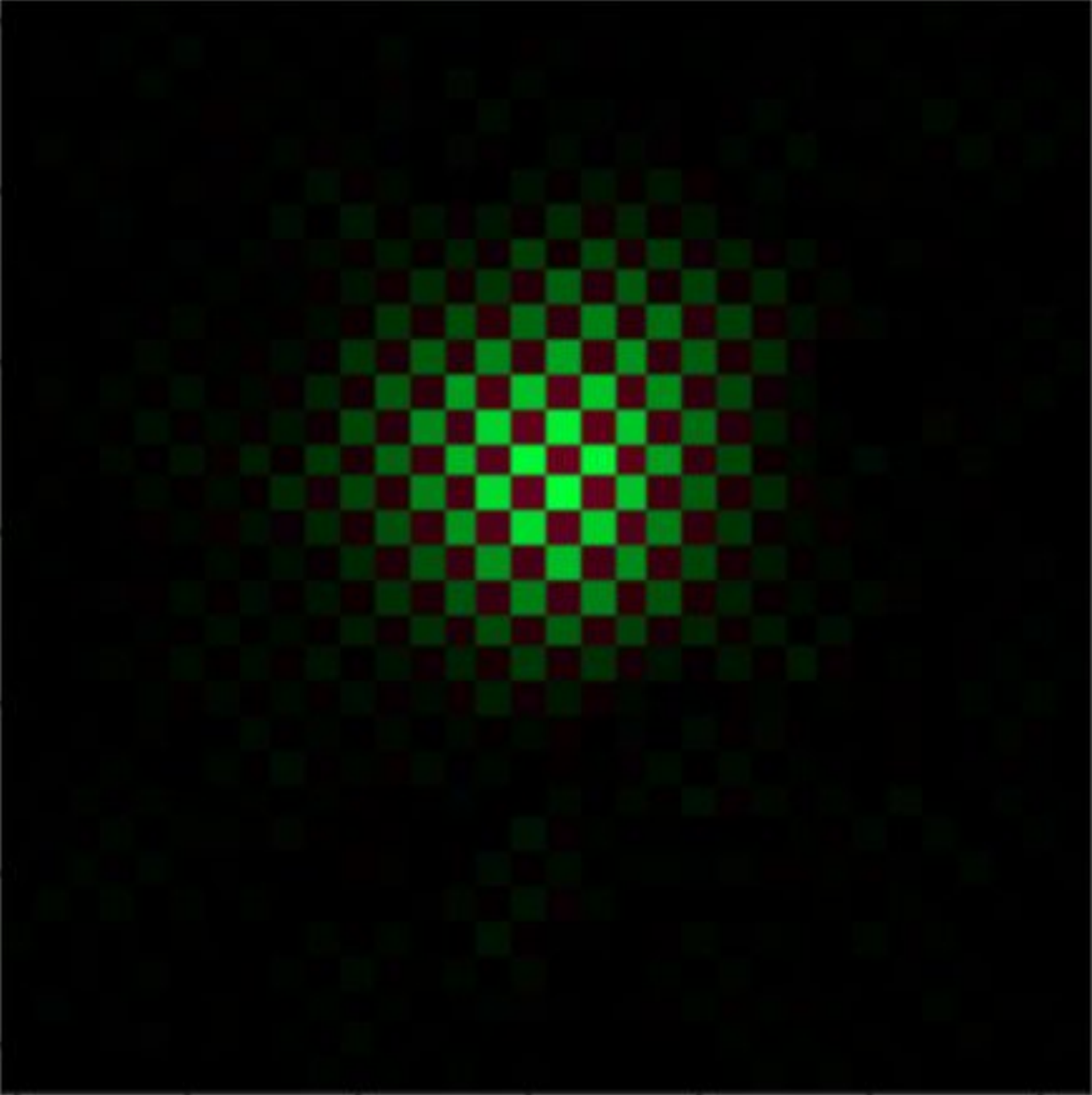}} & \parbox[l]{1em}{\includegraphics[width=3em]{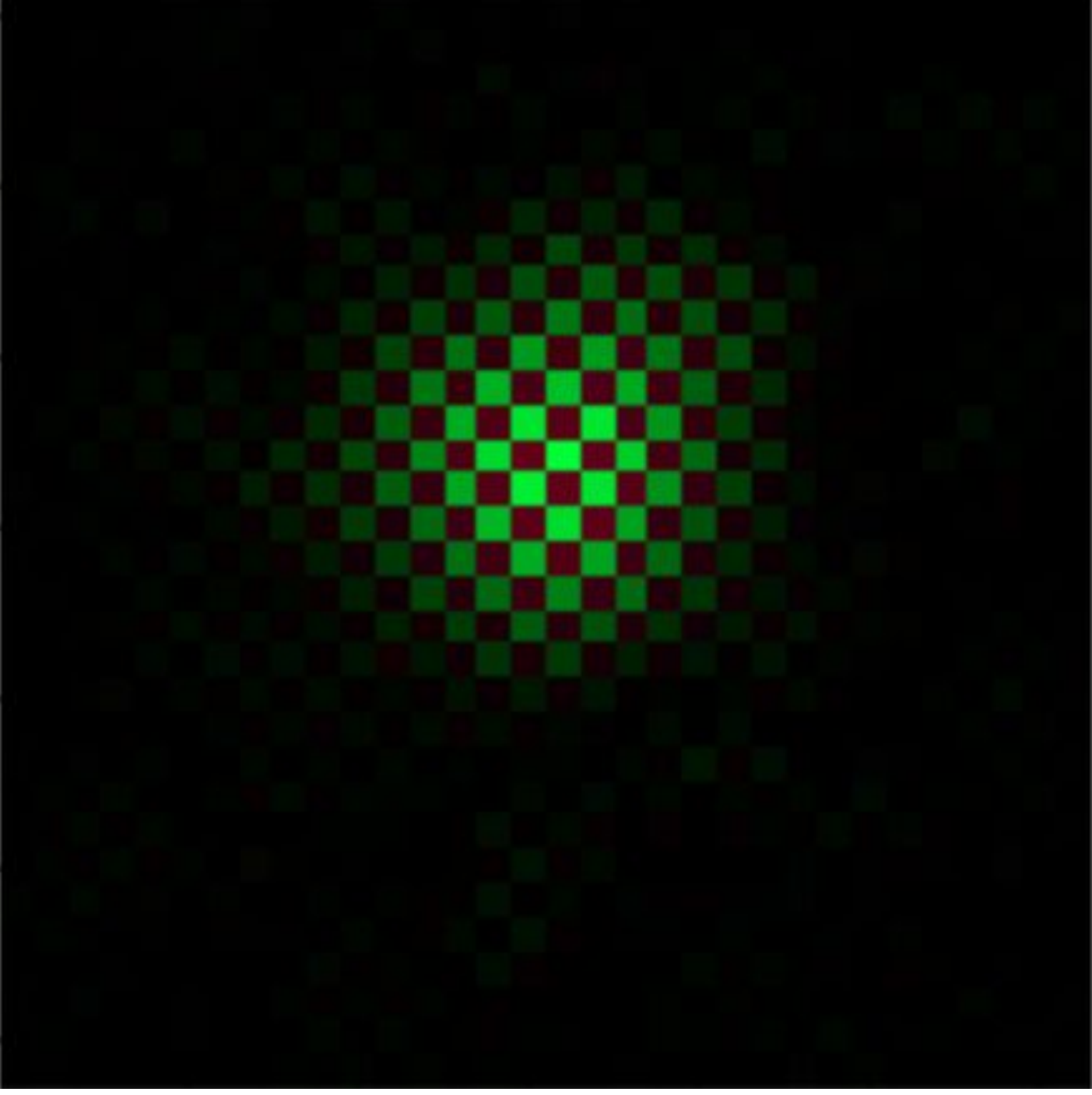}} \\
        \parbox[l]{1em}{\includegraphics[width=3em]{images/f_dot_sample.pdf}} & Bird & \parbox[l]{1em}{\includegraphics[width=3em]{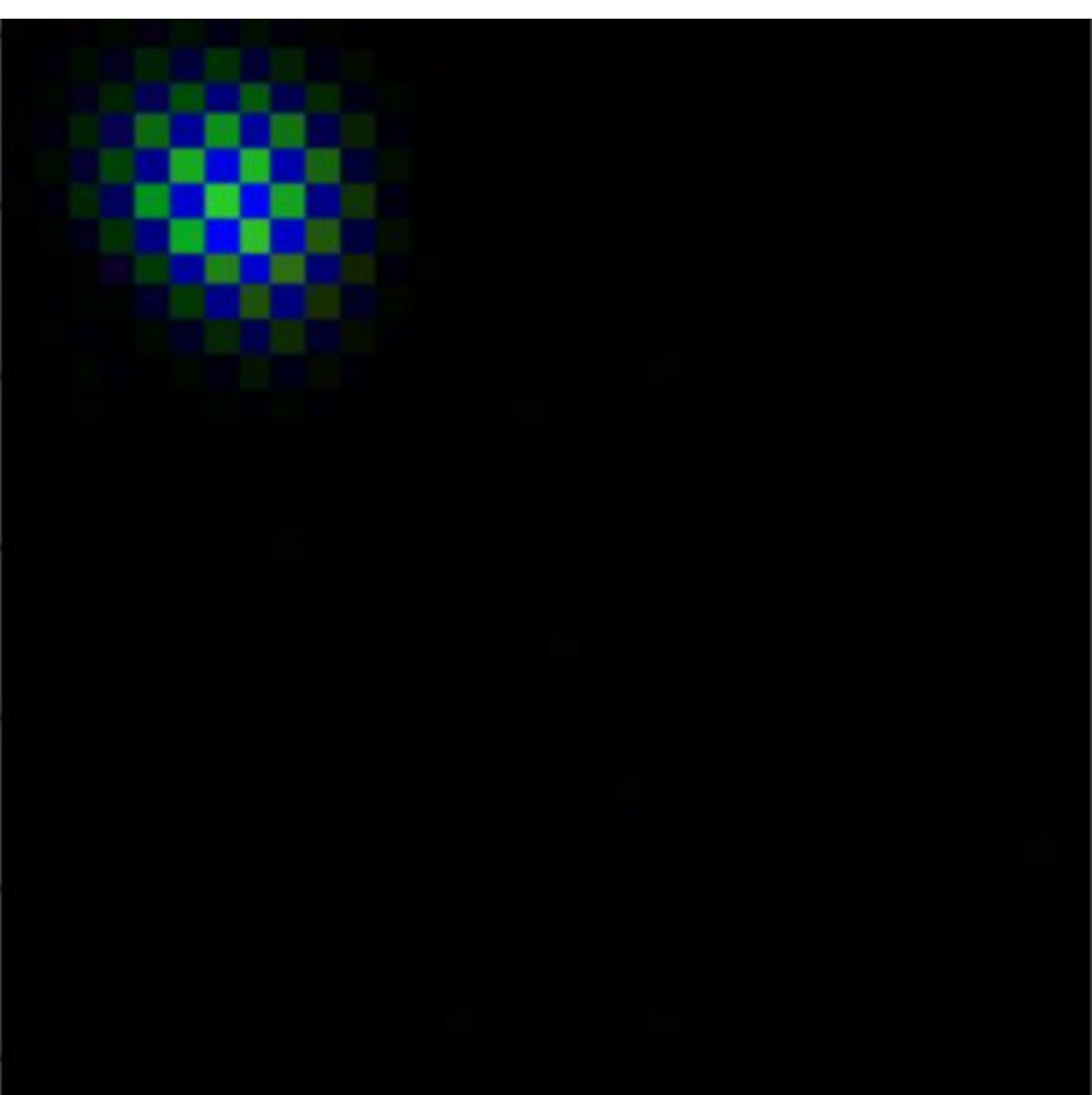}} & \parbox[l]{1em}{\includegraphics[width=3em]{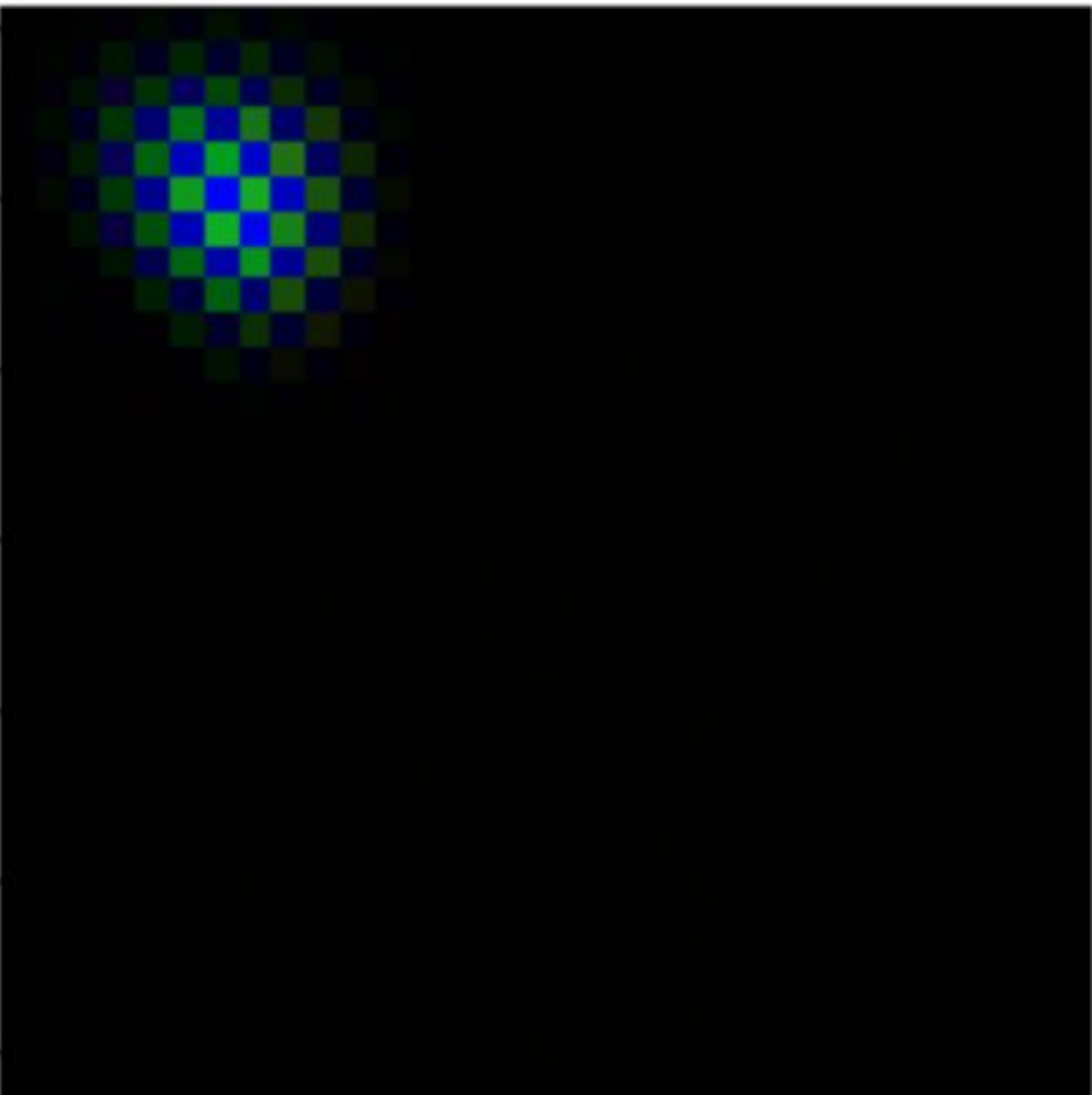}} && \parbox[l]{1em}{\includegraphics[width=3em]{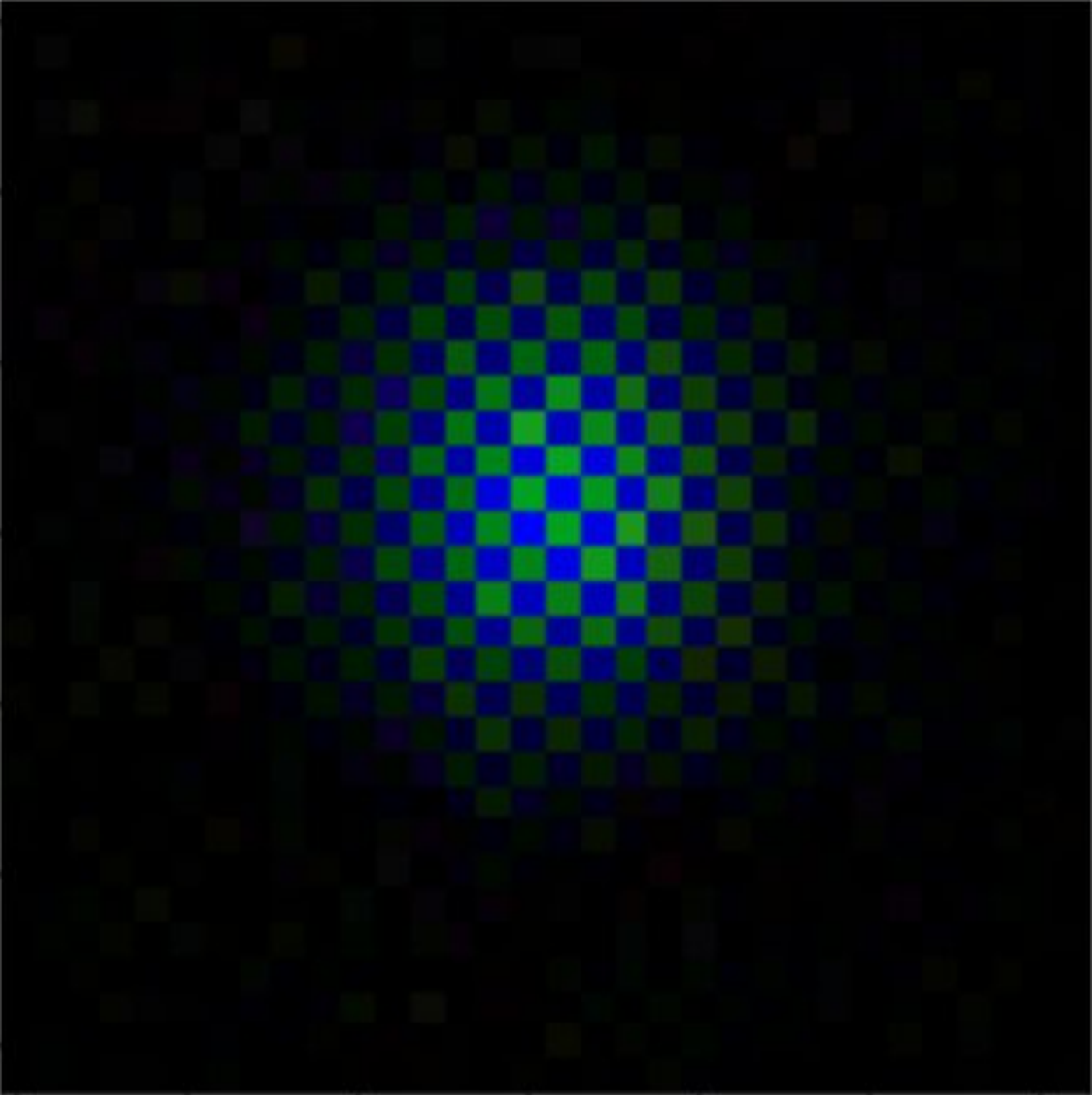}} & \parbox[l]{1em}{\includegraphics[width=3em]{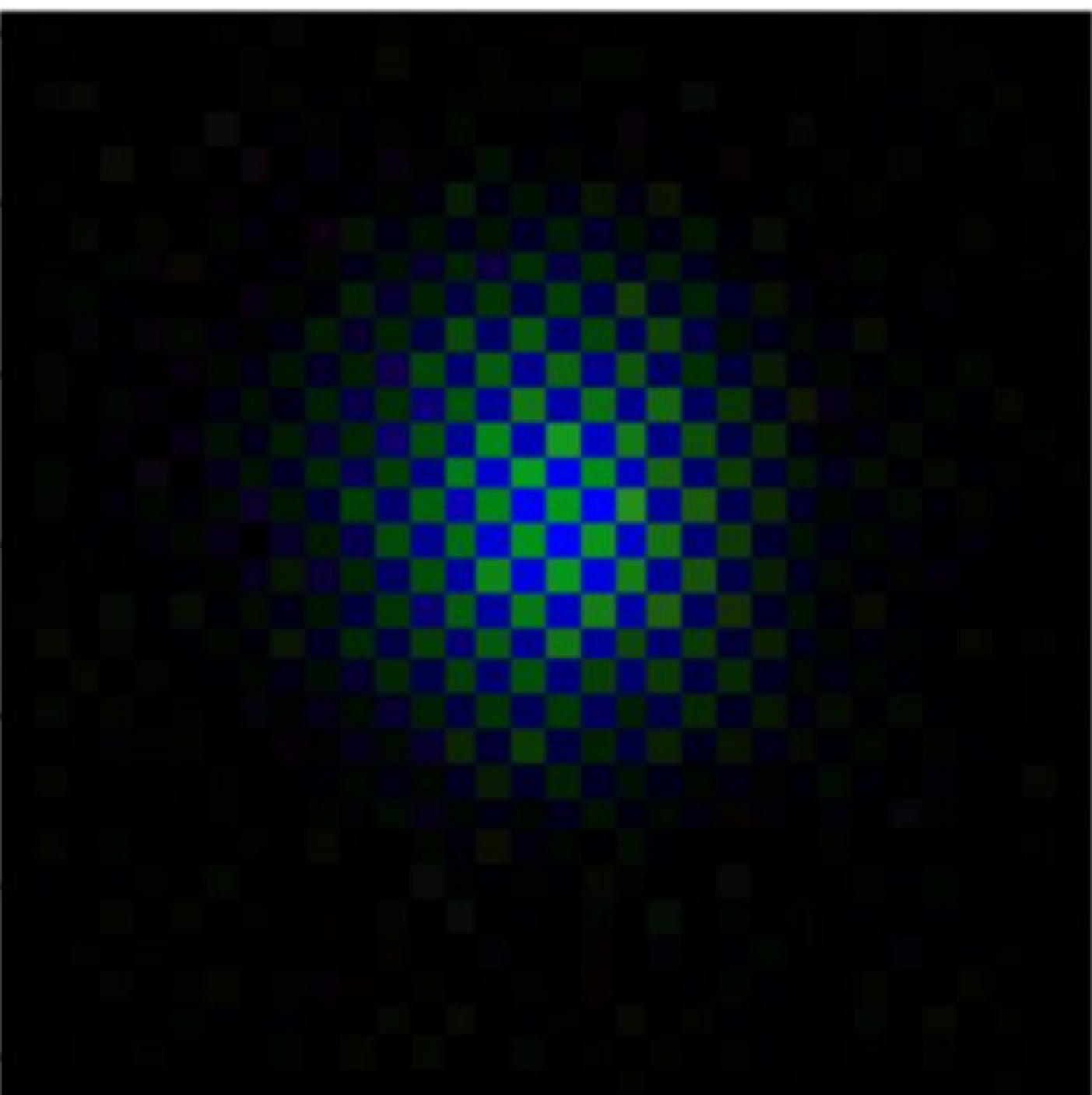}} \\
        \parbox[l]{1em}{\includegraphics[width=3em]{images/g_dot_sample.pdf}} & Horse & \parbox[l]{1em}{\includegraphics[width=3em]{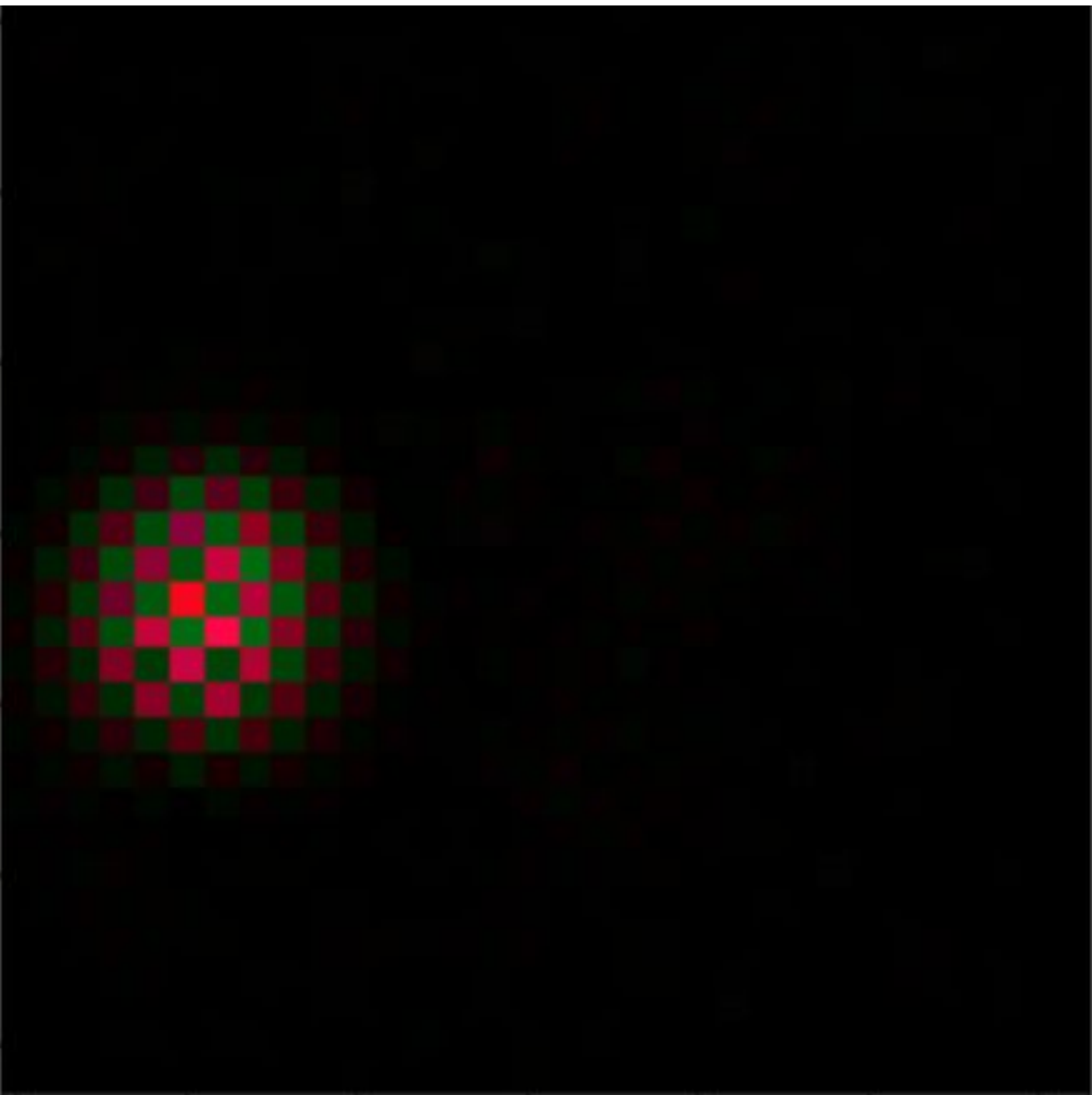}} & \parbox[l]{1em}{\includegraphics[width=3em]{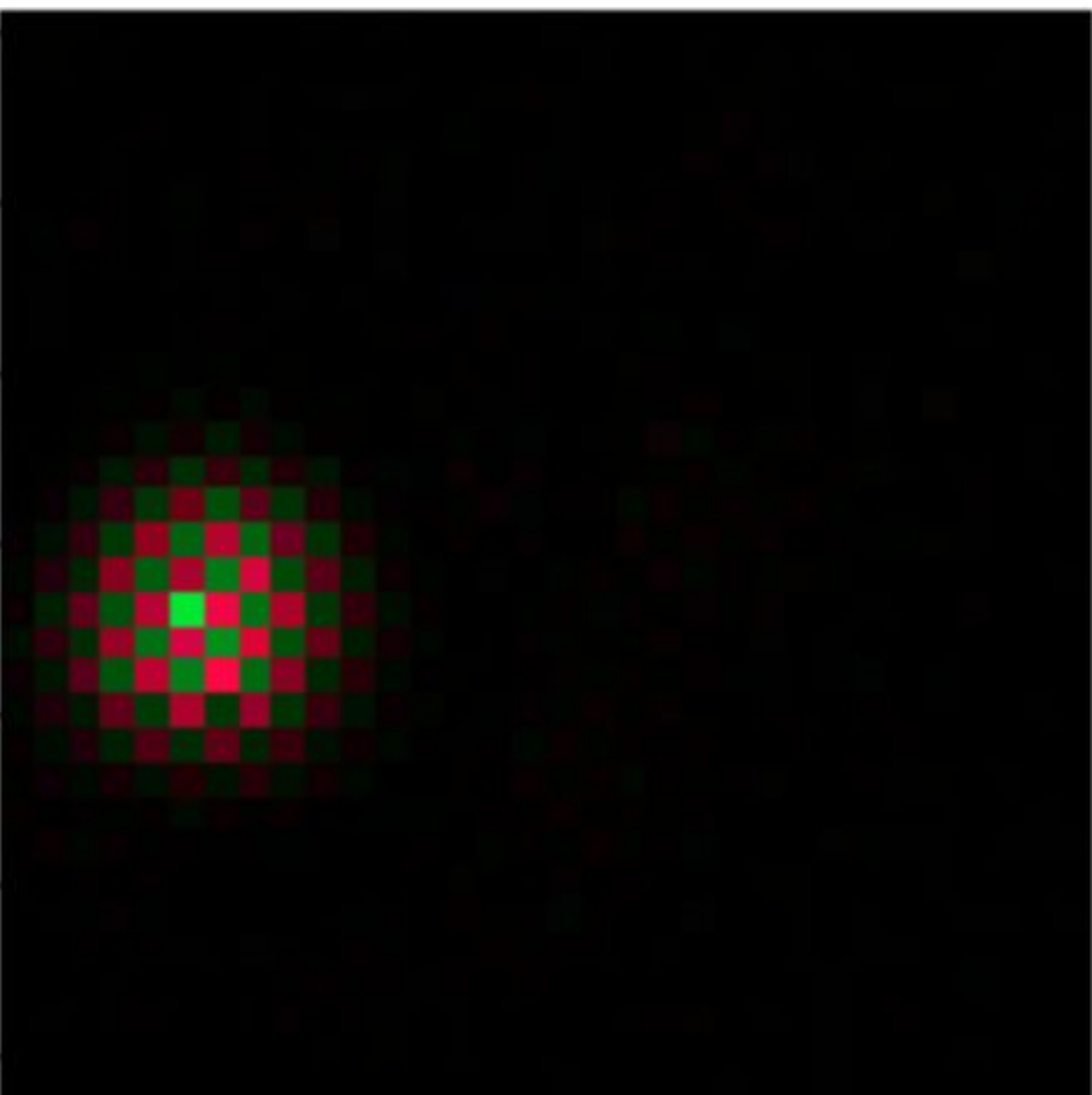}} && \parbox[l]{1em}{\includegraphics[width=3em]{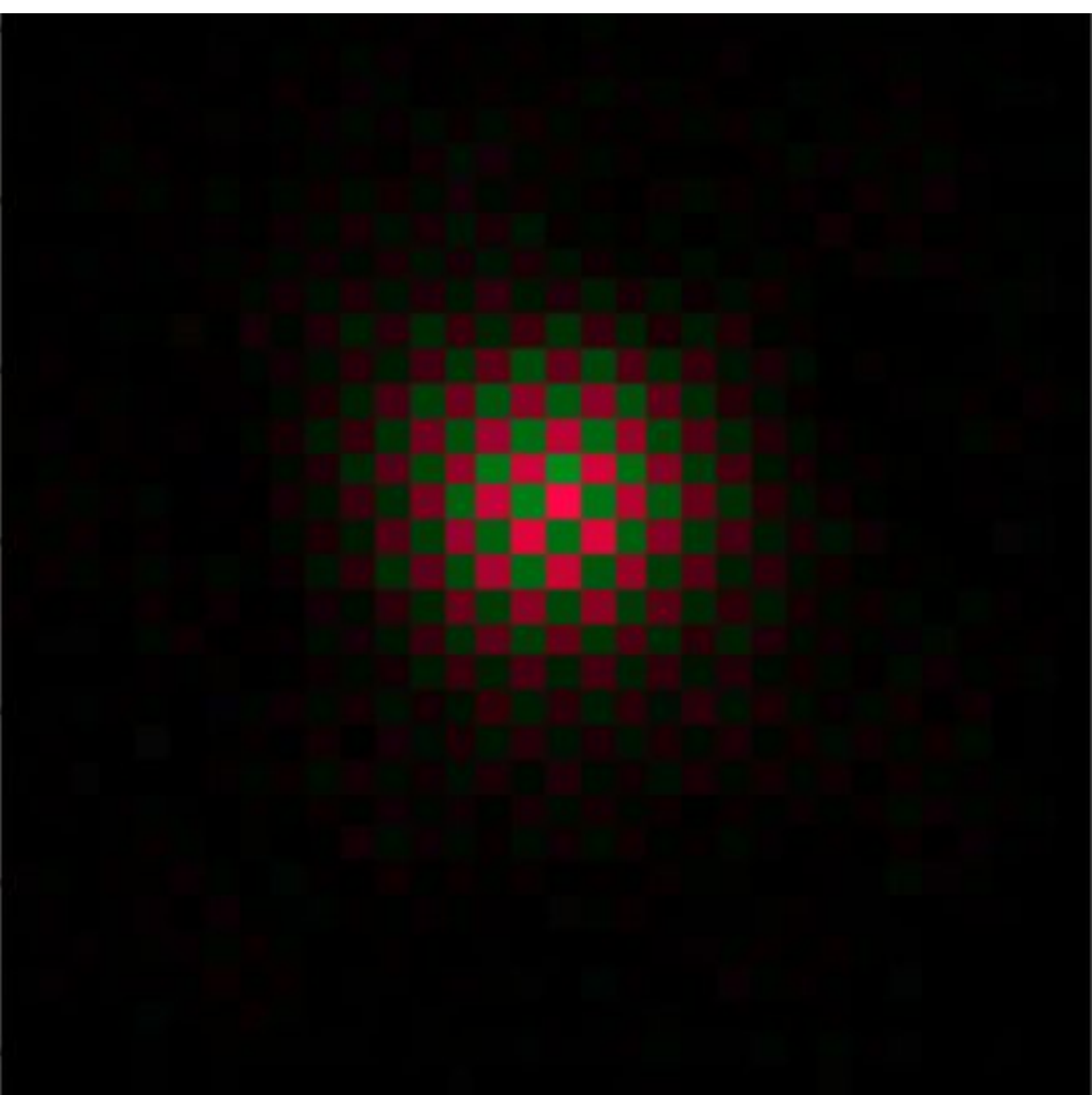}} & \parbox[l]{1em}{\includegraphics[width=3em]{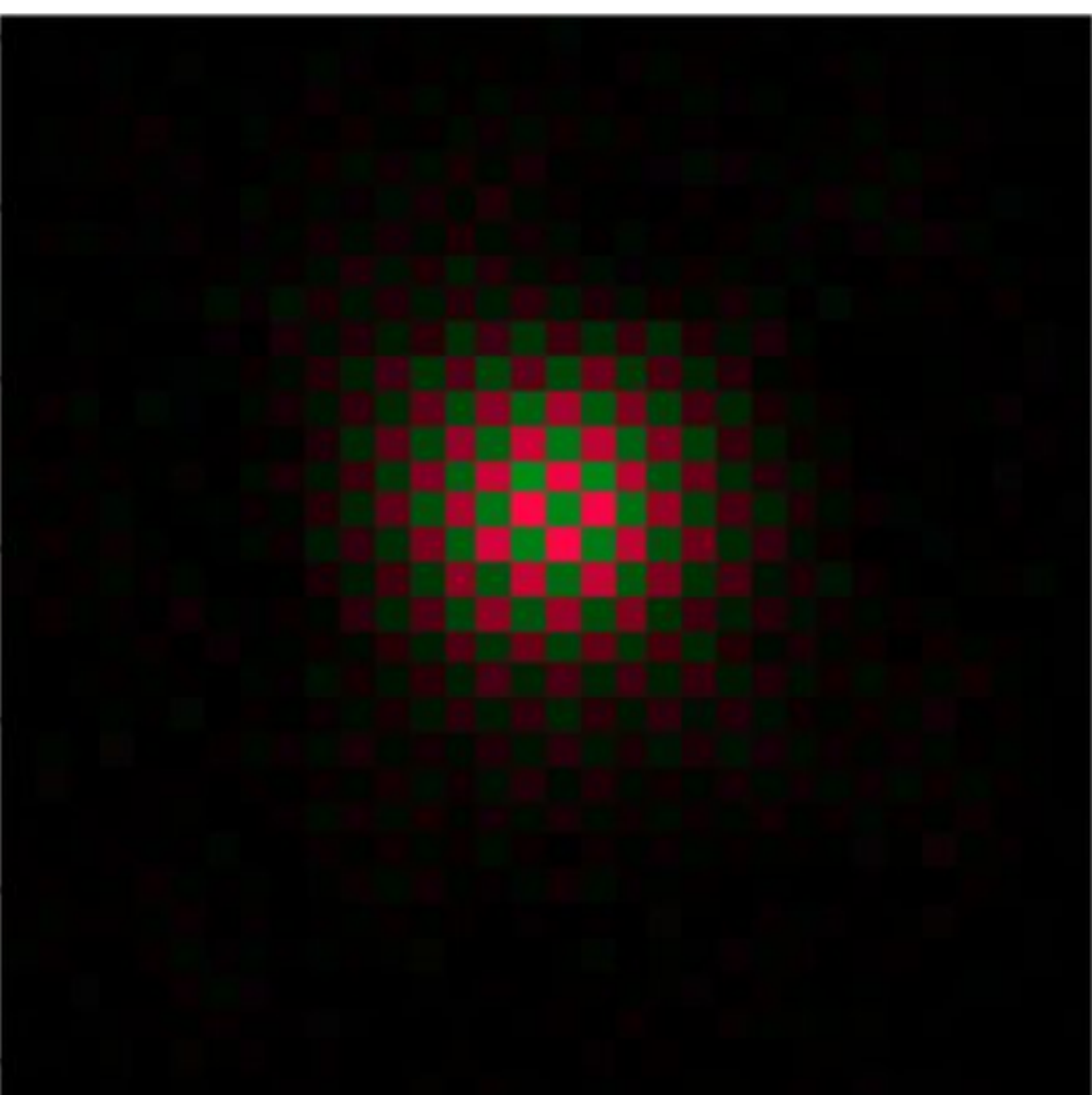}} \\
        \parbox[l]{1em}{\includegraphics[width=3em]{images/h_dot_sample.pdf}} & Cat & \parbox[l]{1em}{\includegraphics[width=3em]{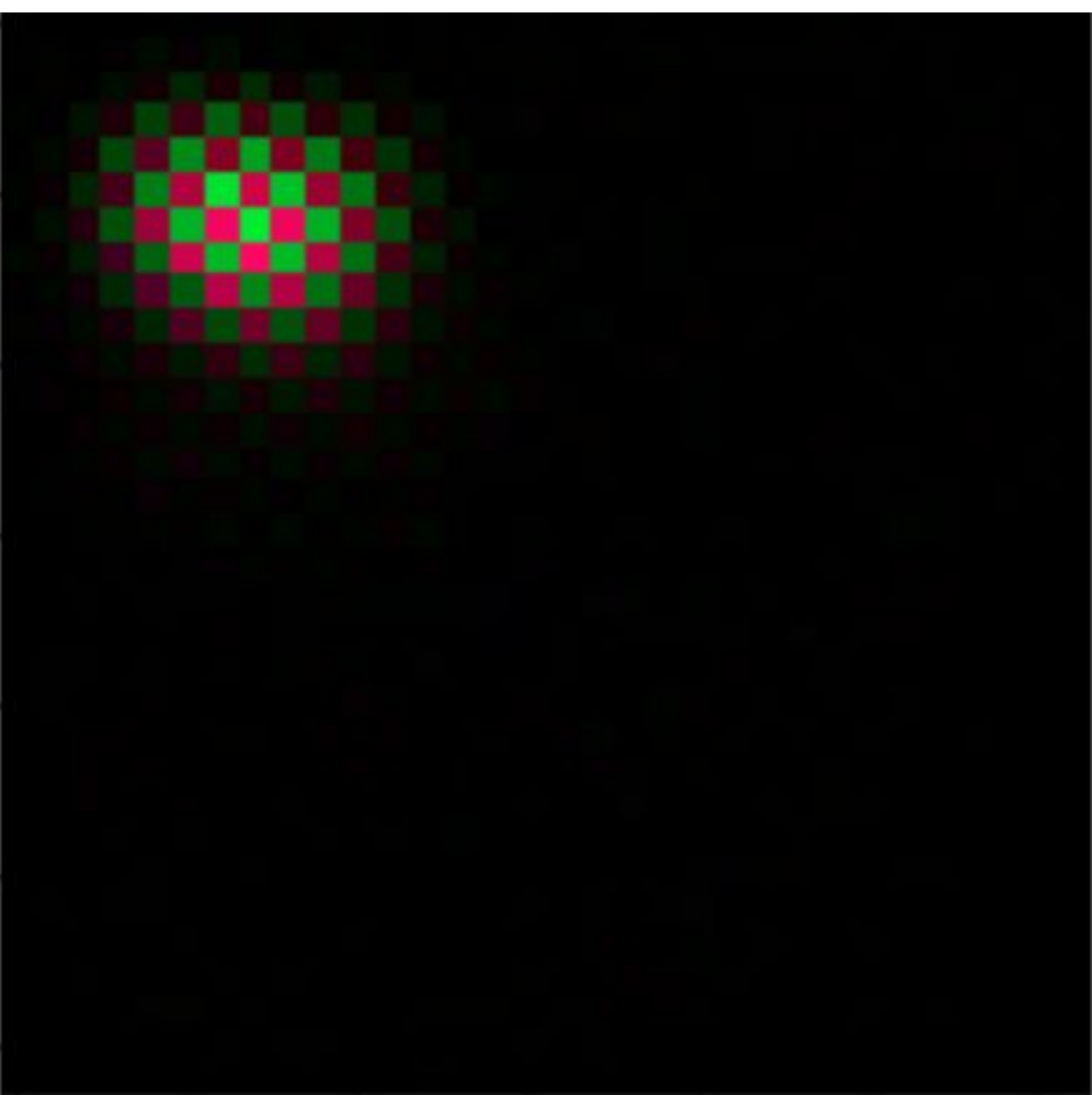}} & \parbox[l]{1em}{\includegraphics[width=3em]{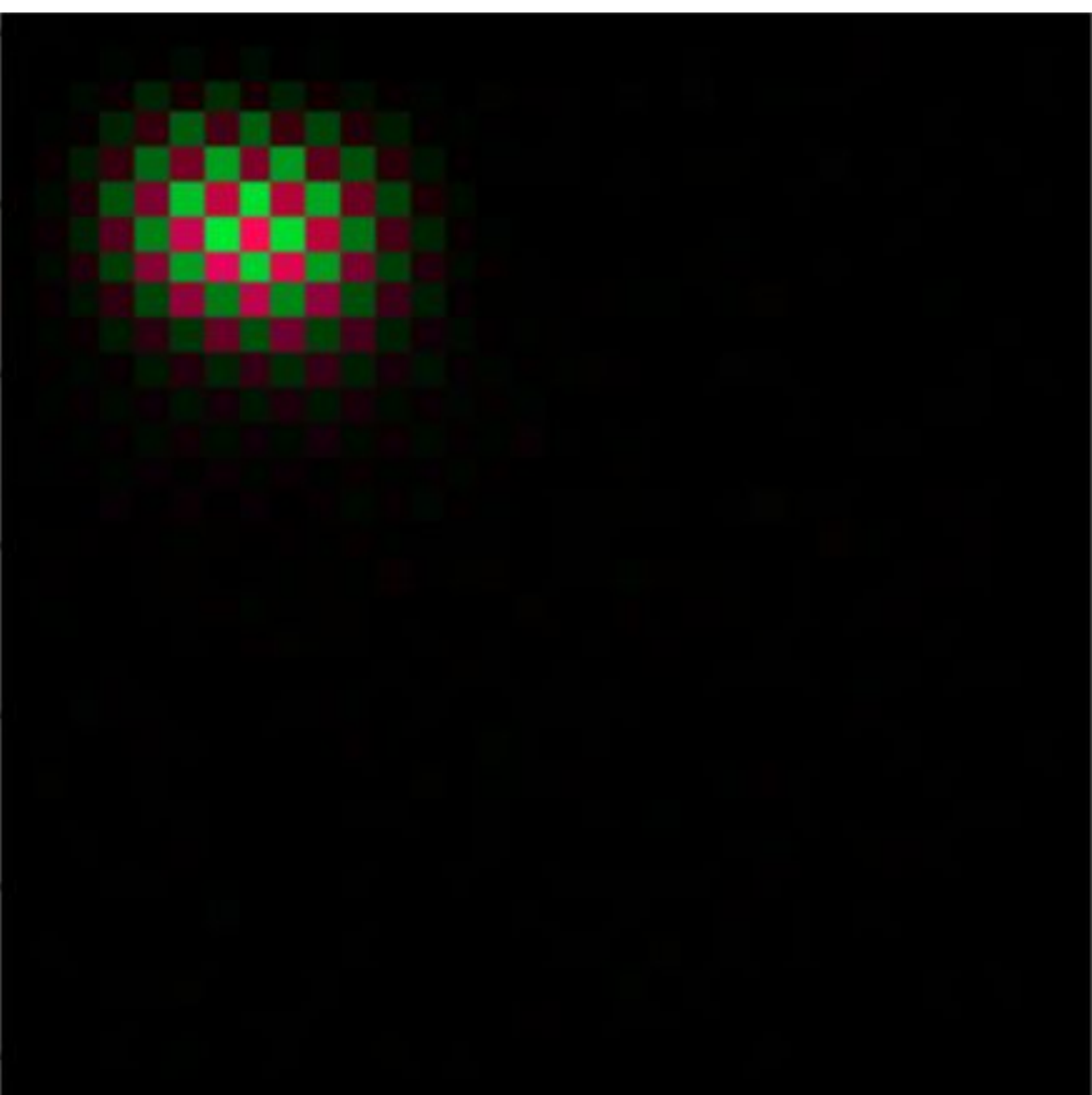}} && \parbox[l]{1em}{\includegraphics[width=3em]{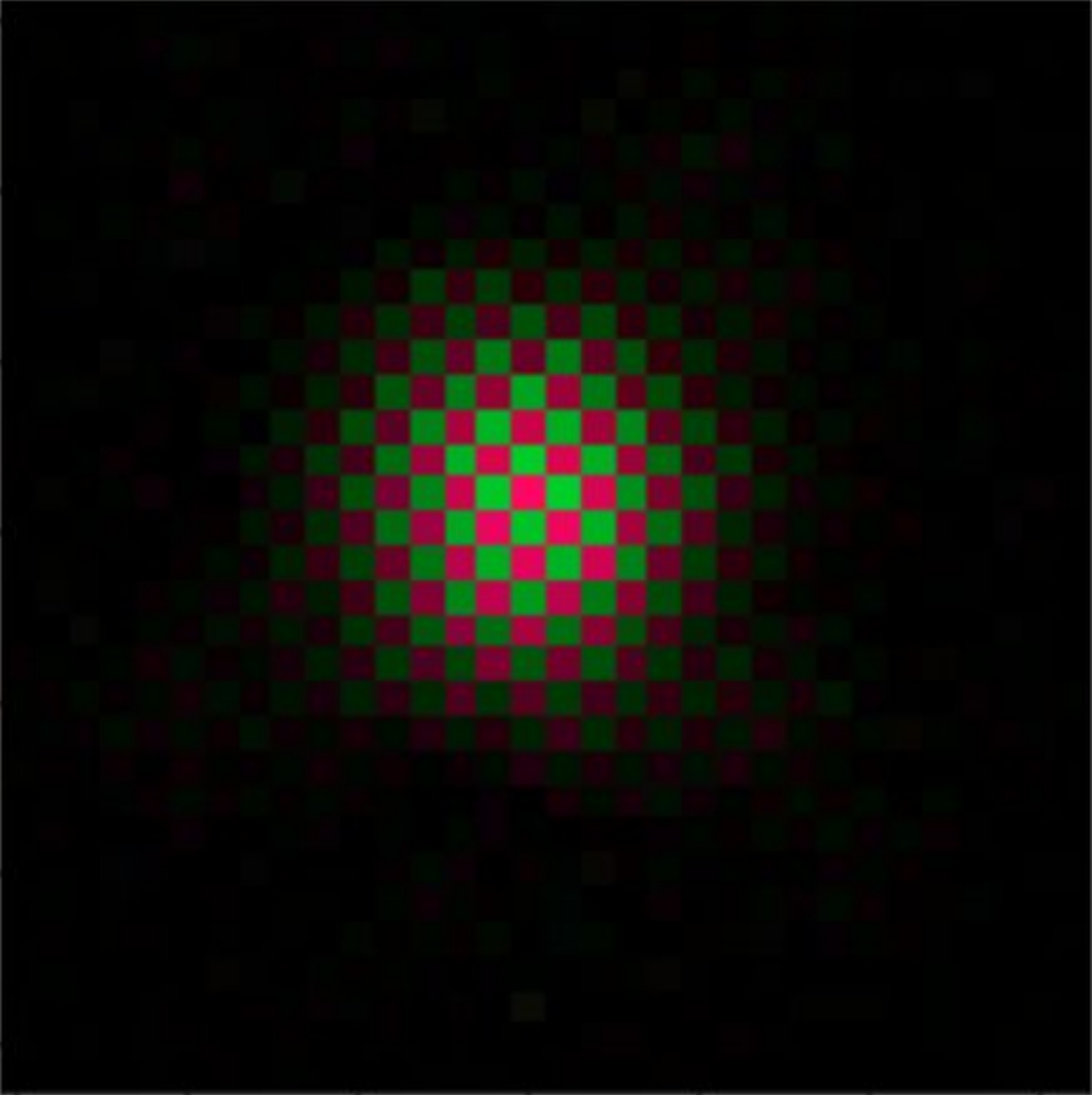}} & \parbox[l]{1em}{\includegraphics[width=3em]{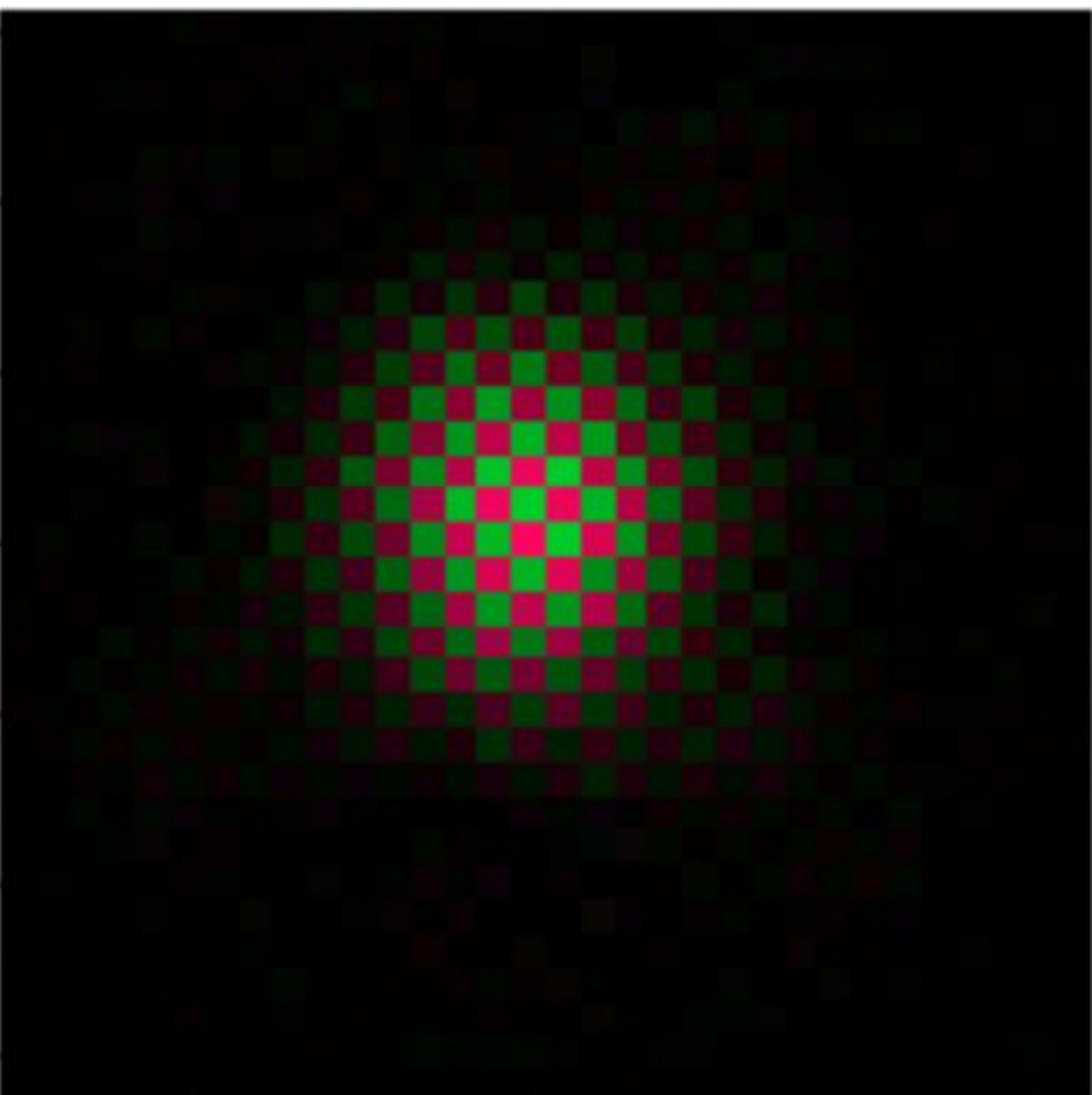}} \\
        \parbox[l]{1em}{\includegraphics[width=3em]{images/i_dot_sample.pdf}} & Dog & \parbox[l]{1em}{\includegraphics[width=3em]{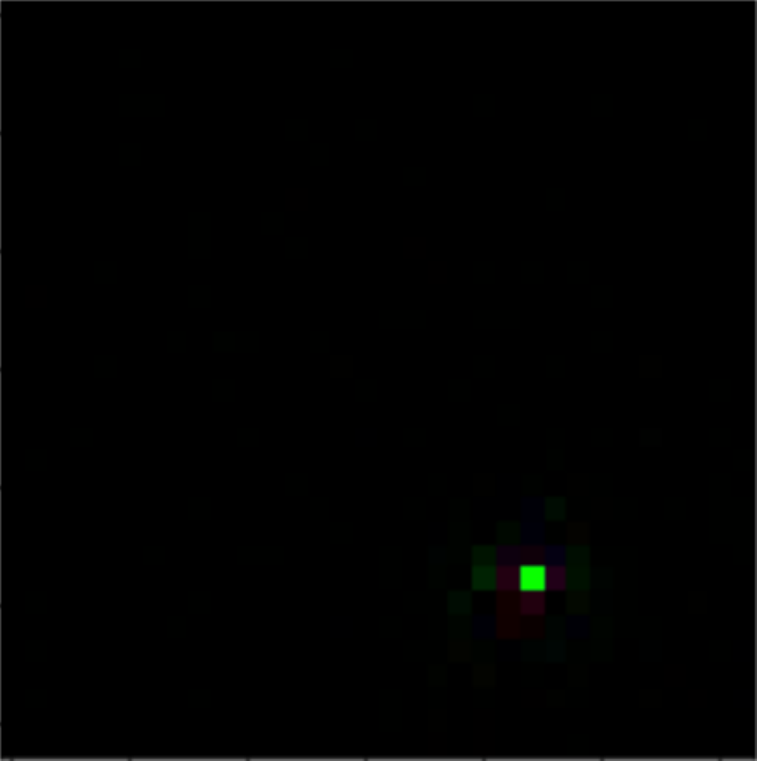}} & \parbox[l]{1em}{\includegraphics[width=3em]{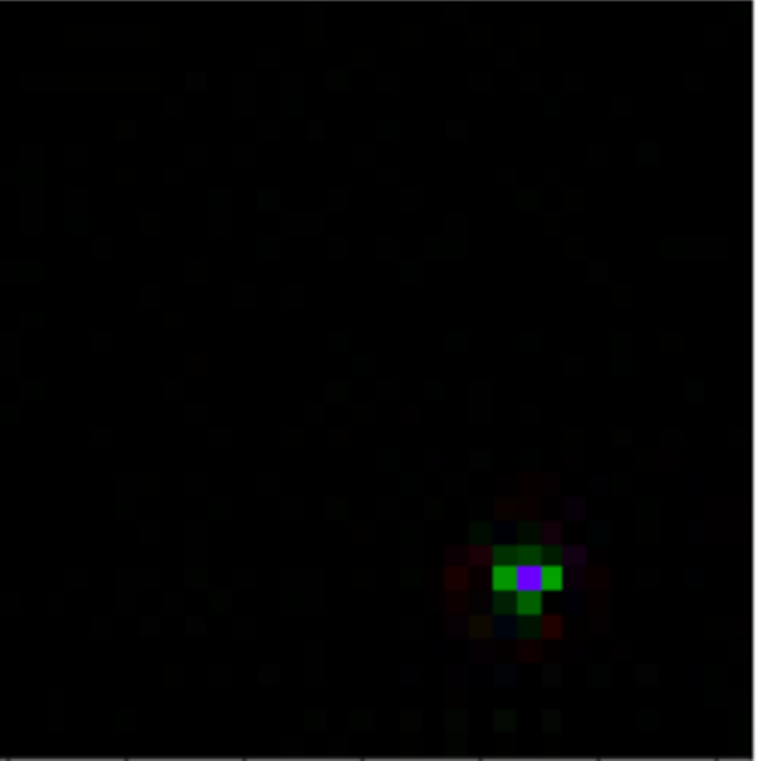}} && \parbox[l]{1em}{\includegraphics[width=3em]{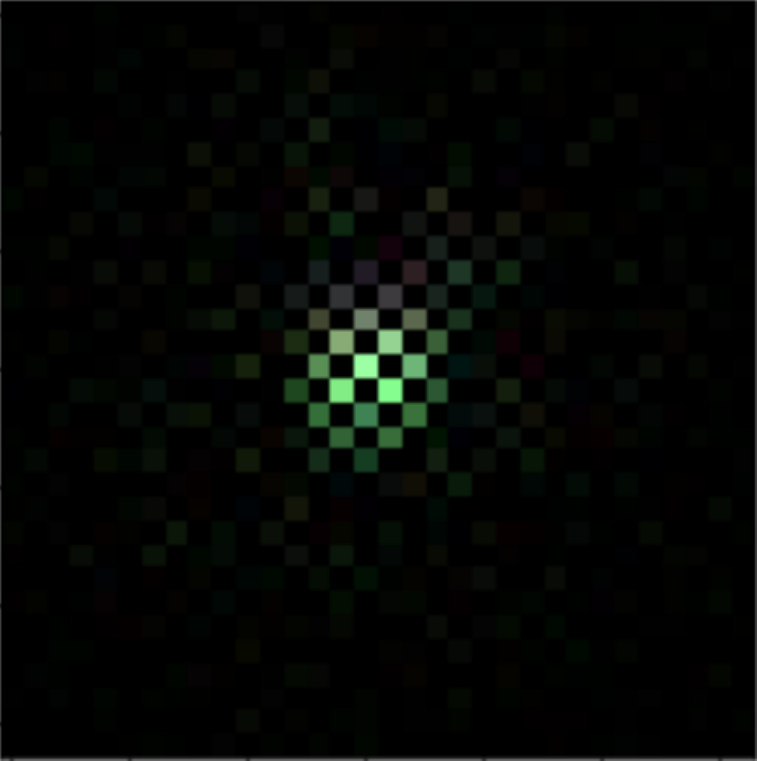}} & \parbox[l]{1em}{\includegraphics[width=3em]{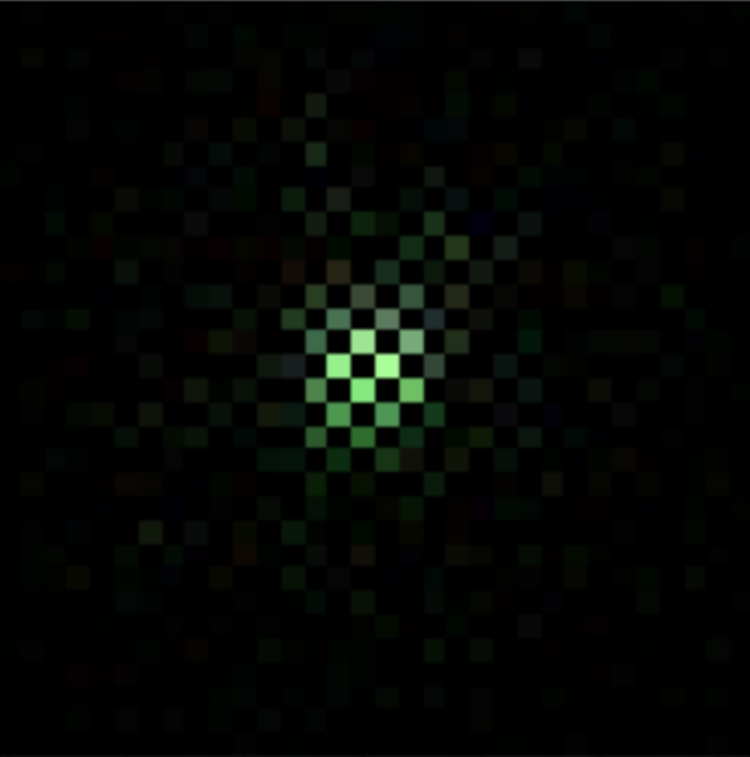}} \\
         \hline
        \end{tabular}
\label{tab:dot_poison_v_all}
\end{table*}

\begin{figure*}[p]
    \centering
    \includegraphics[width=0.5\textwidth]{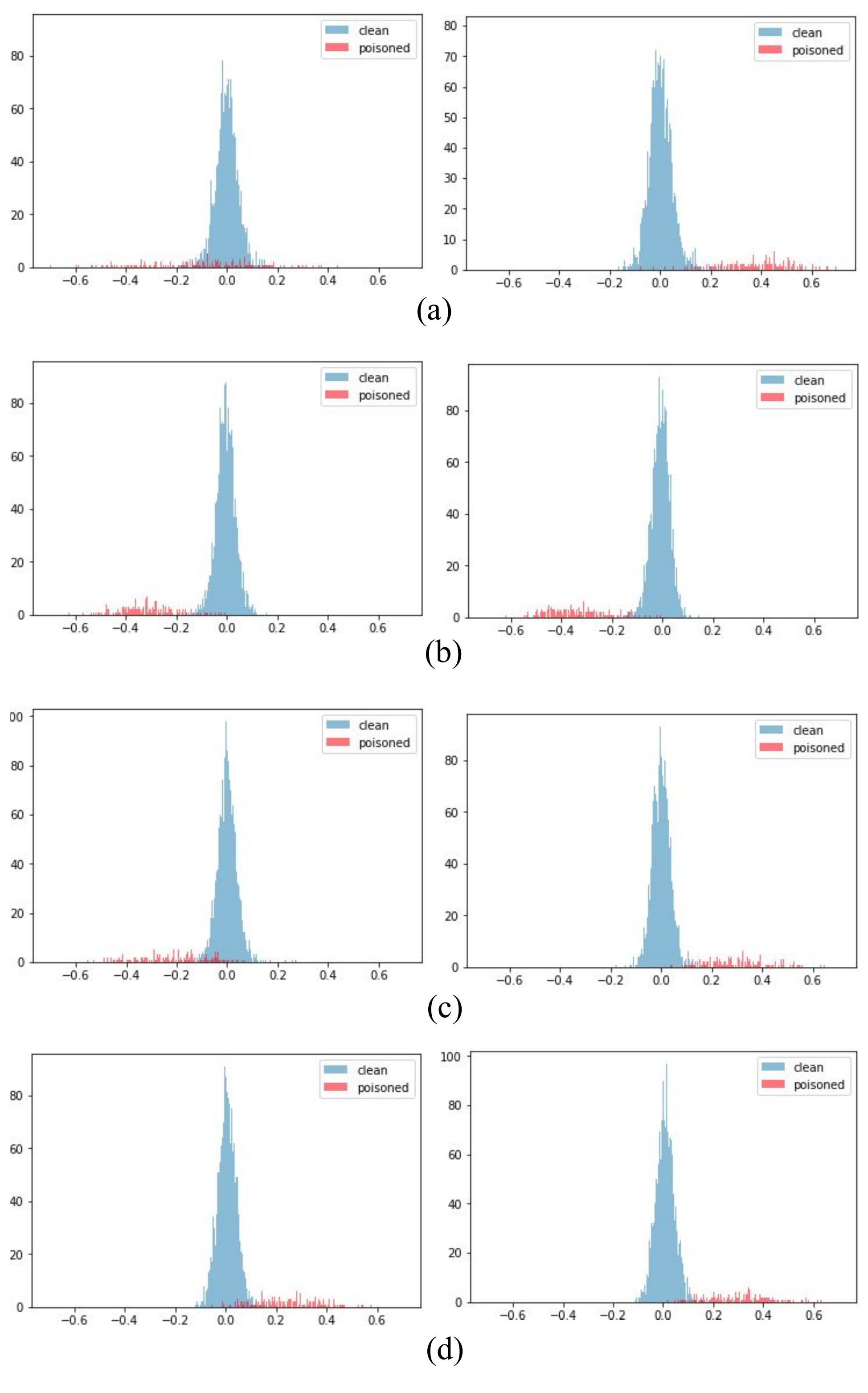}
    \caption{First principal component of poisoned and clean target class input gradients in an overlay image BP attack. The components on the left are derived with the target class as cross-entropy label while the ones on the right are derived with the base class as cross-entropy label. (a) Target: `Dog', Base: `Cat' (b) Target: `Frog', Base: `Ship' (c) Target: `Cat', Base: `Car' (d) Target: `Bird', Base: `Airplane'}
    \label{fig:appendix first_components a}
\end{figure*}

\begin{figure*}[p]
    \centering
    \includegraphics[width=0.5\textwidth]{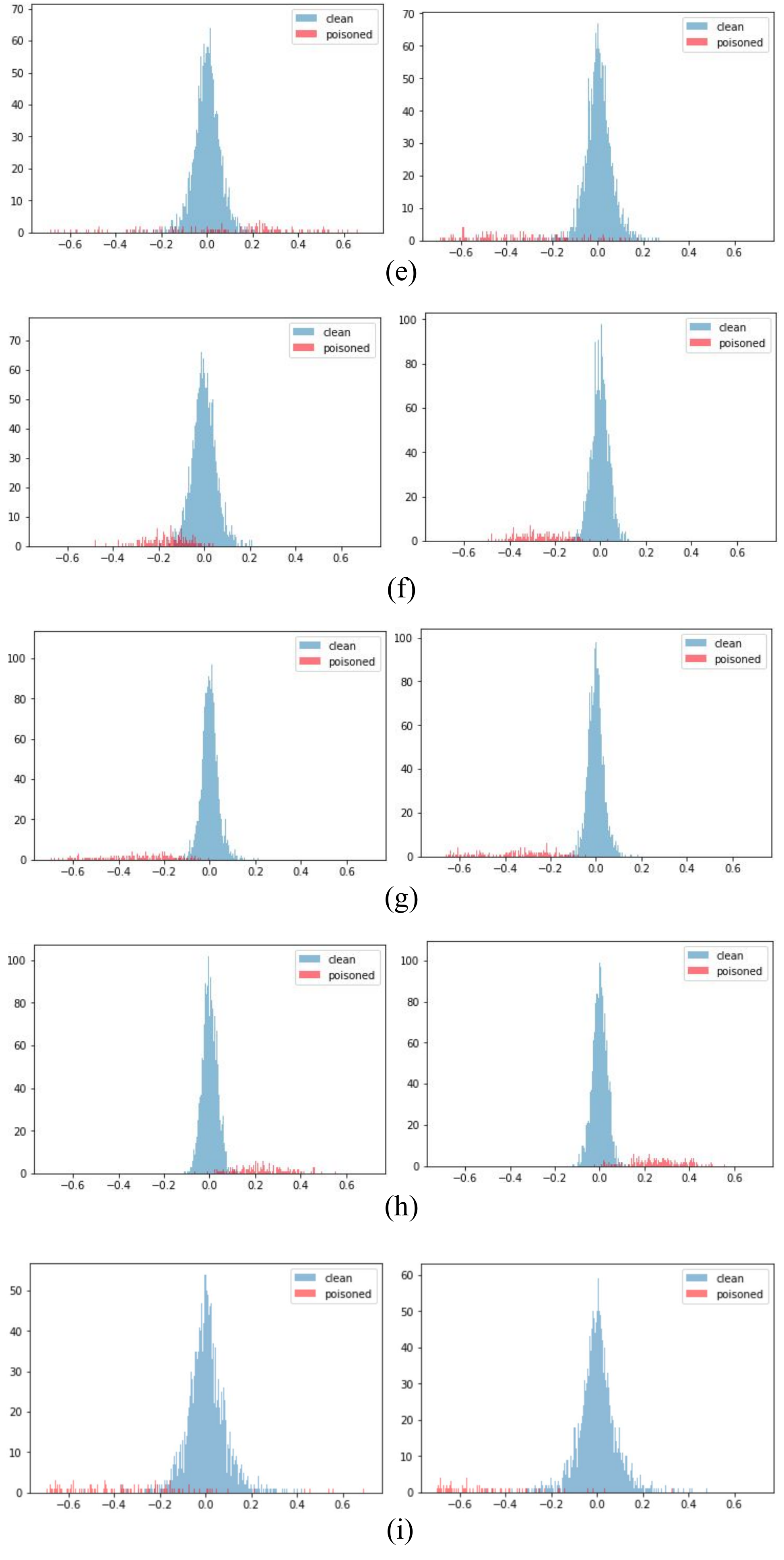}
    \caption{Continued from Figure~\ref{fig:appendix first_components a}; (e) Target: `Deer', Base: `Horse' (f) Target: `Bird', Base: `Truck' (g) Target: `Horse', Base: `Cat' (h) Target: `Cat', Base: `Dog' (i) Target: `Dog', Base: `Car'}
    \label{fig:appendix first_components b}
\end{figure*}

\newpage

\begin{table*}[p]
    \centering
    \footnotesize
    \setlength{\tabcolsep}{0.3em}
    \caption{Wasserstein distance between GMM clusters of input gradient first principal components with under overlay image BP attacks. The target class is identified as the class with highest distance value.}
        \begin{tabular}{ lcc|cccccccccc }
         \hline
         Poison & Target & Base & \multicolumn{10}{c}{Wasserstein Distance}\\
         ~ & ~ & ~ & 0 & 1 & 2 & 3 & 4 & 5 & 6 & 7 & 8 & 9 \\
         \hline
        \parbox[l]{1em}{\includegraphics[width=3em]{images/a_overlay.pdf}} & 5 & 3 & 0.00166 & 0.00320 & 0.00219 & 0.00233 & 0.00259 & \textbf{0.0427} & 0.00134 & 0.00249 & 0.00178 & 0.00160\\
        \parbox[l]{1em}{\includegraphics[width=3em]{images/b_overlay.pdf}} & 6 & 8 & 0.00235 & 0.00229 & 0.00272 & 0.00326 & 0.00305 & 0.00238 & \textbf{0.103} & 0.00212 & 0.0186 & 0.00243\\
        \parbox[l]{1em}{\includegraphics[width=3em]{images/c_overlay.pdf}} & 3 & 1 & 0.00236 & 0.00426 & 0.00298 & \textbf{0.0454} & 0.00212 & 0.00196 & 0.00170 & 0.00245 & 0.00274 & 0.00260\\
        \parbox[l]{1em}{\includegraphics[width=3em]{images/d_overlay.pdf}} & 2 & 0 & 0.00440 & 0.00176 & \textbf{0.0824} & 0.00230 & 0.00231 & 0.00259 & 0.00152 & 0.00169 & 0.00195 & 0.00295\\
        \parbox[l]{1em}{\includegraphics[width=3em]{images/e_overlay.pdf}} & 4 & 7 & 0.00215 & 0.00319 & 0.00254 & 0.00374 & \textbf{0.0655} & 0.00404 & 0.00269 & 0.0161 & 0.00146 & 0.00449\\
        \parbox[l]{1em}{\includegraphics[width=3em]{images/f_overlay.pdf}} & 2 & 9 & 0.00328 & 0.00131 & \textbf{0.0156} & 0.00194 & 0.00222 & 0.00169 & 0.0016 & 0.00364 & 0.00297 & 0.00960\\
        \parbox[l]{1em}{\includegraphics[width=3em]{images/g_overlay.pdf}} & 7 & 3 & 0.00288 & 0.00176 & 0.00234 & 0.0111 & 0.00355 & 0.00224 & 0.00307 & \textbf{0.0995} & 0.00149 & 0.00229\\
        \parbox[l]{1em}{\includegraphics[width=3em]{images/h_overlay.pdf}} & 3 & 5 & 0.00250 & 0.00213 & 0.00183 & \textbf{0.0612} & 0.00243 & 0.00219 & 0.00179 & 0.00223 & 0.00340 & 0.00221\\
        \parbox[l]{1em}{\includegraphics[width=3em]{images/i_overlay.pdf}} & 5 & 1 & 0.00228 & 0.00360 & 0.00319 & 0.00201 & 0.00218 & \textbf{0.00365} & 0.00164 & 0.00334 & 0.00287 & 0.00209\\
         \hline
        \end{tabular}
\label{tab:ws_dist_overlay}
\end{table*}

\begin{table*}[p]
    \centering
    \footnotesize
    \caption{Mean first principal component of input gradient with varying cross entropy label with overlay poison. The base class is identified as the class with highest mean component value.}
        \begin{tabular}{ lcc|cccccccccc }
         \hline
         Poison & Target & Base & \multicolumn{10}{c}{Mean 1st component}\\
         ~ & ~ & ~ & 0 & 1 & 2 & 3 & 4 & 5 & 6 & 7 & 8 & 9 \\
         \hline
        \parbox[l]{1em}{\includegraphics[width=3em]{images/a_overlay.pdf}} & 5 & 3 & 0.109 & 0.014 & 0.028 & \textbf{0.324} & 0.006 & 0.157 & 0.093 & 0.039 & 0.021 & 0.074\\
        \parbox[l]{1em}{\includegraphics[width=3em]{images/b_overlay.pdf}} & 6 & 8 & 0.288 & 0.292 & 0.312 & 0.316 & 0.296 & 0.314 & 0.324 & 0.316 & \textbf{0.346} & 0.306\\
        \parbox[l]{1em}{\includegraphics[width=3em]{images/c_overlay.pdf}} & 3 & 1 & 0.219 & \textbf{0.301} & 0.158 & 0.222 & 0.223 & 0.197 & 0.199 & 0.228 & 0.24 & 0.233\\
        \parbox[l]{1em}{\includegraphics[width=3em]{images/d_overlay.pdf}} & 2 & 0 & \textbf{0.321} & 0.292 & 0.297 & 0.285 & 0.289 & 0.286 & 0.294 & 0.284 & 0.286 & 0.299\\
        \parbox[l]{1em}{\includegraphics[width=3em]{images/e_overlay.pdf}} & 4 & 7 & 0.104 & 0.015 & 0.113 & 0.146 & 0.005 & 0.126 & 0.125 & \textbf{0.303} & 0.087 & 0.061\\
        \parbox[l]{1em}{\includegraphics[width=3em]{images/f_overlay.pdf}} & 2 & 9 & 0.187 & 0.156 & 0.186 & 0.163 & 0.178 & 0.174 & 0.161 & 0.177 & 0.191 & \textbf{0.233}\\
        \parbox[l]{1em}{\includegraphics[width=3em]{images/g_overlay.pdf}} & 7 & 3 & 0.306 & 0.301 & 0.307 & \textbf{0.332} & 0.294 & 0.294 & 0.31 & 0.312 & 0.308 & 0.309\\
        \parbox[l]{1em}{\includegraphics[width=3em]{images/h_overlay.pdf}} & 3 & 5 & 0.244 & 0.243 & 0.236 & 0.249 & 0.224 & \textbf{0.279} & 0.221 & 0.225 & 0.242 & 0.246\\
        \parbox[l]{1em}{\includegraphics[width=3em]{images/i_overlay.pdf}} & 5 & 1 & 0.004 & \textbf{0.093} & 0.019 & 0.010 & 0.014 & 0.012 & 0.012 & 0.001 & 0.004 & 0.026\\
         \hline
        \end{tabular}
\label{tab:first_principal_component_overlay}
\end{table*}

\begin{table*}[p]
    \centering
    \footnotesize
    \setlength{\tabcolsep}{0.3em}
    \caption{Wasserstein distance between GMM clusters of input gradient first principal components with under dot-sized BP attacks. The target class is identified as the class with highest distance value.}
        \begin{tabular}{ lcc|cccccccccc }
         \hline
         Sample & Target & Base & \multicolumn{10}{c}{Wasserstein Distance}\\
         ~ & ~ & ~ & 0 & 1 & 2 & 3 & 4 & 5 & 6 & 7 & 8 & 9 \\
         \hline
        \parbox[l]{1em}{\includegraphics[width=3em]{images/a_dot_sample.pdf}} & 5 & 3 & 0.0139 & 0.0111 & 0.0145 & 0.0184 & 0.0162 & \textbf{0.241} & 0.0121 & 0.0104 & 0.0077 & 0.0154\\
        \parbox[l]{1em}{\includegraphics[width=3em]{images/b_dot_sample.pdf}} & 6 & 8 & 0.0213 & 0.0197 & 0.0185 & 0.0227 & 0.0214 & 0.0193 & \textbf{0.0462} & 0.0145 & 0.0173 & 0.0157\\
        \parbox[l]{1em}{\includegraphics[width=3em]{images/c_dot_sample.pdf}} & 3 & 1 & 0.00288 & 0.00172 & 0.00220 & \textbf{0.248} & 0.00287 & 0.00215 & 0.00182 & 0.00174 & 0.00333 & 0.00266\\
        \parbox[l]{1em}{\includegraphics[width=3em]{images/d_dot_sample.pdf}} & 2 & 0 & 0.00888 & 0.00439 & \textbf{0.0787} & 0.00452 & 0.00445 & 0.00415 & 0.00248 & 0.00306 & 0.00327 & 0.00403\\
        \parbox[l]{1em}{\includegraphics[width=3em]{images/e_dot_sample.pdf}} & 4 & 7 & 0.0172 & 0.018 & 0.0146 & 0.0173 & \textbf{0.410} & 0.0159 & 0.015 & 0.0119 & 0.00984 & 0.0128\\
        \parbox[l]{1em}{\includegraphics[width=3em]{images/f_dot_sample.pdf}} & 2 & 9 & 0.0111 & 0.00543 & \textbf{0.360} & 0.00468 & 0.00344 & 0.00471 & 0.00435 & 0.00376 & 0.00320 & 0.00383\\
        \parbox[l]{1em}{\includegraphics[width=3em]{images/g_dot_sample.pdf}} & 7 & 3 & 0.0123 & 0.0134 & 0.0161 & 0.0183 & 0.0135 & 0.0146 & 0.0109 & \textbf{0.229} & 0.00675 & 0.0115\\
        \parbox[l]{1em}{\includegraphics[width=3em]{images/h_dot_sample.pdf}} & 3 & 5 & 0.00799 & 0.0130 & 0.0113 & \textbf{0.160} & 0.0137 & 0.0104 & 0.0127 & 0.00921 & 0.00678 & 0.00964\\
        \parbox[l]{1em}{\includegraphics[width=3em]{images/i_dot_sample.pdf}} & 5 & 1 & 0.00257 & 0.00325 & 0.00240 & 0.00278 & 0.00259 & \textbf{0.175} & 0.00184 & 0.00224 & 0.00194 & 0.00236\\
         \hline
        \end{tabular}
\label{tab:ws_dist_dot}
\end{table*}

\begin{table*}[p]
    \centering
    \footnotesize
    \caption{Mean first principal component of input gradient with varying cross entropy label with dot-sized poison. The base class is identified as the class with highest mean component value.}
        \begin{tabular}{ lcc|cccccccccc }
         \hline
         Sample & Target & Base & \multicolumn{10}{c}{Mean 1st component}\\
         ~ & ~ & ~ & 0 & 1 & 2 & 3 & 4 & 5 & 6 & 7 & 8 & 9 \\
         \hline
        \parbox[l]{1em}{\includegraphics[width=3em]{images/a_dot_sample.pdf}} & 5 & 3 & 0.466 & 0.464 & 0.453 & \textbf{0.583} & 0.358 & 0.511 & 0.420 & 0.375 & 0.417 & 0.477\\
        \parbox[l]{1em}{\includegraphics[width=3em]{images/b_dot_sample.pdf}} & 6 & 8 & 0.062 & 0.044 & 0.031 & 0.035 & 0.029 & 0.074 & 0.042 & 0.019 & \textbf{0.278} & 0.044\\
        \parbox[l]{1em}{\includegraphics[width=3em]{images/c_dot_sample.pdf}} & 3 & 1 & 0.443 & \textbf{0.657} & 0.347 & 0.378 & 0.28 & 0.24 & 0.352 & 0.289 & 0.409 & 0.302\\
        \parbox[l]{1em}{\includegraphics[width=3em]{images/d_dot_sample.pdf}} & 2 & 0 & \textbf{0.299} & 0.17 & 0.212 & 0.204 & 0.128 & 0.168 & 0.229 & 0.129 & 0.179 & 0.196\\
        \parbox[l]{1em}{\includegraphics[width=3em]{images/e_dot_sample.pdf}} & 4 & 7 & 0.662 & 0.485 & 0.448 & 0.471 & 0.639 & 0.631 & 0.237 & \textbf{0.825} & 0.593 & 0.161\\
        \parbox[l]{1em}{\includegraphics[width=3em]{images/f_dot_sample.pdf}} & 2 & 9 & 0.479 & 0.51 & 0.542 & 0.501 & 0.529 & 0.505 & 0.495 & 0.556 & 0.501 & \textbf{0.632}\\
        \parbox[l]{1em}{\includegraphics[width=3em]{images/g_dot_sample.pdf}} & 7 & 3 & 0.466 & 0.422 & 0.503 & \textbf{0.542} & 0.374 & 0.464 & 0.444 & 0.485 & 0.475 & 0.458\\
        \parbox[l]{1em}{\includegraphics[width=3em]{images/h_dot_sample.pdf}} & 3 & 5 & 0.3 & 0.239 & 0.255 & 0.336 & 0.281 & \textbf{0.473} & 0.130 & 0.259 & 0.237 & 0.284\\
        \parbox[l]{1em}{\includegraphics[width=3em]{images/i_dot_sample.pdf}} & 5 & 1 & 0.278 & \textbf{0.513} & 0.122 & 0.335 & 0.332 & 0.362 & 0.308 & 0.335 & 0.287 & 0.271\\
         \hline
        \end{tabular}
\label{tab:first_principal_component_dot}
\end{table*}

\newpage

\begin{table*}[p]
    \centering
    \footnotesize
    \caption{Poison clustering accuracy for overlay poison. Specificity is the accuracy of clean sample classification while sensitivity is the accuracy of poisoned sample classification.}
        \begin{tabular}{ lcc|ccccc }
         \hline
         Poison & Target & Base & \multicolumn{2}{c}{Target Class Xent} && \multicolumn{2}{c}{Base Class Xent}\\
         \cline{4-5}
         \cline{7-8}
         ~ & ~ & ~ & Specificity(\%) & Sensitivity(\%) && Specificity(\%) & Sensitivity(\%) \\
         \hline
        \parbox[l]{1em}{\includegraphics[width=3em]{images/a_overlay.pdf}} & 5 & 3 & 98.0 & 63.8 && 99.4 & 94.6\\
        \parbox[l]{1em}{\includegraphics[width=3em]{images/b_overlay.pdf}} & 6 & 8 & 99.7 & 93.4 && 99.6 & 96.0\\
        \parbox[l]{1em}{\includegraphics[width=3em]{images/c_overlay.pdf}} & 3 & 1 & 99.2 & 74.4 && 99.2 & 95.6\\
        \parbox[l]{1em}{\includegraphics[width=3em]{images/d_overlay.pdf}} & 2 & 0 & 99.8 & 84.2 && 99.7 & 89.8\\
        \parbox[l]{1em}{\includegraphics[width=3em]{images/e_overlay.pdf}} & 4 & 7 & 97.4 & 73.4 && 97.5 & 87.4\\
        \parbox[l]{1em}{\includegraphics[width=3em]{images/f_overlay.pdf}} & 2 & 9 & 97.3 & 68.8 && 99.5 & 95.4\\
        \parbox[l]{1em}{\includegraphics[width=3em]{images/g_overlay.pdf}} & 7 & 3 & 98.7 & 94.2 && 98.9 & 95.8\\
        \parbox[l]{1em}{\includegraphics[width=3em]{images/h_overlay.pdf}} & 3 & 5 & 99.5 & 83.8 && 99.7 & 89.2\\
        \parbox[l]{1em}{\includegraphics[width=3em]{images/i_overlay.pdf}} & 5 & 1 & 85.3 & 61.6 && 99.6 & 93.6\\
         \hline
        \end{tabular}
\label{tab:clustering accuracies overlay}
\end{table*}

\begin{table*}[p]
    \centering
    \footnotesize
    \caption{Poisoned sample filtering accuracy for dot-sized poison. Specificity is the accuracy of clean sample classification while sensitivity is the accuracy of poisoned sample classification.}
        \begin{tabular}{ lcc|ccccc }
         \hline
         Sample & Target & Base & \multicolumn{2}{c}{Target Class Xent} && \multicolumn{2}{c}{Base Class Xent}\\
         \cline{4-5}
         \cline{7-8}
         ~ & ~ & ~ & Specificity(\%) & Sensitivity(\%) && Specificity(\%) & Sensitivity(\%) \\
         \hline
        \parbox[l]{1em}{\includegraphics[width=3em]{images/a_dot_sample.pdf}} & 5 & 3 & 97.9 & 86.6 && 99.6 & 92.8\\
        \parbox[l]{1em}{\includegraphics[width=3em]{images/b_dot_sample.pdf}} & 6 & 8 & 89.8 & 67.0 && 99.5 & 88.6\\
        \parbox[l]{1em}{\includegraphics[width=3em]{images/c_dot_sample.pdf}} & 3 & 1 & 98.9 & 92.4 && 99.7 & 99.0\\
        \parbox[l]{1em}{\includegraphics[width=3em]{images/d_dot_sample.pdf}} & 2 & 0 & 96.4 & 70.0 && 96.8 & 84.4\\
        \parbox[l]{1em}{\includegraphics[width=3em]{images/e_dot_sample.pdf}} & 4 & 7 & 99.1 & 83.6 && 99.9 & 99.0\\
        \parbox[l]{1em}{\includegraphics[width=3em]{images/f_dot_sample.pdf}} & 2 & 9 & 96.2 & 99.2 && 99.7 & 100\\
        \parbox[l]{1em}{\includegraphics[width=3em]{images/g_dot_sample.pdf}} & 7 & 3 & 98.9 & 92.0 && 99.1 & 95.8\\
        \parbox[l]{1em}{\includegraphics[width=3em]{images/h_dot_sample.pdf}} & 3 & 5 & 95.6 & 83.2 && 99.3 & 94.0\\
        \parbox[l]{1em}{\includegraphics[width=3em]{images/i_dot_sample.pdf}} & 5 & 1 & 98.6 & 96.8 && 99.5 & 99.8\\
         \hline
        \end{tabular}
\label{tab:clustering accuracies dot}
\end{table*}



\begin{table*}[h]
    \centering
    \caption{Model accuracy on full test set and poisoned base class test images, before and after neutralization (Neu.) for full-sized overlay poison attacks with 5\% poison ratio.}
        \begin{tabular}{ llcccccc }
         \hline
         Poison & Sample & Target & \multicolumn{2}{c}{Acc Before Neu. (\%)} && \multicolumn{2}{c}{Acc After Neu. (\%)} \\
         \cline{4-5}
         \cline{7-8}
         ~ & ~ & ~ & All & Poisoned && All & Poisoned \\
         \hline
        \parbox[l]{1em}{\includegraphics[width=2em]{images/a_overlay.pdf}} & \parbox[l]{1em}{\includegraphics[width=2em]{images/a_overlay_sample.pdf}} & Dog & 95.1 & 11.9 && 94.5 & 80 \\
        \parbox[l]{1em}{\includegraphics[width=2em]{images/b_overlay.pdf}} & \parbox[l]{1em}{\includegraphics[width=2em]{images/b_overlay_sample.pdf}} & Frog & 95.1 & 24.3 && 95.1 & 96.3 \\
        \parbox[l]{1em}{\includegraphics[width=2em]{images/c_overlay.pdf}} & \parbox[l]{1em}{\includegraphics[width=2em]{images/c_overlay_sample.pdf}} & Cat & 95.3 & 6.8 && 94.7 & 93.5 \\
        \parbox[l]{1em}{\includegraphics[width=2em]{images/d_overlay.pdf}} & \parbox[l]{1em}{\includegraphics[width=2em]{images/d_overlay_sample.pdf}} & Bird & 95.0 & 46.5 && 94.4 & 92.2  \\
        \parbox[l]{1em}{\includegraphics[width=2em]{images/e_overlay.pdf}} & \parbox[l]{1em}{\includegraphics[width=2em]{images/e_overlay_sample.pdf}} & Deer & 95.1 & 5.0 && 94.8 & 90.4 \\
        \parbox[l]{1em}{\includegraphics[width=2em]{images/f_overlay.pdf}} & \parbox[l]{1em}{\includegraphics[width=2em]{images/f_overlay_sample.pdf}} & Bird & 95.3 & 11.3 && 94.9 & 90.3 \\
        \parbox[l]{1em}{\includegraphics[width=2em]{images/g_overlay.pdf}} & \parbox[l]{1em}{\includegraphics[width=2em]{images/g_overlay_sample.pdf}} & Horse & 95.0 & 49.0 && 94.7 & 89.3 \\
        \parbox[l]{1em}{\includegraphics[width=2em]{images/h_overlay.pdf}} & \parbox[l]{1em}{\includegraphics[width=2em]{images/h_overlay_sample.pdf}} & Cat & 95.4 & 23.9 && 95.0 & 89.6 \\
        \parbox[l]{1em}{\includegraphics[width=2em]{images/i_overlay.pdf}} & \parbox[l]{1em}{\includegraphics[width=2em]{images/i_overlay_sample.pdf}} & Dog & 95.3 & 15.8 && 94.5 & 95.6 \\
         \hline
        \end{tabular}
\label{tab:neutralization_acc_overlay_poison_ep005}
\end{table*}

\begin{table*}[h]
    \centering
    \caption{Model accuracy on full test set and poisoned base class test images, before and after neutralization (Neu.) for dot poison attacks with 5\% poison ratio.}
        \begin{tabular}{ lccccccc }
         \hline
         Sample & Target & \multicolumn{2}{c}{Acc Before Neu. (\%)} && \multicolumn{2}{c}{Acc After Neu. (\%)} \\
         \cline{3-4}
         \cline{6-7}
         ~ & ~ & All & Poisoned && All & Poisoned \\
         \hline
        \parbox[l]{1em}{\includegraphics[width=2.5em]{images/a_dot_sample.pdf}} & Dog & 95.3 & 0.8 && 94.9 & 90 \\
        \parbox[l]{1em}{\includegraphics[width=2.5em]{images/b_dot_sample.pdf}} & Frog & 94.9 & 0.5 && 94.7 & 95.7 \\
        \parbox[l]{1em}{\includegraphics[width=2.5em]{images/c_dot_sample.pdf}} & Cat & 95.1 & 1.0 && 94.8 & 97.7 \\
        \parbox[l]{1em}{\includegraphics[width=2.5em]{images/d_dot_sample.pdf}} & Bird & 95.3 & 1.7 && 95.1 & 96.3 \\
        \parbox[l]{1em}{\includegraphics[width=2.5em]{images/e_dot_sample.pdf}} & Deer & 95.1 & 2.2 && 94.7 & 96.7 \\
        \parbox[l]{1em}{\includegraphics[width=2.5em]{images/f_dot_sample.pdf}} & Bird & 95.4 & 1.8 && 95.2 & 96.6 \\
        \parbox[l]{1em}{\includegraphics[width=2.5em]{images/g_dot_sample.pdf}} & Horse & 95.0 & 0.3 && 94.9 & 87.9 \\
        \parbox[l]{1em}{\includegraphics[width=2.5em]{images/h_dot_sample.pdf}} & Cat & 95.2 & 2.7 && 94.9 & 90.5 \\
        \parbox[l]{1em}{\includegraphics[width=2.5em]{images/i_dot_sample.pdf}} & Dog & 95.4 & 8.2 && 95.2 & 97.3 \\
         \hline
        \end{tabular}
\label{tab:neutralization_acc_dot_poison_ep005}
\end{table*}


\begin{table*}[h]
    \centering
    \caption{Model accuracy on full test set and poisoned base class test images, before and after neutralization (Neu.) for full-sized overlay poison attacks on VGG with 10\% poison ratio.}
        \begin{tabular}{ llcccccc }
         \hline
         Poison & Sample & Target & \multicolumn{2}{c}{Acc Before Neu. (\%)} && \multicolumn{2}{c}{Acc After Neu. (\%)} \\
         \cline{4-5}
         \cline{7-8}
         ~ & ~ & ~ & All & Poisoned && All & Poisoned \\
         \hline
        \parbox[l]{1em}{\includegraphics[width=2em]{images/a_overlay.pdf}} & \parbox[l]{1em}{\includegraphics[width=2em]{images/a_overlay_sample.pdf}} & Dog & 93.9 & 7.6 && 92.9 & 81.2 \\
        \parbox[l]{1em}{\includegraphics[width=2em]{images/b_overlay.pdf}} & \parbox[l]{1em}{\includegraphics[width=2em]{images/b_overlay_sample.pdf}} & Frog & 93.5 & 15.4 && 92.9 & 96.1 \\
        \parbox[l]{1em}{\includegraphics[width=2em]{images/c_overlay.pdf}} & \parbox[l]{1em}{\includegraphics[width=2em]{images/c_overlay_sample.pdf}} & Cat & 93.6 & 7.3 && 92.2 & 87.9 \\
        \parbox[l]{1em}{\includegraphics[width=2em]{images/d_overlay.pdf}} & \parbox[l]{1em}{\includegraphics[width=2em]{images/d_overlay_sample.pdf}} & Bird & 93.1 & 30.7 && 92.7 & 89.8 \\
        \parbox[l]{1em}{\includegraphics[width=2em]{images/e_overlay.pdf}} & \parbox[l]{1em}{\includegraphics[width=2em]{images/e_overlay_sample.pdf}} & Deer & 93.6 & 4.5 && 93.1 & 86.7 \\
        \parbox[l]{1em}{\includegraphics[width=2em]{images/f_overlay.pdf}} & \parbox[l]{1em}{\includegraphics[width=2em]{images/f_overlay_sample.pdf}} & Bird & 93.8 & 6.4 && 93.0 & 93.5 \\
        \parbox[l]{1em}{\includegraphics[width=2em]{images/g_overlay.pdf}} & \parbox[l]{1em}{\includegraphics[width=2em]{images/g_overlay_sample.pdf}} & Horse & 93.4 & 48.9 && 93.5 & 86.7 \\
        \parbox[l]{1em}{\includegraphics[width=2em]{images/h_overlay.pdf}} & \parbox[l]{1em}{\includegraphics[width=2em]{images/h_overlay_sample.pdf}} & Cat & 93.4 & 21.5 && 92.2 & 74.5 \\
        \parbox[l]{1em}{\includegraphics[width=2em]{images/i_overlay.pdf}} & \parbox[l]{1em}{\includegraphics[width=2em]{images/i_overlay_sample.pdf}} & Dog & 93.7 & 11.3 && 92.6 & 94.8 \\
         \hline
        \end{tabular}
\label{tab:neutralization_acc_overlay_poison_VGG}
\end{table*}

\begin{table*}[h]
    \centering
    \caption{Model accuracy on full test set and poisoned base class test images, before and after neutralization (Neu.) for dot poison attacks on VGG with 10\% poison ratio.}
        \begin{tabular}{ lccccccc }
         \hline
         Sample & Target & \multicolumn{2}{c}{Acc Before Neu. (\%)} && \multicolumn{2}{c}{Acc After Neu. (\%)} \\
         \cline{3-4}
         \cline{6-7}
         ~ & ~ & All & Poisoned && All & Poisoned \\
         \hline
        \parbox[l]{1em}{\includegraphics[width=2.5em]{images/a_dot_sample.pdf}} & Dog & 93.7 & 1.1 && 93.1 & 80.6 \\
        \parbox[l]{1em}{\includegraphics[width=2.5em]{images/b_dot_sample.pdf}} & Frog & 93.6 & 0.2 && 93.1 & 96.2 \\
        \parbox[l]{1em}{\includegraphics[width=2.5em]{images/c_dot_sample.pdf}} & Cat & 93.6 & 1.0 && 93.0 & 73.5 \\
        \parbox[l]{1em}{\includegraphics[width=2.5em]{images/d_dot_sample.pdf}} & Bird & 93.6 & 2.0 && 93.4 & 93.3 \\
        \parbox[l]{1em}{\includegraphics[width=2.5em]{images/e_dot_sample.pdf}} & Deer & 93.8 & 0.3 && 93.5 & 94.5 \\
        \parbox[l]{1em}{\includegraphics[width=2.5em]{images/f_dot_sample.pdf}} & Bird & 93.3 & 2.4 && 93.2 & 95.8 \\
        \parbox[l]{1em}{\includegraphics[width=2.5em]{images/g_dot_sample.pdf}} & Horse & 93.4 & 0.8 && 93.1 & 88.0 \\
        \parbox[l]{1em}{\includegraphics[width=2.5em]{images/h_dot_sample.pdf}} & Cat & 93.5 & 2.9 && 93.4 & 86.6 \\
        \parbox[l]{1em}{\includegraphics[width=2.5em]{images/i_dot_sample.pdf}} & Dog & 93.8 & 6.5 && 93.3 & 97.6 \\
         \hline
        \end{tabular}
\label{tab:neutralization_acc_dot_poison_VGG}
\end{table*}

\end{appendices}

\end{document}